\documentclass[sigconf]{aamas}

\usepackage{balance} %
\usepackage{pgfplots}
\pgfplotsset{compat=1.18}
\usepackage{graphicx}
\usepackage{wrapfig}
\usepackage{circuitikz}
\usepackage{multirow}
\usepackage{amsthm}
\usepackage{amsmath}
\usepackage{amsfonts}
\usepackage{tablefootnote}
\usepackage{diagbox}
\usepackage{booktabs}
\usepackage{balance} %
\usepackage{makecell}
\usepackage{bbm}
\usepackage{enumitem}
\usepackage{xcolor}
\usetikzlibrary{arrows.meta, positioning}
\usepackage{subcaption}

\doi{TXUQ4572}

\makeatletter
\gdef\@copyrightpermission{
  \begin{minipage}{0.2\columnwidth}
   \href{https://creativecommons.org/licenses/by/4.0/}{\includegraphics[width=0.90\textwidth]{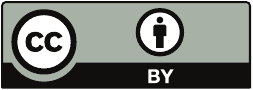}}
  \end{minipage}\hfill
  \begin{minipage}{0.8\columnwidth}
   \href{https://creativecommons.org/licenses/by/4.0/}{This work is licensed under a Creative Commons Attribution International 4.0 License.}
  \end{minipage}
  \vspace{5pt}
}
\makeatother

\setcopyright{ifaamas}
\acmConference[AAMAS '26]{Proc.\@ of the 25th International Conference
on Autonomous Agents and Multiagent Systems (AAMAS 2026)}{May 25 -- 29, 2026}
{Paphos, Cyprus}{C.~Amato, L.~Dennis, V.~Mascardi, J.~Thangarajah (eds.)}
\copyrightyear{2026}
\acmYear{2026}
\acmDOI{}
\acmPrice{}
\acmISBN{}

\title[Robust Counterfactual Inference in Markov Decision Processes]{Robust Counterfactual Inference in Markov Decision Processes}

\author{Jessica Lally}
\affiliation{
  \institution{King's College London}
  \city{London}
  \country{United Kingdom}}
\email{jessica.lally@kcl.ac.uk}

\author{Milad Kazemi}
\affiliation{
  \institution{King's College London}
  \city{London}
  \country{United Kingdom}}
\email{milad.kazemi@kcl.ac.uk}

\author{Nicola Paoletti}
\affiliation{
  \institution{King's College London}
  \city{London}
  \country{United Kingdom}}
\email{nicola.paoletti@kcl.ac.uk}

\begin{abstract}
This paper addresses a key limitation in existing counterfactual inference methods for Markov Decision Processes (MDPs). To make counterfactual distributions identifiable, existing approaches assume a specific causal model of the system; however, there are typically many causal models consistent with the observed and interventional distributions of an MDP, each yielding different counterfactual probabilities. Thus, relying on a single model can limit the validity and usefulness of counterfactual inference. We propose a novel \textit{non-parametric} approach that computes tight bounds on counterfactual transition probabilities across all compatible causal models. Unlike previous methods that require solving prohibitively large optimisation problems, our approach provides closed-form expressions for these bounds, making computation highly efficient even for large-scale MDPs. Using these bounds, we construct an \textit{interval} counterfactual MDP, and identify
robust counterfactual policies that optimise the worst-case reward over the uncertain MDP probabilities. We evaluate our method on various case studies, demonstrating improved robustness over existing methods.
\end{abstract}


\newcommand{\BibTeX}{\rm B\kern-.05em{\sc i\kern-.025em b}\kern-.08em\TeX}
\newcommand{\camera}[1]{{\color{black}#1}}

\newcommand{\interv}[5]{{#1}_{(#2=#3,#4=#5)}}

\DeclareMathOperator*{\argmax}{arg\,max}

\pgfplotsset{
  every axis plot/.append style={line width=1.2pt},
  every axis plot post/.append style={
    every mark/.append style={line width=1.0pt}
  }
}

\newcommand{\jl}[1]{#1}

\newtheorem{definition}{Definition}[section]
\newtheorem{theorem}{Theorem}[section]
\newtheorem{lemma}{Lemma}[section]

\renewcommand{\tilde}[1]{\widetilde{#1}}

\begin{document}

\pagestyle{fancy}
\fancyhead{}

\maketitle 

\section{Introduction}
\jl{Markov Decision Processes (MDPs) are a fundamental mathematical framework for modelling sequential decision-making processes under uncertainty, including reinforcement learning (RL) problems. However, evaluating RL-learnt policies can be challenging, especially in safety-critical domains like healthcare, where testing these policies directly on patients would be both risky and unethical.}

\jl{Counterfactual inference of MDPs enables offline policy evaluation, i.e., without ``deploying'' the alternative policy into the environment. 
Given an observed sequence of actions and outcomes,
counterfactual inference estimates what the outcome would have been if different actions had been taken.} Counterfactual outcomes yielding higher rewards than the observation can serve as explanations for how the observed policy could be improved.

\jl{Counterfactual inference is increasingly being applied to MDPs for various RL applications, e.g., for generating counterfactual explanations \citep{oberst2019counterfactual, tsirtsis2024finding, tsirtsis2021counterfactual}, facilitating policy transfer \citep{killian2022counterfactually}, and augmenting datasets with counterfactual paths to address data scarcity \citep{lu2020sample, sun2024acamda}, \camera{as well as to other sequential models, such as LLMs}
\citep{chatzi2024counterfactual,pona2025abstract}}. \jl{These works compute counterfactual probabilities by assuming a specific causal model of the system (e.g., the Gumbel-max structural causal model (SCM) \citep{oberst2019counterfactual}). However, given an observation and MDP, the system's causal model is generally \textit{non-identifiable}: there can be many causal models compatible with this data, each yielding different counterfactual probabilities \citep{pmlr-v162-zhang22ab}. As a result, any counterfactual analysis based on a single assumed causal model may be inaccurate, which is particularly concerning in safety-critical domains.}

\jl{\textit{Partial counterfactual inference} methods address this problem by computing bounds (as opposed to sharp values) on counterfactual probabilities, over \textit{all} causal models compatible with the given data.}
\jl{One important work in this area is the canonical SCM approach developed by \citet{pmlr-v162-zhang22ab}, which formulates partial counterfactual inference as an optimisation problem.} \jl{However, while this approach can compute counterfactual probability bounds in a wide range of settings (including those with unobserved confounders), their optimisation procedure is highly inefficient, as the number of constraints grows exponentially with the size of the MDP \citep{duarte2023, zaffalon2024,pmlr-v162-zhang22ab}.}
\paragraph{Contributions} In this paper, we demonstrate how the partial counterfactual inference approach of \citet{pmlr-v162-zhang22ab} can be applied to MDPs, and prove this optimisation problem reduces to exact analytical bounds in the MDP setting (i.e., a Markovian setting with no unobserved confounders), thus successfully addressing the complexity of this optimisation problem. Next, using these bounds, we construct interval counterfactual MDPs, which we solve using pessimistic value iteration \citep{mathiesen2024intervalmdp} to derive counterfactual policies that optimise the worst-case counterfactual rewards across all possible causal models compatible with the given data, \camera{ensuring robustness to uncertainty in the true underlying causal model}\footnote{\camera{This notion of robustness aligns with existing work on robust MDPs, which design policies resilient to uncertainty in the transition probabilities \citep{nilim2003robustness}, but differs from other notions, e.g., robustness to out-of-distribution environments \citep{pinto2017robust}.}}. Finally, we evaluate the average performance and robustness of our approach on a range of MDP benchmarks, demonstrating that our policies are more robust to causal model uncertainty than those derived from the Gumbel-max SCM. Thanks to the analytical bounds, our approach results in a speedup of 4-251x over the Gumbel-max SCM. \camera{For the full paper, including appendices and proofs, see~\cite{lally2025robust}.}

\begin{figure}[h]
\centering
\begin{subfigure}[t]{0.45\linewidth} %
\centering
\resizebox{\linewidth}{!}{%
\begin{circuitikz}
\tikzstyle{every node}=[font=\fontsize{32}{36}\selectfont]

\definecolor{a0color}{RGB}{0, 0, 255} %
\definecolor{a1color}{RGB}{255,0,0} %

\draw  (0,0) circle (1.25cm) node {$s_0$} ;
\draw  (6,0) circle (1.25cm) node {$s_1$} ;
\draw  (12,0) circle (1.25cm) node {$s_2$} ;

\draw  (0,-5) circle (1.25cm) node {$s_0$} ;
\draw  (6,-5) circle (1.25cm) node {$s_1$} ;
\draw  (12,-5) circle (1.25cm) node {$s_2$} ;

\draw [->, >=Stealth, color=a0color, thick] (-0.2,-1.25) -- (-0.2,-3.75) node[pos=0.3, left, fill=white]{0.1};
\draw [->, >=Stealth, color=a0color, thick] (-0.2,-1.25) -- (5.8,-3.75) node[pos=0.4, below, fill=white]{0.5};
\draw [->, >=Stealth, color=a0color, thick] (-0.2,-1.25) -- (11.8,-3.75) node[pos=0.8, below, fill=white]{0.4};

\draw [->, >=Stealth, color=a1color, thick, dashed] (0.2,-1.25) -- (0.2,-3.75) node[pos=0.7, left, fill=white]{0.08};
\draw [->, >=Stealth, color=a1color, thick, dashed] (0.2,-1.25) -- (6.2,-3.75) node[pos=0.7, above, fill=white]{0.1};
\draw [->, >=Stealth, color=a1color, thick, dashed] (0.2,-1.25) -- (12.2,-3.75) node[pos=0.3, above, fill=white]{0.82};

\draw [->, >=Stealth, color=a0color, thick] (6,-1.25) -- (6,-3.75) node[pos=0.3, fill=white]{1.0};
\draw [->, >=Stealth, color=a1color, thick, dashed] (6,-1.25) -- (12,-3.75) node[pos=0.3, fill=white]{1.0};

\draw [->, >=Stealth, color=a0color, thick] (12,-1.25) -- (6,-3.75) node[pos=0.3, fill=white]{1.0};
\draw [->, >=Stealth, color=a1color, thick, dashed] (12,-1.25) -- (12,-3.75) node[pos=0.3, fill=white]{1.0};

\node[draw, rectangle, fill=a0color!20] at (3,-8) {Action $a_0$};
\node[draw, rectangle, fill=a1color!20] at (9,-8) {Action $a_1$};

\node at (0,-6.75) {$r(s_0)=0$};
\node at (6,-6.75) {$r(s_1)=2$};
\node at (12,-6.75) {$r(s_2)=1$};

\end{circuitikz}
}
\caption{Interventional probabilities for two actions $a_0$ (blue, solid) and $a_1$ (red, dashed). Optimal policy: $\forall s \in \mathcal{S}, \pi^*(s) = a_0$.}
\end{subfigure}
\hfill
\begin{subfigure}[t]{0.53\linewidth} %
\centering
\resizebox{0.85\linewidth}{!}{%
\begin{circuitikz}
\tikzstyle{every node}=[font=\fontsize{32}{36}\selectfont]

\definecolor{a0color}{RGB}{0, 0, 255} %
\definecolor{a1color}{RGB}{255,0,0} %

\draw  (0,0) circle (1.25cm) node {$s_0$} ;
\draw  (6,0) circle (1.25cm) node {$s_1$} ;
\draw  (12,0) circle (1.25cm) node {$s_2$} ;

\draw  (0,-5) circle (1.25cm) node {$s_0$} ;
\draw  (6,-5) circle (1.25cm) node {$s_1$} ;
\draw  (12,-5) circle (1.25cm) node {$s_2$} ;

\draw [->, >=Stealth, color=a0color, thick] (-0.2,-1.25) -- (-0.2,-3.75) node[pos=0.3, left, fill=white]{1.0};
\draw [->, >=Stealth, color=a1color, thick, dashed] (0.2,-1.25) -- (0.2,-3.75) node[pos=0.7, right, fill=white]{[0.18, 0.8]};
\draw [->, >=Stealth, color=a1color, thick, dashed] (0.2,-1.25) -- (12.2,-3.75) node[pos=0.25, above, yshift=1mm, fill=white]{[0.2, 0.82]};

\draw [->, >=Stealth, color=a0color, thick] (6,-1.25) -- (6,-3.75) node[pos=0.3, fill=white]{1.0};
\draw [->, >=Stealth, color=a1color, thick, dashed] (6,-1.25) -- (12,-3.75) node[pos=0.3, fill=white]{1.0};

\draw [->, >=Stealth, color=a0color, thick] (12,-1.25) -- (6,-3.75) node[pos=0.3, fill=white]{1.0};
\draw [->, >=Stealth, color=a1color, thick, dashed] (12,-1.25) -- (12,-3.75) node[pos=0.3, fill=white]{1.0};

\node[draw, rectangle, fill=a0color!20] at (3,-8) {Action $a_0$};
\node[draw, rectangle, fill=a1color!20] at (9,-8) {Action $a_1$};

\node at (0,-6.75) {$r(s_0)=0$};
\node at (6,-6.75) {$r(s_1)=2$};
\node at (12,-6.75) {$r(s_2)=1$};

\end{circuitikz}
}
\caption{Counterfactual probabilities given observed transition, $s_0, a_0 \rightarrow s_0$. Optimal counterfactual policy: $\pi_{\it CF}^*(s_0) = a_1, \pi_{\it CF}^*(s_1) = a_0, \pi_{\it CF}^*(s_2) = a_0$.}
\end{subfigure}
\caption{\camera{Counterfactual inference in a toy example MDP.}}
\end{figure}

\section{Background}
\label{sec: background}
In this section, we provide background on counterfactual inference and an overview of related work on counterfactual inference in MDPs and partial counterfactual inference.
\paragraph{Markov Decision Processes}
MDPs are a class of stochastic models for representing sequential decision-making processes. In an MDP $\mathcal{M}$, at each step $t$, an agent in state $s_t$ performs some action $a_t$ determined by a policy $\pi$. \jl{The agent then transitions to a new state $s_{t+1} \sim P(\cdot \mid s_t, a_t)$, and receives a reward $R(s_t,a_t)$.}
Formally, an MDP is a tuple $\mathcal{M}=(\mathcal{S},\mathcal{A},P,P_I,R)$ where $\mathcal{S}$ is the discrete \emph{state space}, $\mathcal{A}$ is the set of \emph{actions}, $P: (\mathcal{S} \times \mathcal{A} \times \mathcal{S}) \rightarrow [0,1]$ is the \emph{transition probability} function, $P_I:\mathcal{S} \rightarrow [0,1]$ is the \textit{initial state distribution}, and $R: (\mathcal{S} \times \mathcal{A}) \rightarrow \mathbb{R}$ is the \emph{reward function}. 
A (deterministic) \emph{policy} $\pi$ for $\mathcal{M}$ is a function $\pi: \mathcal{S} \rightarrow \mathcal{A}$. A path $\tau$ of $\mathcal{M}$ under policy $\pi$
is a sequence $\tau=(s_0, a_0), \ldots,(s_{T-1},a_{T-1})$ where $T=|\tau|$ is the path length, $P_I(s_0)>0$, $a_t=\pi(s_t)$ for all $t=0,\ldots,T-1$, and $P(s_{t+1}\mid s_t, a_t)>0$ for all $t=0,\ldots,T-2$.

\paragraph{Counterfactual Inference} Structural causal models (SCMs) \citep{halpern2005causes,pearl_2009} provide a mathematical framework for causal inference. Formally, a SCM is a tuple $\mathcal{C}=(\mathbf{U},\mathbf{V},\mathcal{F},P(\mathbf{U}))$, where $\mathbf{V}$ is the set of endogenous (observed) variables, $\mathbf{U}$ is a set of exogenous (unobserved) variables, $P(\mathbf{U})$ is a joint distribution over the possible values of each $U \in \mathbf{U}$, and $\mathcal{F}$ is a set of structural equations where each $f_V \in \mathcal{F}$ determines the value of endogenous variable $V$, given a fixed realisation of $U_V\in \mathbf{U}$ and set of direct causes (parents) $\mathbf{PA}_V \subseteq \mathbf{V}$. \jl{A causal graph can be defined from an SCM by drawing, for every $V\in\mathbf{V}$, directed edges from the nodes in $\mathbf{PA}_V$ and $U_V$ into $V$.} In the following, we denote random variables with capital letters and specific values of the variables in lowercase.

\jl{SCMs enable the evaluation of causal effects, that is, how the distribution of some (outcome) variable $Y$ changes after applying a so-called \textit{intervention} $X \leftarrow x$ on a (treatment) variable $X$. Such an intervention corresponds to replacing the structural assignment of $X$ with the constant value $x$. 
We can also perform \textit{counterfactual inference} to estimate, given an observation $\mathbf{V}=\mathbf{v}$, the hypothetical values of $\mathbf{V}$ had we applied some intervention. } 
\jl{Counterfactual inference involves inferring the value of the exogenous variables present in the observation $\mathbf{v}$, i.e., $P(\mathbf{U} \mid \mathbf{v})$, then evaluating the structural equations by replacing $P(\mathbf{U})$ with the inferred $P(\mathbf{U} \mid \mathbf{v})$, and applying the intervention.}

\paragraph{Counterfactual Inference in MDPs}
\begin{figure}
    \centering
    \begin{circuitikz}[scale=0.35, transform shape]

    \node[circle, draw, minimum size=2.3cm] (S_tp1) at (0, 0) {\fontsize{24}{40}\selectfont $S_{t+1}$};
    \node[circle, draw, minimum size=2.3cm] (S_tp2) at (3, 0) {\fontsize{24}{40}\selectfont $S_{t+2}$};
    \node[circle, draw, minimum size=2.3cm] (S_t) at (-3, 0) {\fontsize{24}{40} $S_t$};
    \node[circle, draw, minimum size=2.3cm] (A_t) at (-3, -3) {\fontsize{24}{40} $A_t$};
    \node[circle, draw, minimum size=2.3cm] (A_tp1) at (0, -3) {\fontsize{24}{40} $A_{t+1}$};
    \node[circle, draw, minimum size=2.3cm] (A_tp2) at (3, -3) {\fontsize{24}{40} $A_{t+2}$};
    \node[circle, draw, fill=gray!20, minimum size=2.3cm] (U_tm1) at (-3, 3) {\fontsize{24}{40} $U_{t-1}$};
    \node[circle, draw, fill=gray!20, minimum size=2.3cm] (U_t) at (0, 3) {\fontsize{24}{40} $U_{t}$};
    \node[circle, draw, fill=gray!20, minimum size=2.3cm] (U_tp1) at (3, 3) {\fontsize{24}{40} $U_{t+1}$};
    \node at ([xshift=-1.7cm]S_t.west) {\fontsize{24}{40} $\dots$};
    \node at ([xshift=1.7cm]S_tp2.east) {\fontsize{24}{40} $\dots$};

    \draw[->, >=Stealth]([xshift=-1cm]S_t.west) -- (S_t);
    \draw[->, >=Stealth] (S_t) -- (S_tp1);
    \draw[->, >=Stealth] (S_tp1) -- (S_tp2);
    \draw[->, >=Stealth] (S_t) -- (A_t);
    \draw[->, >=Stealth] (S_tp1) -- (A_tp1);
    \draw[->, >=Stealth] (S_tp2) -- (A_tp2);
    \draw[->, >=Stealth] (A_t) -- (S_tp1);
    \draw[->, >=Stealth] (A_tp1) -- (S_tp2);
    \draw[->, >=Stealth] (U_t) -- (S_tp1);
    \draw[->, >=Stealth] (U_tm1) -- (S_t);
    \draw[->, >=Stealth] (U_tp1) -- (S_tp2);
    \draw[->, >=Stealth] (S_tp2) -- ([xshift=1cm]S_tp2.east);
    \draw[->, dashed, >=Stealth] (A_tp2) -- ++(2.24,2.24);

\end{circuitikz}
    \caption{MDP causal graph. White nodes represent endogenous/observable variables; grey nodes represent exogenous/unobserved variables.}
    \label{fig:MDP_DAG}
\end{figure}
\jl{MDPs can be represented as SCMs with the causal graph in Figure \ref{fig:MDP_DAG} and the following structural equations:}
{%
\begin{equation}\label{eq:mdp_scm}
        S_{t+1} = f(S_{t},A_{t}, U_{t}); \ \ A_t = \pi(S_t); 
        \ \ S_0 = f_I(U_I)
\end{equation}
}

\jl{In the MDP context, the interventional distribution is the MDP's transition probabilities\footnote{\jl{in that if we apply an intervention $S_t \gets s$ and $A_t \gets a$, then the interventional distribution of $S_{t+1}$ corresponds to the MDP transition probabilities $P(\cdot \mid s,a)$.}}, and an observation (or realisation) is a path of the MDP, i.e., a sequence of transitions $(s_t, a_t, s_{t+1})$. Counterfactual inference is applied at the level of the individual timesteps: given the observed transition $s_t, a_t \rightarrow s_{t+1}$ at time $t$, we can evaluate the counterfactual probability of reaching outcome $S_{t+1} = \tilde{s}’$, had we been in state $S_t = \tilde{s}$ and performed action $A_t = \tilde{a}$. Formally, we define this counterfactual probability $\tilde{P}_t(\tilde{s}' \mid \tilde{s}, \tilde{a})$ as:
\[
\tilde{P}_t(\tilde{s}' \mid \tilde{s}, \tilde{a}) = P(\interv{S_{t+1}}{S_t}{\tilde{s}}{A_t}{\tilde{a}} = \tilde{s}' \mid S_t = s_t, A_t = a_t, S_{t+1} = s_{t+1}),
\]
where $\interv{S_{t+1}}{S_t}{\tilde{s}}{A_t}{\tilde{a}}$ denotes variable $S_{t+1}$ in the SCM after the intervention $S_t \gets \tilde{s}$ and $A_t \gets \tilde{a}$.

These counterfactual probabilities can be combined across multiple timesteps to construct a non-stationary \textit{counterfactual} MDP \citep{tsirtsis2021counterfactual}. This construction allows us to reason about counterfactual outcomes given an observed path.}
\jl{\begin{definition}[Counterfactual MDP (CFMDP)]
Given an observed path $\tau$ and MDP $\mathcal{M}$, the corresponding counterfactual MDP is a tuple ${\tilde{\mathcal{M}}_{\tau}}=(\mathcal{S}^{+}, \mathcal{A}, \tilde{P}, \tilde{P}_{I}, \mathcal{R}^{+})$ where $\mathcal{S}^{+}$ and $\mathcal{R}^{+}$ are the state space and reward function of $\mathcal{M}$, augmented with time; and:
\begin{enumerate}
    \item $\tilde{P}_I(s) = 1$ if $s=s_0$, $0$ otherwise;
    \item $\tilde{P}_{t}(\tilde{s}' \mid \tilde{s}, \tilde{a}) = P({S_{t+1}}_{(S_t=\tilde{s}, A_t=\tilde{a})} = \tilde{s}' \mid S_t = s_t, A_t = a_t, S_{t+1} = s_{t+1}), \forall t \in \{0,\ldots, |\tau|-1\}$.
\end{enumerate}
\end{definition}}
\jl{Counterfactual inference in MDPs is challenging because there are typically many SCMs consistent with the same observation and interventional distribution, yet they produce different counterfactual probabilities \citep{pmlr-v162-zhang22ab}. Most existing works avoid this problem by assuming a particular causal model to make counterfactual probabilities identifiable.} \jl{One such model is the Gumbel-max SCM \cite{oberst2019counterfactual}, which has been widely used for counterfactual analysis in MDPs }\citep{benz2022counterfactual,
kazemi2024counterfactual,
killian2022counterfactually, noorbakhsh2022counterfactual,
tsirtsis2021counterfactual, zhu2020counterfactual} due to its desirable property of counterfactual stability (see Definition \ref{def: counterfactual stability}). \jl{Notably, \citet{tsirtsis2021counterfactual} applied the Gumbel-max SCM to derive alternative sequences of actions, i.e., counterfactual paths, that would have produced a higher reward than the observed path\footnote{See Appendix \ref{app: gumbel} for more details on the construction of counterfactual MDPs using the Gumbel-max SCM. There also exist other works which apply bijective causal models \citep{nasr2023counterfactual, tsirtsis2024finding}, or genetic algorithms \citep{gajcin2024acter} to generate counterfactual paths.}. In this sense, counterfactual paths can be seen as \textit{counterfactual explanations} for how the observed policy could have been improved.
However, since the Gumbel-max SCM is just one of many possible SCMs compatible with a given MDP, the derived counterfactual MDP may be inaccurate, leading to unreliable counterfactual explanations.
}

\paragraph{Partial Counterfactual Inference}
Partial counterfactual inference methods bound counterfactual outcomes by considering all SCMs compatible with observational and interventional data. Key works are summarised in Table \ref{table: partial counterfactual inference}. Notably, the \textit{canonical SCM} approach, developed by \citet{pmlr-v162-zhang22ab}, constructs a linear or polynomial optimisation problem (depending on the causal graph structure) that identifies exact counterfactual probability bounds. However, this optimisation is very inefficient, as the number of linear constraints grows exponentially with the number and cardinality of the endogenous variables. Thus, research has shifted from exact methods to developing approximation methods. One notable exception is recent work by \citet{li2024probabilities}, who identified counterfactual probability bounds for categorical systems, given observational and interventional distributions. \jl{While their bounds are exact (that is, capturing all causal models compatible with the observation and interventional distributions), their bounds are often uninformative, i.e., very wide or trivial ($[0,1]$), because they make no additional assumptions about the underlying causal model.}
In a multi-step setting such as MDPs, having many transitions with loose or trivial bounds can quickly compound over several time steps, severely limiting the usefulness of counterfactual inference. \jl{In this paper, we propose adding two common and reasonable assumptions to tighten the bounds and improve the usefulness of counterfactual inference. We discuss the similarities and differences between their work and ours in more detail in Section \ref{sec: bounds}.
}
\begin{table}
\centering
\resizebox{0.8\linewidth}{!}{
\begin{tabular}{|l|l|l|}
\hline
& \textbf{Binary} & \textbf{Categorical} \\ 
& \textbf{outcomes} & \textbf{outcomes} \\ \hline
\textbf{Exact bounds}       & \citep{balke1994counterfactual, kang2006inequality, cai2008bounds}       & \citep{pmlr-v162-zhang22ab, li2024probabilities}, \textbf{ours} \\ \hline
\textbf{Approximate bounds} & \citep{robins1989analysis, manski1990nonparametric} & \citep{zaffalon2020, zaffalon2021, zaffalon2022bounding, zaffalon2023approximating, zaffalon2024, duarte2023, pmlr-v162-zhang22ab}\tablefootnote{\citet{pmlr-v162-zhang22ab} provide both an exact and approximation method in their paper.}            \\ \hline
\end{tabular}
}
\caption{Related works on partial counterfactual inference.}
\label{table: partial counterfactual inference}
\end{table}

\section{Partial Counterfactual Inference via Canonical SCMs}
\label{sec: partial CF inference}
\citet{pmlr-v162-zhang22ab} introduced a family of canonical SCMs capable of capturing all possible counterfactual distributions for any causal graph. \jl{Their partial counterfactual inference approach optimises over these canonical SCMs to identify exact counterfactual probability bounds}. In this section, we demonstrate how this approach can be applied to MDPs, and how additional assumptions can be incorporated into the optimisation problem to ensure counterfactuals are plausible and useful, addressing issues in existing work.

\subsection{Optimisation Procedure}
\jl{To perform partial counterfactual inference in MDPs, we must convert the MDP SCM \eqref{eq:mdp_scm} to its equivalent canonical SCM. To this end, we need to identify the c-component of the exogenous variable $U_t$ in the MDP causal graph (Figure \ref{fig:MDP_DAG}), which is defined as follows:}

\begin{definition}[c-component \citep{tian2002general}]
    Given a causal graph $\mathcal{G}$, a subset of its endogenous variables $\mathbf{C} \subseteq \mathbf{V}$ is a \textbf{c-component} if any two variables $V_i, V_j \in \mathbf{C}$ are connected by a sequence of bi-directed edges $V_i \leftarrow U_k$ and $V_j \leftarrow U_k$, where each $U_k$ is an exogenous parent shared by $V_i$ and $V_j$.
\end{definition}
\begin{definition}[Canonical SCM \citep{pmlr-v162-zhang22ab}]
    A \textbf{canonical SCM} is a tuple $\mathcal{C} = \langle \boldsymbol{V}, \boldsymbol{U}, \mathcal{F}, P \rangle$, where:
    \begin{enumerate}
        \item For every endogenous variable $V \in \boldsymbol{V}$, its values $v$ are determined by a structural equation $v \leftarrow f_V(pa_V, u_V)$ where, for any $pa_V$ and $u_V$, $f_V(pa_V, u_V)$ is contained within a finite domain $\Omega_V$.
        \item For every exogenous $U \in \boldsymbol{U}$, its values $u$ are drawn from a finite domain $\Omega_U$, where the cardinality of \jl{$U$} is equal to the total number of functions that map all possible inputs $pa_V \in \Omega_{\mathit{PA}_V}$ to values $v \in V$ for every endogenous $V$ in the c-component covering $U$\footnotemark, i.e., $
        \jl{|U|} = \prod_{V \in \mathbf{C}(U)}|\Omega_{{PA}_V} \mapsto \Omega_V|
        $
    \end{enumerate}
\end{definition}
\footnotetext{$\Omega_{\mathit{PA}_V} \mapsto \Omega_V$ denotes the set of all possible mappings from values in $\Omega_{{\mathit{PA}_V}}$ to values in $\Omega_V$.}
The c-component covering the single exogenous variable $U_t$ in an MDP is $\mathbf{C}(U_t) = \{S_{t+1}\}$. Therefore, for an MDP, \jl{\( |U| \)} is equal to the total number of functions that map all possible values of \( S_t \) and \( A_t \) (i.e., all possible combinations of states \( s \in \mathcal{S} \) and actions \( a \in \mathcal{A} \)) to all possible values of \( S_{t+1} \). Thus, the cardinality of \( U_t \) in the canonical MDP SCM is $|{U_t}| = |\mathcal{S}|^{|\mathcal{S}| \times |\mathcal{A}|}
$. Each value $u_t \in U_t$ indexes a unique structural equation, which deterministically maps all possible state-action pairs $(s, a)$ to a next state $s'$.

Given the canonical SCM representation of an MDP and an observed transition $s_t, a_t \rightarrow s_{t+1}$, we can define an optimisation procedure to find the minimum and maximum counterfactual probabilities for every transition \citep{pmlr-v162-zhang22ab}. The first step is to define a mapping between exogenous values $u_t \in U_t$ and the structural equation they index. We define an indicator variable $\mu \in \{0,1\}^{|\mathcal{S}| \times |\mathcal{A}| \times |{U_t}| \times |\mathcal{S}|}$ such that, for any $s_t, a_t, u_t$ and $s_{t+1}$:
\[
\mu_{s_t, a_t, u_t, s_{t+1}} = 
\begin{cases}
    1 & \text{if $f(s_t, a_t, u_t) = s_{t+1}$}\\
    0 & \text{otherwise}\\
\end{cases}
\]
Next, we define a vector $\theta \in \mathbb{R}^{|{U_t}|}$, where each $\theta_{u_t}$ represents the probability $P(U_t = u_t)$. Each instantiation of $\theta$ defines a unique SCM. Given an observed transition $s_t, a_t \rightarrow s_{t+1}$, we can write the counterfactual probability $\tilde{P}_t(\tilde{s}' \mid \tilde{s}, \tilde{a})$ of any transition $\tilde{s}, \tilde{a} \rightarrow \tilde{s}'$ in terms of $\mu$ and $\theta$:
\begin{equation}\label{eq:cf_probs}
    \tilde{P}_t(\tilde{s}' \mid \tilde{s}, \tilde{a}) = \frac{\sum_{u_t = 1}^{|U_t|} \mu_{\tilde{s}, \tilde{a}, u_t, \tilde{s}'} \cdot \mu_{s_t, a_t, u_t, s_{t+1}} \cdot \theta_{u_t}}{P(s_{t+1} \mid s_t, a_t)}
\end{equation}

\jl{To identify the counterfactual probability bounds, we search over all values of $\theta$ consistent with the observed transition and interventional data (i.e., the MDP's transition probabilities). This optimisation problem is defined as:}
{
\begin{equation}
\label{eq: optimisation}
\begin{aligned}
{\min / \max}_{\theta} \sum_{u_t = 1}^{|U_t|} \mu_{\tilde{s}, \tilde{a}, u_t, \tilde{s}'} \cdot \mu_{s_t, a_t, u_t, s_{t+1}} \cdot \theta_{u_t}\\
\textrm{s.t.} \sum_{u_t = 1}^{|U_t|} \mu_{s, a, u_t, s'} \cdot \theta_{u_t} = P(s' \mid s, a), \forall s, a, s'\\
0 \leq \theta_{u_t} \leq 1, \forall u_t, \qquad
\sum_{u_t = 1}^{|U_t|} \theta_{u_t} = 1
\end{aligned}
\end{equation}
}

\subsection{Incorporating Additional Assumptions}
While this optimisation correctly considers all SCMs consistent with the interventional and observational data, it can result in wide or trivial $[0,1]$ bounds. As previously discussed, having many MDP transitions with loose or trivial bounds can severely limit the usefulness of counterfactual inference in MDP settings. \jl{However, incorporating reasonable assumptions about the causal model can help to tighten counterfactual probability bounds enough to derive useful (but still robust) counterfactual policies. In this work, we adopt two assumptions (similar to those in \citep{haugh2023bounding, lorberbom2021learning}) \jl{when deriving counterfactual MDPs}.} The first assumption is \textit{counterfactual stability} \citep{oberst2019counterfactual}, which, in the context of an MDP, is defined as:
\begin{definition}[Counterfactual stability]\label{def: counterfactual stability} An MDP SCM \eqref{eq:mdp_scm} satisfies \textit{counterfactual stability} \jl{if, given we observed the transition $s_t, a_t \rightarrow s_{t+1}$ at time $t$}, the counterfactual outcome under a different state-action pair $(\tilde{s}, \tilde{a})$ will not change to some $S_{t+1} = \tilde{s}'\neq s_{t+1}$ unless $\dfrac{P(s_{t+1} \mid \tilde{s}, \tilde{a})}{P(s_{t+1} \mid s_t, a_t)} < \dfrac{P(\tilde{s}' \mid \tilde{s}, \tilde{a})}{P(\tilde{s}' \mid s_t, a_t)}$.
\end{definition}
\newcolumntype{P}[1]{>{\centering\arraybackslash}p{#1}}
\renewcommand{\arraystretch}{0.9} %
\begin{figure}
    \centering
    \resizebox{0.55\linewidth}{!}{%
            \begin{circuitikz}
            \tikzstyle{every node}=[font=\fontsize{32}{36}\selectfont]
            \draw  (10.75,18.5) circle (1.25cm) node {\fontsize{32}{36}\selectfont $s_1$} ;
            \draw [ fill={rgb,255:red,177; green,170; blue,170} ] (4.25,18.5) circle (1.25cm) node {\fontsize{32}{36}\selectfont $s_0$} ;
            \draw  (16.75,18.5) circle (1.25cm) node {\fontsize{32}{36}\selectfont $s_2$} ;
            \draw [ fill={rgb,255:red,177; green,170; blue,170} ] (10.75,14.75) circle (1.25cm) node {\fontsize{32}{36}\selectfont $s_1$} ;
            \draw  (4.25,14.75) circle (1.25cm) node {\fontsize{32}{36}\selectfont $s_0$} ;
            \draw  (16.75,14.75) circle (1.25cm) node {\fontsize{32}{36}\selectfont $s_2$} ;
            \draw [->, >=Stealth] (4.25,17.25) -- (4.25,16)node[pos=0.5, fill=white]{0.3};
            \draw [->, >=Stealth] (16.75,17.25) -- (16.75,16)node[pos=0.5, fill=white]{1.0};
            \draw [->, >=Stealth] (10,17.5) -- (5.5,15);
            \draw [->, >=Stealth] (11.5,17.5) -- (15.5,15);
            \draw [->, >=Stealth] (5.25,17.75) -- (15.5,14.75);
            \draw [->, >=Stealth] (5,17.5) -- (9.5,14.75);
            \node [] at (5.75,16.5) {0.4};
            \node [] at (9,17.5) {0.4};
            \node [] at (6.75,17.75) {0.3};
            \node [] at (12.5,17.5) {0.6};
            \end{circuitikz}
        }
    \caption{Example MDP where Gumbel-max produces unintuitive CF probabilities. The observed path is $s_0 \rightarrow s_1$.}
    \label{fig:gumbel-max-scm-unintuitive-probs}
\end{figure}
However, even causal models that assume counterfactual stability can yield unreasonable counterfactual probabilities. Consider the toy MDP in Figure \ref{fig:gumbel-max-scm-unintuitive-probs}. Table \ref{tab:gumbel-max-scm} compares the counterfactual transition probabilities obtained from the optimisation approach in Eq.~\eqref{eq: optimisation} (with no assumptions) and the Gumbel-max SCM (which satisfies counterfactual stability). \camera{Under the Gumbel-max SCM,} the counterfactual probability of transition $s=1, a=0 \rightarrow s'=2$ (highlighted in Table \ref{tab:gumbel-max-scm}) is greater than its nominal probability, even though state $s'=2$ was reachable from the observed state $s_t=0$, but not observed (see Appendix \ref{app: unintuitive probs explanation} for further explanation). Arguably, a state that was reachable from the observed state-action pair, but was not observed, should not become more likely in the counterfactual world (i.e., under the counterfactual state-action pair). To formalise this intuition, we introduce \textit{counterfactual monotonicity}\footnote{This is different from other definitions of monotonicity, which assume an ordering on interventions and outcomes, e.g., \citep{vlontzos2023estimating}.}:
\begin{definition}[Counterfactual monotonicity]
    \jl{An MDP SCM \eqref{eq:mdp_scm} satisfies \textit{counterfactual monotonicity} if, upon observing the transition $s_t, a_t \rightarrow s_{t+1}$, then, $\forall \tilde{s} \in \mathcal{S}, \tilde{a} \in \mathcal{A}$: 
    \begin{enumerate}[label=(Mon\arabic*)]
        \item $\tilde{P}_t(s_{t+1} \mid \tilde{s}, \tilde{a}) \geq P(s_{t+1} \mid \tilde{s}, \tilde{a})$ (i.e., observing an outcome cannot make it less likely in the counterfactual world), and
        \item $\forall \tilde{s}'\neq s_{t+1}$ with $P(\tilde{s}' \mid s_t, a_t)>0$, $\tilde{P}_t(\tilde{s}' \mid \tilde{s}, \tilde{a}) \leq P(\tilde{s}' \mid \tilde{s}, \tilde{a})$ (i.e., not observing a possible outcome cannot make it more likely in the counterfactual world).
    \end{enumerate}
  }  
\end{definition}
\begin{table}[H]
    \centering
    \resizebox{0.95\linewidth}{!}{
        \begin{tabular}{|c|c|c|c|P{0.7cm}|P{0.7cm}|c|P{0.7cm}|P{0.7cm}|}
        \hline
        \multirow{2}{*}{\text{$s$}} & \multirow{2}{*}{\text{$a$}} & \multirow{2}{*}{\text{$s'$}} & \multirow{2}{*}{\text{$P(s' \mid s, a)$}} & \multicolumn{2}{c|}{\makecell{\text{Optimisation} \\ \text{(\ref{eq: optimisation})}}} & \multirow{2}{*}{\makecell{\text{Gumbel-} \\ \text{Max \citep{oberst2019counterfactual}}}} & \multicolumn{2}{c|}{\makecell{\text{Optimisation} \\ \text{\eqref{eq: optimisation}-\eqref{eq:mon2}}}} \\ \cline{5-6} \cline{8-9} 
                                    &                                 &                 &                     & \text{LB}                             & \text{UB}                            &                                      & \text{LB}                              & \text{UB}                             \\ \Xhline{1pt}
        0                               & 0                               & 0      & 0.3                             & 0.0                                     & 0.0                                    & 0.0                                  & 0.0                                      & 0.0                                     \\ \hline
        0                               & 0                               & 1      & 0.4                             & 1.0                                     & 1.0                                    & 1.0                                  & 1.0                                      & 1.0                                     \\ \hline
        0                               & 0                               & 2      & 0.3                             & 0.0                                     & 0.0                                    & 0.0                                  & 0.0                                      & 0.0                                     \\ \hline
        1                               & 0                               & 0      & 0.4                             & 0.0                                     & 1.0                                    & 0.35                                 & 0.4                                      & 0.4                                     \\ \hline
        1                               & 0                               & 1      & 0.0                             & 0.0                                     & 0.0                                    & 0.0                                  & 0.0                                      & 0.0                                     \\ \hline
        \textbf{1}                      & \textbf{0}                      & \textbf{2} & \textbf{0.6}             & \textbf{0.0}                            & \textbf{1.0}                           & \textbf{0.65}                        & \textbf{0.6}                             & \textbf{0.6}                            \\ \hline
        2                               & 0                               & 0      & 0.0                             & 0.0                                     & 0.0                                    & 0.0                                  & 0.0                                      & 0.0                                     \\ \hline
        2                               & 0                               & 1      & 0.0                             & 0.0                                     & 0.0                                    & 0.0                                  & 0.0                                      & 0.0                                     \\ \hline
        2                               & 0                               & 2      & 1.0                             & 1.0                                     & 1.0                                    & 1.0                                  & 1.0                                      & 1.0                                     \\ \hline
        \end{tabular}
        }
    \caption{Counterfactual transition probabilities produced by various methods.}
    \label{tab:gumbel-max-scm}
\end{table}
The counterfactual stability \eqref{eq:cs} and monotonicity assumptions (\ref{eq:mon1},\ref{eq:mon2}) can be added as constraints to the optimisation problem \eqref{eq: optimisation} as:
\begin{alignat}{2}
&\tilde{P}_t(\tilde{s}' \mid \tilde{s}, \tilde{a}) = 0 &&\quad \text{if } \frac{P(s_{t+1} \mid \tilde{s}, \tilde{a})}{P(s_{t+1} \mid s_t, a_t)} \geq \frac{P(\tilde{s}' \mid \tilde{s}, \tilde{a})}{P(\tilde{s}' \mid s_t, a_t)} \notag \\
& &&\quad \text{and } P(\tilde{s}' \mid s_t, a_t) > 0, \notag \\ & &&\quad \forall (\tilde{s}, \tilde{a})\neq(s_{t}, a_t), \forall \tilde{s}' \neq s_{t+1} \label{eq:cs} \\[6pt]
&\tilde{P}_t(s_{t+1} \mid \tilde{s}, \tilde{a}) \geq P(s_{t+1} \mid \tilde{s}, \tilde{a}) &&\quad \text{if } P(s_{t+1} \mid \tilde{s}, \tilde{a}) > 0 \label{eq:mon1} \\[6pt]
&\tilde{P}_t(\tilde{s}' \mid \tilde{s}, \tilde{a}) \leq P(\tilde{s}' \mid \tilde{s}, \tilde{a}) &&\quad \text{if } P(\tilde{s}' \mid s_t, a_t) > 0, \forall \tilde{s}' \neq s_{t+1} \label{eq:mon2}
\end{alignat}

\section{Deriving Analytical Bounds}
\label{sec: bounds}
\jl{While the optimisation in Eq. \eqref{eq: optimisation}-\eqref{eq:mon2} yields exact bounds, it can be very inefficient as the number of constraints grows exponentially with the sizes of the state and action spaces. However, in the MDP setting (Markovian, no unobserved confounders), we have proven that this linear optimisation problem always reduces to exact closed-form solutions, illustrated below}.

\subsection{Analytical Bounds}
\jl{The support of a state-action pair $(\tilde{s}, \tilde{a})$ is defined as the set of next states with nonzero probability when action $\tilde{a}$ is taken in state $\tilde{s}$.} Given the observed transition $s_t, a_t \rightarrow s_{t+1}$, the counterfactual probability of any transition $\tilde{s}, \tilde{a} \rightarrow \tilde{s}'$ depends only on whether the state-action pair $(\tilde{s}, \tilde{a})$ is the observed state-action pair $(s_t, a_t)$, or whether its support is disjoint from or overlaps with the support of the observed state-action pair. The lower and upper bounds, denoted by $\tilde{P}_{t}^{\it LB}$ and $\tilde{P}_{t}^{\it UB}$, are given in Theorems \ref{theorem: observed state}-\ref{theorem:ub overlapping}. \camera{For clarity, we provide a proof sketch below and defer to the full proof in Appendix \ref{sec: probability bounds proof}.}

\begin{theorem}
\label{theorem: observed state}
For the observed state-action pair $(s_t, a_t)$, the linear program will produce the following bounds:
{
\[\tilde{P}_{t}^{\it LB}(s_{t+1} \mid s_t, a_t) = \tilde{P}_{t}^{\it UB}(s_{t+1} \mid s_t, a_t) = 1\] 
\[\forall \tilde{s}' \in \mathcal{S} \setminus \{s_{t+1}\}, \tilde{P}_{t}^{\it LB}(\tilde{s}' \mid s_t, a_t) = \tilde{P}_{t}^{\it UB}(\tilde{s}' \mid s_t, a_t) = 0\]
}
\end{theorem}

\begin{theorem}
\label{theorem: ub disjoint}
For state-action pairs $(\tilde{s}, \tilde{a})\neq (s_t, a_t)$ which have disjoint support from the observed $(s_t, a_t)$, the linear program will produce, $\forall \tilde{s}' \in \mathcal{S}$, the following bounds: 
\[
\tilde{P}_{t}^{\it UB}(\tilde{s}' \mid \tilde{s}, \tilde{a}) =  
\begin{cases}
\dfrac{P(\tilde{s}' \mid \tilde{s}, \tilde{a})}{P(s_{t+1} \mid s_t, a_t)} & \textnormal{ if }P(\tilde{s}' \mid \tilde{s}, \tilde{a}) < P(s_{t+1} \mid s_t, a_t) \\
1 &\textnormal{ otherwise}
\end{cases}
\]

\[
\tilde{P}_{t}^{\it LB}(\tilde{s}' \mid \tilde{s}, \tilde{a}) =  
\begin{cases}
\dfrac{%
  P(\tilde{s}' \mid \tilde{s}, \tilde{a}) - \bigl(1 - P(s_{t+1} \mid s_t, a_t)\bigr)%
}{%
  P(s_{t+1} \mid s_t, a_t)%
} 
& 
\begin{aligned}
&\textnormal{if } P(\tilde{s}' \mid \tilde{s}, \tilde{a}) > \\&1 - P(s_{t+1} \mid s_t, a_t)
\end{aligned} \\[1.5mm]
0 & \textnormal{otherwise.}
\end{cases}
\]
\end{theorem}

\begin{theorem}
\label{theorem:ub overlapping}
For state-action pairs $(\tilde{s}, \tilde{a})\neq (s_t, a_t)$ which have overlapping support with the observed $(s_t, a_t)$, the linear program will produce, $\forall \tilde{s}' \in \mathcal{S}$, the following upper bounds for $\tilde{P}_{t}^{\it UB}(\tilde{s}' \mid \tilde{s}, \tilde{a})$:
\[
\begin{cases}
\dfrac{\min\left(\begin{array}{@{}c@{}} P(s_{t+1} \mid s_t, a_t), P(s_{t+1} \mid \tilde{s}, \tilde{a}) \end{array}\right)}{P(s_{t+1} \mid s_t, a_t)} & \text{ if } \tilde{s}' = s_{t+1} \\
0 & \textnormal{ if CS conditions} \\
\min\left(\begin{array}{@{}c@{}} P(\tilde{s}' \mid \tilde{s}, \tilde{a}), 1 - P(s_{t+1} \mid \tilde{s}, \tilde{a}) \end{array}\right) & \textnormal{ if } P({\tilde{s}' \mid s_t, a_t}) > 0 \\
\min\left(\begin{array}{@{}c@{}} 1 - P(s_{t+1} \mid \tilde{s}, \tilde{a}), \dfrac{P(\tilde{s}' \mid \tilde{s}, \tilde{a})}{P(s_{t+1} \mid s_t, a_t)} \end{array}\right) & \textnormal{ otherwise}
\end{cases}
\]

{
and the following lower bounds for $\tilde{P}_{t}^{\it LB}(\tilde{s}' \mid \tilde{s}, \tilde{a})$:
\[
\begin{cases}
\max\left(P(\tilde{s}' \mid \tilde{s}, \tilde{a}), 1 - \sum_{s' \in \mathcal{S}\setminus\{\tilde{s}'\}}{\tilde{P}_{t}^{\it UB}(s' \mid \tilde{s}, \tilde{a})}\right) & \textnormal{ if $\tilde{s}'= s_{t+1}$}\\
0 & \textnormal{ if CS conditions}\\
\max\left(0, 1 - \sum_{s' \in \mathcal{S}\setminus\{\tilde{s}'\}}{\tilde{P}_{t}^{\it UB}(s' \mid \tilde{s}, \tilde{a})}\right) & \textnormal{ otherwise}
\end{cases}
\]
}
\noindent where the counterfactual stability (CS) conditions are $P({\tilde{s}' \mid s_t, a_t}) > 0$  and  $\dfrac{P(s_{t+1} \mid \tilde{s}, \tilde{a})}{P(s_{t+1} \mid s_t, a_t)} \geq \dfrac{P(\tilde{s}' \mid \tilde{s}, \tilde{a})}{P({\tilde{s}' \mid s_t, a_t})}$.
\end{theorem}
\camera{
\paragraph{Proof Sketch}
Consider an observed transition $s_t, a_t \rightarrow s_{t+1}$ and counterfactual transition $\tilde{s}, \tilde{a} \rightarrow \tilde{s}'$. For a fixed $\theta$, we can compute the counterfactual probability $\tilde{P}_t(\tilde{s}' \mid \tilde{s}, \tilde{a})$ as in Eq. \eqref{eq:cf_probs}. Since $P(s_{t+1} \mid s_t, a_t)$ is fixed, $\tilde{P}_t(\tilde{s}' \mid \tilde{s}, \tilde{a})$ depends only on:
\begin{equation}
    \sum_{u_t = 1}^{|U_t|} \mu_{\tilde{s}, \tilde{a}, u_t, \tilde{s}'} \cdot \mu_{s_t, a_t, u_t, s_{t+1}} \cdot \theta_{u_t} = \sum_{\substack{u_t \in U_t \\f(s_t, a_t, u_t) = s_{t+1} \\ f(\tilde{s}, \tilde{a}, u_t) = \tilde{s}'}} \theta_{u_t}
\label{eq: optimisation sum}
\end{equation}
Therefore, to minimise/maximise the counterfactual probability of a particular transition, we choose $\theta$ to minimise/maximise the value of this sum. The main challenge is that $\theta$ must also satisfy the optimisation constraints for \emph{all} state-action pairs in the MDP. To address this, the proof proceeds in two steps:
\begin{enumerate}
    \item \textbf{Induction over state-action pairs (\ref{sec: probability bounds induction})}. We show that, for any MDP $\mathcal{M}$, if we can identify assignments of $\theta$ that satisfy the constraints individually for all $|\mathcal{S}| \times |\mathcal{A}|$ state-action pairs in $\mathcal{M}$, then these assignments can be combined to form a single $\theta$ that satisfies the constraints of all $|\mathcal{S}| \times |\mathcal{A}|$ state-action pairs simultaneously.

    \item \textbf{Closed-form Probability Bounds (\ref{sec: probability bounds cases})} 
    We then analyse how to minimise/maximise a particular counterfactual transition probability, $\tilde{P}_t(\tilde{s}’ \mid \tilde{s}, \tilde{a})$.
    When the supports of the observed and counterfactual state-action pairs overlap, the counterfactual stability and counterfactual monotonicity assumptions influence the counterfactual probability bounds, leading to different closed-form expressions. In D.1.3 we systematically cover all possible cases in an MDP and prove, subject to the constraints of the optimisation problem, the minimum and maximum values of Eq. \eqref{eq: optimisation sum} to identify the counterfactual probability bounds.
\end{enumerate}
}

\paragraph{Equivalence with Existing Work} \jl{As discussed in Section \ref{sec: background}, \citet{li2024probabilities} have also derived counterfactual probability bounds for categorical models, given observational and interventional distributions. In this work, we have independently derived equivalent counterfactual probability bounds in the MDP setting, using the canonical SCM framework \citep{pmlr-v162-zhang22ab}. However, unlike \citet{li2024probabilities}, we additionally incorporate reasonable assumptions about the causal model to tighten the counterfactual probability bounds and enhance the practical utility of counterfactual inference for MDPs. In Appendix \ref{app: equivalence proofs}, we prove our closed-form bounds are equivalent to the bounds of \citet{li2024probabilities} when the counterfactual stability and monotonicity assumptions are removed (or equivalently, for the case where the state-action pair has \textit{disjoint support} with the observed state-action pair\footnote{When the state-action pair has disjoint support, counterfactual stability and monotonicity hold vacuously.}).} This equivalence reaffirms the correctness of our bounds and suggests we could reformulate them for an observational distribution \jl{(as opposed to a single observed path)}, which would be an interesting future direction.
\paragraph{Flexibility of Assumptions}
\jl{While the bounds given in Theorems \ref{theorem: observed state}-\ref{theorem:ub overlapping} incorporate the assumptions of counterfactual stability and counterfactual monotonicity, our approach is flexible: these assumptions can be removed if they do not hold in a particular environment (for example, if a domain expert determines they are not applicable). This modularity ensures that our approach can adapt to a variety of settings without requiring major changes to the underlying procedure. Appendix \ref{app: assumption plausibility} provides further discussion on the plausibility of our assumptions, as well as closed-form solutions for the case where only the standard counterfactual stability assumption is adopted and for the case where both assumptions are removed.}

\section{Robust counterfactual policies}
With our analytical bounds, we can compute a non-stationary interval counterfactual MDP for any MDP $\mathcal{M}$ and observed path $\tau$. Formally, an interval Markov decision process (IMDP) \citep{GIVAN200071} is a tuple $(\mathcal{S}, \mathcal{A}, P_{\updownarrow}, \mathcal{R})$, extending a standard MDP with an uncertain transition probability function $P_{\updownarrow}$ that maps each transition to a probability within its bounds $P_{\updownarrow}(s' \mid s, a) = [P^{\it LB}(s' \mid s, a), P^{\it UB}(s' \mid s, a)]$. To construct an interval counterfactual MDP, we compute the counterfactual probability bounds for every transition in the MDP, given each observed transition in $\tau$:
\begin{definition}[Interval Counterfactual MDP (ICFMDP)]
Given an observed path $\tau$ and MDP $\mathcal{M}$, the counterfactual probability for each transition in the interval counterfactual MDP (ICFMDP) ${\tilde{\mathcal{M}}^{\updownarrow}_{\tau}}$ is defined, for $t=0,\ldots,T-1$, as $\tilde{P}^{\updownarrow}_{t}(\tilde{s}' \mid \tilde{s}, \tilde{a}) = [\tilde{P}_{t}^{\it LB}(\tilde{s'} \mid \tilde{s}, \tilde{a}), \tilde{P}_{t}^{\it UB}(\tilde{s'} \mid \tilde{s}, \tilde{a})]$.
\end{definition}
To derive optimal counterfactual policies for this ICFMDP, we apply robust value iteration  \citep{mathiesen2024intervalmdp}, which pessimistically optimises the expected reward over the worst-case CFMDP within the ICFMDP:
\begin{equation}\label{eq:pessimistic_v}
V^{*}_t(s) = \max_{\pi} \min_{\tilde{P}_t \in \tilde{P}^{\updownarrow}_{t}}\mathbb{E}_{s' \sim \tilde{P}_t(\cdot \mid s, \pi(s))}\left[R(s, \pi(s)) + V^{*}_{t+1}(s')\right]
\end{equation}
\camera{This yields a \textit{robust} counterfactual policy: its performance on the true (unknown) causal model is guaranteed to be at least as good as the worst-case performance over the ICFMDP. Moreover, as we prove in Appendix \ref{sec: sampled interval cfmdp proof}, any CFMDP entailed by the ICFMDP is a valid (i.e., the ICFMDP formulation does not introduce spurious CFMDPs), meaning the value function bounds are tight}\footnote{For some applications, one might be interested in an \textit{optimistic} variant of the problem. This is supported by our method and boils down to maximising (instead of minimising) the expected reward under the best-case CFMDP within the ICFMDP.}.

\definecolor{darkgreen}{rgb}{0.0, 0.5, 0.0}
\definecolor{violet}{rgb}{0.56, 0.0, 1.0}
\section{Evaluation}
We apply our methodology to derive counterfactual policies for various MDPs, addressing four main research questions: 
(1) how does our CF inference approach perform in off-policy evaluation problems; 
(2) how does our policy's performance and robustness compare to the Gumbel-max SCM approach; (3) how do the counterfactual stability and monotonicity assumptions impact the probability bounds; and (4) how fast is our approach compared with the Gumbel-max SCM method? \camera{We conduct experiments in four environments, spanning simple navigation (GridWorld and Frozen Lake), clinical decision-making (Sepsis), and safety-critical control (Aircraft), with varying levels of stochasticity and complexity (see Appendix \ref{app: environments} for details). Experiments were run on a 128-core Intel Xeon CPU (see Appendix \ref{app: training details} for more information), and the code is available at:
\url{https://github.com/ddv-lab/robust-cf-inference-in-MDPs}.}

\subsection{Robust CF Inference and Off-Policy Evaluation (OPE)}
We assess the validity of our robust CF inference method in an OPE setting~\cite{oberst2019counterfactual,buesing2018woulda}, where we want to predict the expected return (cumulative reward) of a target policy $\pi^*$ given only paths observed under a behavioural \camera{policy $\pi$}. In this experiment, $\pi$ is a suboptimal policy and $\pi^*$ an improved one. 
Procedurally, we sample a number of paths (up to $100$) under $\pi$; for each sampled path, we use our method to derive upper and lower bounds on the counterfactual return of that path under $\pi^*$. If our method is unbiased (and enough paths are sampled), then the average of these bounds should contain the true expected return under $\pi^*$. Using the same set of paths, we perform a similar evaluation for the Gumbel-max approach, except that this method yields a crisp return value rather than bounds. 

\jl{
Figure \ref{fig:cf_correctness_gridworld_0.4} illustrates the results of this experiment for the GridWorld ($p=0.4$) MDP. As expected, the averaged pessimistic and optimistic returns obtained from our approach correctly bound the true return under the target policy. The Gumbel-max approach also closely approximates the target return, indicating that it is also a correct and unbiased approach. However, as we demonstrate in the following sections, the Gumbel-max approach is less robust to causal model uncertainty compared to our method. 
}

\begin{figure}[h]
  \centering
  \resizebox{0.7\linewidth}{!}{%

  \begin{tikzpicture}
    \begin{axis}[
        xlabel={Number of sampled behavioural paths},
        ylabel={Expected return},
        xtick={0,20, 40, 60, 80, 100},
        xmin=0, xmax=105,
        ymin=-350, ymax=250,
        grid=both,
        width=\columnwidth,
        height=0.6\columnwidth,
        label style={font=\Large},          %
        tick label style={font=\Large},     %
        legend style={
            at={(0.5,-0.3)},
            anchor=north,
            font=\Large,
        },
        every axis plot/.append style={thick, mark size=3pt, line width=1.2pt},
    ]

    \addplot+[
        mark=o,
        color=blue,
    ] coordinates {
        (1, -15.212) +- (0, 0)
        (10, -73.48405927272725) +- (87.67556429452748, 87.67556429452748)
        (20, -129.686975651135) +- (128.49656293499737, 128.49656293499737)
        (40, -116.59003952070941) +- (110.76950637683635, 110.76950637683635)
        (60, -119.58905890984825) +- (104.72948204744685, 104.72948204744685)
        (80, -119.22869654166988) +- (111.52668072637644, 111.52668072637644)
        (100, -120.7506028565535) +- (119.01032720226385, 119.01032720226385)
    };
    \addlegendentry{Average CF return bounds (ours)}

    \addplot+[
        mark=x,
        color=red,
    ] coordinates {
        (1, 5.372312802224) +- (0, 0)
        (10, -0.7498667513463873) +- (35.29153241540447, 35.29153241540447)
        (20, -28.66833271087721) +- (119.30356816623626, 119.30356816623626)
        (40, -30.43986491176804) +- (98.27251718278436, 98.27251718278436)
        (60, -23.850403760009538) +- (94.20382165756313, 94.20382165756313)
        (80, -21.73502100716342) +- (98.97485142137681, 98.97485142137681)
        (100, -26.765922803280006) +- (107.82654949437878, 107.82654949437878)
    };
    \addlegendentry{Average CF return (Gumbel-max)}

    \addplot[dotted, green!50!black, thick] coordinates {(0,-40.56411852799998) (100,-40.56411852799998)};
    \addlegendentry{True target return} 
    \addplot+[
        mark=o,
        color=blue,
    ] coordinates {
        (1, 9.2) +- (0, 0)
        (10, 65.75115264) +- (65.98085879353107, 65.98085879353107)
        (20, 76.09143818551972) +- (156.70812134225434, 156.70812134225434)
        (40, 70.21961568150236) +- (124.24118049564295, 124.24118049564295)
        (60, 82.35106626083876) +- (115.55075865894253, 115.55075865894253)
        (80, 86.35395872035576) +- (115.72846189642432, 115.72846189642432)
        (100, 80.29546332706126) +- (125.52914107287741, 125.52914107287741)
    };
    \end{axis}
  \end{tikzpicture}
  }
\caption{CF inference approaches for off-policy evaluation (GridWorld ($p=0.4$))}
\label{fig:cf_correctness_gridworld_0.4}
\end{figure}

\subsection{Performance of Counterfactual Policies}
To compare policy performance, we measure average rewards of counterfactual paths induced by our policy and the Gumbel-max policy by uniformly sampling $200$ counterfactual MDPs from the ICFMDP and generating $10,000$ counterfactual paths over each sampled CFMDP. Since the interval CFMDP depends on the observed path, we select three paths of varying optimality: \camera{a slightly suboptimal path needing minimal changes to reach the goal, a catastrophic path ending in a terminal low-reward state, and an almost catastrophic path that narrowly avoided a catastrophic state. }

\camera{Figures \ref{fig: reward p=0.9}-\ref{fig: reward sepsis} show the average counterfactual rewards induced by our policy vs. the Gumbel-max policy in the GridWorld and Sepsis MDPs (similar results are observed in Frozen Lake and Aircraft, see Appendix \ref{app: additional experiments}). Table \ref{tab: lowest cumulative rewards} reports the worst-case cumulative reward across all induced counterfactual paths, demonstrating that the worst-case cumulative reward obtained by our robust policy is never below that obtained by the Gumbel-max policy.}

\begin{figure*}[h]
    \centering
    \begin{subfigure}{\linewidth}
        \centering
        \begin{tikzpicture}[scale=1.0, every node/.style={scale=1.0}]
            \draw[thick, black] (-3, -0.25) rectangle (10, 0.25);
            \draw[black, line width=1pt] (-2.5, 0.0) -- (-2,0.0);
            \fill[black] (-2.25,0.0) circle (2pt); %
            \node[right] at (-2,0.0) {\small Observed Path};

            \draw[blue, line width=1pt, dashed] (1.0,0.0) -- (1.5,0.0);
            \node[draw=blue, circle, minimum size=4pt, inner sep=0pt] at (1.25,0.0) {}; %
            \node[right] at (1.5,0.0) {\small Interval CFMDP Policy (ours)};

            \draw[red, line width=1pt, dashed] (5.5,0) -- (6,0);
            \node[red] at (5.75,0) {$\boldsymbol{\times}$}; %
            \node[right] at (6,0) {\small Gumbel-max SCM Policy};
        \end{tikzpicture}
    \end{subfigure}
    \\
    \vspace{0.2cm}
    \begin{subfigure}{0.33\linewidth}  
        \begin{tikzpicture}
            \begin{axis}[
                width=\linewidth,
                height=4.2cm,
                xlabel={$t$},
                xtick={0,2,4,6,8},
                ylabel={Mean reward at time step $t$},
                grid=both,
                ymin=-10, ymax=105,
                ]
                \addplot[color=black, mark=*, line width=1pt, mark size=1.5pt]
                coordinates { (0, 0.0) (1, 1.0) (2, 2.0) (3, 2.0) (4, 3.0) (5, 3.0) (6, 4.0) (7, 3.0) (8, 3.0) (9, 4.0) };
                \addplot[color=blue, mark=o,    mark options={solid},line width=1pt, mark size=1.5pt, dashed,
                    error bars/.cd, y dir=both, y explicit, error bar style={line width=1pt,dashed}, error mark options={line width=1pt,mark size=4pt,rotate=90}
                ]
                coordinates { (0, 0.0)  +- (0, 0.0) (1, 1.0)  +- (0, 0.0) (2, 2.0)  +- (0, 0.0) (3, 2.8661345)  +- (0, 0.42603466) (4, 3.798614)  +- (0, 0.52865932) (5, 4.6613305)  +- (0, 1.02983565) (6, 74.7305995) += (0, 25.2694005) -= (0, 42.20242079) (7, 84.802457)  += (0, 15.197543)  -= (0, 35.08575705) (8, 96.455942)  += (0, 3.544058)   -= (0, 18.4617318) (9, 97.9790975) += (0, 2.0209025)  -= (0, 14.29029466) };
                \addplot[color=red, mark=x, line width=1pt, mark size=3pt, dashed,
                    error bars/.cd, y dir=both, y explicit, error bar style={line width=1pt,dashed}, error mark options={line width=1pt,mark size=4pt,rotate=90}
                ]
                coordinates { (0, 0.0)  +- (0, 0.0) (1, 1.0)  +- (0, 0.0) (2, 2.0)  +- (0, 0.0) (3, 2.866351)  +- (0, 0.42523987) (4, 3.799144)  +- (0, 0.52810308) (5, 4.662749)  +- (0, 1.00452614) (6, 74.773081) += (0, 25.226919)   -= (0, 42.17816553) (7, 84.833464) += (0, 15.166536)   -= (0, 35.06110031) (8, 96.4349945) += (0, 3.5650055)  -= (0, 18.51953372) (9, 97.9669475) += (0, 2.0330525)  -= (0, 14.33075378) };
            \end{axis}
        \end{tikzpicture}
        \subcaption{Slightly Suboptimal Path}
    \end{subfigure}
    \hfill
    \begin{subfigure}{0.33\linewidth}
        \begin{tikzpicture}
            \begin{axis}[
                width=\linewidth,
                height=4.2cm,
                ylabel={Mean reward at time step $t$},
                xlabel={$t$},
                xtick={0,2,4,6,8},
                grid=both,
                ymin=-30, ymax=105,
                ]
                \addplot[color=black, mark=*, line width=1pt, mark size=1.5pt]
                coordinates { (0, 0.0) (1, 1.0) (2, 2.0) (3, 1.0) (4, 2.0) (5, 1.0) (6, 2.0) (7, 1.0) (8, 0.0) (9, 1.0) };
                \addplot[color=blue, mark=o,    mark options={solid},line width=1pt, mark size=1.5pt, dashed,
                    error bars/.cd, y dir=both, y explicit, error bar style={line width=1pt,dashed}, error mark options={line width=1pt,mark size=4pt,rotate=90}
                ]
                coordinates { (0, 0.0)  +- (0, 0.0) (1, 1.0)  +- (0, 0.0) (2, 2.0)  +- (0, 0.0) (3, -0.445837)  +- (0, 18.2323085) (4, -0.4585655)  +- (0, 18.40684148) (5, 0.3053175)  +- (0, 18.62866744) (6, 0.0956245)  +- (0, 24.55051607) (7, 71.54727)  += (0, 28.45273)   -= (0, 54.83199135) (8, 76.670842) += (0, 23.329158)  -= (0, 52.9386188) (9, 83.902088) += (0, 16.097912)  -= (0, 49.00201108) };
                \addplot[color=red, mark=x, line width=1pt, mark size=3pt, dashed,
                    error bars/.cd, y dir=both, y explicit, error bar style={line width=1pt,dashed}, error mark options={line width=1pt,mark size=4pt,rotate=90}
                ]
                coordinates { (0, 0.0)  +- (0, 0.0) (1, 1.0)  +- (0, 0.0) (2, 1.96988)  +- (0, 0.21136647) (3, 0.602197)  +- (0, 13.77136924) (4, 1.4765265)  +- (0, 14.02177032) (5, 1.0409715)  +- (0, 17.50019462) (6, 44.898029)  +- (0, 54.28467191) (7, 68.319994)  += (0, 31.680006) -= (0, 52.5192166) (8, 84.6988915) += (0, 15.3011085) -= (0, 44.36762745) (9, 87.540126)  += (0, 12.459874)  -= (0, 42.13634702) };
            \end{axis}
        \end{tikzpicture}
        \subcaption{Almost Catastrophic Path}
    \end{subfigure}
    \hfill
    \begin{subfigure}{0.33\linewidth}
        \begin{tikzpicture}
            \begin{axis}[
                width=\linewidth,
                height=4.2cm,
                ylabel={Mean reward at time step $t$},
                xlabel={$t$},
                xtick={0,2,4,6,8},
                grid=both,
                ymin=-125, ymax=105,
                ]
                \addplot[color=black, mark=*, line width=1pt, mark size=1.5pt]
                coordinates { (0, 1.0) (1, 2.0) (2, -100.0) (3, -100.0) (4, -100.0) (5, -100.0) (6, -100.0) (7, -100.0) (8, -100.0) (9, -100.0) };
                \addplot[color=blue, mark=o,    mark options={solid},line width=1pt, mark size=1.5pt, dashed,
                    error bars/.cd, y dir=both, y explicit, error bar style={line width=1pt,dashed}, error mark options={line width=1pt,mark size=4pt,rotate=90}
                ]
                coordinates { (0, 0.0)  +- (0, 0.0) (1, 0.934776)  +- (0, 0.2469207) (2, 1.836178)  +- (0, 0.41747018) (3, 2.589682)  +- (0, 4.25205246) (4, 3.498908)  +- (0, 4.43656509) (5, 4.294412)  +- (0, 4.79783311) (6, 64.6238925) += (0, 35.3761075) -= (0, 46.69660801) (7, 78.4716035) += (0, 21.5283965) -= (0, 40.66803792) (8, 91.745166)  += (0, 8.254834)  -= (0, 28.04970295) (9, 95.393843)  += (0, 4.606157)  -= (0, 22.08917601) };
                \addplot[color=red, mark=x, line width=1pt, mark size=3pt, dashed,
                    error bars/.cd, y dir=both, y explicit, error bar style={line width=1pt,dashed}, error mark options={line width=1pt,mark size=4pt,rotate=90}
                ]
                coordinates { (0, 0.0)  +- (0, 0.0) (1, 0.934791)  +- (0, 0.24689428) (2, 1.8360415)  +- (0, 0.41772372) (3, 2.596742)  +- (0, 4.16891772) (4, 3.503747)  +- (0, 4.38250693) (5, 4.300445)  +- (0, 4.74124465) (6, 64.6758075) += (0, 35.3241925) -= (0, 46.67192883) (7, 78.483407)  += (0, 21.516593)  -= (0, 40.64493202) (8, 91.75123)   += (0, 8.24877)   -= (0, 28.02106014) (9, 95.407927)  += (0, 4.592073)   -= (0, 22.05066531) };
            \end{axis}
        \end{tikzpicture}
        \subcaption{Catastrophic Path}
    \end{subfigure}
    \caption{Average instant reward of CF paths induced by policies on GridWorld $p=0.9$. Error bars denote the standard deviation in reward at each time step.}
    \label{fig: reward p=0.9}
\end{figure*}

\paragraph{GridWorld ($p=0.9$)}
When $p=0.9$, the counterfactual probability bounds are typically narrow (see Table \ref{tab: bounds} for average measurements). \jl{Consequently, as shown in Figure \ref{fig: reward p=0.9}, both policies are nearly identical and perform similarly well across the slightly suboptimal and catastrophic paths.}
However, for the almost catastrophic path, our interval CFMDP path is more conservative and follows the observed path more closely (as this is where the probability bounds are narrowest), which typically requires one additional step to reach the goal state than the Gumbel-max SCM policy.

\paragraph{GridWorld ($p=0.4$)}
\jl{When $p=0.4$, the GridWorld environment is more uncertain, increasing the risk of entering the dangerous state even if correct actions are chosen. Thus, as shown in Figure \ref{fig: reward p=0.4}, the interval CFMDP policy adopts a more conservative approach, and does not deviate from the observed path if it cannot guarantee higher counterfactual rewards (see the slightly suboptimal and almost catastrophic paths). \jl{The Gumbel-max policy is more inconsistent: it can yield higher rewards, but also much lower rewards, as reflected by the wide error bars.} For the catastrophic path, both policies must \jl{significantly} deviate from the observed path to achieve a higher reward and perform similarly.}
\begin{figure*}
    \centering
    \begin{subfigure}{0.33\linewidth}
        \begin{tikzpicture}
            \begin{axis}[
                width=\linewidth,
                height=4.2cm,
                xlabel={$t$},
                ylabel={Mean reward at time step $t$},
                grid=both,
                xtick={0,2,4,6,8},
            ]
            \addplot[color=black, mark=*, line width=1pt, mark size=1.5pt]
            coordinates {
                (0, 0.0) (1, 1.0) (2, 1.0) (3, 1.0) (4, 2.0) (5, 3.0) (6, 3.0) (7, 2.0) (8, 2.0) (9, 4.0)
            };
            \addplot[color=blue, mark=o,    mark options={solid},dashed, line width=1pt, mark size=1.5pt, error bars/.cd, y dir=both, y explicit, error bar style={line width=1pt,dashed}, error mark options={line width=1pt,mark size=4pt,rotate=90}]
            coordinates {
                (0, 0.0) +- (0, 0.0) (1, 1.0) +- (0, 0.0) (2, 1.0) +- (0, 0.0) (3, 1.0) +- (0, 0.0) (4, 2.0) +- (0, 0.0) (5, 3.0) +- (0, 0.0) (6, 3.0) +- (0, 0.0) (7, 2.0) +- (0, 0.0) (8, 2.0) +- (0, 0.0) (9, 4.0) +- (0, 0.0)
            };
            \addplot[color=red, mark=x, dashed, line width=1pt, mark size=3pt, error bars/.cd, y dir=both, y explicit, error bar style={line width=1pt,dashed}, error mark options={line width=1pt,mark size=4pt,rotate=90}]
            coordinates {
                (0, 0.0) +- (0, 0.0) (1, 1.0) +- (0, 0.0) (2, 1.0) +- (0, 0.0) (3, 1.0) +- (0, 0.0) (4, 2.0) += (0, 0.0) (5, 3.0) += (0, 0.0) (6, 3.17847) += (0, 0.62606746) -= (0, 0.62606746) (7, 2.5832885) += (0, 1.04598233) -= (0, 1.04598233) (8, 5.978909) += (0, 17.60137623) -= (0, 17.60137623) (9, 5.297059) += (0, 27.09227512) -= (0, 27.09227512)
            };
            \end{axis}
        \end{tikzpicture}
        \subcaption{Slightly Suboptimal Path}
    \end{subfigure}
    \hfill
    \begin{subfigure}{0.33\linewidth}
        \begin{tikzpicture}
            \begin{axis}[
                width=\linewidth,
                height=4.2cm,
                xtick={0,2,4,6,8},
                xlabel={$t$},
                ylabel={Mean reward at time step $t$},
                grid=both,
            ]
            \addplot[color=black, mark=*, line width=1pt, mark size=1.5pt]
            coordinates {
                (0, 0.0) (1, 1.0) (2, 2.0) (3, 1.0) (4, 0.0) (5, 1.0) (6, 2.0) (7, 2.0) (8, 3.0) (9, 2.0)
            };
            \addplot[color=blue, mark=o,    mark options={solid},dashed, line width=1pt, mark size=1.5pt, error bars/.cd, y dir=both, y explicit, error bar style={line width=1pt,dashed}, error mark options={line width=1pt,mark size=4pt,rotate=90}]
            coordinates {
                (0, 0.0) +- (0, 0.0) (1, 1.0) +- (0, 0.0) (2, 2.0) +- (0, 0.0) (3, 1.0) +- (0, 0.0) (4, 0.0) +- (0, 0.0) (5, 1.0) +- (0, 0.0) (6, 2.0) +- (0, 0.0) (7, 2.0) +- (0, 0.0) (8, 3.0) +- (0, 0.0) (9, 2.0) +- (0, 0.0)
            };
            \addplot[color=red, mark=x, dashed, line width=1pt, mark size=3pt, error bars/.cd, y dir=both, y explicit, error bar style={line width=1pt,dashed}, error mark options={line width=1pt,mark size=4pt,rotate=90}]
            coordinates {
                (0, 0.0) +- (0, 0.0) (1, 0.7065655) +- (0, 0.4553358) (2, 1.341673) +- (0, 0.67091621) (3, 1.122926) +- (0, 0.61281824) (4, -1.1821935) +- (0, 13.82444042) (5, -0.952399) +- (0, 15.35195457) (6, -0.72672) +- (0, 20.33508414) (7, -0.268983) +- (0, 22.77861454) (8, -0.1310835) +- (0, 26.31013314) (9, 0.65806) +- (0, 28.50670214)
            };
            \end{axis}
        \end{tikzpicture}
        \subcaption{Almost Catastrophic Path}
    \end{subfigure}
    \hfill
    \begin{subfigure}{0.33\linewidth}
        \begin{tikzpicture}
            \begin{axis}[
                width=\linewidth,
                height=4.2cm,
                xtick={0,2,4,6,8},
                xlabel={$t$},
                ylabel={Mean reward at time step $t$},
                grid=both,
            ]
            \addplot[color=black, mark=*, line width=1pt, mark size=1.5pt]
            coordinates {
                (0, 1.0) (1, 2.0) (2, -100.0) (3, -100.0) (4, -100.0) (5, -100.0) (6, -100.0) (7, -100.0) (8, -100.0) (9, -100.0)
            };
            \addplot[color=blue, mark=o,    mark options={solid},dashed, line width=1pt, mark size=1.5pt, error bars/.cd, y dir=both, y explicit, error bar style={line width=1pt,dashed}, error mark options={line width=1pt,mark size=4pt,rotate=90}]
            coordinates {
                (0, 0.0) +- (0, 0.0) (1, 0.504814) +- (0, 0.49997682) (2, 0.8439835) +- (0, 0.76831917) (3, -8.2709165) +- (0, 28.93656754) (4, -9.981082) +- (0, 31.66825363) (5, -12.1776325) +- (0, 34.53463233) (6, -13.556076) +- (0, 38.62845372) (7, -14.574418) +- (0, 42.49603359) (8, -15.1757075) +- (0, 46.41913968) (9, -15.3900395) +- (0, 50.33563368)
            };
            \addplot[color=red, mark=x, dashed, line width=1pt, mark size=3pt, error bars/.cd, y dir=both, y explicit, error bar style={line width=1pt,dashed}, error mark options={line width=1pt,mark size=4pt,rotate=90}]
            coordinates {
                (0, 0.0) +- (0, 0.0) (1, 0.701873) +- (0, 0.45743556) (2, 1.1227805) +- (0, 0.73433129) (3, -8.7503255) +- (0, 30.30257976) (4, -10.722092) +- (0, 33.17618589) (5, -13.10721) +- (0, 36.0648089) (6, -13.7631645) +- (0, 40.56553451) (7, -13.909043) +- (0, 45.23829402) (8, -13.472517) +- (0, 49.96270296) (9, -12.8278835) +- (0, 54.38618735)
            };
            \end{axis}
        \end{tikzpicture}
        \subcaption{Catastrophic Path}
    \end{subfigure}
    \caption{Average instant reward of CF paths induced by policies on GridWorld $p=0.4$.}
    \label{fig: reward p=0.4}
\end{figure*}

\paragraph{Sepsis}
\begin{figure*}
    \centering
    \begin{subfigure}{0.33\linewidth}
        \centering
        \begin{tikzpicture}
            \begin{axis}[
                width=\linewidth,
                height=4.2cm,
                xlabel={$t$},
                xtick={0,2,4,6,8},
                ylabel={Mean reward at time step $t$},
                grid=both,
                ]
               \addplot[
                    color=black, %
                    mark=*, %
                    line width=1pt,
                    mark size=1.5pt,
                ]
                coordinates {
                    (0, -50.0) (1, 50.0) (2, -50.0) (3, -50.0) (4, -1000.0) (5, -1000.0) (6, -1000.0) (7, -1000.0) (8, -1000.0) (9, -1000.0)
                };
                \addplot[
                    color=blue, %
                    dashed,
                    mark=o, %
                    mark options={solid},
                    line width=1pt,
                    mark size=1.5pt,
                    error bars/.cd,
                    y dir=both, %
                    y explicit, %
                    error bar style={line width=1pt,dashed},
                    error mark options={line width=1pt,mark size=4pt,rotate=90}
                ]
                coordinates {
                    (0, -50.0)  +- (0, 0.0) (1, 50.0)  +- (0, 0.0) (2, -50.0)  +- (0, 0.0) (3, 20.0631)  +- (0, 49.97539413) (4, 71.206585)  +- (0, 226.02033693) (5, 151.60797) +- (0, 359.23292559) (6, 200.40593) +- (0, 408.86185176) (7, 257.77948) +- (0, 466.10372804) (8, 299.237465) +- (0, 501.82579506) (9, 338.9129) +- (0, 532.06124996)
                };
                \addplot[
                    color=red, %
                    dashed,
                    mark=x, %
                    line width=1pt,
                    mark size=3pt,
                    error bars/.cd,
                    y dir=both, %
                    y explicit, %
                    error bar style={line width=1pt,dashed},
                    error mark options={line width=1pt,mark size=4pt,rotate=90}
                ]
                coordinates {
                    (0, -50.0)  +- (0, 0.0) (1, 20.00736)  +- (0, 49.99786741) (2, -12.282865)  +- (0, 267.598755) (3, -47.125995)  +- (0, 378.41755832) (4, -15.381965)  +- (0, 461.77616558) (5, 41.15459) +- (0, 521.53189262) (6, 87.01595) +- (0, 564.22243126) (7, 132.62376) +- (0, 607.31338037) (8, 170.168145) +- (0, 641.48013693) (9, 201.813135) +- (0, 667.29441777)
                };
            \end{axis}
        \end{tikzpicture}
        \subcaption{Slightly Suboptimal Path}
    \end{subfigure}
    \hfill
    \begin{subfigure}{0.33\linewidth}
        \centering
        \begin{tikzpicture}
            \begin{axis}[
                width=\linewidth,
                height=4.2cm,
                xlabel={$t$},
                xtick={0,2,4,6,8},
                ylabel={Mean reward at time step $t$}, %
                grid=both,
                ]
               \addplot[
                    color=black, %
                    mark=*, %
                    line width=1pt,
                    mark size=1.5pt,
                ]
                coordinates {
                    (0, -50.0) (1, 50.0) (2, 50.0) (3, 50.0) (4, -50.0) (5, 50.0) (6, -50.0) (7, 50.0) (8, -50.0) (9, 50.0)
                };
                \addplot[
                    color=blue, %
                    dashed,
                    mark=o, %
                    mark options={solid},
                    line width=1pt,
                    mark size=1.5pt,
                    error bars/.cd,
                    y dir=both, %
                    y explicit, %
                    error bar style={line width=1pt,dashed},
                    error mark options={line width=1pt,mark size=4pt,rotate=90}
                ]
                coordinates {
                    (0, -50.0) +- (0, 0.0) (1, 50.0) +- (0, 0.0) (2, 50.0) +- (0, 0.0) (3, 50.0) +- (0, 0.0) (4, -50.0) +- (0, 0.0) (5, 50.0) +- (0, 0.0) (6, -50.0) +- (0, 0.0) (7, 50.0) +- (0, 0.0) (8, -50.0) +- (0, 0.0) (9, 50.0) +- (0, 0.0)
                };
                \addplot[
                    color=red, %
                    dashed,
                    mark=x, %
                    line width=1pt,
                    mark size=3pt,
                    error bars/.cd,
                    y dir=both, %
                    y explicit, %
                    error bar style={line width=1pt,dashed},
                    error mark options={line width=1pt,mark size=4pt,rotate=90}
                ]
                coordinates {
                    (0, -50.0) +- (0, 0.0) (1, 50.0) +- (0, 0.0) (2, 50.0) +- (0, 0.0) (3, 50.0) +- (0, 0.0) (4, -50.0) +- (0, 0.0) (5, 50.0) +- (0, 0.0) (6, -50.0) +- (0, 0.0) (7, 50.0) +- (0, 0.0) (8, -50.0) +- (0, 0.0) (9, 50.0) +- (0, 0.0)
                };
            \end{axis}
        \end{tikzpicture}
        \subcaption{Almost Catastrophic Path}
    \end{subfigure}
    \hfill
    \begin{subfigure}{0.33\linewidth}
        \centering
        \begin{tikzpicture}
            \begin{axis}[
                width=\linewidth,
                height=4.2cm,
                xtick={0,2,4,6,8},
                xlabel={$t$},
                ylabel={Mean reward at time step $t$}, %
                grid=both,
                ]
               \addplot[
                    color=black, %
                    mark=*, %
                    line width=1pt,
                    mark size=1.5pt,
                ]
                coordinates {
                    (0, -50.0) (1, -50.0) (2, -1000.0) (3, -1000.0) (4, -1000.0) (5, -1000.0) (6, -1000.0) (7, -1000.0) (8, -1000.0) (9, -1000.0)
                };
                \addplot[
                    color=blue, %
                    dashed,
                    mark=o, %
                    mark options={solid},
                    line width=1pt,
                    mark size=1.5pt,
                    error bars/.cd,
                    y dir=both, %
                    y explicit, %
                    error bar style={line width=1pt,dashed},
                    error mark options={line width=1pt,mark size=4pt,rotate=90}
                ]
                coordinates {
                    (0, -50.0)  +- (0, 0.0) (1, -50.0)  +- (0, 0.0) (2, -50.0)  +- (0, 0.0) (3, -841.440725)  += (0, 354.24605512) -= (0, 158.559275) (4, -884.98225)  += (0, 315.37519669) -= (0, 115.01775) (5, -894.330425) += (0, 304.88572805) -= (0, 105.669575) (6, -896.696175) += (0, 301.19954514) -= (0, 103.303825) (7, -897.4635) += (0, 299.61791279) -= (0, 102.5365) (8, -897.77595) += (0, 298.80392585) -= (0, 102.22405) (9, -897.942975) += (0, 298.32920557) -= (0, 102.057025)
                };
                \addplot[
                    color=red, %
                    dashed,
                    mark=x, %
                    line width=1pt,
                    mark size=3pt,
                    error bars/.cd,
                    y dir=both, %
                    y explicit, %
                    error bar style={line width=1pt,dashed},
                    error mark options={line width=1pt,mark size=4pt,rotate=90}
                ]
                coordinates {
                    (0, -50.0) +- (0, 0.0) (1, -360.675265) +- (0, 479.39812699) (2, -432.27629) +- (0, 510.38620897) (3, -467.029545) += (0, 526.36009628) -= (0, 526.36009628) (4, -439.17429)  += (0, 583.96638919) -= (0, 560.82571) (5, -418.82704) += (0, 618.43027478) -= (0, 581.17296) (6, -397.464895) += (0, 652.67322574) -= (0, 602.535105) (7, -378.49052) += (0, 682.85407033) -= (0, 621.50948) (8, -362.654195) += (0, 707.01412023) -= (0, 637.345805) (9, -347.737935) += (0, 729.29076479) -= (0, 652.262065)
                };
            \end{axis}
        \end{tikzpicture}
        \subcaption{Catastrophic Path}
    \end{subfigure}
    \caption{Average instant reward of CF paths induced by policies on Sepsis.}
    \label{fig: reward sepsis}
\end{figure*}
\jl{Like in the above experiment, the Sepsis MDP is highly stochastic, with many states equally likely to lead to optimal and poor outcomes. As shown in Figure \ref{fig: reward sepsis}, both policies follow the observed almost-catastrophic path to ensure rewards are no worse than the observation.} However, to improve the catastrophic path, both policies must deviate from the observation. Here, on average, the Gumbel-max SCM policy outperforms the interval CFMDP policy, but since both have lower \jl{error bars} clipped at $-1000$, neither reliably improves upon the observation. In contrast, for the slightly suboptimal path, the interval CFMDP policy performs significantly better, as shown by \jl{the higher lower end of its error bars}. Moreover, in these two cases, the worst-case counterfactual path generated by the interval CFMDP policy is better than that of the Gumbel-max SCM policy (see Table \ref{tab: lowest cumulative rewards}), indicating its greater robustness.
\begin{table*}
\centering
\resizebox{0.85\textwidth}{!}{%
\begin{tabular}{|c|cc|cc|cc|}
\hline
\multirow{2}{*}{\textbf{Environment}} &
  \multicolumn{2}{c|}{\textbf{Slightly Suboptimal Path}} &
  \multicolumn{2}{c|}{\textbf{Almost Catastrophic}} &
  \multicolumn{2}{c|}{\textbf{Catastrophic Path}} \\ \cline{2-7} 
 &
  \multicolumn{1}{c|}{\textbf{ICFMDP (ours)}} &
  \textbf{Gumbel-max} &
  \multicolumn{1}{c|}{\textbf{ICFMDP (ours)}} &
  \textbf{Gumbel-max} &
  \multicolumn{1}{c|}{\textbf{ICFMDP (ours)}} &
  \textbf{Gumbel-max} \\ \hline
GridWorld ($p=0.9$) & \multicolumn{1}{c|}{-495} & -495 & \multicolumn{1}{c|}{-697} & -698 & \multicolumn{1}{c|}{-698} & -698 \\ \hline
GridWorld ($p=0.4$)  & \multicolumn{1}{c|}{19} & -88 & \multicolumn{1}{c|}{14} & -598 & \multicolumn{1}{c|}{-698} & -698 \\ \hline
Sepsis              & \multicolumn{1}{c|}{-5980} & 8000 & \multicolumn{1}{c|}{100} & 100 & \multicolumn{1}{c|}{-7150} & -9050 \\ \hline
Frozen Lake         & \multicolumn{1}{c|}{41} & -97  & \multicolumn{1}{c|}{-68} & -87 & \multicolumn{1}{c|}{-87} & -87 \\ \hline
Aircraft             & \multicolumn{1}{c|}{231} & 231 & \multicolumn{1}{c|}{-776} & -776 & \multicolumn{1}{c|}{-776} & -776 \\ \hline
\end{tabular}
}
\caption{Lowest cumulative rewards achieved by Interval CFMDP and Gumbel-max SCM policies across sampled counterfactual paths and CFMDPs, for observed paths of varying optimality.}
\label{tab: lowest cumulative rewards}
\end{table*}
\subsection{Robustness of Counterfactual Policies}
\label{sec: robustness experiment}
\jl{
The primary advantage of our approach is that it produces policies that are robust with respect to the unknown (true) causal model. To experimentally assess robustness, we examine how our policy and the Gumbel-max policy perform when deployed in the worst-case environment, i.e., when selecting the counterfactual MDP (among the admissible ones) with the lowest reward. 
Procedurally, given an MDP, we sample from it $100$ observed paths following a random policy. For each path, we derive the corresponding Gumbel-max policy and use our approach to construct the corresponding ICFMDP. Recall that the ICFMDP entails all and only the counterfactual MDPs that are compatible with the model and data. Thus, we use the ICFMDP to compute the worst-case performance of the two policies across all CFMDPs within the ICFMDP. Note that, for our policy, the worst-case performance is readily available as the solution of the pessimistic value iteration problem~\eqref{eq:pessimistic_v}. 
Results are reported in Table~\ref{tab: pessimistic value random paths}. \camera{Across all environments, the worst-case reward obtained by our approach is consistently much higher than the Gumbel-max approach, demonstrating that our counterfactual inference method is substantially more robust to causal model uncertainty.}}
\begin{table}
\centering
\resizebox{0.8\linewidth}{!}{%
\begin{tabular}{|c|cc|}
\hline
\multirow{2}{*}{\textbf{Environment}} & \multicolumn{2}{c|}{\textbf{Pessimistic $V(s_0)$}} \\ \cline{2-3} 
                             & \multicolumn{1}{c|}{\textcolor{blue}{\textbf{ICFMDP (ours)}}}  & \textcolor{red}{\textbf{Gumbel-max}} \\ \hline
1. GridWorld ($p=0.9$)          & \multicolumn{1}{c|}{$346 \pm 104$}        &   $304 \pm 211$         \\ \hline %
2. GridWorld ($p=0.4$)          & \multicolumn{1}{c|}{$-81.4 \pm 197$}        &   $-230 \pm 213$         \\ \hline %
3. Sepsis                       & \multicolumn{1}{c|}{$1660\pm 1010$}        &    $-85.4\pm 2860$        \\ \hline %
4. Frozen Lake                       & \multicolumn{1}{c|}{$37.3 \pm 17.9$}        &    $2.56 \pm 51.3$      \\ \hline  %
5. Aircraft                     & \multicolumn{1}{c|}{$-99.0 \pm 380$}        &       $-221 \pm 421$   \\ \hline %

\end{tabular}
}
\caption{Average worst-case counterfactual $V(s_0)$ for the ICFMDP and Gumbel-max policies over $100$ randomly sampled observed trajectories. Our policy significantly outperforms the Gumbel-max one (Welch T-test resulted in $p<10^{-4}$ for rows 2,3,4; $p=0.0381$ for row 1; $p=0.0163$ for row 5).}
\label{tab: pessimistic value random paths}
\end{table}
\subsection{Interval CFMDP Bounds and Runtimes}
Table \ref{tab: bounds} reports average counterfactual probability bound widths (excluding transitions where the upper bound is $0$) for each MDP, averaged over $20$ observed paths. We compare the bounds under counterfactual stability (CS) and monotonicity (M) assumptions, CS alone, and no assumptions. \jl{These results indicate that the assumptions are not overly restrictive (i.e., they do not exclude too many causal models), as they do not significantly narrow the bounds, yet still exclude implausible counterfactuals such as those in the example in Figure \ref{fig:gumbel-max-scm-unintuitive-probs}.} \camera{In Appendix \ref{app: additional experiments} we further examine the impact of these assumptions on policy robustness, showing that relaxing the assumptions leads to a slight reduction in performance in most environments, but the policies still outperform the worst-case performance of the Gumbel-max approach.} Table \ref{tab: times} compares the average time needed to generate the interval CFMDP vs.\ the Gumbel-max SCM CFMDP for 20 observations. \jl{The GridWorld and Frozen Lake experiments were run single-threaded, while Sepsis and Aircraft were run in parallel.} Generating the interval CFMDP is significantly faster as it uses exact analytical bounds, whereas the Gumbel-max CFMDP requires sampling from the Gumbel distribution to estimate counterfactual transition probabilities. \jl{In particular, across the five case studies, our approach is $4$ to $251$ times faster than the Gumbel-max method.} 
Since constructing counterfactual MDPs is the main bottleneck in both approaches, ours is more efficient overall and suitable for larger MDPs.
\begin{table}
\centering
\resizebox{0.7\linewidth}{!}{%
\begin{tabular}{|c|c|c|c|}
\hline
\multirow{2}{*}{\textbf{Environment}} 
  & \multicolumn{3}{c|}{\textbf{Mean Bound Widths}} \\ \cline{2-4}
  & \textbf{CS + M} & \textbf{CS} & \textbf{None} \\ \hline
{GridWorld ($p=0.9$)} & 0.0817 & 0.0977 & 0.100 \\ \hline
{GridWorld ($p=0.4$)} & 0.552  & 0.638  & 0.646 \\ \hline
{Sepsis}              & 0.138  & 0.140  & 0.140 \\ \hline
{Frozen Lake} & 0.307  & 0.353  & 0.359 \\ \hline
{Aircraft} & 0.180  & 0.190  & 0.190 \\ \hline
\end{tabular}
}
\caption{Mean width of counterfactual probability bounds in generated CFMDPs.}
\label{tab: bounds}
\end{table}
\begin{table}
\centering
\resizebox{0.8\linewidth}{!}{%
\begin{tabular}{|c|c|c|}
\hline
\textbf{Environment}
  & \textbf{ICFMDP (ours)} & \textbf{Gumbel-max} \\ \hline
{GridWorld ($p=0.9$)} & 0.261 & 56.1 \\ \hline %
{GridWorld ($p=0.4$)} & 0.336 & 54.5 \\ \hline %
{Sepsis}              & 688   & 2940 \\ \hline %
{Frozen Lake} & 0.398 & 100 \\ \hline %
{Aircraft} & 6.99 & 74.3 \\ \hline %
\end{tabular}
}
\caption{Mean execution time \camera{(s)} for generating CFMDPs.}
\label{tab: times}
\end{table}

\section{Conclusion} We introduced a non-parametric partial counterfactual inference approach for MDPs, leveraging tight analytical bounds on counterfactual probabilities to accelerate the construction of counterfactual models, enabling scaling to larger MDPs. Our interval CFMDP policies are more robust to uncertainty about the true (unknown) causal model, particularly in highly stochastic environments. This robustness yields more reliable counterfactual explanations for improving the agent's policy, which is crucial in safety-critical domains.
\jl{
\paragraph{Future Work}
Like existing work on counterfactual inference in MDPs \cite{buesing2018woulda,oberst2019counterfactual, tsirtsis2021counterfactual}, our approach assumes \camera{access to} the MDP's transition probabilities. With estimated probabilities, the counterfactual policy may be sensitive to misspecification. A natural extension is to generalise our method to work to uncertain MDPs learned from data, where the probabilities are bounded by confidence intervals learned from observed trajectories. This extension will build directly on the theoretical framework established in this paper. Future work will also explore extending our approach to partially observable and continuous-state settings.}

\begin{acks}
We thank Frederik Mathiesen and 
Luca Laurenti for their support with the IntervalMDP.jl package. This work was supported by UK Research and Innovation [grant EP/S023356/1] in the UKRI Centre for Doctoral Training in Safe and Trusted Artificial Intelligence (www.safeandtrustedai.org) and by the EPSRC grant EP/W014785/2.
\end{acks}

\bibliographystyle{ACM-Reference-Format}
\balance
\bibliography{bibliography}

\newpage
\onecolumn
\appendix
\section{Gumbel-max SCMs}
\label{app: gumbel}
The Gumbel-max SCM for an MDP is expressed as:
\begin{equation}\label{eq:gumbel-max-scm}
    S_{t+1} = f(S_{t},A_{t}, U_{t}=(G_{s,t})_{s\in \mathcal{S}}) 
    = \argmax_{s\in \mathcal{S}}\left\{\log\left(P_{\mathcal{M}}(s \mid S_t, A_t)\right) +G_{s,t}\right\}
\end{equation}
where $P_{\mathcal{M}}$ is the MDP's transition probabilities, $S_t$, $A_t$ and $S_{t+1}$ are endogenous variables representing the state, action and next state, and the values of the exogenous variable $U_{t}=(G_{s,t})_{s\in \mathcal{S}}$ are sampled from the standard Gumbel distribution\footnote{This formulation relies on the Gumbel-max trick \citep{maddison2014sampling} by which sampling from a categorical distribution with $n$ categories is equivalent to sampling $n$ values $g_0,\dots,g_n$ from the standard Gumbel distribution and evaluating $\argmax_{j}\left\{\log\left(P(Y=j)\right)+ g_j\right\}$ over all the categories, where $Y$ is the output.}. In this model, the Gumbel values $G_{s,t}$ represent the exogenous noise. We can perform (approximate) posterior inference of $P((G_{s,t})_{s\in \mathcal{S}} \mid s_t, a_t, s_{t+1})$ through rejection sampling \citep{oberst2019counterfactual} or top-down Gumbel sampling \citep{maddison2014sampling}.

We can define a so-called \textit{counterfactual MDP} $\mathcal{M}_{\tau}$ by solving the SCM \eqref{eq:gumbel-max-scm} for each transition along an observed path $\tau$ in an MDP $\mathcal{M}$. The counterfactual probability for each transition is defined, for $t=0, ..., T-1$, as:
\begin{equation}
\label{eq:cf_mdp_probs}
\begin{aligned}
P_{\mathcal{M},t,\tau}(s' \mid s, a) &= P(s' = \argmax_{q \in \mathcal{S}} \left\{\log\left(P_{\mathcal{M}}(q \mid s, a)\right) + 
     \jl{G'_{\tau, q,t}}\right\}) \\
     &\approx \dfrac{1}{N} \sum_{j=0}^{N} \mathbbm{1}\left(s' = \argmax_{q \in \mathcal{S}} \left\{\log\left(P_{\mathcal{M}}(q \mid s, a)\right) + 
     \jl{G_{\tau, q,t}^{\prime(j)}}\right\}\right) \\
\end{aligned}
\end{equation}
where we sample $N$ values of $G_{q,t}^{\prime(j)}$ from the true posterior distribution $G'_{q,t}$ through rejection sampling or top-down Gumbel sampling. The indicator function $\mathbbm{1}(\mathbbm{X})$ takes the value $1$ if the condition $\mathbbm{X}$ is satisfied and $0$ otherwise.

\section{Explanation of Unintuitive Counterfactual Probabilities Produced by Gumbel-max SCMs}
\label{app: unintuitive probs explanation}

Take the example MDP from Figure \ref{fig:appendix-gumbel-max-scm-unintuitive-probs}, which contains three states and a single action. In this example, we observe the transition $s_t=0, a_t=0 \rightarrow s_{t+1}=1$. Neither state 0 nor state 2 was observed, although both were possible. This means that the observed transition gives us no information as to whether state 0 or state 2 would have been more likely where, counterfactually, $s_t = 1$, so intuitively these counterfactual probabilities should be equal to the nominal probabilities. However, the Gumbel-max SCM scales the probabilities, making state 2 even more likely than in the nominal MDP, and vice versa for state 0. This happens because the Gumbel distributions that the Gumbel values are drawn from for each transition are truncated in proportion to the log probability of the transition's nominal probability (instead of in proportion to their nominal probabilities).

\begin{figure*}[h]
    \centering
    \begin{minipage}{0.45\textwidth}
        \centering
        \resizebox{0.9\linewidth}{!}{%
            \begin{circuitikz}
            \tikzstyle{every node}=[font=\Huge]
            \draw  (10.75,18.5) circle (1.25cm) node {\Huge $s_1$} ;
            \draw [ fill={rgb,255:red,177; green,170; blue,170} ] (4.25,18.5) circle (1.25cm) node {\Huge $s_0$} ;
            \draw  (16.75,18.5) circle (1.25cm) node {\Huge $s_2$} ;
            \draw [ fill={rgb,255:red,177; green,170; blue,170} ] (10.75,14.75) circle (1.25cm) node {\Huge $s_1$} ;
            \draw  (4.25,14.75) circle (1.25cm) node {\Huge $s_0$} ;
            \draw  (16.75,14.75) circle (1.25cm) node {\Huge $s_2$} ;
            \draw [->, >=Stealth] (4.25,17.25) -- (4.25,16)node[pos=0.5, fill=white]{0.3};
            \draw [->, >=Stealth] (16.75,17.25) -- (16.75,16)node[pos=0.5, fill=white]{1.0};
            \draw [->, >=Stealth] (10,17.5) -- (5.5,15);
            \draw [->, >=Stealth] (11.5,17.5) -- (15.5,15);
            \draw [->, >=Stealth] (5.25,17.75) -- (15.5,14.75);
            \draw [->, >=Stealth] (5,17.5) -- (9.5,14.75);
            \node [font=\Huge] at (5.75,16.5) {0.4};
            \node [font=\Huge] at (9,17.5) {0.4};
            \node [font=\Huge] at (6.75,17.75) {0.3};
            \node [font=\Huge] at (12.5,17.5) {0.6};
            \end{circuitikz}
        }
        \captionof{figure}{Example MDP where Gumbel-Max produces unintuitive CF probabilities. The observed path is $s_0 \rightarrow s_1$.}
        \label{fig:appendix-gumbel-max-scm-unintuitive-probs}
    \end{minipage}%
    \hfill
    \begin{minipage}{0.48\textwidth}
        \centering
        \resizebox{\linewidth}{!}{
        \begin{tabular}{|c|c|c|c|P{0.7cm}|P{0.7cm}|c|P{0.7cm}|P{0.7cm}|}
        \hline
        \multirow{2}{*}{\text{$s$}} & \multirow{2}{*}{\text{$a$}} & \multirow{2}{*}{\text{$s'$}} & \multirow{2}{*}{\text{$P(s' \mid s, a)$}} & \multicolumn{2}{c|}{\makecell{\text{Optimisation} \\ \text{(\ref{eq: optimisation})}}} & \multirow{2}{*}{\makecell{\text{Gumbel-} \\ \text{Max (\ref{eq:cf_mdp_probs})}}} & \multicolumn{2}{c|}{\makecell{\text{Optimisation} \\ \text{(\ref{eq: optimisation}-\ref{eq:mon2})}}} \\ \cline{5-6} \cline{8-9} 
                                    &                                 &                 &                     & \text{LB}                             & \text{UB}                            &                                      & \text{LB}                              & \text{UB}                             \\ \Xhline{1pt}
        0                               & 0                               & 0      & 0.3                             & 0.0                                     & 0.0                                    & 0.0                                  & 0.0                                      & 0.0                                     \\ \hline
        0                               & 0                               & 1      & 0.4                             & 1.0                                     & 1.0                                    & 1.0                                  & 1.0                                      & 1.0                                     \\ \hline
        0                               & 0                               & 2      & 0.3                             & 0.0                                     & 0.0                                    & 0.0                                  & 0.0                                      & 0.0                                     \\ \hline
        1                               & 0                               & 0      & 0.4                             & 0.0                                     & 1.0                                    & 0.35                                 & 0.4                                      & 0.4                                     \\ \hline
        1                               & 0                               & 1      & 0.0                             & 0.0                                     & 0.0                                    & 0.0                                  & 0.0                                      & 0.0                                     \\ \hline
        \textbf{1}                      & \textbf{0}                      & \textbf{2} & \textbf{0.6}             & \textbf{0.0}                            & \textbf{1.0}                           & \textbf{0.65}                        & \textbf{0.6}                             & \textbf{0.6}                            \\ \hline
        2                               & 0                               & 0      & 0.0                             & 0.0                                     & 0.0                                    & 0.0                                  & 0.0                                      & 0.0                                     \\ \hline
        2                               & 0                               & 1      & 0.0                             & 0.0                                     & 0.0                                    & 0.0                                  & 0.0                                      & 0.0                                     \\ \hline
        2                               & 0                               & 2      & 1.0                             & 1.0                                     & 1.0                                    & 1.0                                  & 1.0                                      & 1.0                                     \\ \hline
        \end{tabular}
        }
        \captionof{table}{Counterfactual transition probabilities produced by various methods.}
        \label{tab:appendix-gumbel-max-scm}
    \end{minipage}
\end{figure*}

\section{Further Details on Counterfactual Stability and Monotonicity Assumptions}
\label{app: assumption plausibility}

In Section \ref{sec: partial CF inference}, we formulate partial counterfactual inference as an optimisation problem, using the canonical SCM framework. Without any assumptions, the counterfactual probability bounds produced by this optimisation procedure (Eq. \eqref{eq: optimisation}) may be too wide to be useful. This issue is illustrated in our toy example in Figure \ref{fig:appendix-gumbel-max-scm-unintuitive-probs} and Table \ref{tab:appendix-gumbel-max-scm}: Eq. \eqref{eq: optimisation} produces trivial [0,1] bounds for the transitions $s=1,a=0 \rightarrow s'=0$ and $s=1,a=0 \rightarrow s'=2$. In a multi-timestep setting, like MDPs, having many transitions with loose or trivial bounds can quickly compound over several timesteps, severely limiting the usefulness of counterfactual inference for generating counterfactual explanations.

This motivates the use of reasonable assumptions about the underlying causal model to obtain informative bounds from partial counterfactual inference. One widely used assumption in the literature is counterfactual stability \citep{benz2022counterfactual, kazemi2024counterfactual, killian2022counterfactually, lorberbom2021learning, noorbakhsh2022counterfactual, zhu2020counterfactual}. This says that, for the outcome to change from the observed outcome under a counterfactual distribution (in our MDP setting, the counterfactual state and action), the relative likelihood of an alternative outcome must have increased relative to that of the observed outcome \citep{oberst2019counterfactual}. \citet{oberst2019counterfactual} constructed the counterfactual stability assumption as a generalisation of the monotonicity assumption used in binary settings (which we note is not the same as our counterfactual monotonicity assumption) to categorical distributions. This encodes a natural intuition we have about counterfactuals: if the probability of seeing the observed outcome increases more relative to another possible outcome under some counterfactual, then we would still expect to see the observed outcome. There already exist partial counterfactual inference approaches that incorporate counterfactual stability as an assumption into their optimisation problem, e.g., \citep{haugh2023bounding}.

However, in our toy example in Figure \ref{fig:appendix-gumbel-max-scm-unintuitive-probs} and Table \ref{tab:appendix-gumbel-max-scm}, the bounds for the transitions $s=1,a=0 \rightarrow s'=0$ and $s=1,a=0 \rightarrow s'=2$ are still trivial $[0,1]$ even with the counterfactual stability assumption. We can infer this from Table \ref{tab:appendix-gumbel-max-scm}: the counterfactual stability assumption sets the counterfactual probability of a transition to $0$ if the inequality holds, but this condition is not satisfied (otherwise the counterfactual probabilities of the Gumbel-max SCM, which satisfies counterfactual stability, would have been 0).

To prevent this, we introduce our assumption of counterfactual monotonicity: if a state could have been reached from the observed state-action pair, but was not observed/realised, then the counterfactual probabilities of all transitions to that state cannot increase from their nominal probabilities; similarly, if a state was realised/observed, then the counterfactual probabilities of all transitions to that state cannot decrease from their nominal probabilities.

This assumption aligns with the core (Bayesian) principle that observations carry useful information about the underlying dynamics and should be used to update the agent's belief. In our toy example, we did not observe either state $0$ or $2$ (although both states were possible and had equal probability from the observed state-action pair $(s=0, a=0)$), so the observation does not bear any useful information in the counterfactual context. If an outcome did not occur despite being possible, we treat its likelihood as relatively lower under alternative counterfactual scenarios. This prevents counterfactuals from contradicting evidence embedded in the factual trajectory.

Importantly, our framework is modular: these assumptions can be removed from the optimisation problem and the counterfactual probability bounds if desired, without requiring changes to the core inference procedure. The closed-form solutions can be proven similarly to the proofs in Appendix \ref{app: proofs}:

\begin{theorem}
\label{theorem: no assumptions}
For the observed state-action pair $(s_{t}, a_{t})$, the linear program without either assumption will produce the following bounds:
{
\[\tilde{P}_{t}^{\it LB}(s_{t+1} \mid s_t, a_t) = \tilde{P}_{t}^{\it UB}(s_{t+1} \mid s_t, a_t) = 1\] 
\[\forall \tilde{s}' \in \mathcal{S} \setminus \{s_{t+1}\}, \tilde{P}_{t}^{\it LB}(\tilde{s}' \mid s_t, a_t) = \tilde{P}_{t}^{\it UB}(\tilde{s}' \mid s_t, a_t) = 0\]
}
For all state-action pairs $(\tilde{s}, \tilde{a}) \neq (s_{t}, a_{t})$, the linear program without either assumption will produce, $\forall \tilde{s}' \in \mathcal{S}$, the following bounds:
{
\[\tilde{P}_{t}^{\it UB}(\tilde{s}' \mid \tilde{s}, \tilde{a}) = \min(1, \dfrac{P(\tilde{s}' \mid \tilde{s}, \tilde{a})}{P(s_{t+1} \mid s_t, a_t)}) \] 

\[
\tilde{P}_{t}^{\it LB}(\tilde{s}' \mid \tilde{s}, \tilde{a}) = \max(0, \dfrac{P(\tilde{s}' \mid \tilde{s}, \tilde{a}) - (1 - P(s_{t+1} \mid s_t, a_t))}{P(s_{t+1} \mid s_t, a_t)})
\]
}
\end{theorem}

\begin{theorem}
\label{theorem: just cs assumption}
For the observed state-action pair $(s_{t}, a_{t})$, the linear program with only the counterfactual stability assumption will produce the following bounds:
{
\[\tilde{P}_{t}^{\it LB}(s_{t+1} \mid s_t, a_t) = \tilde{P}_{t}^{\it UB}(s_{t+1} \mid s_t, a_t) = 1\] 
\[\forall \tilde{s}' \in \mathcal{S} \setminus \{s_{t+1}\}, \tilde{P}_{t}^{\it LB}(\tilde{s}' \mid s_t, a_t) = \tilde{P}_{t}^{\it UB}(\tilde{s}' \mid s_t, a_t) = 0\]
}
For all state-action pairs $(\tilde{s}, \tilde{a}) \neq (s_{t}, a_{t})$, the linear program with only the counterfactual stability assumption will produce, $\forall \tilde{s}' \in \mathcal{S}$, the following bounds:
{
\[\tilde{P}_{t}^{\it UB}(\tilde{s}' \mid \tilde{s}, \tilde{a}) = 
\begin{cases}
0 & \textnormal{ if CS conditions (*)} \\
\min\left(1, \dfrac{P(\tilde{s}' \mid \tilde{s}, \tilde{a})}{P(s_{t+1} \mid s_t, a_t)} \right) & \textnormal{ otherwise}
\end{cases}
\] 
}
{
\[
\tilde{P}_{t}^{\it LB}(\tilde{s}' \mid \tilde{s}, \tilde{a}) = 
\begin{cases}
\max \left(0, \dfrac{P(\tilde{s}' \mid \tilde{s}, \tilde{a}) - (1 - P(s_{t+1} \mid s_t, a_t))}{P(s_{t+1} \mid s_t, a_t)}\right) & \textnormal{if the support of $(\tilde{s}, \tilde{a})$ is disjoint from the support of $(s_t, a_t)$} \\
0 & \textnormal{ if CS conditions (*)}\\
\max\left(0, 1 - \sum_{s' \in \mathcal{S}\setminus\{\tilde{s}'\}}{\tilde{P}_{t}^{\it UB}(s' \mid \tilde{s}, \tilde{a})} \right) & \textnormal{ otherwise}
\end{cases}
\]
}
\end{theorem}

(*) counterfactual stability (CS) conditions: $\tilde{s}' \neq s_{t+1} \text{ and } P({\tilde{s}' \mid s_t, a_t}) > 0 \text{ and } \dfrac{P(s_{t+1} \mid \tilde{s}, \tilde{a})}{P(s_{t+1} \mid s_t, a_t)} \geq \dfrac{P(\tilde{s}' \mid \tilde{s}, \tilde{a})}{P({\tilde{s}' \mid s_t, a_t})}$

\section{Proofs}
\label{app: proofs}

\paragraph{Optimisation Problem}
For a given counterfactual transition $s, a \rightarrow s'$ and given observed transition $s_t, a_t \rightarrow s_{t+1}$, the optimisation problem is defined as follows:

\begin{align}
&{\min / \max}_{\theta} \sum_{u_t = 1}^{|U_t|} \mu_{s, a, u_t, s'} \cdot \mu_{s_t, a_t, u_t, s_{t+1}} \cdot \theta_{u_t} \label{proofeq:objective}\\
&\text{s.t.} \sum_{u_t = 1}^{|U_t|} \mu_{\tilde{s}, \tilde{a}, u_t, \tilde{s}'} \cdot \theta_{u_t} = P(\tilde{s}' \mid \tilde{s}, \tilde{a}), \forall \tilde{s}, \tilde{a}, \tilde{s}' \label{proofeq:interventional constraint} \\
&\tilde{P}_t(s_{t+1} \mid \tilde{s}, \tilde{a}) \geq {P}(s_{t+1} \mid \tilde{s}, \tilde{a}) \text{ if } P(s_{t+1} \mid \tilde{s}, \tilde{a})>0, \forall\tilde{s},\tilde{a} \text{ (Mon1)} \label{proofeq:monotonicity1} \\
&\tilde{P}_t(\tilde{s}' \mid \tilde{s}, \tilde{a}) \leq {P}(\tilde{s}' \mid \tilde{s}, \tilde{a}) \text{ if } P(\tilde{s}' \mid \tilde{s}, \tilde{a})>0 \text{ and } P(\tilde{s}' \mid s_t, a_t)>0, \forall \tilde{s}, \tilde{a}, \tilde{s}'\neq s_{t+1} \text{ (Mon2)} \label{proofeq:monotonicity2} \\
&\tilde{P}_t(\tilde{s}' \mid \tilde{s}, \tilde{a}) = 0 \text{ if } \dfrac{P(s_{t+1} \mid \tilde{s}, \tilde{a})}{P(s_{t+1} \mid s_t, a_t)}\geq\dfrac{P(\tilde{s}' \mid \tilde{s}, \tilde{a})}{P(\tilde{s}' \mid s_t, a_t)} \text{ and } P(\tilde{s}' \mid s_t, a_t) > 0, \forall (\tilde{s}, \tilde{a})\neq(s_{t}, a_t), \forall \tilde{s}' \neq s_{t+1} \text{ (CS)} \label{proofeq:counterfactual stability} \\
&0 \leq \theta_{u_t} \leq 1, \forall u_t \label{proofeq:valid prob1} \\
&\sum_{u_t = 1}^{|U_t|} \theta_{u_t} = 1 \label{proofeq:valid prob2}
\end{align}

where $\theta$ and $\mu$ are defined as follows:

\[\theta \in \mathbb{R}^{|\Omega_{U_t}|}\]
\[\mu \in \{0,1\}^{\mathcal{S} \times \mathcal{A} \times |\Omega_{U_t}| \times \mathcal{S}}\]
\[\forall s \in \mathcal{S}, a \in \mathcal{A}, s' \in \mathcal{S}, u_t \in U_t, \mu_{s, a, u_t, s'} = 
\begin{cases}
    1 & \text{if $f(s, a, u_t) = s'$}\\
    0 & \text{otherwise}\\
\end{cases}
\]

\paragraph{Counterfactual Probabilities}
For any transition $s, a \rightarrow s'$ and observed transition $s_t, a_t \rightarrow s_{t+1}$, the counterfactual transition probability $\tilde{P}_t(s' \mid s, a)$ can be calculated as follows:

\begin{equation}
\label{eq: counterfactual probability}
   \tilde{P}_t(s' \mid s, a)= \dfrac{\sum_{u_t = 1}^{|U_t|} \mu_{s, a, u_t, s'} \cdot \mu_{s_t, a_t, u_t, s_{t+1}} \cdot \theta_{u_t}}{P(s_{t+1} \mid s_t, a_t)} 
\end{equation}
\jl{
\paragraph{Shorthand Notation}
For compactness in the proofs, we adjust the notation of the optimisation problem in \eqref{proofeq:objective}-\eqref{proofeq:valid prob2} as follows. $\forall s,s' \in \mathcal{S}, \forall a \in \mathcal{A}, \forall u_t \in U_t$

\begin{equation}
    \begin{aligned}
    \sum_{u_t = 1}^{|U_t|} \mu_{\tilde{s},\tilde{a}, u_t, \tilde{s}'} \cdot \mu_{s_t, a_t, u_t, s_{t+1}} \cdot \theta_{u_t}
    &= \sum_{\substack{u_t \in U_t\\f(s_t, a_t, u_t) = s_{t+1}}} \mu_{\tilde{s},\tilde{a}, u_t, \tilde{s}'} \cdot \theta_{u_t} \text{ (by def, $\mu_{s_t, a_t, u_t, s_{t+1}} = 1 \iff f(s_t, a_t, u_t) = s_{t+1}$)}\\
    &= \sum_{\substack{u_t \in U_t\\ f(s_t, a_t, u_t) = s_{t+1} \\ f(\tilde{s}, \tilde{a}, u_t) = \tilde{s}'}} \theta_{u_t} \text{ (by def, $\mu_{\tilde{s}, \tilde{a}, u_t, \tilde{s}'} = 1 \iff f(\tilde{s}, \tilde{a}, u_t) = \tilde{s}'$)}\\
\end{aligned}
\end{equation}
}

\subsection{Counterfactual Transition Probability Bounds}
\label{sec: probability bounds proof}

In this section, we prove that the above optimisation problem reduces to the exact analytical bounds for counterfactual transition probabilities given in Section \ref{sec: bounds}, for any MDP $\mathcal{M}$.

\begin{theorem}
    For any MDP $\mathcal{M}$ and path $\tau$, the optimisation problem defined in \eqref{proofeq:objective}-\eqref{proofeq:valid prob2} reduces to the exact analytical bounds given in Theorems \ref{theorem: observed state}-\ref{theorem:ub overlapping} for the counterfactual transition probabilities of every transition in $\mathcal{M}$.
\end{theorem}

\jl{
\subsubsection{Proof Sketch}
\label{sec: proof sketch}
Consider an arbitrary observed transition $s_t, a_t \rightarrow s_{t+1}$ and counterfactual transition $\tilde{s}, \tilde{a} \rightarrow \tilde{s}'$. As shown in Eq. \eqref{eq: counterfactual probability}, the counterfactual probability of the counterfactual transition can be computed given a fixed assignment of $\theta$. Since the denominator in Eq. \eqref{eq: counterfactual probability} ($P(s_{t+1} \mid s_t, a_t)$) is fixed, the counterfactual probability depends only on:

\[\sum_{u_t = 1}^{|U_t|} \mu_{s, a, u_t, s'} \cdot \mu_{s_t, a_t, u_t, s_{t+1}} \cdot \theta_{u_t} = \sum_{\substack{u_t \in U_t \\f(s_t, a_t, u_t) = s_{t+1} \\ f(\tilde{s}, \tilde{a}, u_t) = \tilde{s}'}} \theta_{u_t}\]

Therefore, to minimise/maximise the counterfactual probability of a particular transition, we seek assignments of $\theta$ that optimise the value of \[\sum_{\substack{u_t \in U_t \\f(s_t, a_t, u_t) = s_{t+1} \\ f(\tilde{s}, \tilde{a}, u_t) = \tilde{s}'}} \theta_{u_t}\] subject to the other constraints of the optimisation problem. In particular, this means that, for any state-action pair $(\tilde{s}, \tilde{a})$ in $\mathcal{M}$, the counterfactual probability of the transition $\tilde{s}, \tilde{a} \rightarrow \tilde{s}'$ depends only on: (i) the other transitions from $(\tilde{s}, \tilde{a})$, and (ii) the counterfactual probabilities of the transitions from the observed state-action pair $(s_t, a_t)$.

However, the main difficulty is that $\theta$ must satisfy the optimisation constraints for \emph{all} state-action pairs in the MDP simultaneously. To address this, the proof proceeds in two steps:

\begin{enumerate}
    \item \textbf{Induction over state-action pairs (\ref{sec: probability bounds induction})}. We show that, for any MDP $\mathcal{M}$, if we can identify assignments of $\theta$ that satisfy the constraints of the optimisation problem individually for all $|\mathcal{S}| \times |\mathcal{A}|$ state-action pairs in $\mathcal{M}$, then these assignments can be combined to form a single $\theta$ that satisfies the constraints of all $|\mathcal{S}| \times |\mathcal{A}|$ state-action pairs simultaneously. This step is crucial for the second part of the proof, because it allows us to focus on optimising the counterfactual probability of a particular transition without violating constraints elsewhere.

    \item \textbf{Closed-form Probability Bounds (\ref{sec: probability bounds cases})} 
    We then analyse how to minimise/maximise a particular counterfactual transition probability, $\tilde{P}_t(s’ \mid s, a)$, dividing the proof into three disjoint and exhaustive cases, based on the relationships between the observed and counterfactual state-action pairs:
    
    \begin{itemize}
        \item When the counterfactual state-action pair is the observed state-action pair, $(s, a) = (s_t, a_t)$
        \item When the counterfactual state-action pair has disjoint support with the observed state-action pair under the interventional probability distribution
        \item When the counterfactual state-action pair has overlapping support with the observed state-action pair under the interventional probability distribution.
    \end{itemize}
    
    When the supports of the observed and counterfactual state-action pairs overlap, the counterfactual stability and counterfactual monotonicity assumptions influence the counterfactual probability bounds, leading to different closed-form expressions. In \ref{sec: probability bounds cases} we cover all possible cases in an MDP systematically and prove, subject to the constraints of the optimisation problem on the observed and counterfactual state-action pairs, the minimum and maximum values of the sum \[\sum_{\substack{u_t \in U_t \\f(s_t, a_t, u_t) = s_{t+1} \\ f(\tilde{s}, \tilde{a}, u_t) = \tilde{s}'}} \theta_{u_t}\] to identify the lower and upper counterfactual probability bounds.
\end{enumerate}
}

\subsubsection{Inductive Proof}
\label{sec: probability bounds induction}
By induction, we will prove, given any MDP $\mathcal{M}$, that if we can assign $\theta$ such that the constraints of the optimisation problem \eqref{proofeq:interventional constraint}-\eqref{proofeq:valid prob2} are satisfied for all $|\mathcal{S}| \times |\mathcal{A}|$ state-action pairs in $\mathcal{M}$ separately, we know these can be combined to form a single valid $\theta$ across all $|\mathcal{S}| \times |\mathcal{A}|$ state-action pairs in $\mathcal{M}$.

Take an arbitrary MDP $\mathcal{M}$ with state space $|\mathcal{S}|$ and action space $|\mathcal{A}|$. In total, there are $N = |\mathcal{S}| \times |\mathcal{A}|$ state-action pairs in the MDP, meaning there are $|U_t| = |\mathcal{S}|^{|\mathcal{S}| \times |\mathcal{A}|}$ possible unique structural equation mechanisms. At each stage of the inductive proof, we will add a new state-action pair value from the MDP $\mathcal{M}$ to the causal model, and adjust $\theta$ to satisfy the constraints for this new state-action pair without affecting whether the constraints hold for all state-action pair values already considered in the causal model. When we have considered $k$ state-action pairs, this results in $\theta \in \mathbb{R}^{|\mathcal{S}|^k}$. Therefore, once we have added all $|\mathcal{S}| \times |\mathcal{A}|$ state-action pairs, we will have $\theta \in \mathbb{R}^{|U_t|}$, as required by the optimisation problem.

\paragraph{Base Case 1} Assume $(s_t, a_t)$ (the observed state-action pair) is the only state-action pair value we want to consider in the causal model (this always exists, as we always have an observed transition). We require a structural equation mechanism for each possible next state (of which there are $|\mathcal{S}|$). This leads to an assignment of $\theta^1 \in \mathbb{R}^{|\mathcal{S}|}$\footnote{The superscript of each $\theta^k$ indicates the number of state-action pair values we have considered in the causal model. When we have all $|\mathcal{S}| \times |\mathcal{A}|$ state-action pairs, we will have our final $\theta = \theta^{|\mathcal{S}| \times |\mathcal{A}|}$} as follows, which is the only possible assignment that satisfies all the constraints \eqref{proofeq:interventional constraint}-\eqref{proofeq:valid prob2}:

\begin{table}[h]
\centering
\begin{tabular}{c|c|c|c}
$\theta^1_{u_1}$                                  & $\theta^1_{u_2}$                                 & ... & $\theta^1_{u_{|\mathcal{S}|}}$\\
\hline
$P(s_1 \mid s_t, a_t)$ & $P(s_2 \mid s_t, a_t)$ & ... & $P(s_{|\mathcal{S}|} \mid s_t, a_t)$
\end{tabular}
\end{table}

$\theta^1 = [P(s_1 \mid s_t, a_t), P(s_2 \mid s_t, a_t), ..., P(s_{|\mathcal{S}|} \mid s_t, a_t)]$. Given $\theta^1$, the counterfactual transition probability for each transition $s_t, a_t \rightarrow s'$ can be calculated with:

\[
\tilde{P}_t(s' \mid s_t, a_t) = \dfrac{\sum_{u_t = 1}^{|U_t|} \mu_{s_t, a_t, u_t, s'} \cdot \mu_{s_{t}, a_{t}, u_t, s_{t+1}} \cdot \theta_{u_t}}{P(s_{t+1} \mid s_t, a_t)}
\]

Therefore, given $\theta^1$, $\tilde{P}_t(s_{t+1} \mid s_t, a_t) = 1$, and $\forall s' \in \mathcal{S}\setminus\{s_{t+1}\}$, $\tilde{P}_t(s' \mid s_t, a_t) = 0$. These are the only possible counterfactual probabilities, so these probabilities are produced by the closed-form bounds in Section \ref{sec: probability bounds cases}. These counterfactual probabilities satisfy the monotonicity and counterfactual stability constraints of the optimisation problem, as follows:

\begin{itemize}
    \item Because $(\tilde{s}, \tilde{a})=(s_t, a_t)$, CS \eqref{proofeq:counterfactual stability} doesn't apply.

    \item Because $\tilde{P}_{t}(s_{t+1} \mid s_t, a_t) = 1 \geq P(s_{t+1} \mid s_t, a_t)$, this satisfies Mon1 \eqref{proofeq:monotonicity1}.
    
    \item Because $\forall s' \in \mathcal{S}\setminus \{s_{t+1}\}, \tilde{P}_{t}(s' \mid s_t, a_t) = 0 \leq P(s' \mid s_t, a_t)$, this satisfies Mon2 \eqref{proofeq:monotonicity2}.
\end{itemize}

\paragraph{Base Case 2} Let us assume that we only have the observed state-action pair in the causal model, and assume we have a valid $\theta^1$ with $|\mathcal{S}|$ structural equation mechanisms. 

Now, take the state-action pair for which we wish to calculate the probability bounds, $(s, a)$, to add to the causal model. We now require $|\mathcal{S}|^2$ structural equation mechanisms, one for each possible combination of transitions from the two state-action pairs. We can view this as separating each of the existing structural equation mechanisms from Base Case 1 into $|\mathcal{S}|$ new structural equation mechanisms, one for every possible transition from the new state-action pair, $(s, a)$:

\begin{figure}[!ht]
\centering
\resizebox{0.5\textwidth}{!}{%
\begin{circuitikz}
\tikzstyle{every node}=[font=\LARGE]
\node [font=\LARGE] at (3,16.25) {$u_1$};
\node [font=\LARGE] at (6.0,16.25) {$...$};
\node [font=\LARGE] at (8.75,16.25) {$u_{|S|}$};
\draw [short] (3,15.75) -- (1.25,13.5);
\draw [short] (3,15.75) -- (5,13.5);
\draw [short] (3,15.75) -- (3,13.5);
\node [font=\LARGE] at (1,13) {$u_{1,1}$};
\node [font=\LARGE] at (5,13) {$u_{1,|S|}$};
\node [font=\LARGE] at (3,13) {$u_{1,2}$};
\node [font=\LARGE] at (4.0,13) {$...$};
\node [font=\LARGE] at (9.7,13) {$...$};

\draw [short] (8.75,15.75) -- (7,13.5);
\draw [short] (8.75,15.75) -- (10.75,13.5);
\draw [short] (8.75,15.75) -- (8.75,13.5);
\node [font=\LARGE] at (6.75,13) {$u_{|S|,1}$};
\node [font=\LARGE] at (10.75,13) {$u_{|S|,|S|}$};
\node [font=\LARGE] at (8.75,13) {$u_{|S|,2}$};
\end{circuitikz}
}%
\end{figure}

where, for $1 \leq i \leq |\mathcal{S}|, 1 \leq n \leq |\mathcal{S}|$, $u_{i, n}$ leads to the same next state for $(s_t, a_t)$ as $u_i$, and produces the transition $s, a \rightarrow s_n$. \\

In the same way, we can split each $\theta^1_{u_1}, ..., \theta^1_{u_{|\mathcal{S}|}}$ across these new structural equation mechanisms to obtain $\theta^2 \in \mathbb{R}^{|\mathcal{S}|^2}$:

\begin{figure}[!ht]
\centering
\resizebox{0.5\textwidth}{!}{%
\begin{circuitikz}
\tikzstyle{every node}=[font=\LARGE]
\node [font=\LARGE] at (3,16.25) {$\theta^1_{u_1}$};
\node [font=\LARGE] at (6.0,16.25) {$...$};
\node [font=\LARGE] at (8.75,16.25) {$\theta^1_{u_{|S|}}$};
\draw [short] (3,15.75) -- (1.25,13.5);
\draw [short] (3,15.75) -- (5,13.5);
\draw [short] (3,15.75) -- (3,13.5);
\node [font=\LARGE] at (1,13) {$\theta^2_{u_{1,1}}$};
\node [font=\LARGE] at (5,13) {$\theta^2_{u_{1,|S|}}$};
\node [font=\LARGE] at (3,13) {$\theta^2_{u_{1,2}}$};
\node [font=\LARGE] at (3.9,13) {$...$};
\node [font=\LARGE] at (9.6,13) {$...$};
\draw [short] (8.75,15.75) -- (7,13.5);
\draw [short] (8.75,15.75) -- (10.75,13.5);
\draw [short] (8.75,15.75) -- (8.75,13.5);
\node [font=\LARGE] at (6.75,13) {$\theta^2_{u_{|S|,1}}$};
\node [font=\LARGE] at (10.75,13) {$\theta^2_{u_{|S|,|S|}}$};
\node [font=\LARGE] at (8.75,13) {$\theta^2_{u_{|S|,2}}$};
\end{circuitikz}
}%
\end{figure}

By splitting $\theta^1$ in this way, we guarantee that the total probability across each set of mechanisms (where each set produces a different transition from $(s_t, a_t)$) remains the same, i.e., \[\forall i \in \{1, .., |\mathcal{S}|\}, \sum_{n=1}^{|\mathcal{S}|}\theta^2_{u_{i,n}} = \theta_{u_i}^1\]

This means that the counterfactual probabilities of all transitions from $(s_t, a_t)$ will be exactly the same when using $\theta^2$ as with $\theta^1$. Because we know all the constraints were satisfied when using $\theta^1$, this guarantees that all the constraints of the optimisation problem will continue to hold for $(s_t, a_t)$ with $\theta^2$.\\

Now, we only need to ensure we assign $\theta^2$ such that all the constraints hold for the transitions from $(s, a)$. Firstly, we need $\sum_{i=1}^{|\mathcal{S}|}{\theta^2_{u_{i, n}}} = P(s_n \mid s, a)$ for every possible next state $s_n \in \mathcal{S}$, to satisfy the interventional constraint \eqref{proofeq:interventional constraint}. For each transition $s, a \rightarrow s_n$, there is a $\theta^2_{u_{i, n}}$ for every $u_i$, and for every $\theta^2_{u_{i, n}}$, $0 \leq \theta^2_{u_{i, n}} \leq \theta^1_{u_n}$. Therefore, $0 \leq \sum_{i=1}^{|\mathcal{S}|}{\theta^2_{u_{i, n}}} \leq 1, \forall s_n \in \mathcal{S}$. We also know that:

\begin{align*}
\sum_{n=1}^{|\mathcal{S}|}\sum_{i=1}^{|\mathcal{S}|}\theta^2_{u_i,n} &= \sum_{i=1}^{|\mathcal{S}|}\theta^1_{u_i} \text{ (because we split the probabilities of each $\theta^1_{u_i}$)}\\
&= \sum_{i=1}^{|\mathcal{S}|}P(s_i|s_t, a_t) \\&= 1
\end{align*}

and 

\[
\sum_{n=1}^{|\mathcal{S}|}\sum_{i=1}^{|\mathcal{S}|}{\theta^2_{u_{i, n}}} = 
\sum_{n=1}^{|\mathcal{S}|}P(s_n|s,a) = 1
\]

Therefore, no matter the assignment of $\theta^1$, we can satisfy the interventional probability constraints \eqref{proofeq:interventional constraint} for all transitions from $(s, a)$. In Section \ref{sec: probability bounds cases}, we prove valid assignments of $\theta$ that would minimise and maximise the counterfactual probability for some transition from $(s, a)$, and satisfy all the constraints \eqref{proofeq:interventional constraint}-\eqref{proofeq:valid prob2}. Therefore, we can pick the assignment of $\theta^2$ from these cases to minimise/maximise the counterfactual probability of the transition in question.

\paragraph{Inductive Case} Let us assume that we now have $k$ state-action pair values from the MDP $\mathcal{M}$ in our causal model, and we have found a valid assignment $\theta^k \in \mathbb{R}^{|\mathcal{S}|^k}$ that satisfies all the constraints for these $k$ state-action pairs, including $(s_t, a_t)$. Because there are $k$ state-action pairs, there are $|\mathcal{S}|^{k}$ possible unique structural equation mechanisms for these state-action pairs\footnote{Note we have changed the indexing of each $\theta^k_{i, n}$ to $\theta^k_{j}$, where $j = (i-1)\cdot|\mathcal{S}| + n$.}, with probabilities as follows:

\begin{table}[h]
\centering
\begin{tabular}{lllllll}
$\theta^k_{u_1}$ & ... & $\theta^k_{u_{|\mathcal{S}|}}$ & ... & $\theta^k_{u_{|\mathcal{S}|^2}}$ & ... & $\theta^k_{u_{|\mathcal{S}|^{k}}}$ \\
\end{tabular}
\end{table}

Now, we wish to add the $k+1^{th}$ state-action pair from the MDP to the causal model, resulting in $|\mathcal{S}|^{k+1}$ possible structural equation mechanisms. Let the $k+1^{th}$ state-action pair be $(s, a)$ arbitrarily. We can view this as separating the existing structural equation mechanisms each into $|\mathcal{S}|$ new structural equation mechanisms, one for every possible transition from the $k+1^{th}$ state-action pair:

\begin{figure}[!ht]
\centering
\resizebox{1\textwidth}{!}{%
\begin{circuitikz}
\tikzstyle{every node}=[font=\LARGE]
\node [font=\LARGE] at (3.0,16.25) {$u_1$};
\node [font=\LARGE] at (6.0,16.25) {$...$};
\node [font=\LARGE] at (8.75,16.25) {$u_{|S|}$};
\node [font=\LARGE] at (12,16.25) {$...$};
\node [font=\LARGE] at (14.75,16.25) {$u_{|S|^2}$};
\node [font=\LARGE] at (20.75,16.25) {$u_{|S|^k}$};
\node [font=\LARGE] at (18,16.25) {$...$};
\draw [short] (3,15.75) -- (1.25,13.5);
\draw [short] (3,15.75) -- (5,13.5);
\draw [short] (3,15.75) -- (3,13.5);
\node [font=\LARGE] at (1,13) {$u_{1,1}$};
\node [font=\LARGE] at (5,13) {$u_{1,|S|}$};
\node [font=\LARGE] at (3,13) {$u_{1,2}$};
\node [font=\LARGE] at (3.9,13) {$...$};
\draw [short] (8.75,15.75) -- (7,13.5);
\draw [short] (8.75,15.75) -- (10.75,13.5);
\draw [short] (8.75,15.75) -- (8.75,13.5);
\node [font=\LARGE] at (6.75,13) {$u_{|S|,1}$};
\node [font=\LARGE] at (10.75,13) {$u_{|S|,|S|}$};
\node [font=\LARGE] at (8.75,13) {$u_{|S|,2}$};
\node [font=\LARGE] at (9.7,13) {$...$};
\draw [short] (14.75,15.75) -- (13,13.5);
\draw [short] (14.75,15.75) -- (16.75,13.5);
\draw [short] (14.75,15.75) -- (14.75,13.5);
\node [font=\LARGE] at (12.75,13) {$u_{|S|^2,1}$};
\node [font=\LARGE] at (17,13) {$u_{|S|^2,|S|}$};
\node [font=\LARGE] at (14.75,13) {$u_{|S|^2,2}$};
\node [font=\LARGE] at (15.8,13) {$...$};
\draw [short] (20.75,15.75) -- (19,13.5);
\draw [short] (20.75,15.75) -- (22.75,13.5);
\draw [short] (20.75,15.75) -- (20.75,13.5);
\node [font=\LARGE] at (18.75,13) {$u_{|S|^{k},1}$};
\node [font=\LARGE] at (23,13) {$u_{|S|^{k},|S|}$};
\node [font=\LARGE] at (20.75,13) {$u_{|S|^{k}, 2}$};
\node [font=\LARGE] at (21.8,13) {$...$};
\end{circuitikz}
}%
\end{figure}
where each $u_{i, n}$ produces the same next states for all of the first $k$ state-action pairs as $u_i$, and produces the transition $s, a \rightarrow s_n$.

In the same way, we can split each $\theta^k_{u_1}, ..., \theta^k_{u_{|\mathcal{S}|}}, ...$ across these new structural equation mechanisms to find $\theta^{k+1} \in \mathbb{R}^{|\mathcal{S}|^{k+1}}$.

\begin{figure}[!ht]
\centering
\resizebox{1\textwidth}{!}{%
\begin{circuitikz}
\tikzstyle{every node}=[font=\LARGE]
\node [font=\LARGE] at (3,16.25) {$\theta^k_{u_1}$};
\node [font=\LARGE] at (6.0,16.25) {$...$};
\node [font=\LARGE] at (8.75,16.25) {$\theta^k_{u_{|S|}}$};
\node [font=\LARGE] at (11.8,16.25) {$...$};
\node [font=\LARGE] at (14.75,16.25) {$\theta^k_{u_{|S|^2}}$};
\node [font=\LARGE] at (20.75,16.25) {$\theta^k_{u_{|S|^k}}$};
\node [font=\LARGE] at (17.8,16.25) {$...$};
\draw [short] (3,15.75) -- (1.25,13.5);
\draw [short] (3,15.75) -- (5,13.5);
\draw [short] (3,15.75) -- (3,13.5);
\node [font=\LARGE] at (1,13) {$\theta^{k+1}_{u_{1,1}}$};
\node [font=\LARGE] at (5,13) {$\theta^{k+1}_{u_{1,|S|}}$};
\node [font=\LARGE] at (3,13) {$\theta^{k+1}_{u_{1,2}}$};
\node [font=\LARGE] at (3.9,13) {$...$};
\draw [short] (8.75,15.75) -- (7,13.5);
\draw [short] (8.75,15.75) -- (10.75,13.5);
\draw [short] (8.75,15.75) -- (8.75,13.5);
\node [font=\LARGE] at (6.75,13) {$\theta^{k+1}_{u_{|S|,1}}$};
\node [font=\LARGE] at (10.9,13) {$\theta^{k+1}_{u_{|S|,|S|}}$};
\node [font=\LARGE] at (8.75,13) {$\theta^{k+1}_{u_{|S|,2}}$};
\node [font=\LARGE] at (9.7,13) {$...$};
\draw [short] (14.75,15.75) -- (13,13.5);
\draw [short] (14.75,15.75) -- (16.75,13.5);
\draw [short] (14.75,15.75) -- (14.75,13.5);
\node [font=\LARGE] at (12.75,13) {$\theta^{k+1}_{u_{|S|^2,1}}$};
\node [font=\LARGE] at (17,13) {$\theta^{k+1}_{u_{|S|^2,|S|}}$};
\node [font=\LARGE] at (14.75,13) {$\theta^{k+1}_{u_{|S|^2,2}}$};
\node [font=\LARGE] at (15.75,13) {$...$};
\draw [short] (20.75,15.75) -- (19,13.5);
\draw [short] (20.75,15.75) -- (22.75,13.5);
\draw [short] (20.75,15.75) -- (20.75,13.5);
\node [font=\LARGE] at (18.75,13) {$\theta^{k+1}_{u_{|S|^{k},1}}$};
\node [font=\LARGE] at (23,13) {$\theta^{k+1}_{u_{|S|^{k},|S|}}$};
\node [font=\LARGE] at (20.75,13) {$\theta^{k+1}_{u_{|S|^{k}, 2}}$};
\node [font=\LARGE] at (21.75,13) {$...$};
\end{circuitikz}
}%
\end{figure}

By splitting $\theta^k$ in this way, we have $\forall i \in \{1, .., |\mathcal{S}|\}, \sum_{n=1}^{|\mathcal{S}|}\theta^{k+1}_{u_{i,n}} = \theta_{u_i}^k$. Each $\theta^{k+1}_{u_{i, n}}$ is only different to $\theta^k_{u_i}$ in the transition from $(s, a)$. Therefore, this guarantees that for every state-action pair $(s', a') \neq (s, a)$ in the first $k$ pairs added to the causal model, the total probability across each set of mechanisms (where each set produces a different transition from $(s', a')$) remains the same in $\theta^{k+1}$ as in $\theta^k$. As a result, the counterfactual probabilities of all transitions from each $(s', a')$ will be exactly the same when using $\theta^{k+1}$ as with $\theta^k$. Because we know all the constraints were satisfied when using $\theta^k$, this guarantees that all the constraints of the optimisation problem will continue to hold for $(s', a')$ with $\theta^{k+1}$.\\

Now, we only need to make sure that we can assign $\theta^{k+1}$ such that the constraints also hold for all the transitions from $(s, a)$. We need $\sum_{i=1}^{|\mathcal{S}|^k}{\theta^{k+1}_{u_{i, n}}} = P(s_n \mid s, a)$ for every possible next state $s_n \in \mathcal{S}$, to satisfy the interventional probability constraints \eqref{proofeq:interventional constraint}. For each transition $s, a \rightarrow s_n$, there is a $\theta^{k+1}_{u_{i, n}}$ for every $u_i$, and for every $\theta^{k+1}_{u_{i, n}}$, $0 \leq \theta^{k+1}_{u_{i, n}} \leq \theta^k_{u_n}$, therefore $0 \leq \sum_{i=1}^{|\mathcal{S}|^k}{\theta^{k+1}_{u_{i, n}}} \leq 1, \forall s_n \in \mathcal{S}$. We also know that:

\begin{align*}
\sum_{n=1}^{|\mathcal{S}|}\sum_{i=1}^{|\mathcal{S}|^k}\theta^{k+1}_{u_i,n} &= \sum_{i=1}^{|\mathcal{S}|^k}\theta^k_{u_i} \text{ (because we split the probabilities of each $\theta^k_{u_i}$)}\\&= 1
\end{align*}

and 

\[
\sum_{n=1}^{|\mathcal{S}|}\sum_{i=1}^{|\mathcal{S}|^k}{\theta^2_{u_{i, n}}} = 
\sum_{n=1}^{|\mathcal{S}|}P(s_n|s,a) = 1
\]

Therefore, no matter the assignment of $\theta^k$, we can satisfy the interventional probability constraints \eqref{proofeq:interventional constraint} for all transitions from $(s, a)$. In Section \ref{sec: probability bounds cases}, we show that we can assign $\theta^{k+1}$ such that it would minimise and maximise the counterfactual probability for some transition from $(s, a)$, and satisfy all the constraints \eqref{proofeq:interventional constraint}-\eqref{proofeq:valid prob2}. Therefore, we know we can assign $\theta^{k+1}$ such that it satisfies the constraints for all $k+1$ state-action pairs in the causal model.

\paragraph{Conclusion} By induction, we have proven that as long as we can satisfy the constraints for each state-action pair separately, we can combine these into a single valid $\theta$ that satisfies the constraints for all state-action pairs, and minimises/maximises the counterfactual probability for a particular transition in question. Because the state-action pairs are added in an arbitrary order in this inductive proof (other than the observed state-action pair and the state-action pair for the transition in question, which both must always exist), this works for any MDP.

\pagebreak
\subsubsection{Closed-form Probability Bounds}
\label{sec: probability bounds cases}

Take an arbitrary transition $\tilde{s}, \tilde{a} \rightarrow \tilde{s}'$ from the MDP. We need to consider the minimum/maximum possible value of 
$\sum_{\substack{u_t \in U_t\\f(s_t, a_t, u_t) = s_{t+1} \\ f(\tilde{s}, \tilde{a}, u_t) = \tilde{s}'}} \theta_{u_t}$ 
to find the lower and upper bounds, subject to the following constraints over the state-action pair, $(\tilde{s}, \tilde{a})$:

\begin{itemize}
    \item $\forall \tilde{s}' \in \mathcal{S}, \sum_{\substack{u_t \in U_t \\ f(\tilde{s}, \tilde{a}, u_t) = \tilde{s}'}}{\theta_{u_t}} = P(\tilde{s}' \mid \tilde{s}, \tilde{a})$ to satisfy \eqref{proofeq:interventional constraint}
    \item $\tilde{P}_t(s_{t+1} \mid \tilde{s}, \tilde{a}) \geq P(s_{t+1} \mid \tilde{s}, \tilde{a})$ to satisfy Mon1 \eqref{proofeq:monotonicity1}
    \item $\forall s' \in \mathcal{S}\setminus \{s_{t+1} \}, \tilde{P}_t(s' \mid \tilde{s}, \tilde{a}) \leq P(s' \mid \tilde{s}, \tilde{a})$ to satisfy Mon2 \eqref{proofeq:monotonicity2}
    \item If $(\tilde{s}, \tilde{a}) \neq (s_t, a_t)$, then $\forall \tilde{s}' \in \mathcal{S}\setminus\{s_{t+1}\}, \tilde{P}_t(\tilde{s}' \mid \tilde{s}, \tilde{a})$ satisfies CS \eqref{proofeq:counterfactual stability}
    \item $0 \leq \theta_{u_t} \leq 1, \forall u_t$ \eqref{proofeq:valid prob1}
    \item $\sum_{u_t = 1}^{|U_t|} \theta_{u_t} = 1$ \eqref{proofeq:valid prob2}
\end{itemize}

From \eqref{proofeq:interventional constraint}, we also derive the following constraints over the state-action pair $(s_t, a_t)$,  which must hold when calculating the counterfactual probability bounds for every state-action pair.

\begin{itemize}
    \item $\sum_{\substack{u_t \in U_t\\f(s_t, a_t, u_t) = s_{t+1}}}{\theta_{u_t}} = P(s_{t+1} \mid s_t, a_t)$ to satisfy \eqref{proofeq:interventional constraint}
    \item $\sum_{\substack{u_t \in U_t\\f(s_t, a_t, u_t) \neq s_{t+1}}}{\theta_{u_t}} = \sum_{s' \in \mathcal{S}\setminus \{s_{t+1} \}} P(s' \mid s_t, a_t) = 1 - P(s_{t+1} \mid s_t, a_t)$ to satisfy \eqref{proofeq:interventional constraint}
\end{itemize}

Consider the following distinct cases:
\textbf{\begin{enumerate}
    \item $(\tilde{s}, \tilde{a}) = (s_t, a_t)$
    \item $(\tilde{s}, \tilde{a}) \neq (s_t, a_t)$ and $(\tilde{s}, \tilde{a})$ has disjoint support with $(s_t, a_t)$
    \item $(\tilde{s}, \tilde{a}) \neq (s_t, a_t)$, $(\tilde{s}, \tilde{a})$ has overlapping support with $(s_t, a_t)$
\end{enumerate}}

These cases are disjoint, and cover all possible state-action pairs.

\pagebreak
\subsubsection{CF Bounds for Observed $(s_t, a_t)$}
\begin{theorem}
\label{proof theorem: observed state}
For outgoing transitions from the observed state-action pair $(s_t, a_t)$, the linear program will produce bounds of:

\[\tilde{P}_{t}^{LB}(s_{t+1} \mid s_t, a_t) = \tilde{P}_{t}^{UB}(s_{t+1} \mid s_t, a_t) = 1\] 

for the observed next state $s_{t+1}$, and 

\[\tilde{P}_{t}^{LB}(\tilde{s}' \mid s_t, a_t) = \tilde{P}_{t}^{UB}(\tilde{s}' \mid s_t, a_t) = 0\]

for all other $\tilde{s}' \in \mathcal{S}\setminus \{s_{t+1}\}$.
\end{theorem}

\begin{proof}
If $(\tilde{s}, \tilde{a}) = (s_t, a_t)$, then it is impossible to have a $u_t$ such that $f(s_t, a_t, u_t) = s'$ and $f(\tilde{s}, \tilde{a}, u_t) \neq s'$, $\forall s' \in \mathcal{S}$. Therefore, the only possible assignment of $\theta$ is:

\[
    \sum_{\substack{u_t \in U_t\\f(s_t, a_t, u_t) = s_{t+1}}}{\theta_{u_t}} = P(s_{t+1} \mid \tilde{s}, \tilde{a})
\]

\[
    \forall s' \in \mathcal{S}\setminus \{s_{t+1}\}, \sum_{\substack{u_t \in U_t\\f(s_t, a_t, u_t) \neq s'}}{\theta_{u_t}} = P(s' \mid \tilde{s}, \tilde{a})\\
\]
    This assignment satisfies the interventional probability constraints:

    \begin{itemize}
        \item $\sum_{\substack{u_t \in U_t\\f(s_t, a_t, u_t) = s_{t+1}}}{\theta_{u_t}} = P(s_{t+1} \mid s_t, a_t)$ satisfying \eqref{proofeq:interventional constraint}.
        \item $\sum_{\substack{u_t \in U_t\\f(s_t, a_t, u_t) \neq s_{t+1}}}{\theta_{u_t}} = 1 - P(s_{t+1} \mid s_t, a_t)$ satisfying \eqref{proofeq:interventional constraint}.
        \item $\forall s' \in \mathcal{S}, \sum_{\substack{u_t \in U_t \\ f(\tilde{s}, \tilde{a}, u_t) = s'}}{\theta_{u_t}} = P(s' \mid \tilde{s}, \tilde{a})$ satisfying \eqref{proofeq:interventional constraint}.
        \item Because of Lemma \ref{lemma: CF probs of observed state-action pair}, $\tilde{P}_{t}(s_{t+1} \mid s_t, a_t) = 1 \geq P(s_{t+1} \mid s_t, a_t)$, satisfying Mon1 \eqref{proofeq:monotonicity1}, and $\forall s' \in \mathcal{S}\setminus \{s_{t+1}\}, \tilde{P}_{t}(s' \mid s_t, a_t) = 0 \leq P(s' \mid s_t, a_t)$, satisfying Mon2 \eqref{proofeq:monotonicity2}.
        \item Because $(\tilde{s}, \tilde{a}) = (s_t, a_t)$, CS \eqref{proofeq:counterfactual stability} doesn't apply.
        \item $0 \leq \theta_{u_t} \leq 1, \forall u_t$, satisfying \eqref{proofeq:valid prob1}.
        \item $\sum_{u_t = 1}^{|U_t|} \theta_{u_t} = \sum_{s' \in \mathcal{S}} P(s' \mid s_t, a_t) = 1$, satisfying \eqref{proofeq:valid prob2}.
    \end{itemize}
    
Therefore,

\[\tilde{P}_{t}^{LB}(s_{t+1} \mid s_t, a_t) = \tilde{P}_{t}^{UB}(s_{t+1} \mid s_t, a_t) = 1\] 

and $\forall \tilde{s}' \in \mathcal{S}\setminus \{s_{t+1}\}$:

\[\tilde{P}_{t}^{LB}(\tilde{s}' \mid s_t, a_t) = \tilde{P}_{t}^{UB}(\tilde{s}' \mid s_t, a_t) = 0\]

\end{proof}

\subsubsection{CF Bounds for Disjoint $(\tilde{s}, \tilde{a})$}
To enhance readability, we split the proof for Theorem \ref{theorem: ub disjoint} into two separate theorems: Theorem \ref{proof theorem: ub disjoint} proves the upper counterfactual probability bound, and Theorem \ref{proof theorem: lb disjoint} proves the lower counterfactual probability bound.

\begin{theorem}
\label{proof theorem: ub disjoint}
For outgoing transitions from state-action pairs $(\tilde{s}, \tilde{a})$ which have completely disjoint support from the observed $(s_t, a_t)$, the linear program will produce an upper bound of:

\[\tilde{P}_{t}^{UB}(\tilde{s}' \mid \tilde{s}, \tilde{a}) = 
\begin{cases}
\dfrac{P(\tilde{s}' \mid \tilde{s}, \tilde{a})}{P(s_{t+1} \mid s_t, a_t)} & \text{if $P(\tilde{s}' \mid \tilde{s}, \tilde{a}) < P(s_{t+1} \mid s_t, a_t)$} \\
1 & \text{otherwise} \\
\end{cases}
\]
for all next states $\tilde{s}' \in \mathcal{S}$.

\end{theorem}

\begin{proof} Take an arbitrary transition $\tilde{s}, \tilde{a} \rightarrow \tilde{s}'$ where $(\tilde{s}, \tilde{a})$ has disjoint support with the observed state-action pair $(s_t, a_t)$. By definition:

\[\tilde{P}_{t}^{UB}(\tilde{s}' \mid \tilde{s}, \tilde{a}) = \max_\theta \left( \frac{\sum_{u_t = 1}^{|U_t|} \mu_{\tilde{s}, \tilde{a}, u_t, \tilde{s}'} \cdot \mu_{s_t, a_t, u_t, s_{t+1}} \cdot \theta_{u_t}}{P(s_{t+1} \mid s_t, a_t)} \right)\]

Because $P(s_{t+1} \mid s_t, a_t)$ (the interventional probability of the observed transition) is fixed, to find the upper bound for $\tilde{P}_t(\tilde{s}' \mid \tilde{s}, \tilde{a})$ we need to maximise $\sum_{u_t = 1}^{|U_t|} \mu_{\tilde{s}, \tilde{a}, u_t, \tilde{s}'} \cdot \mu_{s_t, a_t, u_t, s_{t+1}} \cdot \theta_{u_t}$. From Lemma \ref{lemma:absolute max cf prob}, we have:

\[
\sum_{u_t = 1}^{|U_t|} \mu_{\tilde{s}, \tilde{a}, u_t, \tilde{s}'} \cdot \mu_{s_t, a_t, u_t, s_{t+1}} \cdot \theta_{u_t} \leq \min\left(P(s_{t+1} \mid s_t, a_t), P(\tilde{s}' \mid \tilde{s}, \tilde{a})\right)
\]

Therefore, if we can show that there exists a $\theta$ such that
\[
\sum_{u_t = 1}^{|U_t|} \mu_{\tilde{s}, \tilde{a}, u_t, \tilde{s}'} \cdot \mu_{s_t, a_t, u_t, s_{t+1}} \cdot \theta_{u_t} = \min\left((P(s_{t+1} \mid s_t, a_t), P(\tilde{s}' \mid \tilde{s}, \tilde{a})\right)
\]

and $\theta$ satisfies the constraints of the optimisation problem, we know that this results in the maximum possible counterfactual transition probability $\tilde{P}_t(\tilde{s}' \mid \tilde{s}, \tilde{a})$. Since $\forall s,s' \in \mathcal{S}, \forall a \in \mathcal{A}, \forall u_t \in U_t, \mu_{s, a, u_t, s'} = 1 \iff f(s, a, u_t) = s'$ by definition of $\mu$, this is equivalent to showing that there exists a $\theta$ such that:

\[
\sum_{\substack{u_t \in U_t \\f(s_t, a_t, u_t) = s_{t+1} \\ f(\tilde{s}, \tilde{a}, u_t) = \tilde{s}'}} \theta_{u_t}= \min\left((P(s_{t+1} \mid s_t, a_t), P(\tilde{s}' \mid \tilde{s}, \tilde{a})\right)
\]

As these notations are equivalent, we will use this second notation for brevity and clarity. Consider the following disjoint cases:

\begin{itemize}
    \item $P(\tilde{s}' \mid \tilde{s}, \tilde{a}) < P(s_{t+1} \mid s_t, a_t)$
    \item $P(\tilde{s}' \mid \tilde{s}, \tilde{a}) \geq P(s_{t+1} \mid s_t, a_t)$
\end{itemize}

\paragraph{Case 1: $P(\tilde{s}' \mid \tilde{s}, \tilde{a}) < P(s_{t+1} \mid s_t, a_t)$}
\noindent
\begin{proof}
In this case, assign $\theta$ as follows:

\[
\begin{cases}
    \sum_{\substack{u_t \in U_t \\f(s_t, a_t, u_t) = s_{t+1} \\ f(\tilde{s}, \tilde{a}, u_t) = \tilde{s}'}} \theta_{u_t}= P(\tilde{s}' \mid \tilde{s}, \tilde{a})\\
    \sum_{\substack{u_t \in U_t \\f(s_t, a_t, u_t) = s_{t+1} \\ f(\tilde{s}, \tilde{a}, u_t) \neq \tilde{s}'}} \theta_{u_t} = P(s_{t+1} \mid s_t, a_t) - P(\tilde{s}' \mid \tilde{s}, \tilde{a})\\
    \sum_{\substack{u_t \in U_t \\f(s_t, a_t, u_t) \neq s_{t+1} \\ f(\tilde{s}, \tilde{a}, u_t) = \tilde{s}'}} \theta_{u_t} = 0\\
    \sum_{\substack{u_t \in U_t \\f(s_t, a_t, u_t) \neq s_{t+1} \\ f(\tilde{s}, \tilde{a}, u_t) \neq \tilde{s}'}} \theta_{u_t} = 1 - P(s_{t+1} \mid s_t, a_t)\\
\end{cases}
\]

    This assignment of $\theta$ satisfies the constraints of the linear optimisation problem, as follows:

    \begin{itemize}        
        \item $\sum_{\substack{u_t \in U_t \\f(s_t, a_t, u_t) = s_{t+1}}} \theta_{u_t} = \sum_{\substack{u_t \in U_t \\f(s_t, a_t, u_t) = s_{t+1} \\ f(\tilde{s}, \tilde{a}, u_t) = \tilde{s}'}} \theta_{u_t} + \sum_{\substack{u_t \in U_t \\f(s_t, a_t, u_t) = s_{t+1} \\ f(\tilde{s}, \tilde{a}, u_t) \neq \tilde{s}'}} \theta_{u_t} = P(s_{t+1} \mid s_t, a_t)$ \\satisfying \eqref{proofeq:interventional constraint}.

        \item $\sum_{\substack{u_t \in U_t\\f(s_t, a_t, u_t) \neq s_{t+1}}}{\theta_{u_t}} = \sum_{\substack{u_t \in U_t \\f(s_t, a_t, u_t) \neq s_{t+1} \\ f(\tilde{s}, \tilde{a}, u_t) = \tilde{s}'}} \theta_{u_t} + \sum_{\substack{u_t \in U_t \\f(s_t, a_t, u_t) \neq s_{t+1} \\ f(\tilde{s}, \tilde{a}, u_t) \neq \tilde{s}'}} \theta_{u_t} =1 -
        P(s_{t+1} \mid s_t, a_t)$ satisfying \eqref{proofeq:interventional constraint}.
        
        \item $\sum_{\substack{u_t \in U_t \\f(\tilde{s}, \tilde{a}, u_t) = \tilde{s}'}} \theta_{u_t} = \sum_{\substack{u_t \in U_t \\f(s_t, a_t, u_t) = s_{t+1} \\ f(\tilde{s}, \tilde{a}, u_t) = \tilde{s}'}} \theta_{u_t} + \sum_{\substack{u_t \in U_t \\f(s_t, a_t, u_t) \neq s_{t+1} \\ f(\tilde{s}, \tilde{a}, u_t) = \tilde{s}'}} \theta_{u_t} = P(\tilde{s}' \mid \tilde{s}, \tilde{a})$ satisfying \eqref{proofeq:interventional constraint}.
        
        \item $\sum_{s' \in \mathcal{S}\setminus\{\tilde{s}'\}}\sum_{\substack{u_t \in U_t \\f(\tilde{s}, \tilde{a}, u_t) = s'}} \theta_{u_t} = \sum_{\substack{u_t \in U_t \\f(s_t, a_t, u_t) = s_{t+1} \\ f(\tilde{s}, \tilde{a}, u_t) \neq \tilde{s}'}} \theta_{u_t} + \sum_{\substack{u_t \in U_t \\f(s_t, a_t, u_t) \neq s_{t+1} \\ f(\tilde{s}, \tilde{a}, u_t) \neq \tilde{s}'}} \theta_{u_t} \\ = 1 - P(\tilde{s}' \mid \tilde{s}, \tilde{a}) = \sum_{s' \in \mathcal{S}\setminus\{\tilde{s}'\}}P(s' \mid \tilde{s}, \tilde{a})$, therefore it is possible to assign $\theta$ such that $\forall s' \in \mathcal{S}\setminus\{\tilde{s}'\}, \sum_{\substack{u_t \in U_t \\ f(\tilde{s}, \tilde{a}, u_t) = s'}}{\theta_{u_t}} = P(s' \mid \tilde{s}, \tilde{a})$ satisfying \eqref{proofeq:interventional constraint}.
        
        \item Mon1 \eqref{proofeq:monotonicity1}, Mon2 \eqref{proofeq:monotonicity2} and CS \eqref{proofeq:counterfactual stability} are all vacuously true for all transitions from $(\tilde{s}, \tilde{a})$ due to Lemma \ref{lemma: constraints do not apply if disjoint}.

        \item $0 \leq \theta_{u_t} \leq 1, \forall u_t$, satisfying \eqref{proofeq:valid prob1}.
        
        \item $\sum_{u_t = 1}^{|U_t|} \theta_{u_t} = \sum_{\substack{u_t \in U_t \\f(s_t, a_t, u_t) = s_{t+1} \\ f(\tilde{s}, \tilde{a}, u_t) = \tilde{s}'}} \theta_{u_t} + \sum_{\substack{u_t \in U_t \\f(s_t, a_t, u_t) = s_{t+1} \\ f(\tilde{s}, \tilde{a}, u_t) \neq \tilde{s}'}} \theta_{u_t} + \sum_{\substack{u_t \in U_t \\f(s_t, a_t, u_t) \neq s_{t+1} \\ f(\tilde{s}, \tilde{a}, u_t) = \tilde{s}'}} \theta_{u_t} \\+ \sum_{\substack{u_t \in U_t \\f(s_t, a_t, u_t) \neq s_{t+1} \\ f(\tilde{s}, \tilde{a}, u_t) \neq \tilde{s}'}} \theta_{u_t} = 1$, satisfying \eqref{proofeq:valid prob2}.
    \end{itemize}

    Therefore,

    \[\tilde{P}_{t}^{UB}(\tilde{s}' \mid \tilde{s}, \tilde{a}) = \left( \max_\theta \frac{ \sum_{\substack{u_t \in U_t \\f(s_t, a_t, u_t) = s_{t+1} \\ f(\tilde{s}, \tilde{a}, u_t) = \tilde{s}'}} \theta_{u_t} }{P(s_{t+1} \mid s_t, a_t)} \right) = \frac{P(\tilde{s}' \mid \tilde{s}, \tilde{a})}{P(s_{t+1} \mid s_t, a_t)}\]
\end{proof}

\paragraph{Case 2: $P(\tilde{s}' \mid \tilde{s}, \tilde{a}) \geq P(s_{t+1} \mid s_t, a_t)$}
\noindent
\begin{proof}
In this case, assign $\theta$ as follows:

\[
\begin{cases}
    \sum_{\substack{u_t \in U_t \\f(s_t, a_t, u_t) = s_{t+1} \\ f(\tilde{s}, \tilde{a}, u_t) = \tilde{s}'}} \theta_{u_t}= P(s_{t+1} \mid s_t, a_t)\\
    \sum_{\substack{u_t \in U_t \\f(s_t, a_t, u_t) = s_{t+1} \\ f(\tilde{s}, \tilde{a}, u_t) \neq \tilde{s}'}} \theta_{u_t} = 0\\
    \sum_{\substack{u_t \in U_t \\f(s_t, a_t, u_t) \neq s_{t+1} \\ f(\tilde{s}, \tilde{a}, u_t) = \tilde{s}'}} \theta_{u_t} = P(\tilde{s}' \mid \tilde{s}, \tilde{a}) - P(s_{t+1} \mid s_t, a_t)\\
    \sum_{\substack{u_t \in U_t \\f(s_t, a_t, u_t) \neq s_{t+1} \\ f(\tilde{s}, \tilde{a}, u_t) \neq \tilde{s}'}} \theta_{u_t} = 1 - P(\tilde{s}' \mid \tilde{s}, \tilde{a})\\
\end{cases}
\]
    This assignment of $\theta$ satisfies the constraints of the linear optimisation problem, as follows:

    \begin{itemize}        
        \item $\sum_{\substack{u_t \in U_t \\f(s_t, a_t, u_t) = s_{t+1}}} \theta_{u_t} = \sum_{\substack{u_t \in U_t \\f(s_t, a_t, u_t) = s_{t+1} \\ f(\tilde{s}, \tilde{a}, u_t) = \tilde{s}'}} \theta_{u_t} + \sum_{\substack{u_t \in U_t \\f(s_t, a_t, u_t) = s_{t+1} \\ f(\tilde{s}, \tilde{a}, u_t) \neq \tilde{s}'}} \theta_{u_t} = P(s_{t+1} \mid s_t, a_t)$ \\satisfying \eqref{proofeq:interventional constraint}.

        \item $\sum_{\substack{u_t \in U_t\\f(s_t, a_t, u_t) \neq s_{t+1}}}{\theta_{u_t}} = \sum_{\substack{u_t \in U_t \\f(s_t, a_t, u_t) \neq s_{t+1} \\ f(\tilde{s}, \tilde{a}, u_t) = \tilde{s}'}} \theta_{u_t} + \sum_{\substack{u_t \in U_t \\f(s_t, a_t, u_t) \neq s_{t+1} \\ f(\tilde{s}, \tilde{a}, u_t) \neq \tilde{s}'}} \theta_{u_t} =1 -
        P(s_{t+1} \mid s_t, a_t)$ satisfying \eqref{proofeq:interventional constraint}.
        
        \item $\sum_{\substack{u_t \in U_t \\f(\tilde{s}, \tilde{a}, u_t) = \tilde{s}'}} \theta_{u_t} = \sum_{\substack{u_t \in U_t \\f(s_t, a_t, u_t) = s_{t+1} \\ f(\tilde{s}, \tilde{a}, u_t) = \tilde{s}'}} \theta_{u_t} + \sum_{\substack{u_t \in U_t \\f(s_t, a_t, u_t) \neq s_{t+1} \\ f(\tilde{s}, \tilde{a}, u_t) = \tilde{s}'}} \theta_{u_t} = P(\tilde{s}' \mid \tilde{s}, \tilde{a})$ satisfying \eqref{proofeq:interventional constraint}.
        
        \item $\sum_{s' \in \mathcal{S}\setminus\{\tilde{s}'\}}\sum_{\substack{u_t \in U_t \\f(\tilde{s}, \tilde{a}, u_t) = s'}} \theta_{u_t} = \sum_{\substack{u_t \in U_t \\f(s_t, a_t, u_t) = s_{t+1} \\ f(\tilde{s}, \tilde{a}, u_t) \neq \tilde{s}'}} \theta_{u_t} + \sum_{\substack{u_t \in U_t \\f(s_t, a_t, u_t) \neq s_{t+1} \\ f(\tilde{s}, \tilde{a}, u_t) \neq \tilde{s}'}} \theta_{u_t} \\ = 1 - P(\tilde{s}' \mid \tilde{s}, \tilde{a}) = \sum_{s' \in \mathcal{S}\setminus\{\tilde{s}'\}}P(s' \mid \tilde{s}, \tilde{a})$, therefore it is possible to assign $\theta$ such that $\forall s' \in \mathcal{S}\setminus\{\tilde{s}'\}, \sum_{\substack{u_t \in U_t \\ f(\tilde{s}, \tilde{a}, u_t) = s'}}{\theta_{u_t}} = P(s' \mid \tilde{s}, \tilde{a})$ satisfying \eqref{proofeq:interventional constraint}.
        
        \item Mon1 \eqref{proofeq:monotonicity1}, Mon2 \eqref{proofeq:monotonicity2} and CS \eqref{proofeq:counterfactual stability} are all vacuously true for all transitions from $(\tilde{s}, \tilde{a})$ due to Lemma \ref{lemma: constraints do not apply if disjoint}.

        \item $0 \leq \theta_{u_t} \leq 1, \forall u_t$, satisfying \eqref{proofeq:valid prob1}.
        
        \item $\sum_{u_t = 1}^{|U_t|} \theta_{u_t} = \sum_{\substack{u_t \in U_t \\f(s_t, a_t, u_t) = s_{t+1} \\ f(\tilde{s}, \tilde{a}, u_t) = \tilde{s}'}} \theta_{u_t} + \sum_{\substack{u_t \in U_t \\f(s_t, a_t, u_t) = s_{t+1} \\ f(\tilde{s}, \tilde{a}, u_t) \neq \tilde{s}'}} \theta_{u_t} + \sum_{\substack{u_t \in U_t \\f(s_t, a_t, u_t) \neq s_{t+1} \\ f(\tilde{s}, \tilde{a}, u_t) = \tilde{s}'}} \theta_{u_t} \\ + \sum_{\substack{u_t \in U_t \\f(s_t, a_t, u_t) \neq s_{t+1} \\ f(\tilde{s}, \tilde{a}, u_t) \neq \tilde{s}'}} \theta_{u_t} = 1$, satisfying \eqref{proofeq:valid prob2}.
    \end{itemize}

     All the constraints are satisfied, so this is a valid assignment of $\theta$ for this state-action pair. Therefore,

        \[\tilde{P}_{t}^{UB}(\tilde{s}' \mid \tilde{s}, \tilde{a}) = \max_\theta \left(\frac{\sum_{\substack{u_t \in U_t \\f(s_t, a_t, u_t) = s_{t+1} \\ f(\tilde{s}, \tilde{a}, u_t) = \tilde{s}'}} \theta_{u_t}}{P(s_{t+1} \mid s_t, a_t)}\right) = \frac{P(s_{t+1} \mid s_t, a_t)}{P(s_{t+1} \mid s_t, a_t)} = 1\]
    
\end{proof}

These two cases can be combined as follows:

\[\tilde{P}_{t}^{UB}(\tilde{s}' \mid \tilde{s}, \tilde{a}) = 
\begin{cases}
\frac{P(\tilde{s}' \mid \tilde{s}, \tilde{a})}{P(s_{t+1} \mid s_t, a_t)} & \text{if $P(\tilde{s}' \mid \tilde{s}, \tilde{a})) < P(s_{t+1} \mid s_t, a_t)$} \\
1 & \text{otherwise} \\
\end{cases}
\]
for all next states $\tilde{s}' \in \mathcal{S}$.
\end{proof}

\pagebreak
\begin{theorem}
\label{proof theorem: lb disjoint}
For outgoing transitions from state-action pairs $(\tilde{s}, \tilde{a})$ which have completely disjoint support from the observed $(s_t, a_t)$, the linear program will produce a lower bound of:
\[\tilde{P}_{t}^{LB}(\tilde{s}' \mid \tilde{s}, \tilde{a}) = 
\begin{cases}
\frac{P(\tilde{s}' \mid \tilde{s}, \tilde{a}) - (1 - P(s_{t+1} \mid s_t, a_t))}{P(s_{t+1} \mid s_t, a_t)} & \text{if $P(\tilde{s}' \mid \tilde{s}, \tilde{a}) > 1 - P(s_{t+1} \mid s_t, a_t)$}\\
0 & \text{otherwise}\\
\end{cases}
\]
for all possible next states $\tilde{s}'$.
\end{theorem}

\begin{proof}
Take an arbitrary transition $\tilde{s}, \tilde{a} \rightarrow \tilde{s}'$ where $(\tilde{s}, \tilde{a})$ is disjoint from the observed state-action pair $(s_t, a_t)$. By definition:

\[\tilde{P}_{t}^{LB}(\tilde{s}' \mid \tilde{s}, \tilde{a}) = \min_\theta \left(\frac{\sum_{u_t = 1}^{|U_t|} \mu_{\tilde{s}, \tilde{a}, u_t, \tilde{s}'} \cdot \mu_{s_t, a_t, u_t, s_{t+1}} \cdot \theta_{u_t}}{P(s_{t+1} \mid s_t, a_t)}\right)\]

Because $P(s_{t+1} \mid s_t, a_t)$ (the interventional probability of the observed transition) is fixed, to find the lower bound for $\tilde{P}_t(\tilde{s}' \mid \tilde{s}, \tilde{a})$ we need to minimise $\sum_{u_t = 1}^{|U_t|} \mu_{\tilde{s}, \tilde{a}, u_t, \tilde{s}'} \cdot \mu_{s_t, a_t, u_t, s_{t+1}} \cdot \theta_{u_t}$. Because $\sum_{u_t = 1}^{|U_t|} \mu_{s_t, a_t, u_t, s_{t+1}} \cdot \theta_{u_t} + \sum_{s' \in \mathcal{S}\setminus{\{s_{t+1}\}}}\sum_{u_t = 1}^{|U_t|} \mu_{s_t, a_t, u_t, s'} \cdot \theta_{u_t} = 1$, this is equivalent to maximising the following equation:

\[
\max_\theta \left( \sum_{s' \in \mathcal{S}\setminus{\{s_{t+1}\}}}\sum_{u_t = 1}^{|U_t|} \mu_{\tilde{s}, \tilde{a}, u_t, \tilde{s}'} \cdot \mu_{s_t, a_t, u_t, s'} \cdot \theta_{u_t}\right)
\]

From the constraints of the optimisation problem, we have:

\begin{equation}
    \label{eq: theorem 4.8 constraint 1}
    \sum_{u_t = 1}^{|U_t|} \mu_{\tilde{s}, \tilde{a}, u_t, \tilde{s}'}\cdot \theta_{u_t} = P(\tilde{s}' \mid \tilde{s}, \tilde{a})
\end{equation}
and
\begin{equation}
\label{eq: theorem 4.8 constraint 2}
    \sum_{s' \in \mathcal{S}\setminus{\{s_{t+1}\}}}\sum_{u_t = 1}^{|U_t|} \mu_{s_t, a_t, u_t, s'} \cdot \theta_{u_t} = \sum_{s' \in \mathcal{S}\setminus{\{s_{t+1}\}}}{P(s' \mid s_t, a_t)} = 1 - P(s_{t+1} \mid s_t, a_t)
\end{equation}

$\forall s,s' \in \mathcal{S}, \forall a \in \mathcal{A}, \forall u_t \in U_t, \mu_{s, a, u_t, s'} \in \{0, 1\}$ (from its definition). Therefore,\\ $\sum_{s' \in \mathcal{S}\setminus{\{s_{t+1}\}}}\sum_{u_t = 1}^{|U_t|} \mu_{\tilde{s}, \tilde{a}, u_t, \tilde{s}'} \cdot \mu_{s_t, a_t, u_t, s'} \cdot \theta_{u_t}$ is maximal when the probability assigned to all $u_t$ where $\mu_{\tilde{s}, \tilde{a}, u_t, \tilde{s}'} = 1$ and $\mu_{s_t, a_t, u_t, s_{t+1}}=0$ is maximal:

\[
\sum_{u_t = 1}^{|U_t|}(\mu_{\tilde{s}, \tilde{a}, u_t, \tilde{s}'} \cdot \theta_{u_t}) \cdot \sum_{s' \in \mathcal{S}\setminus{\{s_{t+1}\}}}\mu_{s_t, a_t, u_t, s'} \leq \sum_{u_t = 1}^{|U_t|} \mu_{\tilde{s}, \tilde{a}, u_t, \tilde{s}'}\cdot \theta_{u_t} = P(\tilde{s}' \mid \tilde{s}, \tilde{a})
\]

\[
\sum_{u_t = 1}^{|U_t|} \mu_{\tilde{s}, \tilde{a}, u_t, \tilde{s}'} \cdot (\sum_{s' \in \mathcal{S}\setminus{\{s_{t+1}\}}}\mu_{s_t, a_t, u_t, s'} \cdot \theta_{u_t}) \leq  \sum_{s' \in \mathcal{S}\setminus{\{s_{t+1}\}}}\sum_{u_t = 1}^{|U_t|} \mu_{s_t, a_t, u_t, s'} \cdot \theta_{u_t} = 1 - P(s_{t+1} \mid s_t, a_t)
\]

Because of Eq. \eqref{eq: theorem 4.8 constraint 1} and Eq. \eqref{eq: theorem 4.8 constraint 2}, the probability assigned to all $u_t$ where $\mu_{\tilde{s}, \tilde{a}, u_t, \tilde{s}'} = 1$ and $\mu_{s_t, a_t, u_t, s_{t+1}}=0$ cannot be greater than $P(\tilde{s}' \mid \tilde{s}, \tilde{a})$ or $1 - P(s_{t+1} \mid s_t, a_t)$. So,

\[
    \sum_{s' \in \mathcal{S}\setminus{\{s_{t+1}\}}}\sum_{u_t = 1}^{|U_t|} \mu_{\tilde{s}, \tilde{a}, u_t, \tilde{s}'} \cdot \mu_{s_t, a_t, u_t, s'} \cdot \theta_{u_t} \leq \min\left(P(\tilde{s}' \mid \tilde{s}, \tilde{a}), (1 - P(s_{t+1} \mid s_t, a_t)\right)
\]

Therefore, if we can show that there exists a $\theta$ such that
\[
    \sum_{s' \in \mathcal{S}\setminus{\{s_{t+1}\}}}\sum_{u_t = 1}^{|U_t|} \mu_{\tilde{s}, \tilde{a}, u_t, \tilde{s}'} \cdot \mu_{s_t, a_t, u_t, s'} \cdot \theta_{u_t} = \min\left(P(\tilde{s}' \mid \tilde{s}, \tilde{a}), (1 - P(s_{t+1} \mid s_t, a_t)\right)
\]

and $\theta$ satisfies the constraints of the optimisation problem, we know that this results in the minimum possible counterfactual transition probability $\tilde{P}_t(\tilde{s}' \mid \tilde{s}, \tilde{a})$. Since $\forall s,s' \in \mathcal{S}, \forall a \in \mathcal{A}, \forall u_t \in U_t, \mu_{s, a, u_t, s'} = 1 \iff f(s, a, u_t) = s'$ by definition of $\mu$, this is equivalent to showing that there exists a $\theta$ such that:

\[
\sum_{s' \in \mathcal{S}\setminus{\{s_{t+1}\}}}\sum_{\substack{u_t \in U_t \\f(s_t, a_t, u_t) = s' \\ f(\tilde{s}, \tilde{a}, u_t) = \tilde{s}'}} \theta_{u_t}= \min\left(P(\tilde{s}' \mid \tilde{s}, \tilde{a}), (1 - P(s_{t+1} \mid s_t, a_t)\right)
\]

As these notations are equivalent, we will use this second notation for brevity and clarity. Consider the following disjoint cases:

\begin{itemize}
    \item $P(\tilde{s}' \mid \tilde{s}, \tilde{a}) > 1 - P(s_{t+1} \mid s_t, a_t)$
    \item $P(\tilde{s}' \mid \tilde{s}, \tilde{a}) \leq 1 - P(s_{t+1} \mid s_t, a_t)$
\end{itemize}

\paragraph{Case 1: $P(\tilde{s}' \mid \tilde{s}, \tilde{a}) > 1 - P(s_{t+1} \mid s_t, a_t)$}
\noindent
\begin{proof}
In this case,

\[
    \sum_{s' \in \mathcal{S}\setminus{\{s_{t+1}\}}}\sum_{\substack{u_t \in U_t \\f(s_t, a_t, u_t) = s' \\ f(\tilde{s}, \tilde{a}, u_t) = \tilde{s}'}} \theta_{u_t} \leq 1 - P(s_{t+1} \mid s_t, a_t)
\]

and so
\[
\sum_{\substack{u_t \in U_t \\f(s_t, a_t, u_t) = s_{t+1} \\ f(\tilde{s}, \tilde{a}, u_t) = \tilde{s}'}} \theta_{u_t} \geq \sum_{\substack{u_t \in U_t \\ f(\tilde{s}, \tilde{a}, u_t) = \tilde{s}'}} \theta_{u_t} - \sum_{s' \in \mathcal{S}\setminus{\{s_{t+1}\}}}\sum_{\substack{u_t \in U_t \\f(s_t, a_t, u_t) = s' \\ f(\tilde{s}, \tilde{a}, u_t) = \tilde{s}'}} \theta_{u_t}= P(\tilde{s}' \mid \tilde{s}, \tilde{a}) - (1 - P(s_{t+1} \mid s_t, a_t))
\]

If we can show that there exists a $\theta$ such that

\[
\sum_{\substack{u_t \in U_t \\f(s_t, a_t, u_t) = s_{t+1} \\ f(\tilde{s}, \tilde{a}, u_t) = \tilde{s}'}} \theta_{u_t} = P(\tilde{s}' \mid \tilde{s}, \tilde{a}) - (1 - P(s_{t+1} \mid s_t, a_t))
\]

and $\theta$ satisfies the constraints in the optimisation problem, we know that this results in the minimum possible counterfactual transition probability for $\tilde{P}_t(\tilde{s}' \mid \tilde{s}, \tilde{a})$. We can assign $\theta$ as follows:

\[
\begin{cases}
    \sum_{\substack{u_t \in U_t \\f(s_t, a_t, u_t) = s_{t+1} \\ f(\tilde{s}, \tilde{a}, u_t) = \tilde{s}'}} \theta_{u_t}= P(\tilde{s}' \mid \tilde{s}, \tilde{a}) - (1 - P(s_{t+1} \mid s_t, a_t))\\
    \sum_{\substack{u_t \in U_t \\f(s_t, a_t, u_t) = s_{t+1} \\ f(\tilde{s}, \tilde{a}, u_t) \neq \tilde{s}'}} \theta_{u_t} = 1 - P(\tilde{s}' \mid \tilde{s}, \tilde{a})\\
    \sum_{\substack{u_t \in U_t \\f(s_t, a_t, u_t) \neq s_{t+1} \\ f(\tilde{s}, \tilde{a}, u_t) = \tilde{s}'}} \theta_{u_t} = 1 - P(s_{t+1} \mid s_t, a_t)\\
    \sum_{\substack{u_t \in U_t \\f(s_t, a_t, u_t) \neq s_{t+1} \\ f(\tilde{s}, \tilde{a}, u_t) \neq \tilde{s}'}} \theta_{u_t} = 0\\
\end{cases}
\]

This assignment of $\theta$ satisfies the constraints of the linear optimisation problem, as follows:

\begin{itemize}        
    \item $\sum_{\substack{u_t \in U_t \\f(s_t, a_t, u_t) = s_{t+1}}} \theta_{u_t} = \sum_{\substack{u_t \in U_t \\f(s_t, a_t, u_t) = s_{t+1} \\ f(\tilde{s}, \tilde{a}, u_t) = \tilde{s}'}} \theta_{u_t} + \sum_{\substack{u_t \in U_t \\f(s_t, a_t, u_t) = s_{t+1} \\ f(\tilde{s}, \tilde{a}, u_t) \neq \tilde{s}'}} \theta_{u_t} = P(s_{t+1} \mid s_t, a_t)$ \\satisfying \eqref{proofeq:interventional constraint}.

    \item $\sum_{\substack{u_t \in U_t\\f(s_t, a_t, u_t) \neq s_{t+1}}}{\theta_{u_t}} = \sum_{\substack{u_t \in U_t \\f(s_t, a_t, u_t) \neq s_{t+1} \\ f(\tilde{s}, \tilde{a}, u_t) = \tilde{s}'}} \theta_{u_t} + \sum_{\substack{u_t \in U_t \\f(s_t, a_t, u_t) \neq s_{t+1} \\ f(\tilde{s}, \tilde{a}, u_t) \neq \tilde{s}'}} \theta_{u_t} =1 -
    P(s_{t+1} \mid s_t, a_t)$ satisfying \eqref{proofeq:interventional constraint}.
    
    \item $\sum_{\substack{u_t \in U_t \\f(\tilde{s}, \tilde{a}, u_t) = \tilde{s}'}} \theta_{u_t} = \sum_{\substack{u_t \in U_t \\f(s_t, a_t, u_t) = s_{t+1} \\ f(\tilde{s}, \tilde{a}, u_t) = \tilde{s}'}} \theta_{u_t} + \sum_{\substack{u_t \in U_t \\f(s_t, a_t, u_t) \neq s_{t+1} \\ f(\tilde{s}, \tilde{a}, u_t) = \tilde{s}'}} \theta_{u_t} = P(\tilde{s}' \mid \tilde{s}, \tilde{a})$ satisfying \eqref{proofeq:interventional constraint}.
    
    \item $\sum_{s' \in \mathcal{S}\setminus\{\tilde{s}'\}}\sum_{\substack{u_t \in U_t \\f(\tilde{s}, \tilde{a}, u_t) = s'}} \theta_{u_t} = \sum_{\substack{u_t \in U_t \\f(s_t, a_t, u_t) = s_{t+1} \\ f(\tilde{s}, \tilde{a}, u_t) \neq \tilde{s}'}} \theta_{u_t} + \sum_{\substack{u_t \in U_t \\f(s_t, a_t, u_t) \neq s_{t+1} \\ f(\tilde{s}, \tilde{a}, u_t) \neq \tilde{s}'}} \theta_{u_t} \\= 1 - P(\tilde{s}' \mid \tilde{s}, \tilde{a}) = \sum_{s' \in \mathcal{S}\setminus\{\tilde{s}'\}}P(s' \mid \tilde{s}, \tilde{a})$, therefore it is possible to assign $\theta$ such that $\forall s' \in \mathcal{S}\setminus\{\tilde{s}'\}, \sum_{\substack{u_t \in U_t \\ f(\tilde{s}, \tilde{a}, u_t) = s'}}{\theta_{u_t}} = P(s' \mid \tilde{s}, \tilde{a})$ satisfying \eqref{proofeq:interventional constraint}.
    
    \item Mon1 \eqref{proofeq:monotonicity1}, Mon2 \eqref{proofeq:monotonicity2} and CS \eqref{proofeq:counterfactual stability} are all vacuously true for all transitions from $(\tilde{s}, \tilde{a})$ due to Lemma \ref{lemma: constraints do not apply if disjoint}.

    \item $0 \leq \theta_{u_t} \leq 1, \forall u_t$, satisfying \eqref{proofeq:valid prob1}.
    
    \item $\sum_{u_t = 1}^{|U_t|} \theta_{u_t} = \sum_{\substack{u_t \in U_t \\f(s_t, a_t, u_t) = s_{t+1} \\ f(\tilde{s}, \tilde{a}, u_t) = \tilde{s}'}} \theta_{u_t} + \sum_{\substack{u_t \in U_t \\f(s_t, a_t, u_t) = s_{t+1} \\ f(\tilde{s}, \tilde{a}, u_t) \neq \tilde{s}'}} \theta_{u_t} + \sum_{\substack{u_t \in U_t \\f(s_t, a_t, u_t) \neq s_{t+1} \\ f(\tilde{s}, \tilde{a}, u_t) = \tilde{s}'}} \theta_{u_t} \\+ \sum_{\substack{u_t \in U_t \\f(s_t, a_t, u_t) \neq s_{t+1} \\ f(\tilde{s}, \tilde{a}, u_t) \neq \tilde{s}'}} \theta_{u_t} = 1$, satisfying \eqref{proofeq:valid prob2}.
\end{itemize}

All the constraints are satisfied, so this is a valid assignment of $\theta$ for this state-action pair. Therefore,
 \[\tilde{P}_{t}^{LB}(\tilde{s}' \mid \tilde{s}, \tilde{a}) = \dfrac{\sum_{\substack{u_t \in U_t \\f(s_t, a_t, u_t) = s_{t+1} \\ f(\tilde{s}, \tilde{a}, u_t) = \tilde{s}'}} \theta_{u_t}}{P(s_{t+1} \mid s_t, a_t)}= \frac{P(\tilde{s}' \mid \tilde{s}, \tilde{a}) - (1 - P(s_{t+1} \mid s_t, a_t))}{P(s_{t+1} \mid s_t, a_t)}\]
\end{proof}

\paragraph{Case 2: $P(\tilde{s}' \mid \tilde{s}, \tilde{a}) \leq 1- P(s_{t+1} \mid s_t, a_t)$}
\begin{proof}
In this case,

\[
    \sum_{s' \in \mathcal{S}\setminus{\{s_{t+1}\}}}\sum_{\substack{u_t \in U_t \\f(s_t, a_t, u_t) = s' \\ f(\tilde{s}, \tilde{a}, u_t) = \tilde{s}'}} \theta_{u_t} \leq P(\tilde{s}' \mid \tilde{s}, \tilde{a})
\]

and so
\[
\sum_{\substack{u_t \in U_t \\f(s_t, a_t, u_t) = s_{t+1} \\ f(\tilde{s}, \tilde{a}, u_t) = \tilde{s}'}} \theta_{u_t} \geq \sum_{\substack{u_t \in U_t \\ f(\tilde{s}, \tilde{a}, u_t) = \tilde{s}'}} \theta_{u_t} - \sum_{s' \in \mathcal{S}\setminus{\{s_{t+1}\}}}\sum_{\substack{u_t \in U_t \\f(s_t, a_t, u_t) = s' \\ f(\tilde{s}, \tilde{a}, u_t) = \tilde{s}'}} \theta_{u_t}= P(\tilde{s}' \mid \tilde{s}, \tilde{a}) - P(\tilde{s}' \mid \tilde{s}, \tilde{a}) = 0
\]

If we can show that there exists a $\theta$ such that

\[
\sum_{\substack{u_t \in U_t \\f(s_t, a_t, u_t) = s_{t+1} \\ f(\tilde{s}, \tilde{a}, u_t) = \tilde{s}'}} \theta_{u_t} = 0
\]

and $\theta$ satisfies the constraints in the optimisation problem, we know that this results in the minimum possible counterfactual transition probability for $\tilde{P}_t(\tilde{s}' \mid \tilde{s}, \tilde{a})$. We can assign $\theta$ as follows:

\[
\begin{cases}
    \sum_{\substack{u_t \in U_t \\f(s_t, a_t, u_t) = s_{t+1} \\ f(\tilde{s}, \tilde{a}, u_t) = \tilde{s}'}} \theta_{u_t}= 0\\
    \sum_{\substack{u_t \in U_t \\f(s_t, a_t, u_t) = s_{t+1} \\ f(\tilde{s}, \tilde{a}, u_t) \neq \tilde{s}'}} \theta_{u_t} = P(s_{t+1} \mid s_t, a_t)\\
    \sum_{\substack{u_t \in U_t \\f(s_t, a_t, u_t) \neq s_{t+1} \\ f(\tilde{s}, \tilde{a}, u_t) = \tilde{s}'}} \theta_{u_t} = P(\tilde{s}' \mid \tilde{s}, \tilde{a})\\
    \sum_{\substack{u_t \in U_t \\f(s_t, a_t, u_t) \neq s_{t+1} \\ f(\tilde{s}, \tilde{a}, u_t) \neq \tilde{s}'}} \theta_{u_t} = 1 - P(s_{t+1} \mid s_t, a_t) - P(\tilde{s}' \mid \tilde{s}, \tilde{a})\\
\end{cases}
\]

This assignment of $\theta$ satisfies the constraints of the linear optimisation problem, as follows:

\begin{itemize}        
    \item $\sum_{\substack{u_t \in U_t \\f(s_t, a_t, u_t) = s_{t+1}}} \theta_{u_t} = \sum_{\substack{u_t \in U_t \\f(s_t, a_t, u_t) = s_{t+1} \\ f(\tilde{s}, \tilde{a}, u_t) = \tilde{s}'}} \theta_{u_t} + \sum_{\substack{u_t \in U_t \\f(s_t, a_t, u_t) = s_{t+1} \\ f(\tilde{s}, \tilde{a}, u_t) \neq \tilde{s}'}} \theta_{u_t} = P(s_{t+1} \mid s_t, a_t)$ \\satisfying \eqref{proofeq:interventional constraint}.

    \item $\sum_{\substack{u_t \in U_t\\f(s_t, a_t, u_t) \neq s_{t+1}}}{\theta_{u_t}} = \sum_{\substack{u_t \in U_t \\f(s_t, a_t, u_t) \neq s_{t+1} \\ f(\tilde{s}, \tilde{a}, u_t) = \tilde{s}'}} \theta_{u_t} + \sum_{\substack{u_t \in U_t \\f(s_t, a_t, u_t) \neq s_{t+1} \\ f(\tilde{s}, \tilde{a}, u_t) \neq \tilde{s}'}} \theta_{u_t} =1 -
    P(s_{t+1} \mid s_t, a_t)$ satisfying \eqref{proofeq:interventional constraint}.
    
    \item $\sum_{\substack{u_t \in U_t \\f(\tilde{s}, \tilde{a}, u_t) = \tilde{s}'}} \theta_{u_t} = \sum_{\substack{u_t \in U_t \\f(s_t, a_t, u_t) = s_{t+1} \\ f(\tilde{s}, \tilde{a}, u_t) = \tilde{s}'}} \theta_{u_t} + \sum_{\substack{u_t \in U_t \\f(s_t, a_t, u_t) \neq s_{t+1} \\ f(\tilde{s}, \tilde{a}, u_t) = \tilde{s}'}} \theta_{u_t} = P(\tilde{s}' \mid \tilde{s}, \tilde{a})$ satisfying \eqref{proofeq:interventional constraint}.
    
    \item $\sum_{s' \in \mathcal{S}\setminus\{\tilde{s}'\}}\sum_{\substack{u_t \in U_t \\f(\tilde{s}, \tilde{a}, u_t) = s'}} \theta_{u_t} = \sum_{\substack{u_t \in U_t \\f(s_t, a_t, u_t) = s_{t+1} \\ f(\tilde{s}, \tilde{a}, u_t) \neq \tilde{s}'}} \theta_{u_t} + \sum_{\substack{u_t \in U_t \\f(s_t, a_t, u_t) \neq s_{t+1} \\ f(\tilde{s}, \tilde{a}, u_t) \neq \tilde{s}'}} \theta_{u_t} \\= 1 - P(\tilde{s}' \mid \tilde{s}, \tilde{a}) = \sum_{s' \in \mathcal{S}\setminus\{\tilde{s}'\}}P(s' \mid \tilde{s}, \tilde{a})$, therefore it is possible to assign $\theta$ such that $\forall s' \in \mathcal{S}\setminus\{\tilde{s}'\}, \sum_{\substack{u_t \in U_t \\ f(\tilde{s}, \tilde{a}, u_t) = s'}}{\theta_{u_t}} = P(s' \mid \tilde{s}, \tilde{a})$ satisfying \eqref{proofeq:interventional constraint}.
    
    \item Mon1 \eqref{proofeq:monotonicity1}, Mon2 \eqref{proofeq:monotonicity2} and CS \eqref{proofeq:counterfactual stability} are all vacuously true for all transitions from $(\tilde{s}, \tilde{a})$ due to Lemma \ref{lemma: constraints do not apply if disjoint}.

    \item $0 \leq \theta_{u_t} \leq 1, \forall u_t$, satisfying \eqref{proofeq:valid prob1}.
    
    \item $\sum_{u_t = 1}^{|U_t|} \theta_{u_t} = \sum_{\substack{u_t \in U_t \\f(s_t, a_t, u_t) = s_{t+1} \\ f(\tilde{s}, \tilde{a}, u_t) = \tilde{s}'}} \theta_{u_t} + \sum_{\substack{u_t \in U_t \\f(s_t, a_t, u_t) = s_{t+1} \\ f(\tilde{s}, \tilde{a}, u_t) \neq \tilde{s}'}} \theta_{u_t} + \sum_{\substack{u_t \in U_t \\f(s_t, a_t, u_t) \neq s_{t+1} \\ f(\tilde{s}, \tilde{a}, u_t) = \tilde{s}'}} \theta_{u_t} \\+ \sum_{\substack{u_t \in U_t \\f(s_t, a_t, u_t) \neq s_{t+1} \\ f(\tilde{s}, \tilde{a}, u_t) \neq \tilde{s}'}} \theta_{u_t} = 1$, satisfying \eqref{proofeq:valid prob2}.
\end{itemize}

All the constraints are satisfied, so this is a valid assignment of $\theta$ for this state-action pair. Therefore,

\[\tilde{P}_{t}^{LB}(\tilde{s}' \mid \tilde{s}, \tilde{a}) = \dfrac{\sum_{\substack{u_t \in U_t \\f(s_t, a_t, u_t) = s_{t+1} \\ f(\tilde{s}, \tilde{a}, u_t) = \tilde{s}'}} \theta_{u_t}}{P(s_{t+1} \mid s_t, a_t)}= \frac{0}{P(s_{t+1} \mid s_t, a_t)} =0\]

\end{proof}
These two cases can be combined as follows:

\[\tilde{P}_{t}^{LB}(\tilde{s}' \mid \tilde{s}, \tilde{a}) = 
\begin{cases}
\frac{P(\tilde{s}' \mid \tilde{s}, \tilde{a}) - (1 - P(s_{t+1} \mid s_t, a_t))}{P(s_{t+1} \mid s_t, a_t)} & \text{if $P(\tilde{s}' \mid \tilde{s}, \tilde{a}) > 1 - P(s_{t+1} \mid s_t, a_t)$}\\
0 & \text{otherwise}\\
\end{cases}
\]
for all next states $\tilde{s}' \in \mathcal{S}$.
\end{proof}

\pagebreak
\subsubsection{CF Bounds for $(\tilde{s}, \tilde{a})$ with Overlapping Support with $(s_t, a_t)$}
To enhance readability, we split the proof for Theorem \ref{theorem:ub overlapping} into two separate theorems: Theorem \ref{proof theorem:ub overlapping} proves the upper counterfactual probability bounds, and Theorem \ref{proof theorem:lb overlapping} proves the lower counterfactual probability bounds.

\begin{theorem}
\label{proof theorem:ub overlapping}
For outgoing transitions from state-action pairs $(\tilde{s}, \tilde{a})$ which have overlapping support with the observed $(s_t, a_t)$, but $(\tilde{s}, \tilde{a}) \neq (s_t, a_t)$, the linear program will produce an upper bound of:

\[\tilde{P}_{t}^{UB}(\tilde{s}' \mid \tilde{s}, \tilde{a}) =
\begin{cases}
\dfrac{\min(P(s_{t+1} \mid s_t, a_t), P(s_{t+1} \mid \tilde{s}, \tilde{a}))}{P(s_{t+1} \mid s_t, a_t)} & \text{if $\tilde{s}' = s_{t+1}$}\\
0 & \text{if $P(\tilde{s}' \mid s_t, a_t) > 0$} \\ & \text{and $\dfrac{P(s_{t+1} \mid \tilde{s}, \tilde{a})}{P(s_{t+1} \mid s_t, a_t)}\geq\dfrac{P(\tilde{s}' \mid \tilde{s}, \tilde{a})}{P(\tilde{s}' \mid s_t, a_t)}$}\\
\min(P(\tilde{s}' \mid \tilde{s}, \tilde{a}), 1 - P(s_{t+1} \mid \tilde{s}, \tilde{a})) & \text{if $P(\tilde{s}' \mid s_t, a_t) > 0$}\\
\min(1 - P(s_{t+1} \mid \tilde{s}, \tilde{a}), \dfrac{P(\tilde{s}' \mid \tilde{s}, \tilde{a})}{P(s_{t+1} \mid s_t, a_t)}) & \text{otherwise}\\
\end{cases}
\]
for all possible next states $\tilde{s}'$.
\end{theorem}

\begin{proof} 
Take an arbitrary transition $\tilde{s}, \tilde{a} \rightarrow \tilde{s}'$ where $(\tilde{s}, \tilde{a})$ has overlapping support with the observed state-action pair $(s_t, a_t)$ (and $(\tilde{s}, \tilde{a}) \neq (s_t, a_t)$). Consider the following distinct cases:
\begin{itemize}
    \item $\tilde{s}' = s_{t+1}$
    \item $\tilde{s}' \neq s_{t+1}$, $\dfrac{P(s_{t+1} \mid \tilde{s}, \tilde{a})}{P(s_{t+1} \mid s_t, a_t)}\geq\dfrac{P(\tilde{s}' \mid \tilde{s}, \tilde{a})}{P(\tilde{s}' \mid s_t, a_t)}$ and $P(\tilde{s}' \mid s_t, a_t) > 0$
    \item $\tilde{s}' \neq s_{t+1}$, $\dfrac{P(s_{t+1} \mid \tilde{s}, \tilde{a})}{P(s_{t+1} \mid s_t, a_t)}<\dfrac{P(\tilde{s}' \mid \tilde{s}, \tilde{a})}{P(\tilde{s}' \mid s_t, a_t)}$ and $P(\tilde{s}' \mid s_t, a_t) >0$
    \item $\tilde{s}' \neq s_{t+1}$ and $P(\tilde{s}' \mid s_t, a_t) =0$
\end{itemize}
These cases are disjoint and cover all possible situations.

\paragraph{Case 1: $\tilde{s}' = s_{t+1}$}
\noindent
\begin{proof}
Consider the following disjoint cases:
\begin{itemize}
    \item $P(s_{t+1} \mid \tilde{s}, \tilde{a}) = 0$
    \item $P(s_{t+1} \mid \tilde{s}, \tilde{a}) > 0$ and $P(s_{t+1} \mid \tilde{s}, \tilde{a}) \geq P(s_{t+1} \mid s_t, a_t)$
    \item $P(s_{t+1} \mid \tilde{s}, \tilde{a}) > 0$ and $P(s_{t+1} \mid \tilde{s}, \tilde{a}) < P(s_{t+1} \mid s_t, a_t)$
\end{itemize}

\paragraph{Case 1(a): $P(s_{t+1} \mid \tilde{s}, \tilde{a}) = 0$}
\noindent
\begin{proof}
If $P(s_{t+1} \mid \tilde{s}, \tilde{a}) = 0$, then \\$\sum_{u_t = 1}^{|U_t|} \mu_{\tilde{s}, \tilde{a}, u_t, s_{t+1}} \cdot \theta_{u_t} = 0$ to satisfy \eqref{proofeq:interventional constraint}. Therefore, the maximum counterfactual probability $\tilde{P}_t(s_{t+1} \mid \tilde{s}, \tilde{a})$ is:

\[
\max_{\theta}\left(\sum_{u_t = 1}^{|U_t|} \mu_{\tilde{s}, \tilde{a}, u_t, s_{t+1}} \cdot \mu_{s_t, a_t, u_t, s_{t+1}} \cdot \theta_{u_t}\right) = 0
\]

If we can show that there exists a $\theta$ such that

\[
\sum_{\substack{u_t \in U_t \\f(s_t, a_t, u_t) = s_{t+1} \\ f(\tilde{s}, \tilde{a}, u_t) = s_{t+1}}} \theta_{u_t} = 0
\]

and $\theta$ satisfies the constraints in the optimisation problem, we know that this results in the maximum possible counterfactual transition probability for $\tilde{P}_t(s_{t+1} \mid \tilde{s}, \tilde{a})$. We can assign $\theta$ as follows:

\[
\forall s' \in \mathcal{S}, \sum_{\substack{u_t \in U_t\\f(s_t, a_t, u_t) = s_{t+1} \\ f(\tilde{s}, \tilde{a}, u_t) = s'}}{\theta_{u_t}} = P(s' \mid \tilde{s}, \tilde{a}) \cdot P(s_{t+1} \mid s_t, a_t)\\
\]

\[
\forall s' \in \mathcal{S}, \sum_{\substack{u_t \in U_t\\f(s_t, a_t, u_t) \neq s_{t+1} \\ f(\tilde{s}, \tilde{a}, u_t) = s'}}{\theta_{u_t}} = P(s' \mid \tilde{s}, \tilde{a}) \cdot (1 - P(s_{t+1} \mid s_t, a_t))\\
\]

which results in a counterfactual probability of $\tilde{P}_t(s_{t+1} \mid \tilde{s}, \tilde{a}) = \sum_{\substack{u_t \in U_t\\f(s_t, a_t, u_t) = s_{t+1}}}{\theta_{u_t}} \\= \sum_{\substack{u_t \in U_t\\f(s_t, a_t, u_t) = s_{t+1} \\ f(\tilde{s}, \tilde{a}, u_t) = s_{t+1}}}{\theta_{u_t}} = P(s_{t+1} \mid \tilde{s}, \tilde{a}) \cdot P(s_{t+1} \mid s_t, a_t) = 0$.\\

This assignment of $\theta$ satisfies the constraints of the linear optimisation problem, as follows:

\begin{itemize}
    \item $\sum_{\substack{u_t \in U_t\\f(s_t, a_t, u_t) = s_{t+1}}}{\theta_{u_t}} = \sum_{s' \in \mathcal{S}}\sum_{\substack{u_t \in U_t\\f(s_t, a_t, u_t) = s_{t+1} \\ f(\tilde{s}, \tilde{a}, u_t) = s'}}{\theta_{u_t}} = P(s_{t+1} \mid s_t, a_t)$, satisfying \eqref{proofeq:interventional constraint}.
    \item $\sum_{\substack{u_t \in U_t\\f(s_t, a_t, u_t) \neq s_{t+1}}}{\theta_{u_t}} = \sum_{s' \in \mathcal{S}}\sum_{\substack{u_t \in U_t\\f(s_t, a_t, u_t) \neq s_{t+1} \\ f(\tilde{s}, \tilde{a}, u_t) = s'}}{\theta_{u_t}} = 1 - P(s_{t+1} \mid s_t, a_t)$, satisfying \eqref{proofeq:interventional constraint}.
    \item $\forall s' \in \mathcal{S}, \sum_{\substack{u_t \in U_t \\ f(\tilde{s}, \tilde{a}, u_t) = s'}}{\theta_{u_t}} = \sum_{\substack{u_t \in U_t\\f(s_t, a_t, u_t) = s_{t+1} \\ f(\tilde{s}, \tilde{a}, u_t) = s'}}{\theta_{u_t}} + \sum_{\substack{u_t \in U_t\\f(s_t, a_t, u_t) \neq s_{t+1} \\ f(\tilde{s}, \tilde{a}, u_t) = s'}}{\theta_{u_t}} = P(s' \mid \tilde{s}, \tilde{a})$, satisfying \eqref{proofeq:interventional constraint}.
    \item $\tilde{P}_t(s_{t+1} \mid \tilde{s}, \tilde{a}) = \dfrac{\sum_{\substack{u_t \in U_t\\f(s_t, a_t, u_t) = s_{t+1} \\ f(\tilde{s}, \tilde{a}, u_t) = s_{t+1}}} \theta_{u_t}}{P(s_{t+1} \mid s_t, a_t)} = \dfrac{P(s_{t+1} \mid \tilde{s}, \tilde{a}) \cdot P(s_{t+1} \mid s_t, a_t)}{P(s_{t+1} \mid s_t, a_t)} \\= P(s_{t+1} \mid \tilde{s}, \tilde{a}) \geq P(s_{t+1} \mid \tilde{s}, \tilde{a})$ so Mon1 \eqref{proofeq:monotonicity1} holds.
    \item $\forall s' \in \mathcal{S}\setminus \{s_{t+1}\}$:
    \begin{align*}
        \tilde{P}_t(s' \mid \tilde{s}, \tilde{a}) 
        &= \dfrac{\sum_{\substack{u_t \in U_t\\f(s_t, a_t, u_t) = s_{t+1} \\ f(\tilde{s}, \tilde{a}, u_t) = s'}} \theta_{u_t}}{P(s_{t+1} | s_t, a_t)}
        \\ &= \dfrac{P(s' \mid \tilde{s}, \tilde{a}) \cdot P(s_{t+1} \mid s_t, a_t)}{P(s_{t+1} \mid s_t, a_t)} \\ &= P(s' \mid \tilde{s}, \tilde{a}) \\ &\leq P(s' \mid \tilde{s}, \tilde{a})
    \end{align*}
    so Mon2 \eqref{proofeq:monotonicity2} holds.
    \item CS \eqref{proofeq:counterfactual stability} doesn't apply for the transition $\tilde{s}, \tilde{a} \rightarrow s_{t+1}$. Then, $\forall s' \in \mathcal{S} \setminus \{s_{t+1}\}$, either $P(s' \mid s_t, a_t) = 0$ (so CS \eqref{proofeq:counterfactual stability} holds vacuously), or $ P(s' \mid s_t, a_t) > 0$ and either:
    \begin{itemize}
        \item $P(s' \mid \tilde{s}, \tilde{a}) > 0$, meaning $
    \dfrac{P(s_{t+1} \mid \tilde{s}, \tilde{a})}{P(s_{t+1} \mid s_t, a_t)} < \dfrac{P(s' \mid \tilde{s}, \tilde{a})}{P(s' \mid s_t, a_t)}
    $ and therefore CS \eqref{proofeq:counterfactual stability} holds vacuously.
    \item or $P(s' \mid \tilde{s}, \tilde{a}) = 0$, so the counterfactual probability for this transition must be $\tilde{P}_t(s' \mid \tilde{s}, \tilde{a})=0$, which satisfies CS \eqref{proofeq:counterfactual stability} since $\dfrac{P(s_{t+1} \mid \tilde{s}, \tilde{a})}{P(s_{t+1} \mid s_t, a_t)} = \dfrac{P(s' \mid \tilde{s}, \tilde{a})}{P(s' \mid s_t, a_t)} =0$.
    \end{itemize} 

    so CS \eqref{proofeq:counterfactual stability} holds $\forall s' \in \mathcal{S}\setminus\{s_{t+1}\}$.
    
    \item $0 \leq \theta_{u_t} \leq 1, \forall u_t$, satisfying \eqref{proofeq:valid prob1}.
    
    \item $\sum_{u_t = 1}^{|U_t|} \theta_{u_t} = \sum_{s' \in \mathcal{S}}\sum_{\substack{u_t \in U_t\\f(s_t, a_t, u_t) = s_{t+1} \\ f(\tilde{s}, \tilde{a}, u_t) = s'}}{\theta_{u_t}} + \sum_{s' \in \mathcal{S}}\sum_{\substack{u_t \in U_t\\f(s_t, a_t, u_t) \neq s_{t+1} \\ f(\tilde{s}, \tilde{a}, u_t) = s'}}{\theta_{u_t}} = 1$, satisfying \eqref{proofeq:valid prob2}.
\end{itemize}

Therefore,

\[\tilde{P}_{t}^{UB}(\tilde{s}' \mid \tilde{s}, \tilde{a}) =
\max_{\theta}\left(\dfrac{\sum_{u_t = 1}^{|U_t|} \mu_{\tilde{s}, \tilde{a}, u_t, s_{t+1}} \cdot \mu_{s_t, a_t, u_t, s_{t+1}} \cdot \theta_{u_t}}{P(s_{t+1} \mid s_t, a_t)}\right) = \dfrac{P(s_{t+1} \mid \tilde{s}, \tilde{a})}{P(s_{t+1} \mid s_t, a_t)}
\]
\end{proof}

\paragraph{Case 1(b): $P(s_{t+1} \mid \tilde{s}, \tilde{a}) > 0$ and $P(s_{t+1} \mid \tilde{s}, \tilde{a}) \geq P(s_{t+1} \mid s_t, a_t)$}
\noindent
\begin{proof}
From Lemma \ref{lemma:absolute max cf prob}, we have:

\[
\sum_{\substack{u_t \in U_t \\f(s_t, a_t, u_t) = s_{t+1} \\ f(\tilde{s}, \tilde{a}, u_t) = s_{t+1}}} \theta_{u_t} \leq \min(P(s_{t+1} \mid s_t, a_t), P(s_{t+1} \mid \tilde{s}, \tilde{a}))
\]

Therefore, 

\[
\sum_{\substack{u_t \in U_t \\f(s_t, a_t, u_t) = s_{t+1} \\ f(\tilde{s}, \tilde{a}, u_t) = s_{t+1}}} \theta_{u_t} \leq P(s_{t+1} \mid s_t, a_t)
\]

If we can show that there exists a $\theta$ such that

\[
\sum_{\substack{u_t \in U_t \\f(s_t, a_t, u_t) = s_{t+1} \\ f(\tilde{s}, \tilde{a}, u_t) = s_{t+1}}} \theta_{u_t} = P(s_{t+1} \mid s_t, a_t)
\]

and $\theta$ satisfies the constraints in the optimisation problem, we know that this results in the maximum possible counterfactual transition probability for $\tilde{P}_t(s_{t+1} \mid \tilde{s}, \tilde{a})$. We can assign $\theta$ as follows:

\begin{align*}
    \sum_{\substack{u_t \in U_t\\f(s_t, a_t, u_t) = s_{t+1} \\ f(\tilde{s}, \tilde{a}, u_t) = s_{t+1}}}{\theta_{u_t}} &= P(s_{t+1} \mid s_t, a_t)\\
    \sum_{\substack{u_t \in U_t\\f(s_t, a_t, u_t) \neq s_{t+1} \\ f(\tilde{s}, \tilde{a}, u_t) = s_{t+1}}}{\theta_{u_t}} &= P(s_{t+1} \mid \tilde{s}, \tilde{a}) - P(s_{t+1} \mid s_t, a_t)\\
    \forall s' \in \mathcal{S}\setminus \{s_{t+1}\} \sum_{\substack{u_t \in U_t\\f(s_t, a_t, u_t) = s_{t+1} \\ f(\tilde{s}, \tilde{a}, u_t) = s'}}{\theta_{u_t}} &= 0\\
   \forall s' \in \mathcal{S}\setminus \{s_{t+1}\} \sum_{\substack{u_t \in U_t\\f(s_t, a_t, u_t) \neq s_{t+1} \\ f(\tilde{s}, \tilde{a}, u_t) = s'}}{\theta_{u_t}} &= P(s' \mid \tilde{s}, \tilde{a})\\
\end{align*}

This assignment of $\theta$ satisfies the constraints of the linear optimisation problem, as follows:

\begin{itemize}
    \item $\sum_{\substack{u_t \in U_t\\f(s_t, a_t, u_t) = s_{t+1}}}{\theta_{u_t}} = \sum_{\substack{u_t \in U_t\\f(s_t, a_t, u_t) = s_{t+1} \\ f(\tilde{s}, \tilde{a}, u_t) = s_{t+1}}}{\theta_{u_t}} + \sum_{\substack{u_t \in U_t\\f(s_t, a_t, u_t) = s_{t+1} \\ f(\tilde{s}, \tilde{a}, u_t) \neq s_{t+1}}}{\theta_{u_t}} = P(s_{t+1} \mid s_t, a_t)$, satisfying \eqref{proofeq:interventional constraint}.
    \item $\sum_{\substack{u_t \in U_t\\f(s_t, a_t, u_t) \neq s_{t+1}}}{\theta_{u_t}} = \sum_{\substack{u_t \in U_t\\f(s_t, a_t, u_t) \neq s_{t+1} \\ f(\tilde{s}, \tilde{a}, u_t) = s_{t+1}}}{\theta_{u_t}} + \sum_{\substack{u_t \in U_t\\f(s_t, a_t, u_t) \neq s_{t+1} \\ f(\tilde{s}, \tilde{a}, u_t) \neq s_{t+1}}}{\theta_{u_t}} = 1 - P(s_{t+1} \mid s_t, a_t)$, satisfying \eqref{proofeq:interventional constraint}.
    \item $\forall s' \in \mathcal{S}, \sum_{\substack{u_t \in U_t \\ f(\tilde{s}, \tilde{a}, u_t) = s'}}{\theta_{u_t}} = \sum_{\substack{u_t \in U_t\\f(s_t, a_t, u_t) = s_{t+1} \\ f(\tilde{s}, \tilde{a}, u_t) = s'}}{\theta_{u_t}} + \sum_{\substack{u_t \in U_t\\f(s_t, a_t, u_t) \neq s_{t+1} \\ f(\tilde{s}, \tilde{a}, u_t) = s'}}{\theta_{u_t}} = P(s' \mid \tilde{s}, \tilde{a})$, satisfying \eqref{proofeq:interventional constraint}.
    \item $\tilde{P}_t(s_{t+1} \mid \tilde{s}, \tilde{a}) = \dfrac{\sum_{\substack{u_t \in U_t\\f(s_t, a_t, u_t) = s_{t+1} \\ f(\tilde{s}, \tilde{a}, u_t) = s_{t+1}}} \theta_{u_t}}{P(s_{t+1} \mid s_t, a_t)} = \dfrac{P(s_{t+1} \mid s_t, a_t)}{P(s_{t+1} \mid s_t, a_t)} = 1 \geq P(s_{t+1} \mid \tilde{s}, \tilde{a})$ so Mon1 \eqref{proofeq:monotonicity1} holds.
    \item $\forall s' \in \mathcal{S}\setminus \{s_{t+1}\}$:
    \begin{align*}
        \tilde{P}_t(s' \mid \tilde{s}, \tilde{a}) 
        &= \dfrac{\sum_{\substack{u_t \in U_t\\f(s_t, a_t, u_t) = s_{t+1} \\ f(\tilde{s}, \tilde{a}, u_t) = s'}} \theta_{u_t}}{P(s_{t+1} | s_t, a_t)}
        \\ &= \dfrac{0}{P(s_{t+1} \mid s_t, a_t)} \\ &= 0 \\ &\leq P(s' \mid \tilde{s}, \tilde{a})
    \end{align*}
    so Mon2 \eqref{proofeq:monotonicity2} holds.
    \item CS \eqref{proofeq:counterfactual stability} doesn't apply for the transition $\tilde{s}, \tilde{a} \rightarrow s_{t+1}$.  $\forall s' \in \mathcal{S}\setminus \{s_{t+1}\}, \tilde{P}_t(s' \mid \tilde{s}, \tilde{a}) = 0$
    which is guaranteed to satisfy CS \eqref{proofeq:counterfactual stability}.
    
    \item $0 \leq \theta_{u_t} \leq 1, \forall u_t$, satisfying \eqref{proofeq:valid prob1}.
    
    \item $\sum_{u_t = 1}^{|U_t|} \theta_{u_t} = \sum_{s' \in \mathcal{S}}\sum_{\substack{u_t \in U_t\\f(s_t, a_t, u_t) = s_{t+1} \\ f(\tilde{s}, \tilde{a}, u_t) = s'}}{\theta_{u_t}} + \sum_{s' \in \mathcal{S}}\sum_{\substack{u_t \in U_t\\f(s_t, a_t, u_t) \neq s_{t+1} \\ f(\tilde{s}, \tilde{a}, u_t) = s'}}{\theta_{u_t}} = 1$, satisfying \eqref{proofeq:valid prob2}.
\end{itemize}

Therefore,
\[
\max_{\theta}\left(\sum_{u_t = 1}^{|U_t|} \mu_{\tilde{s}, \tilde{a}, u_t, s_{t+1}} \cdot \mu_{s_t, a_t, u_t, s_{t+1}} \cdot \theta_{u_t}\right) = P(s_{t+1} \mid s_t, a_t)
\]
and
\[
\tilde{P}_{t}^{UB}(\tilde{s}' \mid \tilde{s}, \tilde{a}) =
\max_{\theta} \left( \dfrac{\sum_{u_t = 1}^{|U_t|} \mu_{\tilde{s}, \tilde{a}, u_t, s_{t+1}} \cdot \mu_{s_t, a_t, u_t, s_{t+1}} \cdot \theta_{u_t}}{P(s_{t+1} \mid s_t, a_t)} \right) = \dfrac{P(s_{t+1} \mid s_t, a_t)}{P(s_{t+1} \mid s_t, a_t)} = 1
\]

\end{proof}

\paragraph{Case 1(c): $P(s_{t+1} \mid \tilde{s}, \tilde{a}) > 0$ and $P(s_{t+1} \mid \tilde{s}, \tilde{a}) < P(s_{t+1} \mid s_t, a_t)$}
\noindent
\begin{proof}
From Lemma \ref{lemma:absolute max cf prob}, we have:

\[
\sum_{\substack{u_t \in U_t \\f(s_t, a_t, u_t) = s_{t+1} \\ f(\tilde{s}, \tilde{a}, u_t) = s_{t+1}}} \theta_{u_t} \leq \min(P(s_{t+1} \mid s_t, a_t), P(s_{t+1} \mid \tilde{s}, \tilde{a}))
\]

Therefore, 
\[
\sum_{\substack{u_t \in U_t \\f(s_t, a_t, u_t) = s_{t+1} \\ f(\tilde{s}, \tilde{a}, u_t) = s_{t+1}}} \theta_{u_t} \leq P(s_{t+1} \mid \tilde{s}, \tilde{a})
\]

If we can show that there exists a $\theta$ such that

\[
\sum_{\substack{u_t \in U_t \\f(s_t, a_t, u_t) = s_{t+1} \\ f(\tilde{s}, \tilde{a}, u_t) = s_{t+1}}} \theta_{u_t} = P(s_{t+1} \mid \tilde{s}, \tilde{a})
\]

and $\theta$ satisfies the constraints in the optimisation problem, we know that this results in the maximum possible counterfactual transition probability for $\tilde{P}_t(s_{t+1} \mid \tilde{s}, \tilde{a})$. We can assign $\theta$ such that:

\begin{align*}
\sum_{\substack{u_t \in U_t\\f(s_t, a_t, u_t) = s_{t+1} \\ f(\tilde{s}, \tilde{a}, u_t) = s_{t+1}}}{\theta_{u_t}} &= P(s_{t+1} \mid \tilde{s}, \tilde{a})\\
\sum_{\substack{u_t \in U_t\\f(s_t, a_t, u_t) = s_{t+1} \\ f(\tilde{s}, \tilde{a}, u_t) \neq s_{t+1}}}{\theta_{u_t}} &= P(s_{t+1} \mid s_t, a_t) - P(s_{t+1} \mid \tilde{s}, \tilde{a})\\
\sum_{\substack{u_t \in U_t\\f(s_t, a_t, u_t) \neq s_{t+1} \\ f(\tilde{s}, \tilde{a}, u_t) = s_{t+1}}}{\theta_{u_t}} &= 0\\
\sum_{\substack{u_t \in U_t\\f(s_t, a_t, u_t) \neq s_{t+1} \\ f(\tilde{s}, \tilde{a}, u_t) \neq s_{t+1}}}{\theta_{u_t}} &= \sum_{s' \in \mathcal{S}\setminus \{s_{t+1}\}}{P(s' \mid \tilde{s}, \tilde{a})} - \sum_{\substack{u_t \in U_t\\f(s_t, a_t, u_t) = s_{t+1} \\ f(\tilde{s}, \tilde{a}, u_t) \neq s_{t+1}}}{\theta_{u_t}}\\
\end{align*}

However, it is not immediately clear whether this assignment of $\theta$ is possible, because the Mon2 \eqref{proofeq:monotonicity2} and CS \eqref{proofeq:counterfactual stability} constraints limit $\sum_{\substack{u_t \in U_t\\f(s_t, a_t, u_t) = s_{t+1} \\ f(\tilde{s}, \tilde{a}, u_t) \neq s'}}{\theta_{u_t}}$ for all $s' \in \mathcal{S}\setminus \{s_{t+1}\}$. In particular, there exists a set $S_{CS} \subset \mathcal{S} \setminus \{s_{t+1}\}$ of states which must have a counterfactual transition probability $\tilde{P}_{t}(s' \mid \tilde{s}, \tilde{a}) = 0$ to satisfy CS \eqref{proofeq:counterfactual stability} (i.e., $\forall s' \in S_{CS}, \dfrac{P(s_{t+1} \mid \tilde{s}, \tilde{a})}{P(s_{t+1} \mid s_t, a_t)}\geq\dfrac{P(s' \mid \tilde{s}, \tilde{a})}{P(s' \mid s_t, a_t)}$ and $P(s' \mid s_t, a_t) > 0$). This means that $\forall s' \in S_{CS}$ we must assign $\theta$ such that:

\begin{align*}
    \sum_{\substack{u_t \in U_t\\f(s_t, a_t, u_t) = s_{t+1} \\ f(\tilde{s}, \tilde{a}, u_t) = s'}}{\theta_{u_t}} &= 0\\
    \sum_{\substack{u_t \in U_t\\f(s_t, a_t, u_t) \neq s_{t+1} \\ f(\tilde{s}, \tilde{a}, u_t) = s'}}{\theta_{u_t}} &= P(s' \mid \tilde{s}, \tilde{a})\\
\end{align*}

As this results in $\tilde{P}_t(s' \mid \tilde{s}, \tilde{a}) = 0$:

\begin{equation}
\label{eq: theta for CS states}
  \begin{aligned}
    \tilde{P}_t(s' \mid \tilde{s}, \tilde{a}) &= \dfrac{\sum_{\substack{u_t \in U_t\\f(s_t, a_t, u_t) = s_{t+1} \\ f(\tilde{s}, \tilde{a}, u_t) = s_{t+1}}} \theta_{u_t}}{P(s_{t+1} | s_t, a_t)} \\ &= \dfrac{0}{P(s_{t+1} \mid s_t, a_t)}\\ &= 0
\end{aligned}  
\end{equation}

Let $S_{other}$ be the set of all other states $s \in \mathcal{S}\setminus{\{s_{t+1}\}\cup S_{CS}}$. $\forall s' \in S_{other}$, if $P(s' \mid \tilde{s}, \tilde{a})>0$ and $P(s' \mid s_t, a_t)>0$, then $\tilde{P}_t(s' \mid \tilde{s}, \tilde{a}) \leq {P}(s' \mid \tilde{s}, \tilde{a})$ to satisfy Mon2 \eqref{proofeq:monotonicity2}. Because we do not know whether $P(s' \mid \tilde{s}, \tilde{a})>0$ and $P(s' \mid s_t, a_t)>0$ for all $s' \in S_{other}$, we have to assume that this condition holds for all $s' \in S_{other}$, and assign $\theta$ with this assumption. This means, $\forall s' \in S_{other}$:

\begin{equation}
\label{eq: theta for other states}
  \begin{aligned}
    \sum_{\substack{u_t \in U_t\\f(s_t, a_t, u_t) = s_{t+1} \\ f(\tilde{s}, \tilde{a}, u_t) = s'}}{\theta_{u_t}} &\leq P(s' \mid \tilde{s}, \tilde{a}) \cdot P(s_{t+1} \mid s_t, a_t)\\
    \sum_{\substack{u_t \in U_t\\f(s_t, a_t, u_t) \neq s_{t+1} \\ f(\tilde{s}, \tilde{a}, u_t) = s'}}{\theta_{u_t}} &\geq P(s' \mid \tilde{s}, \tilde{a}) - (P(s' \mid \tilde{s}, \tilde{a}) \cdot P(s_{t+1} \mid s_t, a_t))\\
\end{aligned}  
\end{equation}

as this ensures each counterfactual transition probability satisfies Mon2 \eqref{proofeq:monotonicity2}, as follows:

\begin{align*}
    \tilde{P}_t(s' \mid \tilde{s}, \tilde{a}) &= \dfrac{\sum_{\substack{u_t \in U_t\\f(s_t, a_t, u_t) = s_{t+1} \\ f(\tilde{s}, \tilde{a}, u_t) = s_{t+1}}} \theta_{u_t}}{P(s_{t+1} | s_t, a_t)} \\ &\leq \dfrac{P(s' \mid \tilde{s}, \tilde{a}) \cdot P(s_{t+1} \mid s_t, a_t)}{P(s_{t+1} \mid s_t, a_t)}\\ &= P(s' \mid \tilde{s}, \tilde{a})
\end{align*}

Therefore, to show that there exists a valid assignment of $\theta$, we must prove that under the Mon2 \eqref{proofeq:monotonicity2} and CS \eqref{proofeq:counterfactual stability} constraints:

\[\max_{\theta} \left(\sum_{s' \in \mathcal{S}\setminus \{s_{t+1}\}}\sum_{\substack{u_t \in U_t\\f(s_t, a_t, u_t) = s_{t+1} \\ f(\tilde{s}, \tilde{a}, u_t) = s'}}{\theta_{u_t}}\right) \geq P(s_{t+1} \mid s_t, a_t) - P(s_{t+1} \mid \tilde{s}, \tilde{a}) \] 

This is proven in Lemma \ref{lemma: existence of theta for overlapping case.}. Now, we can prove that this assignment of $\theta$ satisfies the constraints of the linear optimisation problem, as follows:

\begin{itemize}
    \item $\sum_{\substack{u_t \in U_t\\f(s_t, a_t, u_t) = s_{t+1}}}{\theta_{u_t}} = \sum_{\substack{u_t \in U_t\\f(s_t, a_t, u_t) = s_{t+1} \\ f(\tilde{s}, \tilde{a}, u_t) = s_{t+1}}}{\theta_{u_t}} + \sum_{\substack{u_t \in U_t\\f(s_t, a_t, u_t) = s_{t+1} \\ f(\tilde{s}, \tilde{a}, u_t) \neq s_{t+1}}}{\theta_{u_t}} = P(s_{t+1} \mid s_t, a_t)$, satisfying \eqref{proofeq:interventional constraint}.
    \item $\sum_{\substack{u_t \in U_t\\f(s_t, a_t, u_t) \neq s_{t+1}}}{\theta_{u_t}} = \sum_{\substack{u_t \in U_t\\f(s_t, a_t, u_t) \neq s_{t+1} \\ f(\tilde{s}, \tilde{a}, u_t) = s_{t+1}}}{\theta_{u_t}} + \sum_{\substack{u_t \in U_t\\f(s_t, a_t, u_t) \neq s_{t+1} \\ f(\tilde{s}, \tilde{a}, u_t) \neq s_{t+1}}}{\theta_{u_t}} = 1 - P(s_{t+1} \mid s_t, a_t)$, satisfying \eqref{proofeq:interventional constraint}.
    \item $\sum_{s' \in \mathcal{S}}\sum_{\substack{u_t \in U_t \\ f(\tilde{s}, \tilde{a}, u_t) = s'}}{\theta_{u_t}} = \sum_{s' \in \mathcal{S}}\sum_{\substack{u_t \in U_t\\f(s_t, a_t, u_t) = s_{t+1} \\ f(\tilde{s}, \tilde{a}, u_t) = s'}}{\theta_{u_t}} + \sum_{s' \in \mathcal{S}}\sum_{\substack{u_t \in U_t\\f(s_t, a_t, u_t) \neq s_{t+1} \\ f(\tilde{s}, \tilde{a}, u_t) = s'}}{\theta_{u_t}} \\= P(s' \mid \tilde{s}, \tilde{a})$, therefore it is possible to assign $\theta$ such that \\$\forall s' \in \mathcal{S}\setminus\{\tilde{s}'\}, \sum_{\substack{u_t \in U_t \\ f(\tilde{s}, \tilde{a}, u_t) = s'}}{\theta_{u_t}} = P(s' \mid \tilde{s}, \tilde{a})$, satisfying \eqref{proofeq:interventional constraint}.
    \item $\tilde{P}_t(s_{t+1} \mid \tilde{s}, \tilde{a}) = \dfrac{\sum_{\substack{u_t \in U_t\\f(s_t, a_t, u_t) = s_{t+1} \\ f(\tilde{s}, \tilde{a}, u_t) = s_{t+1}}} \theta_{u_t}}{P(s_{t+1} \mid s_t, a_t)} = \dfrac{P(s_{t+1} \mid \tilde{s}, \tilde{a})}{P(s_{t+1} \mid s_t, a_t)} \geq P(s_{t+1} \mid \tilde{s}, \tilde{a})$ so Mon1 \eqref{proofeq:monotonicity1} holds.
    
    \item Lemma \ref{lemma: existence of theta for overlapping case.} proves that we can assign $\theta$ such that Mon2 \eqref{proofeq:monotonicity2} holds.
    \item CS \eqref{proofeq:counterfactual stability} doesn't apply to the transition $\tilde{s}, \tilde{a} \rightarrow s_{t+1}$. Lemma \ref{lemma: existence of theta for overlapping case.} proves that we can assign $\theta$ such that CS \eqref{proofeq:counterfactual stability} is satisfied for all other transitions.
    
    \item $0 \leq \theta_{u_t} \leq 1, \forall u_t$, satisfying \eqref{proofeq:valid prob1}.
    
    \item $\sum_{u_t = 1}^{|U_t|} \theta_{u_t} = \sum_{s' \in \mathcal{S}}\sum_{\substack{u_t \in U_t\\f(s_t, a_t, u_t) = s_{t+1} \\ f(\tilde{s}, \tilde{a}, u_t) = s'}}{\theta_{u_t}} + \sum_{s' \in \mathcal{S}}\sum_{\substack{u_t \in U_t\\f(s_t, a_t, u_t) \neq s_{t+1} \\ f(\tilde{s}, \tilde{a}, u_t) = s'}}{\theta_{u_t}} = 1$, satisfying \eqref{proofeq:valid prob2}.
\end{itemize}

Therefore,
\[
\max_{\theta}\left(\sum_{u_t = 1}^{|U_t|} \mu_{\tilde{s}, \tilde{a}, u_t, s_{t+1}} \cdot \mu_{s_t, a_t, u_t, s_{t+1}} \cdot \theta_{u_t}\right) = P(s_{t+1} \mid \tilde{s}, \tilde{a})
\]
and
\[\tilde{P}_{t}^{UB}(\tilde{s}' \mid \tilde{s}, \tilde{a}) =
\dfrac{\max_{\theta}\sum_{u_t = 1}^{|U_t|} \mu_{\tilde{s}, \tilde{a}, u_t, s_{t+1}} \cdot \mu_{s_t, a_t, u_t, s_{t+1}} \cdot \theta_{u_t}}{P(s_{t+1} \mid s_t, a_t)} = \dfrac{P(s_{t+1} \mid \tilde{s}, \tilde{a})}{P(s_{t+1} \mid s_t, a_t)}
\]
\end{proof}

These three sub-cases can be simplified to:

\[\tilde{P}_{t}^{UB}(\tilde{s}' \mid \tilde{s}, \tilde{a}) = \dfrac{\min\left(P(s_{t+1} \mid s_t, a_t), P(s_{t+1} \mid \tilde{s}, \tilde{a})\right)}{P(s_{t+1} \mid s_t, a_t)}\]
\end{proof}

\paragraph{Case 2: $\tilde{s}' \neq s_{t+1}$ and $\dfrac{P(s_{t+1} \mid \tilde{s}, \tilde{a})}{P(s_{t+1} \mid s_t, a_t)}\geq\dfrac{P(\tilde{s}' \mid \tilde{s}, \tilde{a})}{P(\tilde{s}' \mid s_t, a_t)}$ and $P(\tilde{s}' \mid s_t, a_t) > 0$}
\noindent
\begin{proof}
Because $\dfrac{P(s_{t+1} \mid \tilde{s}, \tilde{a})}{P(s_{t+1} \mid s_t, a_t)}\geq\dfrac{P(\tilde{s}' \mid \tilde{s}, \tilde{a})}{P(\tilde{s}' \mid s_t, a_t)}$ and $P(\tilde{s}' \mid s_t, a_t) > 0$, to satisfy CS \eqref{proofeq:counterfactual stability} we must have $\tilde{P}_{t}(\tilde{s}' \mid \tilde{s}, \tilde{a}) = 0$, therefore $\tilde{P}_{t}^{UB}(\tilde{s}' \mid \tilde{s}, \tilde{a}) = 0$. From Case 1, we know that we can find an assignment of $\theta$ that satisfies CS \eqref{proofeq:counterfactual stability} for all transitions, when finding the upper bound of the transition $\tilde{s}, \tilde{a} \rightarrow s_{t+1}$. Therefore, we have already proven in Case 1 that there exists an assignment of $\theta$ where $\tilde{P}_{t}(\tilde{s}' \mid \tilde{s}, \tilde{a}) = 0$, that satisfies the constraints \eqref{proofeq:interventional constraint}-\eqref{proofeq:valid prob2}. Therefore, there exists an assignment of $\theta$ that satisfies $\tilde{P}_{t}^{UB}(\tilde{s}' \mid \tilde{s}, \tilde{a}) = 0$.
\end{proof}

\pagebreak
\paragraph{Case 3: $\tilde{s}' \neq s_{t+1}$, $\dfrac{P(s_{t+1} \mid \tilde{s}, \tilde{a})}{P(s_{t+1} \mid s_t, a_t)}<\dfrac{P(\tilde{s}' \mid \tilde{s}, \tilde{a})}{P(\tilde{s}' \mid s_t, a_t)}$, and $P(\tilde{s}' \mid s_t, a_t) >0$}
\noindent
\begin{proof}
Because $P(\tilde{s}' \mid s_t, a_t) >0$, $\tilde{P}_{t}^{UB}(\tilde{s}' \mid \tilde{s}, \tilde{a}) \leq P(\tilde{s}' \mid \tilde{s}, \tilde{a})$ to satisfy Mon2 \eqref{proofeq:monotonicity2}. Therefore,
\begin{equation}
\label{eq: overlapping constraint 1}
\sum_{u_t = 1}^{|U_t|} \mu_{\tilde{s}, \tilde{a}, u_t, \tilde{s}'} \cdot \mu_{s_t, a_t, u_t, s_{t+1}} \cdot \theta_{u_t} \leq P(\tilde{s}' \mid \tilde{s}, \tilde{a}) \cdot P(s_{t+1} \mid s_t, a_t) 
\end{equation}

so that \[\tilde{P}_{t}^{UB}(\tilde{s}' \mid \tilde{s}, \tilde{a}) = \dfrac{\sum_{u_t = 1}^{|U_t|} \mu_{\tilde{s}, \tilde{a}, u_t, \tilde{s}'} \cdot \mu_{s_t, a_t, u_t, s_{t+1}} \cdot \theta_{u_t}}{P(s_{t+1} \mid s_t, a_t)} \leq \dfrac{P(\tilde{s}' \mid \tilde{s}, \tilde{a}) \cdot P(s_{t+1} \mid s_t, a_t)}{P(s_{t+1} \mid s_t, a_t)} = P(\tilde{s}' \mid \tilde{s}, \tilde{a})\]

Clearly, this is less than the maximum possible counterfactual probability proven in Lemma \ref{lemma:absolute max cf prob}: \[\sum_{u_t = 1}^{|U_t|} \mu_{\tilde{s}, \tilde{a}, u_t, \tilde{s}'} \cdot \mu_{s_t, a_t, u_t, s_{t+1}} \cdot \theta_{u_t} \leq P(\tilde{s}' \mid \tilde{s}, \tilde{a}) \cdot P(s_{t+1} \mid s_t, a_t) \leq \min\left(P(s_{t+1} \mid s_t, a_t), P(\tilde{s}' \mid \tilde{s}, \tilde{a})\right)\]

Also, $\tilde{P}_t(s_{t+1} \mid \tilde{s}, \tilde{a}) \geq P(s_{t+1} \mid \tilde{s}, \tilde{a})$ to satisfy Mon1 \eqref{proofeq:monotonicity1}. Therefore,

\begin{equation}
\label{eq: overlapping constraint 2}
\sum_{u_t = 1}^{|U_t|} \mu_{\tilde{s}, \tilde{a}, u_t, s_{t+1}} \cdot \mu_{s_t, a_t, u_t, s_{t+1}} \cdot \theta_{u_t} \geq P(s_{t+1} \mid \tilde{s}, \tilde{a}) \cdot P(s_{t+1} \mid s_t, a_t)    
\end{equation}

Since $\forall s,s' \in \mathcal{S}, \forall a \in \mathcal{A}, \forall u_t \in U_t, \mu_{s, a, u_t, s'} = 1 \iff f(s, a, u_t) = s'$ by definition of $\mu$, Eq. \eqref{eq: overlapping constraint 1} and Eq. \eqref{eq: overlapping constraint 2} are equivalent to:

\begin{equation}
\label{eq: overlapping constraint 3}
\sum_{\substack{u_t \in U_t \\f(s_t, a_t, u_t) = s_{t+1} \\ f(\tilde{s}, \tilde{a}, u_t) = \tilde{s}'}} \theta_{u_t} \leq P(\tilde{s}' \mid \tilde{s}, \tilde{a}) \cdot P(s_{t+1} \mid s_t, a_t)     
\end{equation}
and

\begin{equation}
\label{eq: overlapping constraint 4}
\sum_{\substack{u_t \in U_t \\f(s_t, a_t, u_t) = s_{t+1} \\ f(\tilde{s}, \tilde{a}, u_t) = s_{t+1}}} \theta_{u_t}\geq P(s_{t+1} \mid \tilde{s}, \tilde{a}) \cdot P(s_{t+1} \mid s_t, a_t)     
\end{equation}

respectively. As these notations are equivalent, we will use this second notation for brevity and clarity. Consider the following disjoint cases:
\begin{itemize}
    \item $P(s_{t+1} \mid s_t, a_t) - (P(s_{t+1} \mid s_t, a_t) \cdot P(s_{t+1} \mid \tilde{s}, \tilde{a})) < P(\tilde{s}' \mid \tilde{s}, \tilde{a}) \cdot P(s_{t+1} \mid s_t, a_t)$
    \item $P(s_{t+1} \mid s_t, a_t) - (P(s_{t+1} \mid s_t, a_t) \cdot P(s_{t+1} \mid \tilde{s}, \tilde{a})) \geq P(\tilde{s}' \mid \tilde{s}, \tilde{a}) \cdot P(s_{t+1} \mid s_t, a_t)$
\end{itemize}

\pagebreak
\paragraph{Case 3(a): $P(s_{t+1} \mid s_t, a_t) - (P(s_{t+1} \mid s_t, a_t) \cdot P(s_{t+1} \mid \tilde{s}, \tilde{a})) < P(\tilde{s}' \mid \tilde{s}, \tilde{a}) \cdot P(s_{t+1} \mid s_t, a_t)$}
\noindent
\begin{proof}
We can assign $\theta$ as follows:

\[
\begin{cases}
    \sum_{\substack{u_t \in U_t \\f(s_t, a_t, u_t) = s_{t+1} \\ f(\tilde{s}, \tilde{a}, u_t) = s_{t+1}}}{\theta_{u_t}} = P(s_{t+1} \mid s_t, a_t) \cdot P(s_{t+1} \mid \tilde{s}, \tilde{a})\\
    
    \sum_{\substack{u_t \in U_t \\f(s_t, a_t, u_t) = s_{t+1} \\ f(\tilde{s}, \tilde{a}, u_t) = \tilde{s}'}}{\theta_{u_t}} = P(s_{t+1} \mid s_t, a_t) - (P(s_{t+1} \mid s_t, a_t) \cdot P(s_{t+1} \mid \tilde{s}, \tilde{a})) \\
    
    \sum_{s' \in \mathcal{S}\setminus\{\tilde{s}', s_{t+1}\}}\sum_{\substack{u_t \in U_t \\f(s_t, a_t, u_t) = s_{t+1} \\ f(\tilde{s}, \tilde{a}, u_t) = s'}}{\theta_{u_t}} = 0\\
    
    \sum_{\substack{u_t \in U_t \\f(s_t, a_t, u_t) \neq s_{t+1} \\ f(\tilde{s}, \tilde{a}, u_t) = s_{t+1}}}{\theta_{u_t}} = P(s_{t+1} \mid \tilde{s}, \tilde{a}) - \sum_{\substack{u_t \in U_t \\f(s_t, a_t, u_t) = s_{t+1} \\ f(\tilde{s}, \tilde{a}, u_t) = s_{t+1}}}{\theta_{u_t}} \\
    
    \sum_{\substack{u_t \in U_t \\f(s_t, a_t, u_t) \neq s_{t+1} \\ f(\tilde{s}, \tilde{a}, u_t) = \tilde{s}'}}{\theta_{u_t}} = P(\tilde{s}' \mid \tilde{s}, \tilde{a}) - \sum_{\substack{u_t \in U_t \\f(s_t, a_t, u_t) = s_{t+1} \\ f(\tilde{s}, \tilde{a}, u_t) = \tilde{s}'}}{\theta_{u_t}} \\
    
    \sum_{s' \in \mathcal{S}\setminus\{\tilde{s}', s_{t+1}\}}\sum_{\substack{u_t \in U_t \\f(s_t, a_t, u_t) \neq s_{t+1} \\ f(\tilde{s}, \tilde{a}, u_t) = s'}}{\theta_{u_t}} = \sum_{s' \in \mathcal{S}\setminus\{\tilde{s}', s_{t+1}\}}P(\tilde{s}' \mid \tilde{s}, \tilde{a})\\
\end{cases}
\]

This assignment of $\theta$ satisfies the constraints of the linear optimisation problem, as follows:

\begin{itemize}        
    \item $\sum_{\substack{u_t \in U_t \\f(s_t, a_t, u_t) = s_{t+1}}} \theta_{u_t} = \sum_{\substack{u_t \in U_t \\f(s_t, a_t, u_t) = s_{t+1} \\ f(\tilde{s}, \tilde{a}, u_t) = \tilde{s}'}} \theta_{u_t} + \sum_{\substack{u_t \in U_t \\f(s_t, a_t, u_t) = s_{t+1} \\ f(\tilde{s}, \tilde{a}, u_t) = s_{t+1}}} \theta_{u_t} + \sum_{s' \in \mathcal{S}\setminus\{\tilde{s}', s_{t+1}\}}\sum_{\substack{u_t \in U_t \\f(s_t, a_t, u_t) = s_{t+1} \\ f(\tilde{s}, \tilde{a}, u_t) = s'}}{\theta_{u_t}} \\ = P(s_{t+1} \mid s_t, a_t)$, satisfying \eqref{proofeq:interventional constraint}.

    \item $\sum_{\substack{u_t \in U_t\\f(s_t, a_t, u_t) \neq s_{t+1}}}{\theta_{u_t}} = \sum_{\substack{u_t \in U_t \\f(s_t, a_t, u_t) \neq s_{t+1} \\ f(\tilde{s}, \tilde{a}, u_t) = \tilde{s}'}} \theta_{u_t} + \sum_{\substack{u_t \in U_t \\f(s_t, a_t, u_t) \neq s_{t+1} \\ f(\tilde{s}, \tilde{a}, u_t) = s_{t+1}}} \theta_{u_t} + \sum_{s' \in \mathcal{S}\setminus\{\tilde{s}', s_{t+1}\}}\sum_{\substack{u_t \in U_t \\f(s_t, a_t, u_t) \neq s_{t+1} \\ f(\tilde{s}, \tilde{a}, u_t) = s'}}{\theta_{u_t}} \\=1 -
    P(s_{t+1} \mid s_t, a_t)$, satisfying \eqref{proofeq:interventional constraint}.
    
    \item $\sum_{\substack{u_t \in U_t \\f(\tilde{s}, \tilde{a}, u_t) = \tilde{s}'}} \theta_{u_t} = \sum_{\substack{u_t \in U_t \\f(s_t, a_t, u_t) = s_{t+1} \\ f(\tilde{s}, \tilde{a}, u_t) = \tilde{s}'}} \theta_{u_t} + \sum_{\substack{u_t \in U_t \\f(s_t, a_t, u_t) \neq s_{t+1} \\ f(\tilde{s}, \tilde{a}, u_t) = \tilde{s}'}} \theta_{u_t} = P(\tilde{s}' \mid \tilde{s}, \tilde{a})$, \\satisfying \eqref{proofeq:interventional constraint}.

    \item $\sum_{\substack{u_t \in U_t \\f(\tilde{s}, \tilde{a}, u_t) = s_{t+1}}} \theta_{u_t} = \sum_{\substack{u_t \in U_t \\f(s_t, a_t, u_t) = s_{t+1} \\ f(\tilde{s}, \tilde{a}, u_t) = s_{t+1}}} \theta_{u_t} + \sum_{\substack{u_t \in U_t \\f(s_t, a_t, u_t) \neq s_{t+1} \\ f(\tilde{s}, \tilde{a}, u_t) = s_{t+1}}} \theta_{u_t} = P(s_{t+1} \mid \tilde{s}, \tilde{a})$, \\satisfying \eqref{proofeq:interventional constraint}.
    
    \item $\sum_{s' \in \mathcal{S}\setminus\{\tilde{s}', s_{t+1}\}}\sum_{\substack{u_t \in U_t \\f(\tilde{s}, \tilde{a}, u_t) = s'}} \theta_{u_t} = \sum_{s' \in \mathcal{S}\setminus\{\tilde{s}', s_{t+1}\}}\sum_{\substack{u_t \in U_t \\f(s_t, a_t, u_t) = s_{t+1} \\ f(\tilde{s}, \tilde{a}, u_t) = s'}} \theta_{u_t} \\+ \sum_{s' \in \mathcal{S}\setminus\{\tilde{s}', s_{t+1}\}}\sum_{\substack{u_t \in U_t \\f(s_t, a_t, u_t) \neq s_{t+1} \\ f(\tilde{s}, \tilde{a}, u_t) = s'}} \theta_{u_t} = 1 - P(\tilde{s}' \mid \tilde{s}, \tilde{a}) = \sum_{s' \in \mathcal{S}\setminus\{\tilde{s}'\}}P(s' \mid \tilde{s}, \tilde{a})$, \\therefore it is possible to assign $\theta$ such that \\$\forall s' \in \mathcal{S}\setminus\{\tilde{s}', s_{t+1}\}, \sum_{\substack{u_t \in U_t \\ f(\tilde{s}, \tilde{a}, u_t) = s'}}{\theta_{u_t}} = P(s' \mid \tilde{s}, \tilde{a})$,
    satisfying \eqref{proofeq:interventional constraint}.

    \item \eqref{eq: overlapping constraint 4} holds, therefore Mon1 \eqref{proofeq:monotonicity1} is satisfied.
    
    \item \eqref{eq: overlapping constraint 3} holds, therefore $\tilde{P}_t(\tilde{s}' \mid \tilde{s}, \tilde{a}) \leq P(\tilde{s}' \mid \tilde{s}, \tilde{a})$. Also, because \\$\sum_{s' \in \mathcal{S}\setminus\{\tilde{s}', s_{t+1}\}}\sum_{\substack{u_t \in U_t \\f(s_t, a_t, u_t) = s_{t+1} \\ f(\tilde{s}, \tilde{a}, u_t) = s'}}{\theta_{u_t}} = 0$, the counterfactual probability of all other \\transitions $\tilde{s}, \tilde{a} \rightarrow s'$ where $s' \in \mathcal{S}\setminus\{\tilde{s}', s_{t+1}\}$ will be $\tilde{P}_t(s' \mid \tilde{s}, \tilde{a}) = 0 \leq \tilde{P}_t(s' \mid \tilde{s}, \tilde{a})$. Therefore, Mon2 \eqref{proofeq:monotonicity2} is satisfied.

    \item CS \eqref{proofeq:counterfactual stability} doesn't apply to the transition $\tilde{s}, \tilde{a} \rightarrow s_{t+1}$, and is vacuously satisfied for the transition $\tilde{s}, \tilde{a} \rightarrow \tilde{s}'$  because \\$\dfrac{P(s_{t+1} \mid \tilde{s}, \tilde{a})}{P(s_{t+1} \mid s_t, a_t)}<\dfrac{P(\tilde{s}' \mid \tilde{s}, \tilde{a})}{P(\tilde{s}' \mid s_t, a_t)}$ (from conditions of Case 3). Also, because\\ $\sum_{s' \in \mathcal{S}\setminus\{\tilde{s}', s_{t+1}\}}\sum_{\substack{u_t \in U_t \\f(s_t, a_t, u_t) = s_{t+1} \\ f(\tilde{s}, \tilde{a}, u_t) = s'}}{\theta_{u_t}} = 0$, the counterfactual probability of all other transitions $\tilde{s}, \tilde{a} \rightarrow s'$ where $s' \in \mathcal{S}\setminus\{\tilde{s}', s_{t+1}\}$ will be $\tilde{P}_t(s' \mid \tilde{s}, \tilde{a}) = 0$, which satisfies CS \eqref{proofeq:counterfactual stability} for these transitions.

    \item $0 \leq \theta_{u_t} \leq 1, \forall u_t$, satisfying \eqref{proofeq:valid prob1}.
    
    \item $\sum_{u_t = 1}^{|U_t|} \theta_{u_t} = \sum_{\substack{u_t \in U_t \\f(s_t, a_t, u_t) = s_{t+1} \\ f(\tilde{s}, \tilde{a}, u_t) = \tilde{s}'}} \theta_{u_t} + \sum_{\substack{u_t \in U_t \\f(s_t, a_t, u_t) = s_{t+1} \\ f(\tilde{s}, \tilde{a}, u_t) = s_{t+1}}} \theta_{u_t} +  \sum_{s' \in \mathcal{S}\setminus\{\tilde{s}', s_{t+1}\}}\sum_{\substack{u_t \in U_t \\f(s_t, a_t, u_t) = s_{t+1} \\ f(\tilde{s}, \tilde{a}, u_t) = s'}}{\theta_{u_t}} \\+ \sum_{\substack{u_t \in U_t \\f(s_t, a_t, u_t) \neq s_{t+1} \\ f(\tilde{s}, \tilde{a}, u_t) = \tilde{s}'}} \theta_{u_t} + 
    \sum_{\substack{u_t \in U_t \\f(s_t, a_t, u_t) \neq s_{t+1} \\ f(\tilde{s}, \tilde{a}, u_t) = s_{t+1}}} \theta_{u_t} + \sum_{s' \in \mathcal{S}\setminus\{\tilde{s}', s_{t+1}\}}\sum_{\substack{u_t \in U_t \\f(s_t, a_t, u_t) \neq s_{t+1} \\ f(\tilde{s}, \tilde{a}, u_t) = s'}}{\theta_{u_t}} = 1$, satisfying \eqref{proofeq:valid prob2}.
\end{itemize}

All the constraints are satisfied, so this is a valid assignment of $\theta$ for this state-action pair.
\end{proof}
\paragraph{Case 3(b): $P(s_{t+1} \mid s_t, a_t) - (P(s_{t+1} \mid s_t, a_t) \cdot P(s_{t+1} \mid \tilde{s}, \tilde{a})) \geq P(\tilde{s}' \mid \tilde{s}, \tilde{a}) \cdot P(s_{t+1} \mid s_t, a_t)$}
\noindent
\begin{proof}
If $P(s_{t+1} \mid s_t, a_t) - (P(s_{t+1} \mid s_t, a_t) \cdot P(s_{t+1} \mid \tilde{s}, \tilde{a})) \geq P(\tilde{s}' \mid \tilde{s}, \tilde{a}) \cdot P(s_{t+1} \mid s_t, a_t)$, this implies $P(s_{t+1} \mid s_t, a_t) \geq P(\tilde{s}' \mid \tilde{s}, \tilde{a}) \cdot P(s_{t+1} \mid s_t, a_t)$. Therefore, we can assign $\theta$ as follows:

 \[
    \begin{cases}
        P(s_{t+1} \mid s_t, a_t) \cdot P(s_{t+1} \mid \tilde{s}, \tilde{a}) \leq \sum_{\substack{u_t \in U_t \\f(s_t, a_t, u_t) = s_{t+1} \\ f(\tilde{s}, \tilde{a}, u_t) = s_{t+1}}}{\theta_{u_t}} \leq P(s_{t+1} \mid \tilde{s}, \tilde{a}) \\
        
        \sum_{\substack{u_t \in U_t \\f(s_t, a_t, u_t) = s_{t+1} \\ f(\tilde{s}, \tilde{a}, u_t) = \tilde{s}'}}{\theta_{u_t}} = P(\tilde{s}' \mid \tilde{s}, \tilde{a}) \cdot P(s_{t+1} \mid s_t, a_t)\\
        
        0 \leq \sum_{s' \in \mathcal{S}\setminus\{\tilde{s}', s_{t+1}\}}\sum_{\substack{u_t \in U_t \\f(s_t, a_t, u_t) = s_{t+1} \\ f(\tilde{s}, \tilde{a}, u_t) = s'}}{\theta_{u_t}} \leq \sum_{s' \in \mathcal{S}\setminus\{\tilde{s}', s_{t+1}\}}P(s' \mid \tilde{s}, \tilde{a}) \cdot P(s_{t+1} \mid s_t, a_t)\\
        
        \sum_{\substack{u_t \in U_t \\f(s_t, a_t, u_t) \neq s_{t+1} \\ f(\tilde{s}, \tilde{a}, u_t) = s_{t+1}}}{\theta_{u_t}} = P(s_{t+1} \mid \tilde{s}, \tilde{a}) - \sum_{\substack{u_t \in U_t \\f(s_t, a_t, u_t) = s_{t+1} \\ f(\tilde{s}, \tilde{a}, u_t) = s_{t+1}}}{\theta_{u_t}} \\
        
        \sum_{\substack{u_t \in U_t \\f(s_t, a_t, u_t) \neq s_{t+1} \\ f(\tilde{s}, \tilde{a}, u_t) = \tilde{s}'}}{\theta_{u_t}} = P(\tilde{s}' \mid \tilde{s}, \tilde{a}) - \sum_{\substack{u_t \in U_t \\f(s_t, a_t, u_t) = s_{t+1} \\ f(\tilde{s}, \tilde{a}, u_t) = \tilde{s}'}}{\theta_{u_t}} \\
        
        \sum_{s' \in \mathcal{S}\setminus\{\tilde{s}', s_{t+1}\}}\sum_{\substack{u_t \in U_t \\f(s_t, a_t, u_t) \neq s_{t+1} \\ f(\tilde{s}, \tilde{a}, u_t) = s'}}{\theta_{u_t}} = \sum_{s' \in \mathcal{S}\setminus\{\tilde{s}', s_{t+1}\}}P(\tilde{s}' \mid \tilde{s}, \tilde{a}) - \sum_{s' \in \mathcal{S}\setminus\{\tilde{s}', s_{t+1}\}}\sum_{\substack{u_t \in U_t \\f(s_t, a_t, u_t) = s_{t+1} \\ f(\tilde{s}, \tilde{a}, u_t) = s'}}{\theta_{u_t}}\\
    \end{cases}
\]

s.t. $\sum_{\substack{u_t \in U_t \\f(s_t, a_t, u_t) = s_{t+1} \\ f(\tilde{s}, \tilde{a}, u_t) = \tilde{s}'}}{\theta_{u_t}} + \sum_{\substack{u_t \in U_t \\f(s_t, a_t, u_t) = s_{t+1} \\ f(\tilde{s}, \tilde{a}, u_t) = s_{t+1}}}{\theta_{u_t}} + \sum_{s' \in \mathcal{S}\setminus\{\tilde{s}', s_{t+1}\}}\sum_{\substack{u_t \in U_t \\f(s_t, a_t, u_t) = s_{t+1} \\ f(\tilde{s}, \tilde{a}, u_t) = s'}}{\theta_{u_t}} \\= P(s_{t+1} \mid s_t, a_t)$.\\\\

We do not know exactly how much probability will be assigned to $\sum_{\substack{u_t \in U_t \\f(s_t, a_t, u_t) = s_{t+1} \\ f(\tilde{s}, \tilde{a}, u_t) = s_{t+1}}}{\theta_{u_t}}$ and $\sum_{s' \in \mathcal{S} \setminus \{\tilde{s}', s_{t+1}\}}\sum_{\substack{u_t \in U_t \\f(s_t, a_t, u_t) = s_{t+1} \\ f(\tilde{s}, \tilde{a}, u_t)= s'}}{\theta_{u_t}}$. However, if $\theta$ satisfies the above assignment then $\theta$ will always satisfy the constraints of the optimisation problem, as follows:

\begin{itemize}        
    \item $\sum_{\substack{u_t \in U_t \\f(s_t, a_t, u_t) = s_{t+1}}} \theta_{u_t} = \sum_{\substack{u_t \in U_t \\f(s_t, a_t, u_t) = s_{t+1} \\ f(\tilde{s}, \tilde{a}, u_t) = \tilde{s}'}} \theta_{u_t} + \sum_{\substack{u_t \in U_t \\f(s_t, a_t, u_t) = s_{t+1} \\ f(\tilde{s}, \tilde{a}, u_t) = s_{t+1}}} \theta_{u_t} + \sum_{s' \in \mathcal{S}\setminus\{\tilde{s}', s_{t+1}\}}\sum_{\substack{u_t \in U_t \\f(s_t, a_t, u_t) = s_{t+1} \\ f(\tilde{s}, \tilde{a}, u_t) = s'}}{\theta_{u_t}} \\ = P(s_{t+1} \mid s_t, a_t)$, satisfying \eqref{proofeq:interventional constraint}.

    \item $\sum_{\substack{u_t \in U_t\\f(s_t, a_t, u_t) \neq s_{t+1}}}{\theta_{u_t}} = \sum_{\substack{u_t \in U_t \\f(s_t, a_t, u_t) \neq s_{t+1} \\ f(\tilde{s}, \tilde{a}, u_t) = \tilde{s}'}} \theta_{u_t} + \sum_{\substack{u_t \in U_t \\f(s_t, a_t, u_t) \neq s_{t+1} \\ f(\tilde{s}, \tilde{a}, u_t) = s_{t+1}}} \theta_{u_t} + \sum_{s' \in \mathcal{S}\setminus\{\tilde{s}', s_{t+1}\}}\sum_{\substack{u_t \in U_t \\f(s_t, a_t, u_t) \neq s_{t+1} \\ f(\tilde{s}, \tilde{a}, u_t) = s'}}{\theta_{u_t}} \\=1 -
    P(s_{t+1} \mid s_t, a_t)$, satisfying \eqref{proofeq:interventional constraint}
    
    \item $\sum_{\substack{u_t \in U_t \\f(\tilde{s}, \tilde{a}, u_t) = \tilde{s}'}} \theta_{u_t} = \sum_{\substack{u_t \in U_t \\f(s_t, a_t, u_t) = s_{t+1} \\ f(\tilde{s}, \tilde{a}, u_t) = \tilde{s}'}} \theta_{u_t} + \sum_{\substack{u_t \in U_t \\f(s_t, a_t, u_t) \neq s_{t+1} \\ f(\tilde{s}, \tilde{a}, u_t) = \tilde{s}'}} \theta_{u_t} = P(\tilde{s}' \mid \tilde{s}, \tilde{a})$, \\satisfying \eqref{proofeq:interventional constraint}.

    \item $\sum_{\substack{u_t \in U_t \\f(\tilde{s}, \tilde{a}, u_t) = s_{t+1}}} \theta_{u_t} = \sum_{\substack{u_t \in U_t \\f(s_t, a_t, u_t) = s_{t+1} \\ f(\tilde{s}, \tilde{a}, u_t) = s_{t+1}}} \theta_{u_t} + \sum_{\substack{u_t \in U_t \\f(s_t, a_t, u_t) \neq s_{t+1} \\ f(\tilde{s}, \tilde{a}, u_t) = s_{t+1}}} \theta_{u_t} = P(s_{t+1} \mid \tilde{s}, \tilde{a})$, \\satisfying \eqref{proofeq:interventional constraint}.

    \item $\sum_{s' \in \mathcal{S}\setminus\{\tilde{s}', s_{t+1}\}}\sum_{\substack{u_t \in U_t \\f(\tilde{s}, \tilde{a}, u_t) = s'}} \theta_{u_t} = \sum_{s' \in \mathcal{S}\setminus\{\tilde{s}', s_{t+1}\}}\sum_{\substack{u_t \in U_t \\f(s_t, a_t, u_t) = s_{t+1} \\ f(\tilde{s}, \tilde{a}, u_t) = s'}} \theta_{u_t} + \sum_{s' \in \mathcal{S}\setminus\{\tilde{s}', s_{t+1}\}}\sum_{\substack{u_t \in U_t \\f(s_t, a_t, u_t) \neq s_{t+1} \\ f(\tilde{s}, \tilde{a}, u_t) = s'}} \theta_{u_t} \\= 1 - P(\tilde{s}' \mid \tilde{s}, \tilde{a}) = \sum_{s' \in \mathcal{S}\setminus\{\tilde{s}'\}}P(s' \mid \tilde{s}, \tilde{a})$, therefore it is possible to assign $\theta$ such that $\forall s' \in \mathcal{S}\setminus\{\tilde{s}', s_{t+1}\}, \sum_{\substack{u_t \in U_t \\ f(\tilde{s}, \tilde{a}, u_t) = s'}}{\theta_{u_t}} = P(s' \mid \tilde{s}, \tilde{a})$, satisfying \eqref{proofeq:interventional constraint}.

    \item \eqref{eq: overlapping constraint 4} holds, therefore Mon1 \eqref{proofeq:monotonicity1} is satisfied.
    
    \item \eqref{eq: overlapping constraint 3} holds, so Mon2 \eqref{proofeq:monotonicity2} holds for the transition $\tilde{s}, \tilde{a} \rightarrow \tilde{s}'$.

    \item CS \eqref{proofeq:counterfactual stability} doesn't apply to the transition  $\tilde{s}, \tilde{a} \rightarrow s_{t+1}$, and is vacuously satisfied for the transition $\tilde{s}, \tilde{a} \rightarrow \tilde{s}'$ because $\dfrac{P(s_{t+1} \mid \tilde{s}, \tilde{a})}{P(s_{t+1} \mid s_t, a_t)}<\dfrac{P(\tilde{s}' \mid \tilde{s}, \tilde{a})}{P(\tilde{s}' \mid s_t, a_t)}$ (from conditions of Case 3).

    \item To ensure that Mon2 \eqref{proofeq:monotonicity2} and CS \eqref{proofeq:counterfactual stability} hold for all other transitions, we must consider two possible cases:

    \begin{itemize}
        \item If $P(s_{t+1} \mid \tilde{s}, \tilde{a}) \geq P(s_{t+1} \mid s_t, a_t)$, then we can assign $\theta$ such that:
        
        \[
            \sum_{s' \in \mathcal{S}\setminus\{\tilde{s}', s_{t+1}\}}\sum_{\substack{u_t \in U_t \\f(s_t, a_t, u_t) = s_{t+1} \\ f(\tilde{s}, \tilde{a}, u_t) = s'}}{\theta_{u_t}} = 0
        \]

        and 

        \[
            \sum_{\substack{u_t \in U_t \\f(s_t, a_t, u_t) = s_{t+1} \\ f(\tilde{s}, \tilde{a}, u_t) = s_{t+1}}}{\theta_{u_t}} = P(s_{t+1} \mid s_t, a_t) - \sum_{\substack{u_t \in U_t \\f(s_t, a_t, u_t) = s_{t+1} \\ f(\tilde{s}, \tilde{a}, u_t) = \tilde{s}'}}{\theta_{u_t}}
        \]

        which must be possible in our given assignment of $\theta$ because from the condition of Case 3(b) we have:

        \begin{align*}
            P(s_{t+1} \mid s_t, a_t) - (P(s_{t+1} \mid s_t, a_t) \cdot P(s_{t+1} \mid \tilde{s}, \tilde{a})) \geq P(\tilde{s}' \mid \tilde{s}, \tilde{a}) \cdot P(s_{t+1} \mid s_t, a_t) \\ \implies P(s_{t+1} \mid s_t, a_t) - ( P(\tilde{s}' \mid \tilde{s}, \tilde{a}) \cdot P(s_{t+1} \mid s_t, a_t))\geq P(s_{t+1} \mid s_t, a_t) \cdot P(s_{t+1} \mid \tilde{s}, \tilde{a})
        \end{align*}
        Because $\sum_{s' \in \mathcal{S}\setminus\{\tilde{s}', s_{t+1}\}}\sum_{\substack{u_t \in U_t \\f(s_t, a_t, u_t) = s_{t+1} \\ f(\tilde{s}, \tilde{a}, u_t) = s'}}{\theta_{u_t}} = 0$, $\forall s' \in \mathcal{S} \setminus \{\tilde{s}', s_{t+1}\}, \sum_{\substack{u_t \in U_t \\f(s_t, a_t, u_t) = s_{t+1} \\ f(\tilde{s}, \tilde{a}, u_t) = s'}}{\theta_{u_t}} \\= 0$, and so $\forall s' \in \mathcal{S} \setminus \{\tilde{s}', s_{t+1}\}, \tilde{P}_t(s' \mid \tilde{s}, \tilde{a}) = 0$. This satisfies the Mon2 \eqref{proofeq:monotonicity2} (since $0 \leq \tilde{P}_t(s' \mid \tilde{s}, \tilde{a})$) and CS \eqref{proofeq:counterfactual stability} constraints for all other transitions $\tilde{s}, \tilde{a} \rightarrow s'$.

        \item Otherwise, we must have $P(s_{t+1} \mid \tilde{s}, \tilde{a}) < P(s_{t+1} \mid s_t, a_t)$. We need to prove that there exists an assignment of $\theta$ where the Mon2 \eqref{proofeq:monotonicity2} and CS \eqref{proofeq:counterfactual stability} constraints are satisfied for all transitions $\tilde{s}, \tilde{a} \rightarrow s', \forall s' \in \mathcal{S}\setminus \{\tilde{s}', s_{t+1}\}$, which equates to proving:

        \[
        \max_{\theta}\left(\sum_{s' \in \mathcal{S}\setminus \{s_{t+1}\}}\sum_{\substack{u_t \in U_t\\f(s_t, a_t, u_t) = s_{t+1} \\ f(\tilde{s}, \tilde{a}, u_t) = s'}}{\theta_{u_t}}\right) \geq \sum_{s' \in \mathcal{S}\setminus\{\tilde{s}', s_{t+1}\}}P(s' \mid \tilde{s}, \tilde{a}) \cdot P(s_{t+1} \mid s_t, a_t)
        \]

        since 
        \[
        \sum_{s' \in \mathcal{S}\setminus \{s_{t+1}\}}\sum_{\substack{u_t \in U_t\\f(s_t, a_t, u_t) = s_{t+1} \\ f(\tilde{s}, \tilde{a}, u_t) = s'}}{\theta_{u_t}} \leq 
        \sum_{s' \in \mathcal{S}\setminus\{\tilde{s}', s_{t+1}\}}P(s' \mid \tilde{s}, \tilde{a}) \cdot P(s_{t+1} \mid s_t, a_t)
        \]
        
        in our given assignment of $\theta$. From Lemma \ref{lemma: existence of theta for overlapping case.}, we have:

        \begin{equation}
        \label{eq: implied max prob1}
            \max_{\theta}\left(\sum_{s' \in \mathcal{S}\setminus \{s_{t+1}\}}\sum_{\substack{u_t \in U_t\\f(s_t, a_t, u_t) = s_{t+1} \\ f(\tilde{s}, \tilde{a}, u_t) = s'}}{\theta_{u_t}}\right) \geq P(s_{t+1} \mid s_t, a_t) - P(s_{t+1} \mid \tilde{s}, \tilde{a})
        \end{equation}

        under the Mon2 \eqref{proofeq:monotonicity2} and CS \eqref{proofeq:counterfactual stability} constraints. Since we have assigned the maximum possible probability to $\sum_{\substack{u_t \in U_t \\f(s_t, a_t, u_t) = s_{t+1} \\ f(\tilde{s}, \tilde{a}, u_t) = \tilde{s}'}}{\theta_{u_t}} = P(\tilde{s}' \mid \tilde{s}, \tilde{a}) \cdot P(s_{t+1} \mid s_t, a_t)$ under the Mon2 \eqref{proofeq:monotonicity2} and CS \eqref{proofeq:counterfactual stability} constraints in our given assignment of $\theta$, \eqref{eq: implied max prob1} implies:

        \[\max \sum_{s' \in \mathcal{S}\setminus \{\tilde{s}', s_{t+1}\}}\sum_{\substack{u_t \in U_t\\f(s_t, a_t, u_t) = s_{t+1} \\ f(\tilde{s}, \tilde{a}, u_t) = s'}}{\theta_{u_t}} \geq P(s_{t+1} \mid s_t, a_t) - P(s_{t+1} \mid \tilde{s}, \tilde{a}) -  P(\tilde{s}' \mid \tilde{s}, \tilde{a}) \cdot P(s_{t+1} \mid s_t, a_t)\]

        Therefore, we know that we can assign $\theta$ such that $\sum_{\substack{u_t \in U_t \\f(s_t, a_t, u_t) = s_{t+1} \\ f(\tilde{s}, \tilde{a}, u_t) = \tilde{s}'}}{\theta_{u_t}} + \sum_{\substack{u_t \in U_t \\f(s_t, a_t, u_t) = s_{t+1} \\ f(\tilde{s}, \tilde{a}, u_t) = s_{t+1}}}{\theta_{u_t}} + \sum_{s' \in \mathcal{S}\setminus\{\tilde{s}', s_{t+1}\}}\sum_{\substack{u_t \in U_t \\f(s_t, a_t, u_t) = s_{t+1} \\ f(\tilde{s}, \tilde{a}, u_t) = s'}}{\theta_{u_t}} = P(s_{t+1} \mid s_t, a_t)$, and all transitions satisfy the Mon2 \eqref{proofeq:monotonicity2} and CS \eqref{proofeq:counterfactual stability} constraints.

    \end{itemize}
    
    \item $0 \leq \theta_{u_t} \leq 1, \forall u_t$, satisfying \eqref{proofeq:valid prob1}.
    
    \item $\sum_{u_t = 1}^{|U_t|} \theta_{u_t} = \sum_{\substack{u_t \in U_t \\f(s_t, a_t, u_t) = s_{t+1} \\ f(\tilde{s}, \tilde{a}, u_t) = \tilde{s}'}} \theta_{u_t} + \sum_{\substack{u_t \in U_t \\f(s_t, a_t, u_t) = s_{t+1} \\ f(\tilde{s}, \tilde{a}, u_t) = s_{t+1}}} \theta_{u_t} +  \sum_{s' \in \mathcal{S}\setminus\{\tilde{s}', s_{t+1}\}}\sum_{\substack{u_t \in U_t \\f(s_t, a_t, u_t) = s_{t+1} \\ f(\tilde{s}, \tilde{a}, u_t) = s'}}{\theta_{u_t}} \\+ \sum_{\substack{u_t \in U_t \\f(s_t, a_t, u_t) \neq s_{t+1} \\ f(\tilde{s}, \tilde{a}, u_t) = \tilde{s}'}} \theta_{u_t} + \sum_{\substack{u_t \in U_t \\f(s_t, a_t, u_t) \neq s_{t+1} \\ f(\tilde{s}, \tilde{a}, u_t) = s_{t+1}}} \theta_{u_t} + \sum_{s' \in \mathcal{S}\setminus\{\tilde{s}', s_{t+1}\}}\sum_{\substack{u_t \in U_t \\f(s_t, a_t, u_t) \neq s_{t+1} \\ f(\tilde{s}, \tilde{a}, u_t) = s'}}{\theta_{u_t}} = 1$, satisfying \eqref{proofeq:valid prob2}.
\end{itemize}

All the constraints are satisfied, so this is a valid assignment of $\theta$ for this state-action pair.
\end{proof}

These two cases can be simplified to:

\begin{align*}
\tilde{P}_{t}^{UB}(\tilde{s}' \mid \tilde{s}, \tilde{a}) &= \dfrac{\sum_{\substack{u_t \in U_t \\f(s_t, a_t, u_t) = s_{t+1} \\ f(\tilde{s}, \tilde{a}, u_t) = \tilde{s}'}}{\theta_{u_t}}}{P(s_{t+1} \mid s_t, a_t)}\\ &=
 \dfrac{\min(P(s_{t+1} \mid s_t, a_t) - (P(s_{t+1} \mid s_t, a_t) \cdot P(s_{t+1} \mid \tilde{s}, \tilde{a})), P(\tilde{s}' \mid \tilde{s}, \tilde{a}) \cdot P(s_{t+1} \mid s_t, a_t))}{P(s_{t+1} \mid s_t, a_t)}\\
 &= \min(1 - P(s_{t+1} \mid \tilde{s}, \tilde{a}), P(\tilde{s}' \mid \tilde{s}, \tilde{a}))
\end{align*}
\end{proof}

\pagebreak
\paragraph{Case 4: $\tilde{s}' \neq s_{t+1}$ and $P(\tilde{s}' \mid s_t, a_t) =0$}

\begin{proof}
    Because $P(\tilde{s}' \mid s_t, a_t) =0$, then the monotonicity and counterfactual stability constraints are vacuously satisfied for $\tilde{s}, \tilde{a} \rightarrow \tilde{s}'$. Therefore, as proven in Lemma \ref{lemma:absolute max cf prob}:

    \begin{equation}
    \label{eq:overlapping constraint 5}
        \sum_{u_t = 1}^{|U_t|} \mu_{\tilde{s}, \tilde{a}, u_t, \tilde{s}'} \cdot \mu_{s_t, a_t, u_t, s_{t+1}} \cdot \theta_{u_t} \leq \min\left(P(\tilde{s}' \mid \tilde{s}, \tilde{a}), P(s_{t+1} \mid s_t, a_t)\right)
    \end{equation}
        Also, $\tilde{P}_t(s_{t+1} \mid \tilde{s}, \tilde{a}) \geq P(s_{t+1} \mid \tilde{s}, \tilde{a})$ because of the monotonicity constraint, therefore:

        \begin{equation}
        \label{eq:overlapping constraint 6}
            \sum_{u_t = 1}^{|U_t|} \mu_{\tilde{s}, \tilde{a}, u_t, s_{t+1}} \cdot \mu_{s_t, a_t, u_t, s_{t+1}} \cdot \theta_{u_t} \geq P(s_{t+1} \mid \tilde{s}, \tilde{a}) \cdot P(s_{t+1} \mid s_t, a_t)
        \end{equation}

        Since $\forall s,s' \in \mathcal{S}, \forall a \in \mathcal{A}, \forall u_t \in U_t, \mu_{s, a, u_t, s'} = 1 \iff f(s, a, u_t) = s'$ by definition of $\mu$, Eq. \eqref{eq:overlapping constraint 5} and Eq. \eqref{eq:overlapping constraint 6} are equivalent to:

\begin{equation}
\label{eq: overlapping constraint 7}
\sum_{\substack{u_t \in U_t \\f(s_t, a_t, u_t) = s_{t+1} \\ f(\tilde{s}, \tilde{a}, u_t) = \tilde{s}'}} \theta_{u_t} \leq \min\left(P(\tilde{s}' \mid \tilde{s}, \tilde{a}), P(s_{t+1} \mid s_t, a_t)\right)     
\end{equation}
and

\begin{equation}
\label{eq: overlapping constraint 8}
\sum_{\substack{u_t \in U_t \\f(s_t, a_t, u_t) = s_{t+1} \\ f(\tilde{s}, \tilde{a}, u_t) = s_{t+1}}} \theta_{u_t}\geq P(s_{t+1} \mid \tilde{s}, \tilde{a}) \cdot P(s_{t+1} \mid s_t, a_t)     
\end{equation}

respectively. As these notations are equivalent, we will use this second notation for brevity and clarity. Consider the following disjoint cases:

\begin{itemize}
    \item $P(s_{t+1} \mid s_t, a_t) - (P(s_{t+1} \mid s_t, a_t) \cdot P(s_{t+1} \mid \tilde{s}, \tilde{a})) < P(\tilde{s}' \mid \tilde{s}, \tilde{a})$
    \item $P(s_{t+1} \mid s_t, a_t) - (P(s_{t+1} \mid s_t, a_t) \cdot P(s_{t+1} \mid \tilde{s}, \tilde{a})) \geq P(\tilde{s}' \mid \tilde{s}, \tilde{a})$
\end{itemize}

\pagebreak
\paragraph{Case 4(a): $P(s_{t+1} \mid s_t, a_t) - (P(s_{t+1} \mid s_t, a_t) \cdot P(s_{t+1} \mid \tilde{s}, \tilde{a})) < P(\tilde{s}' \mid \tilde{s}, \tilde{a})$}
\noindent
\begin{proof}
We can assign $\theta$ as follows:

\[
\begin{cases}
    \sum_{\substack{u_t \in U_t \\f(s_t, a_t, u_t) = s_{t+1} \\ f(\tilde{s}, \tilde{a}, u_t) = s_{t+1}}}{\theta_{u_t}} = P(s_{t+1} \mid s_t, a_t) \cdot P(s_{t+1} \mid \tilde{s}, \tilde{a})\\
    
    \sum_{\substack{u_t \in U_t \\f(s_t, a_t, u_t) = s_{t+1} \\ f(\tilde{s}, \tilde{a}, u_t) = \tilde{s}'}}{\theta_{u_t}} = P(s_{t+1} \mid s_t, a_t) - (P(s_{t+1} \mid s_t, a_t) \cdot P(s_{t+1} \mid \tilde{s}, \tilde{a})) \\
    
    \sum_{s' \in \mathcal{S}\setminus\{\tilde{s}', s_{t+1}\}}\sum_{\substack{u_t \in U_t \\f(s_t, a_t, u_t) = s_{t+1} \\ f(\tilde{s}, \tilde{a}, u_t) = s'}}{\theta_{u_t}} = 0\\
    
    \sum_{\substack{u_t \in U_t \\f(s_t, a_t, u_t) \neq s_{t+1} \\ f(\tilde{s}, \tilde{a}, u_t) = s_{t+1}}}{\theta_{u_t}} = P(s_{t+1} \mid \tilde{s}, \tilde{a}) - \sum_{\substack{u_t \in U_t \\f(s_t, a_t, u_t) = s_{t+1} \\ f(\tilde{s}, \tilde{a}, u_t) = s_{t+1}}}{\theta_{u_t}} \\
    
    \sum_{\substack{u_t \in U_t \\f(s_t, a_t, u_t) \neq s_{t+1} \\ f(\tilde{s}, \tilde{a}, u_t) = \tilde{s}'}}{\theta_{u_t}} = P(\tilde{s}' \mid \tilde{s}, \tilde{a}) - \sum_{\substack{u_t \in U_t \\f(s_t, a_t, u_t) = s_{t+1} \\ f(\tilde{s}, \tilde{a}, u_t) = \tilde{s}'}}{\theta_{u_t}} \\
    
    \sum_{s' \in \mathcal{S}\setminus\{\tilde{s}', s_{t+1}\}}\sum_{\substack{u_t \in U_t \\f(s_t, a_t, u_t) \neq s_{t+1} \\ f(\tilde{s}, \tilde{a}, u_t) = s'}}{\theta_{u_t}} = \sum_{s' \in \mathcal{S}\setminus\{\tilde{s}', s_{t+1}\}}P(\tilde{s}' \mid \tilde{s}, \tilde{a})\\
\end{cases}
\]

    This assignment of $\theta$ satisfies the constraints of the linear optimisation problem, as follows:

    \begin{itemize}        
    \item $\sum_{\substack{u_t \in U_t \\f(s_t, a_t, u_t) = s_{t+1}}} \theta_{u_t} = \sum_{\substack{u_t \in U_t \\f(s_t, a_t, u_t) = s_{t+1} \\ f(\tilde{s}, \tilde{a}, u_t) = \tilde{s}'}} \theta_{u_t} + \sum_{\substack{u_t \in U_t \\f(s_t, a_t, u_t) = s_{t+1} \\ f(\tilde{s}, \tilde{a}, u_t) = s_{t+1}}} \theta_{u_t} + \sum_{s' \in \mathcal{S}\setminus\{\tilde{s}', s_{t+1}\}}\sum_{\substack{u_t \in U_t \\f(s_t, a_t, u_t) = s_{t+1} \\ f(\tilde{s}, \tilde{a}, u_t) = s'}}{\theta_{u_t}} \\ = P(s_{t+1} \mid s_t, a_t)$, satisfying \eqref{proofeq:interventional constraint}.

    \item $\sum_{\substack{u_t \in U_t\\f(s_t, a_t, u_t) \neq s_{t+1}}}{\theta_{u_t}} = \sum_{\substack{u_t \in U_t \\f(s_t, a_t, u_t) \neq s_{t+1} \\ f(\tilde{s}, \tilde{a}, u_t) = \tilde{s}'}} \theta_{u_t} + \sum_{\substack{u_t \in U_t \\f(s_t, a_t, u_t) \neq s_{t+1} \\ f(\tilde{s}, \tilde{a}, u_t) = s_{t+1}}} \theta_{u_t} + \sum_{s' \in \mathcal{S}\setminus\{\tilde{s}', s_{t+1}\}}\sum_{\substack{u_t \in U_t \\f(s_t, a_t, u_t) \neq s_{t+1} \\ f(\tilde{s}, \tilde{a}, u_t) = s'}}{\theta_{u_t}} \\=1 -
    P(s_{t+1} \mid s_t, a_t)$, satisfying \eqref{proofeq:interventional constraint}.
    
    \item $\sum_{\substack{u_t \in U_t \\f(\tilde{s}, \tilde{a}, u_t) = \tilde{s}'}} \theta_{u_t} = \sum_{\substack{u_t \in U_t \\f(s_t, a_t, u_t) = s_{t+1} \\ f(\tilde{s}, \tilde{a}, u_t) = \tilde{s}'}} \theta_{u_t} + \sum_{\substack{u_t \in U_t \\f(s_t, a_t, u_t) \neq s_{t+1} \\ f(\tilde{s}, \tilde{a}, u_t) = \tilde{s}'}} \theta_{u_t} = P(\tilde{s}' \mid \tilde{s}, \tilde{a})$, \\satisfying \eqref{proofeq:interventional constraint}.

    \item $\sum_{\substack{u_t \in U_t \\f(\tilde{s}, \tilde{a}, u_t) = s_{t+1}}} \theta_{u_t} = \sum_{\substack{u_t \in U_t \\f(s_t, a_t, u_t) = s_{t+1} \\ f(\tilde{s}, \tilde{a}, u_t) = s_{t+1}}} \theta_{u_t} + \sum_{\substack{u_t \in U_t \\f(s_t, a_t, u_t) \neq s_{t+1} \\ f(\tilde{s}, \tilde{a}, u_t) = s_{t+1}}} \theta_{u_t} = P(s_{t+1} \mid \tilde{s}, \tilde{a})$, \\satisfying \eqref{proofeq:interventional constraint}.
    
    \item $\sum_{s' \in \mathcal{S}\setminus\{\tilde{s}', s_{t+1}\}}\sum_{\substack{u_t \in U_t \\f(\tilde{s}, \tilde{a}, u_t) = s'}} \theta_{u_t} = \sum_{s' \in \mathcal{S}\setminus\{\tilde{s}', s_{t+1}\}}\sum_{\substack{u_t \in U_t \\f(s_t, a_t, u_t) = s_{t+1} \\ f(\tilde{s}, \tilde{a}, u_t) = s'}} \theta_{u_t} + \sum_{s' \in \mathcal{S}\setminus\{\tilde{s}', s_{t+1}\}}\sum_{\substack{u_t \in U_t \\f(s_t, a_t, u_t) \neq s_{t+1} \\ f(\tilde{s}, \tilde{a}, u_t) = s'}} \theta_{u_t} \\= 1 - P(\tilde{s}' \mid \tilde{s}, \tilde{a}) = \sum_{s' \in \mathcal{S}\setminus\{\tilde{s}'\}}P(s' \mid \tilde{s}, \tilde{a})$, therefore it is possible to assign $\theta$ such that $\forall s' \in \mathcal{S}\setminus\{\tilde{s}', s_{t+1}\}, \sum_{\substack{u_t \in U_t \\ f(\tilde{s}, \tilde{a}, u_t) = s'}}{\theta_{u_t}} = P(s' \mid \tilde{s}, \tilde{a})$, satisfying \eqref{proofeq:interventional constraint}.

    \item \eqref{eq: overlapping constraint 8} holds, therefore Mon1 \eqref{proofeq:monotonicity1} is satisfied.
    
    \item Because $P(\tilde{s}' \mid s_t, a_t) =0$, Mon2 \eqref{proofeq:monotonicity2} holds for the transition $\tilde{s}, \tilde{a} \rightarrow \tilde{s}'$. Also, because $\sum_{s' \in \mathcal{S}\setminus\{\tilde{s}', s_{t+1}\}}\sum_{\substack{u_t \in U_t \\f(s_t, a_t, u_t) = s_{t+1} \\ f(\tilde{s}, \tilde{a}, u_t) = s'}}{\theta_{u_t}} = 0$, the counterfactual probability of all other transitions $\tilde{s}, \tilde{a} \rightarrow s'$ where $s' \in \mathcal{S}\setminus\{\tilde{s}', s_{t+1}\}$ will be $\tilde{P}_t(s' \mid \tilde{s}, \tilde{a}) = 0$. Therefore, Mon2 \eqref{proofeq:monotonicity2} is satisfied.

    \item CS \eqref{proofeq:counterfactual stability} doesn't apply to the transition $\tilde{s}, \tilde{a} \rightarrow s_{t+1}$, and is vacuously satisfied for the transition $\tilde{s}, \tilde{a} \rightarrow \tilde{s}'$  because $P(\tilde{s}' \mid s_t, a_t) =0$. Also, because \\$\sum_{s' \in \mathcal{S}\setminus\{\tilde{s}', s_{t+1}\}}\sum_{\substack{u_t \in U_t \\f(s_t, a_t, u_t) = s_{t+1} \\ f(\tilde{s}, \tilde{a}, u_t) = s'}}{\theta_{u_t}} = 0$, the counterfactual probability of all other transitions $\tilde{s}, \tilde{a} \rightarrow s'$ where $s' \in \mathcal{S}\setminus\{\tilde{s}', s_{t+1}\}$ will be $\tilde{P}_t(s' \mid \tilde{s}, \tilde{a}) = 0$, which satisfies CS \eqref{proofeq:counterfactual stability} for all other transitions.

    \item $0 \leq \theta_{u_t} \leq 1, \forall u_t$, satisfying \eqref{proofeq:valid prob1}.
    
    \item $\sum_{u_t = 1}^{|U_t|} \theta_{u_t} = \sum_{\substack{u_t \in U_t \\f(s_t, a_t, u_t) = s_{t+1} \\ f(\tilde{s}, \tilde{a}, u_t) = \tilde{s}'}} \theta_{u_t} + \sum_{\substack{u_t \in U_t \\f(s_t, a_t, u_t) = s_{t+1} \\ f(\tilde{s}, \tilde{a}, u_t) = s_{t+1}}} \theta_{u_t} +  \sum_{s' \in \mathcal{S}\setminus\{\tilde{s}', s_{t+1}\}}\sum_{\substack{u_t \in U_t \\f(s_t, a_t, u_t) = s_{t+1} \\ f(\tilde{s}, \tilde{a}, u_t) = s'}}{\theta_{u_t}} \\+ \sum_{\substack{u_t \in U_t \\f(s_t, a_t, u_t) \neq s_{t+1} \\ f(\tilde{s}, \tilde{a}, u_t) = \tilde{s}'}} \theta_{u_t} + 
    \sum_{\substack{u_t \in U_t \\f(s_t, a_t, u_t) \neq s_{t+1} \\ f(\tilde{s}, \tilde{a}, u_t) = s_{t+1}}} \theta_{u_t} + \sum_{s' \in \mathcal{S}\setminus\{\tilde{s}', s_{t+1}\}}\sum_{\substack{u_t \in U_t \\f(s_t, a_t, u_t) \neq s_{t+1} \\ f(\tilde{s}, \tilde{a}, u_t) = s'}}{\theta_{u_t}} = 1$, satisfying \eqref{proofeq:valid prob2}.
\end{itemize}

All the constraints are satisfied, so this is a valid assignment of $\theta$ for this state-action pair.
\end{proof}

\paragraph{Case 4(b): $P(s_{t+1} \mid s_t, a_t) - (P(s_{t+1} \mid s_t, a_t) \cdot P(s_{t+1} \mid \tilde{s}, \tilde{a})) \geq P(\tilde{s}' \mid \tilde{s}, \tilde{a})$}
\noindent
\begin{proof}
We can assign $\theta$ as follows:

 \[
    \begin{cases}
        P(s_{t+1} \mid s_t, a_t) \cdot P(s_{t+1} \mid \tilde{s}, \tilde{a}) \leq \sum_{\substack{u_t \in U_t \\f(s_t, a_t, u_t) = s_{t+1} \\ f(\tilde{s}, \tilde{a}, u_t) = s_{t+1}}}{\theta_{u_t}} \leq P(s_{t+1} \mid \tilde{s}, \tilde{a}) \\
        \sum_{\substack{u_t \in U_t \\f(s_t, a_t, u_t) = s_{t+1} \\ f(\tilde{s}, \tilde{a}, u_t) = \tilde{s}'}}{\theta_{u_t}} = P(\tilde{s}' \mid \tilde{s}, \tilde{a})\\
        0 \leq \sum_{s' \in \mathcal{S}\setminus\{\tilde{s}', s_{t+1}\}}\sum_{\substack{u_t \in U_t \\f(s_t, a_t, u_t) = s_{t+1} \\ f(\tilde{s}, \tilde{a}, u_t) = s'}}{\theta_{u_t}} \leq \sum_{s' \in \mathcal{S}\setminus\{\tilde{s}', s_{t+1}\}}P(s' \mid \tilde{s}, \tilde{a}) \cdot P(s_{t+1} \mid s_t, a_t)\\
        \sum_{\substack{u_t \in U_t \\f(s_t, a_t, u_t) \neq s_{t+1} \\ f(\tilde{s}, \tilde{a}, u_t) = s_{t+1}}}{\theta_{u_t}} = P(s_{t+1} \mid \tilde{s}, \tilde{a}) - \sum_{\substack{u_t \in U_t \\f(s_t, a_t, u_t) = s_{t+1} \\ f(\tilde{s}, \tilde{a}, u_t) = s_{t+1}}}{\theta_{u_t}} \\
        \sum_{\substack{u_t \in U_t \\f(s_t, a_t, u_t) \neq s_{t+1} \\ f(\tilde{s}, \tilde{a}, u_t) = \tilde{s}'}}{\theta_{u_t}} = 0\\
        \sum_{s' \in \mathcal{S}\setminus\{\tilde{s}', s_{t+1}\}}\sum_{\substack{u_t \in U_t \\f(s_t, a_t, u_t) \neq s_{t+1} \\ f(\tilde{s}, \tilde{a}, u_t) = s'}}{\theta_{u_t}} = \sum_{s' \in \mathcal{S}\setminus\{\tilde{s}', s_{t+1}\}}P(\tilde{s}' \mid \tilde{s}, \tilde{a}) - \sum_{s' \in \mathcal{S}\setminus\{\tilde{s}', s_{t+1}\}}\sum_{\substack{u_t \in U_t \\f(s_t, a_t, u_t) = s_{t+1} \\ f(\tilde{s}, \tilde{a}, u_t) = s'}}{\theta_{u_t}}\\
    \end{cases}
\]

    s.t. $\sum_{\substack{u_t \in U_t \\f(s_t, a_t, u_t) = s_{t+1} \\ f(\tilde{s}, \tilde{a}, u_t) = \tilde{s}'}}{\theta_{u_t}} + \sum_{\substack{u_t \in U_t \\f(s_t, a_t, u_t) = s_{t+1} \\ f(\tilde{s}, \tilde{a}, u_t) = s_{t+1}}}{\theta_{u_t}} + \sum_{s' \in \mathcal{S}\setminus\{\tilde{s}', s_{t+1}\}}\sum_{\substack{u_t \in U_t \\f(s_t, a_t, u_t) = s_{t+1} \\ f(\tilde{s}, \tilde{a}, u_t) = s'}}{\theta_{u_t}} \\= P(s_{t+1} \mid s_t, a_t)$.\\\\
    
    We do not know exactly how much probability will be assigned to $\sum_{\substack{u_t \in U_t \\f(s_t, a_t, u_t) = s_{t+1} \\ f(\tilde{s}, \tilde{a}, u_t) = s_{t+1}}}{\theta_{u_t}}$ and $\sum_{s' \in \mathcal{S} \setminus \{\tilde{s}', s_{t+1}\}}\sum_{\substack{u_t \in U_t \\f(s_t, a_t, u_t) = s_{t+1} \\ f(\tilde{s}, \tilde{a}, u_t)= s'}}{\theta_{u_t}}$. However, if $\theta$ satisfies the above assignment then $\theta$ will always satisfy the constraints of the optimisation problem, as follows:

    \begin{itemize}        
    \item $\sum_{\substack{u_t \in U_t \\f(s_t, a_t, u_t) = s_{t+1}}} \theta_{u_t} = \sum_{\substack{u_t \in U_t \\f(s_t, a_t, u_t) = s_{t+1} \\ f(\tilde{s}, \tilde{a}, u_t) = \tilde{s}'}} \theta_{u_t} + \sum_{\substack{u_t \in U_t \\f(s_t, a_t, u_t) = s_{t+1} \\ f(\tilde{s}, \tilde{a}, u_t) = s_{t+1}}} \theta_{u_t} + \sum_{s' \in \mathcal{S}\setminus\{\tilde{s}', s_{t+1}\}}\sum_{\substack{u_t \in U_t \\f(s_t, a_t, u_t) = s_{t+1} \\ f(\tilde{s}, \tilde{a}, u_t) = s'}}{\theta_{u_t}} \\ = P(s_{t+1} \mid s_t, a_t)$, satisfying \eqref{proofeq:interventional constraint}.

    \item $\sum_{\substack{u_t \in U_t\\f(s_t, a_t, u_t) \neq s_{t+1}}}{\theta_{u_t}} = \sum_{\substack{u_t \in U_t \\f(s_t, a_t, u_t) \neq s_{t+1} \\ f(\tilde{s}, \tilde{a}, u_t) = \tilde{s}'}} \theta_{u_t} + \sum_{\substack{u_t \in U_t \\f(s_t, a_t, u_t) \neq s_{t+1} \\ f(\tilde{s}, \tilde{a}, u_t) = s_{t+1}}} \theta_{u_t} + \sum_{s' \in \mathcal{S}\setminus\{\tilde{s}', s_{t+1}\}}\sum_{\substack{u_t \in U_t \\f(s_t, a_t, u_t) \neq s_{t+1} \\ f(\tilde{s}, \tilde{a}, u_t) = s'}}{\theta_{u_t}} \\=1 -
    P(s_{t+1} \mid s_t, a_t)$, satisfying \eqref{proofeq:interventional constraint}.
    
    \item $\sum_{\substack{u_t \in U_t \\f(\tilde{s}, \tilde{a}, u_t) = \tilde{s}'}} \theta_{u_t} = \sum_{\substack{u_t \in U_t \\f(s_t, a_t, u_t) = s_{t+1} \\ f(\tilde{s}, \tilde{a}, u_t) = \tilde{s}'}} \theta_{u_t} + \sum_{\substack{u_t \in U_t \\f(s_t, a_t, u_t) \neq s_{t+1} \\ f(\tilde{s}, \tilde{a}, u_t) = \tilde{s}'}} \theta_{u_t} = P(\tilde{s}' \mid \tilde{s}, \tilde{a})$, \\satisfying \eqref{proofeq:interventional constraint}.

    \item $\sum_{\substack{u_t \in U_t \\f(\tilde{s}, \tilde{a}, u_t) = s_{t+1}}} \theta_{u_t} = \sum_{\substack{u_t \in U_t \\f(s_t, a_t, u_t) = s_{t+1} \\ f(\tilde{s}, \tilde{a}, u_t) = s_{t+1}}} \theta_{u_t} + \sum_{\substack{u_t \in U_t \\f(s_t, a_t, u_t) \neq s_{t+1} \\ f(\tilde{s}, \tilde{a}, u_t) = s_{t+1}}} \theta_{u_t} = P(s_{t+1} \mid \tilde{s}, \tilde{a})$, \\satisfying \eqref{proofeq:interventional constraint}.
    
    \item $\sum_{s' \in \mathcal{S}\setminus\{\tilde{s}', s_{t+1}\}}\sum_{\substack{u_t \in U_t \\f(\tilde{s}, \tilde{a}, u_t) = s'}} \theta_{u_t} = \sum_{s' \in \mathcal{S}\setminus\{\tilde{s}', s_{t+1}\}}\sum_{\substack{u_t \in U_t \\f(s_t, a_t, u_t) = s_{t+1} \\ f(\tilde{s}, \tilde{a}, u_t) = s'}} \theta_{u_t} + \sum_{s' \in \mathcal{S}\setminus\{\tilde{s}', s_{t+1}\}}\sum_{\substack{u_t \in U_t \\f(s_t, a_t, u_t) \neq s_{t+1} \\ f(\tilde{s}, \tilde{a}, u_t) = s'}} \theta_{u_t} \\= 1 - P(\tilde{s}' \mid \tilde{s}, \tilde{a}) = \sum_{s' \in \mathcal{S}\setminus\{\tilde{s}'\}}P(s' \mid \tilde{s}, \tilde{a})$, therefore it is possible to assign $\theta$ such that $\forall s' \in \mathcal{S}\setminus\{\tilde{s}', s_{t+1}\}, \sum_{\substack{u_t \in U_t \\ f(\tilde{s}, \tilde{a}, u_t) = s'}}{\theta_{u_t}} = P(s' \mid \tilde{s}, \tilde{a})$, satisfying \eqref{proofeq:interventional constraint}.
    
    \item \eqref{eq: overlapping constraint 8} holds, therefore Mon1 \eqref{proofeq:monotonicity1} is satisfied.
    
    \item Because $P(\tilde{s}' \mid s_t, a_t) = 0$, Mon2 \eqref{proofeq:monotonicity2} holds for $\tilde{s}, \tilde{a} \rightarrow \tilde{s}'$.

    \item CS \eqref{proofeq:counterfactual stability} doesn't apply to the transition $\tilde{s}, \tilde{a} \rightarrow s_{t+1}$, and is vacuously satisfied for the transition $\tilde{s}, \tilde{a} \rightarrow \tilde{s}'$ because $P(\tilde{s}' \mid s_t, a_t) =0$.

    \item To ensure that Mon2 \eqref{proofeq:monotonicity2} and CS \eqref{proofeq:counterfactual stability} hold for all other transitions, we must consider two possible cases:

    \begin{itemize}
        \item If $P(s_{t+1} \mid \tilde{s}, \tilde{a}) \geq P(s_{t+1} \mid s_t, a_t)$, then we can assign $\theta$ such that:
        
        \[
            \sum_{s' \in \mathcal{S}\setminus\{\tilde{s}', s_{t+1}\}}\sum_{\substack{u_t \in U_t \\f(s_t, a_t, u_t) = s_{t+1} \\ f(\tilde{s}, \tilde{a}, u_t) = s'}}{\theta_{u_t}} = 0
        \]

        and 

        \[
            \sum_{\substack{u_t \in U_t \\f(s_t, a_t, u_t) = s_{t+1} \\ f(\tilde{s}, \tilde{a}, u_t) = s_{t+1}}}{\theta_{u_t}} = P(s_{t+1} \mid s_t, a_t) - \sum_{\substack{u_t \in U_t \\f(s_t, a_t, u_t) = s_{t+1} \\ f(\tilde{s}, \tilde{a}, u_t) = \tilde{s}'}}{\theta_{u_t}}
        \]

        which must be possible in our given assignment of $\theta$ because, from the condition of Case 4(b), we have:

        \begin{align*}
            P(s_{t+1} \mid s_t, a_t) - (P(s_{t+1} \mid s_t, a_t) \cdot P(s_{t+1} \mid \tilde{s}, \tilde{a})) \geq P(\tilde{s}' \mid \tilde{s}, \tilde{a}) \\ \implies P(s_{t+1} \mid s_t, a_t) - P(\tilde{s}' \mid \tilde{s}, \tilde{a}) \geq P(s_{t+1} \mid s_t, a_t) \cdot P(s_{t+1} \mid \tilde{s}, \tilde{a})
        \end{align*}
        Because $\sum_{s' \in \mathcal{S}\setminus\{\tilde{s}', s_{t+1}\}}\sum_{\substack{u_t \in U_t \\f(s_t, a_t, u_t) = s_{t+1} \\ f(\tilde{s}, \tilde{a}, u_t) = s'}}{\theta_{u_t}} = 0$, $\forall s' \in \mathcal{S} \setminus \{\tilde{s}', s_{t+1}\}, \sum_{\substack{u_t \in U_t \\f(s_t, a_t, u_t) = s_{t+1} \\ f(\tilde{s}, \tilde{a}, u_t) = s'}}{\theta_{u_t}} = 0$, and so $\forall s' \in \mathcal{S} \setminus \{\tilde{s}', s_{t+1}\}, \tilde{P}_t(s' \mid \tilde{s}, \tilde{a}) = 0$. This satisfies the Mon2 \eqref{proofeq:monotonicity2} (since $0 \leq \tilde{P}_t(s' \mid \tilde{s}, \tilde{a})$) and CS \eqref{proofeq:counterfactual stability} constraints for all other transitions $\tilde{s}, \tilde{a} \rightarrow s'$.

        \item Otherwise, we must have $P(s_{t+1} \mid \tilde{s}, \tilde{a}) < P(s_{t+1} \mid s_t, a_t)$. We need to prove that there exists an assignment of $\theta$ where the Mon2 \eqref{proofeq:monotonicity2} and CS \eqref{proofeq:counterfactual stability} constraints are satisfied for all transitions $\tilde{s}, \tilde{a} \rightarrow s', \forall s' \in \mathcal{S}\setminus \{\tilde{s}', s_{t+1}\}$, which equates to proving:

        \[
        \max_{\theta} \left(\sum_{s' \in \mathcal{S}\setminus \{s_{t+1}\}}\sum_{\substack{u_t \in U_t\\f(s_t, a_t, u_t) = s_{t+1} \\ f(\tilde{s}, \tilde{a}, u_t) = s'}}{\theta_{u_t}}\right) \geq \sum_{s' \in \mathcal{S}\setminus\{\tilde{s}', s_{t+1}\}}P(s' \mid \tilde{s}, \tilde{a}) \cdot P(s_{t+1} \mid s_t, a_t)
        \]

        since 
        \[
        \sum_{s' \in \mathcal{S}\setminus \{s_{t+1}\}}\sum_{\substack{u_t \in U_t\\f(s_t, a_t, u_t) = s_{t+1} \\ f(\tilde{s}, \tilde{a}, u_t) = s'}}{\theta_{u_t}} \leq 
        \sum_{s' \in \mathcal{S}\setminus\{\tilde{s}', s_{t+1}\}}P(s' \mid \tilde{s}, \tilde{a}) \cdot P(s_{t+1} \mid s_t, a_t)
        \]
        
        in our given assignment of $\theta$. From Lemma \ref{lemma: existence of theta for overlapping case.}, we have:

        \begin{equation}
        \label{eq: implied max prob}
            \max_{\theta}\left(\sum_{s' \in \mathcal{S}\setminus \{s_{t+1}\}}\sum_{\substack{u_t \in U_t\\f(s_t, a_t, u_t) = s_{t+1} \\ f(\tilde{s}, \tilde{a}, u_t) = s'}}{\theta_{u_t}}\right) \geq P(s_{t+1} \mid s_t, a_t) - P(s_{t+1} \mid \tilde{s}, \tilde{a})
        \end{equation}

        under the Mon2 \eqref{proofeq:monotonicity2} and CS \eqref{proofeq:counterfactual stability} constraints. Since we have assigned the maximum possible probability to $\sum_{\substack{u_t \in U_t \\f(s_t, a_t, u_t) = s_{t+1} \\ f(\tilde{s}, \tilde{a}, u_t) = \tilde{s}'}}{\theta_{u_t}} = P(\tilde{s}' \mid \tilde{s}, \tilde{a})$ under the Mon2 \eqref{proofeq:monotonicity2} and CS \eqref{proofeq:counterfactual stability} constraints in our given assignment of $\theta$, \eqref{eq: implied max prob} implies:

        \[\max_{\theta}\left(\sum_{s' \in \mathcal{S}\setminus \{\tilde{s}', s_{t+1}\}}\sum_{\substack{u_t \in U_t\\f(s_t, a_t, u_t) = s_{t+1} \\ f(\tilde{s}, \tilde{a}, u_t) = s'}}{\theta_{u_t}}\right) \geq P(s_{t+1} \mid s_t, a_t) - P(s_{t+1} \mid \tilde{s}, \tilde{a}) -  P(\tilde{s}' \mid \tilde{s}, \tilde{a})\]

        Therefore, we know that we can assign $\theta$ such that $\sum_{\substack{u_t \in U_t \\f(s_t, a_t, u_t) = s_{t+1} \\ f(\tilde{s}, \tilde{a}, u_t) = \tilde{s}'}}{\theta_{u_t}} + \sum_{\substack{u_t \in U_t \\f(s_t, a_t, u_t) = s_{t+1} \\ f(\tilde{s}, \tilde{a}, u_t) = s_{t+1}}}{\theta_{u_t}} \\+ \sum_{s' \in \mathcal{S}\setminus\{\tilde{s}', s_{t+1}\}}\sum_{\substack{u_t \in U_t \\f(s_t, a_t, u_t) = s_{t+1} \\ f(\tilde{s}, \tilde{a}, u_t) = s'}}{\theta_{u_t}} = P(s_{t+1} \mid s_t, a_t)$, and all transitions satisfy the Mon2 \eqref{proofeq:monotonicity2} and CS \eqref{proofeq:counterfactual stability} constraints.

        \end{itemize}
    
    \item $0 \leq \theta_{u_t} \leq 1, \forall u_t$, satisfying \eqref{proofeq:valid prob1}.
    
    \item $\sum_{u_t = 1}^{|U_t|} \theta_{u_t} = \sum_{\substack{u_t \in U_t \\f(s_t, a_t, u_t) = s_{t+1} \\ f(\tilde{s}, \tilde{a}, u_t) = \tilde{s}'}} \theta_{u_t} + \sum_{\substack{u_t \in U_t \\f(s_t, a_t, u_t) = s_{t+1} \\ f(\tilde{s}, \tilde{a}, u_t) = s_{t+1}}} \theta_{u_t} +  \sum_{s' \in \mathcal{S}\setminus\{\tilde{s}', s_{t+1}\}}\sum_{\substack{u_t \in U_t \\f(s_t, a_t, u_t) = s_{t+1} \\ f(\tilde{s}, \tilde{a}, u_t) = s'}}{\theta_{u_t}} \\+ \sum_{\substack{u_t \in U_t \\f(s_t, a_t, u_t) \neq s_{t+1} \\ f(\tilde{s}, \tilde{a}, u_t) = \tilde{s}'}} \theta_{u_t} + \sum_{\substack{u_t \in U_t \\f(s_t, a_t, u_t) \neq s_{t+1} \\ f(\tilde{s}, \tilde{a}, u_t) = s_{t+1}}} \theta_{u_t} + \sum_{s' \in \mathcal{S}\setminus\{\tilde{s}', s_{t+1}\}}\sum_{\substack{u_t \in U_t \\f(s_t, a_t, u_t) \neq s_{t+1} \\ f(\tilde{s}, \tilde{a}, u_t) = s'}}{\theta_{u_t}} = 1$, satisfying \eqref{proofeq:valid prob2}.
\end{itemize}

All the constraints are satisfied, so this is a valid assignment of $\theta$ for this state-action pair.
\end{proof}    

Therefore, these two cases can be simplified to:

\begin{align*}
\tilde{P}_{t}^{UB}(\tilde{s}' \mid \tilde{s}, \tilde{a}) &= \dfrac{\sum_{\substack{u_t \in U_t \\f(s_t, a_t, u_t) = s_{t+1} \\ f(\tilde{s}, \tilde{a}, u_t) = \tilde{s}'}}{\theta_{u_t}}}{P(s_{t+1} \mid s_t, a_t)}\\ &=
 \dfrac{\min\left(P(s_{t+1} \mid s_t, a_t) - (P(s_{t+1} \mid s_t, a_t) \cdot P(s_{t+1} \mid \tilde{s}, \tilde{a})), P(\tilde{s}' \mid \tilde{s}, \tilde{a})\right)}{P(s_{t+1} \mid s_t, a_t)}\\
 &= \min\left(1 - P(s_{t+1} \mid \tilde{s}, \tilde{a}), \dfrac{P(\tilde{s}' \mid \tilde{s}, \tilde{a})}{P(s_{t+1} \mid s_t, a_t)}\right)
\end{align*}

\end{proof}

We have proven all the cases for Theorem \ref{proof theorem:ub overlapping}, therefore Theorem \ref{proof theorem:ub overlapping} holds.

\end{proof}

\pagebreak
\begin{theorem}
\label{proof theorem:lb overlapping}
For outgoing transitions from state-action pairs $(\tilde{s}, \tilde{a})$ which have overlapping support with the observed $(s_t, a_t)$, but $(\tilde{s}, \tilde{a}) \neq (s_t, a_t)$, the linear program will produce a lower bound of:

\[\tilde{P}_{t}^{LB}(\tilde{s}' \mid \tilde{s}, \tilde{a}) = 
\begin{cases}
\max(P(\tilde{s}' \mid \tilde{s}, \tilde{a}), 1 - \sum_{s' \in \mathcal{S}\setminus\{s_{t+1}\}}{\tilde{P}_{t}^{UB}(s' \mid \tilde{s}, \tilde{a})}) & \text{if $\tilde{s}' = s_{t+1}$}\\
0 & \text{if $P(\tilde{s}' \mid s_t, a_t) > 0$} \\ & \text{and $\dfrac{P(s_{t+1} \mid \tilde{s}, \tilde{a})}{P(s_{t+1} \mid s_t, a_t)}\geq\dfrac{P(\tilde{s}' \mid \tilde{s}, \tilde{a})}{P(\tilde{s}' \mid s_t, a_t)}$}\\
\max(0, 1 - \sum_{s' \in \mathcal{S}\setminus\{\tilde{s}'\}}{\tilde{P}_{t}^{UB}(s' \mid \tilde{s}, \tilde{a})}) & \text{otherwise}\\
\end{cases}
\]
\end{theorem}

\begin{proof}
Take an arbitrary transition $\tilde{s}, \tilde{a} \rightarrow \tilde{s}'$ where $(\tilde{s}, \tilde{a})$ has overlapping support with the observed state-action pair $(s_t, a_t)$. Consider the following disjoint cases:

\begin{itemize}
    \item $\tilde{s}' = s_{t+1}$
    \item $\tilde{s}' \neq s_{t+1}$, $\dfrac{P(s_{t+1} \mid \tilde{s}, \tilde{a})}{P(s_{t+1} \mid s_t, a_t)}\geq\dfrac{P(\tilde{s}' \mid \tilde{s}, \tilde{a})}{P(\tilde{s}' \mid s_t, a_t)}$ and $P(\tilde{s}' \mid s_t, a_t) > 0$
    \item $\tilde{s}' \neq s_{t+1}$, and $\dfrac{P(s_{t+1} \mid \tilde{s}, \tilde{a})}{P(s_{t+1} \mid s_t, a_t)}<\dfrac{P(\tilde{s}' \mid \tilde{s}, \tilde{a})}{P(\tilde{s}' \mid s_t, a_t)}$ or $P(\tilde{s}' \mid s_t, a_t) = 0$
\end{itemize}

These cases are disjoint and cover all possible situations.

\paragraph{Case 1: $\tilde{s}' = s_{t+1}$}
\noindent
\begin{proof}
    To satisfy Mon1 \eqref{proofeq:monotonicity1}, $\tilde{P}_t^{LB}(s_{t+1} \mid \tilde{s}, \tilde{a}) \geq P(s_{t+1} \mid \tilde{s}, \tilde{a})$. To do this, we must assign $\theta$ such that:

    \begin{equation}
        \label{eq:overlapping constraint 9}
        \sum_{\substack{u_t \in U_t\\f(s_t, a_t, u_t) = s_{t+1}\\f(\tilde{s}, \tilde{a}, u_t) = s_{t+1}}}\theta_{u_t}\geq P(s_{t+1} \mid s_t, a_t) \cdot P(s_{t+1} \mid \tilde{s}, \tilde{a})
    \end{equation}
    
    Consider the following disjoint cases:

    \begin{enumerate}
        \item $P(s_{t+1} \mid \tilde{s}, \tilde{a}) \geq 1 - \sum_{s' \in \mathcal{S}\setminus\{s_{t+1}\}}{\tilde{P}_{t}^{UB}(s' \mid \tilde{s}, \tilde{a})}$
        \item $P(s_{t+1} \mid \tilde{s}, \tilde{a}) < 1 - \sum_{s' \in \mathcal{S}\setminus\{s_{t+1}\}}{\tilde{P}_{t}^{UB}(s' \mid \tilde{s}, \tilde{a})}$
    \end{enumerate}

    \paragraph{Case 1(a): $P(s_{t+1} \mid \tilde{s}, \tilde{a}) \geq 1 - \sum_{s' \in \mathcal{S}\setminus\{s_{t+1}\}}{\tilde{P}_{t}^{UB}(s' \mid \tilde{s}, \tilde{a})}$}
\noindent
    \begin{proof}

    In this case, $P(s_{t+1} \mid \tilde{s}, \tilde{a}) \geq 1 - \sum_{s' \in \mathcal{S}\setminus\{s_{t+1}\}}{\tilde{P}_{t}^{UB}(s' \mid \tilde{s}, \tilde{a})}$, therefore:

    \begin{equation}
        \label{eq: lbcase1a}
        P(s_{t+1} \mid s_t, a_t) \cdot P(s_{t+1} \mid \tilde{s}, \tilde{a}) \geq P(s_{t+1} \mid s_t, a_t) - P(s_{t+1} \mid s_t, a_t) \cdot \sum_{s' \in \mathcal{S}\setminus\{s_{t+1}\}}{\tilde{P}_{t}^{UB}(s' \mid \tilde{s}, \tilde{a})}
    \end{equation}
    
    Therefore, we can assign $\theta$ as follows:
    
    \[
    \begin{cases}
        \sum_{\substack{u_t \in U_t \\f(s_t, a_t, u_t) = s_{t+1} \\ f(\tilde{s}, \tilde{a}, u_t) = s_{t+1}}}{\theta_{u_t}} = P(s_{t+1} \mid s_t, a_t) \cdot P(s_{t+1} \mid \tilde{s}, \tilde{a})\\
        \sum_{s' \in \mathcal{S}\setminus\{s_{t+1}\}}\sum_{\substack{u_t \in U_t \\f(s_t, a_t, u_t) = s_{t+1} \\ f(\tilde{s}, \tilde{a}, u_t) = s'}}{\theta_{u_t}} = P(s_{t+1} \mid s_t, a_t) - \sum_{\substack{u_t \in U_t \\f(s_t, a_t, u_t) = s_{t+1} \\ f(\tilde{s}, \tilde{a}, u_t) = s_{t+1}}}{\theta_{u_t}}\\
        \sum_{\substack{u_t \in U_t \\f(s_t, a_t, u_t) \neq s_{t+1} \\ f(\tilde{s}, \tilde{a}, u_t) = s_{t+1}}}{\theta_{u_t}} = P(s_{t+1} \mid \tilde{s}, \tilde{a}) - \sum_{\substack{u_t \in U_t \\f(s_t, a_t, u_t) = s_{t+1} \\ f(\tilde{s}, \tilde{a}, u_t) = s_{t+1}}}{\theta_{u_t}} \\
        \sum_{s' \in \mathcal{S}\setminus\{ s_{t+1}\}}\sum_{\substack{u_t \in U_t \\f(s_t, a_t, u_t) \neq s_{t+1} \\ f(\tilde{s}, \tilde{a}, u_t) = s'}}{\theta_{u_t}} = \sum_{s' \in \mathcal{S}\setminus\{ s_{t+1}\}}P(\tilde{s}' \mid \tilde{s}, \tilde{a}) - \sum_{s' \in \mathcal{S}\setminus\{s_{t+1}\}}\sum_{\substack{u_t \in U_t \\f(s_t, a_t, u_t) = s_{t+1} \\ f(\tilde{s}, \tilde{a}, u_t) = s'}}{\theta_{u_t}}\\
    \end{cases}
    \]
   
    This assignment of $\theta$ satisfies the constraints of the linear optimisation problem, as follows:

    \begin{itemize}        
    \item $\sum_{\substack{u_t \in U_t \\f(s_t, a_t, u_t) = s_{t+1}}} \theta_{u_t} = \sum_{\substack{u_t \in U_t \\f(s_t, a_t, u_t) = s_{t+1} \\ f(\tilde{s}, \tilde{a}, u_t) = s_{t+1}}} \theta_{u_t} + \sum_{s' \in \mathcal{S}\setminus\{ s_{t+1}\}}\sum_{\substack{u_t \in U_t \\f(s_t, a_t, u_t) = s_{t+1} \\ f(\tilde{s}, \tilde{a}, u_t) = s'}}{\theta_{u_t}} \\ = P(s_{t+1} \mid s_t, a_t)$, satisfying \eqref{proofeq:interventional constraint}.

    \item $\sum_{\substack{u_t \in U_t\\f(s_t, a_t, u_t) \neq s_{t+1}}}{\theta_{u_t}} = \sum_{\substack{u_t \in U_t \\f(s_t, a_t, u_t) \neq s_{t+1} \\ f(\tilde{s}, \tilde{a}, u_t) = s_{t+1}}} \theta_{u_t} + \sum_{s' \in \mathcal{S}\setminus\{ s_{t+1}\}}\sum_{\substack{u_t \in U_t \\f(s_t, a_t, u_t) \neq s_{t+1} \\ f(\tilde{s}, \tilde{a}, u_t) = s'}}{\theta_{u_t}} \\=1 -
    P(s_{t+1} \mid s_t, a_t)$, satisfying \eqref{proofeq:interventional constraint}.

    \item $\sum_{\substack{u_t \in U_t \\f(\tilde{s}, \tilde{a}, u_t) = s_{t+1}}} \theta_{u_t} = \sum_{\substack{u_t \in U_t \\f(s_t, a_t, u_t) = s_{t+1} \\ f(\tilde{s}, \tilde{a}, u_t) = s_{t+1}}} \theta_{u_t} + \sum_{\substack{u_t \in U_t \\f(s_t, a_t, u_t) \neq s_{t+1} \\ f(\tilde{s}, \tilde{a}, u_t) = s_{t+1}}} \theta_{u_t} = P(s_{t+1} \mid \tilde{s}, \tilde{a})$, \\satisfying \eqref{proofeq:interventional constraint}.
    
    \item $\sum_{s' \in \mathcal{S}\setminus\{s_{t+1}\}}\sum_{\substack{u_t \in U_t \\f(\tilde{s}, \tilde{a}, u_t) = s'}} \theta_{u_t} = \sum_{s' \in \mathcal{S}\setminus\{ s_{t+1}\}}\sum_{\substack{u_t \in U_t \\f(s_t, a_t, u_t) = s_{t+1} \\ f(\tilde{s}, \tilde{a}, u_t) = s'}} \theta_{u_t} + \sum_{s' \in \mathcal{S}\setminus\{s_{t+1}\}}\sum_{\substack{u_t \in U_t \\f(s_t, a_t, u_t) \neq s_{t+1} \\ f(\tilde{s}, \tilde{a}, u_t) = s'}} \theta_{u_t} \\= 1 - P(\tilde{s}' \mid \tilde{s}, \tilde{a}) = \sum_{s' \in \mathcal{S}\setminus\{\tilde{s}'\}}P(s' \mid \tilde{s}, \tilde{a})$, therefore it is possible to assign $\theta$ such that $\forall s' \in \mathcal{S}\setminus\{s_{t+1}\}, \sum_{\substack{u_t \in U_t \\ f(\tilde{s}, \tilde{a}, u_t) = s'}}{\theta_{u_t}} = P(s' \mid \tilde{s}, \tilde{a})$, satisfying \eqref{proofeq:interventional constraint}.
    
    \item Eq. \eqref{eq:overlapping constraint 9} holds, so Mon1 \eqref{proofeq:monotonicity1} is satisfied. 
    
    \item CS \eqref{proofeq:counterfactual stability} doesn't apply to the transition $\tilde{s}, \tilde{a} \rightarrow s_{t+1}$.
    
    \item From \eqref{eq: lbcase1a}, we get:

    \[
    P(s_{t+1} \mid s_t, a_t) - P(s_{t+1} \mid s_t, a_t) \cdot P(\tilde{s}' \mid \tilde{s}, \tilde{a}) \leq \sum_{s' \in \mathcal{S}\setminus\{s_{t+1}\}}{\tilde{P}_{t}^{UB}(s' \mid \tilde{s}, \tilde{a})} \cdot P(s_{t+1} \mid s_t, a_t)
    \]

    This means that $\theta$ can be assigned such that $\forall s' \in \mathcal{S}\setminus\{s_{t+1}\}$, the counterfactual probability of each transition $\tilde{s}, \tilde{a} \rightarrow s'$ will be $\tilde{P}_t(s' \mid \tilde{s}, \tilde{a}) \leq \tilde{P}_{t}^{UB}(s' \mid \tilde{s}, \tilde{a})$. We have already proven that each $\tilde{P}_{t}^{UB}(s' \mid \tilde{s}, \tilde{a})$ satisfies the \eqref{proofeq:monotonicity2} and CS \eqref{proofeq:counterfactual stability} constraints, and because $\tilde{P}_t(s' \mid \tilde{s}, \tilde{a}) \leq \tilde{P}_{t}^{UB}(s' \mid \tilde{s}, \tilde{a})$, $\tilde{P}_t(s'\mid \tilde{s}, \tilde{a})$ must satisfy CS as well.

    \item $0 \leq \theta_{u_t} \leq 1, \forall u_t$, satisfying \eqref{proofeq:valid prob1}.
    
    \item $\sum_{u_t = 1}^{|U_t|} \theta_{u_t} = \sum_{\substack{u_t \in U_t \\f(s_t, a_t, u_t) = s_{t+1} \\ f(\tilde{s}, \tilde{a}, u_t) = s_{t+1}}} \theta_{u_t} +  \sum_{s' \in \mathcal{S}\setminus\{ s_{t+1}\}}\sum_{\substack{u_t \in U_t \\f(s_t, a_t, u_t) = s_{t+1} \\ f(\tilde{s}, \tilde{a}, u_t) = s'}}{\theta_{u_t}} + \sum_{\substack{u_t \in U_t \\f(s_t, a_t, u_t) \neq s_{t+1} \\ f(\tilde{s}, \tilde{a}, u_t) = s_{t+1}}} \theta_{u_t} \\+ \sum_{s' \in \mathcal{S}\setminus\{ s_{t+1}\}}\sum_{\substack{u_t \in U_t \\f(s_t, a_t, u_t) \neq s_{t+1} \\ f(\tilde{s}, \tilde{a}, u_t) = s'}}{\theta_{u_t}} = 1$, satisfying \eqref{proofeq:valid prob2}.
\end{itemize}

All the constraints are satisfied, so this is a valid assignment of $\theta$ for this state-action pair.
\end{proof}

\paragraph{Case 1(b): $P(s_{t+1} \mid \tilde{s}, \tilde{a}) < 1 - \sum_{s' \in \mathcal{S}\setminus\{s_{t+1}\}}{\tilde{P}_{t}^{UB}(s' \mid \tilde{s}, \tilde{a})}$}
\noindent
\begin{proof}
In this case $P(s_{t+1} \mid \tilde{s}, \tilde{a}) < 1 - \sum_{s' \in \mathcal{S}\setminus\{s_{t+1}\}}{\tilde{P}_{t}^{UB}(s' \mid \tilde{s}, \tilde{a})}$, therefore:

\begin{equation}
\label{eq: lbcase1b}
    P(s_{t+1} \mid s_t, a_t) \cdot P(s_{t+1} \mid \tilde{s}, \tilde{a}) < P(s_{t+1} \mid s_t, a_t) \cdot (1 - \sum_{s' \in \mathcal{S}\setminus\{s_{t+1}\}}{\tilde{P}_{t}^{UB}(s' \mid \tilde{s}, \tilde{a})})
\end{equation}

Therefore, we can assign $\theta$ as follows:

\[
\begin{cases}
    \sum_{\substack{u_t \in U_t \\f(s_t, a_t, u_t) = s_{t+1} \\ f(\tilde{s}, \tilde{a}, u_t) = s_{t+1}}}{\theta_{u_t}} = P(s_{t+1} \mid s_t, a_t) \cdot (1 - \sum_{s' \in \mathcal{S}\setminus\{s_{t+1}\}}{\tilde{P}_{t}^{UB}(s' \mid \tilde{s}, \tilde{a})})\\
    
    \sum_{s' \in \mathcal{S}\setminus\{s_{t+1}\}}\sum_{\substack{u_t \in U_t \\f(s_t, a_t, u_t) = s_{t+1} \\ f(\tilde{s}, \tilde{a}, u_t) = s'}}{\theta_{u_t}} = P(s_{t+1} \mid s_t, a_t) \cdot (\sum_{s' \in \mathcal{S}\setminus\{s_{t+1}\}}{\tilde{P}_{t}^{UB}(s' \mid \tilde{s}, \tilde{a})})\\
    
    \sum_{\substack{u_t \in U_t \\f(s_t, a_t, u_t) \neq s_{t+1} \\ f(\tilde{s}, \tilde{a}, u_t) = s_{t+1}}}{\theta_{u_t}} = P(s_{t+1} \mid \tilde{s}, \tilde{a}) - \sum_{\substack{u_t \in U_t \\f(s_t, a_t, u_t) = s_{t+1} \\ f(\tilde{s}, \tilde{a}, u_t) = s_{t+1}}}{\theta_{u_t}} \\
    
    \sum_{s' \in \mathcal{S}\setminus\{ s_{t+1}\}}\sum_{\substack{u_t \in U_t \\f(s_t, a_t, u_t) \neq s_{t+1} \\ f(\tilde{s}, \tilde{a}, u_t) = s'}}{\theta_{u_t}} = \sum_{s' \in \mathcal{S}\setminus\{s_{t+1}\}}{P(s' \mid \tilde{s}, \tilde{a})} - \sum_{s' \in \mathcal{S}\setminus\{s_{t+1}\}}\sum_{\substack{u_t \in U_t \\f(s_t, a_t, u_t) = s_{t+1} \\ f(\tilde{s}, \tilde{a}, u_t) = s'}}{\theta_{u_t}}\\
\end{cases}
\]

  This assignment of $\theta$ satisfies the constraints of the linear optimisation problem, as follows:

    \begin{itemize}        
    \item $\sum_{\substack{u_t \in U_t \\f(s_t, a_t, u_t) = s_{t+1}}} \theta_{u_t} = \sum_{\substack{u_t \in U_t \\f(s_t, a_t, u_t) = s_{t+1} \\ f(\tilde{s}, \tilde{a}, u_t) = s_{t+1}}} \theta_{u_t} + \sum_{s' \in \mathcal{S}\setminus\{ s_{t+1}\}}\sum_{\substack{u_t \in U_t \\f(s_t, a_t, u_t) = s_{t+1} \\ f(\tilde{s}, \tilde{a}, u_t) = s'}}{\theta_{u_t}} \\ = P(s_{t+1} \mid s_t, a_t)$, satisfying \eqref{proofeq:interventional constraint}.

    \item $\sum_{\substack{u_t \in U_t\\f(s_t, a_t, u_t) \neq s_{t+1}}}{\theta_{u_t}} = \sum_{\substack{u_t \in U_t \\f(s_t, a_t, u_t) \neq s_{t+1} \\ f(\tilde{s}, \tilde{a}, u_t) = s_{t+1}}} \theta_{u_t} + \sum_{s' \in \mathcal{S}\setminus\{ s_{t+1}\}}\sum_{\substack{u_t \in U_t \\f(s_t, a_t, u_t) \neq s_{t+1} \\ f(\tilde{s}, \tilde{a}, u_t) = s'}}{\theta_{u_t}} \\=1 -
    P(s_{t+1} \mid s_t, a_t)$, satisfying \eqref{proofeq:interventional constraint}.

    \item $\sum_{\substack{u_t \in U_t \\f(\tilde{s}, \tilde{a}, u_t) = s_{t+1}}} \theta_{u_t} = \sum_{\substack{u_t \in U_t \\f(s_t, a_t, u_t) = s_{t+1} \\ f(\tilde{s}, \tilde{a}, u_t) = s_{t+1}}} \theta_{u_t} + \sum_{\substack{u_t \in U_t \\f(s_t, a_t, u_t) \neq s_{t+1} \\ f(\tilde{s}, \tilde{a}, u_t) = s_{t+1}}} \theta_{u_t} = P(s_{t+1} \mid \tilde{s}, \tilde{a})$, \\satisfying \eqref{proofeq:interventional constraint}.
    
    \item $\sum_{s' \in \mathcal{S}\setminus\{s_{t+1}\}}\sum_{\substack{u_t \in U_t \\f(\tilde{s}, \tilde{a}, u_t) = s'}} \theta_{u_t} = \sum_{s' \in \mathcal{S}\setminus\{ s_{t+1}\}}\sum_{\substack{u_t \in U_t \\f(s_t, a_t, u_t) = s_{t+1} \\ f(\tilde{s}, \tilde{a}, u_t) = s'}} \theta_{u_t} + \sum_{s' \in \mathcal{S}\setminus\{s_{t+1}\}}\sum_{\substack{u_t \in U_t \\f(s_t, a_t, u_t) \neq s_{t+1} \\ f(\tilde{s}, \tilde{a}, u_t) = s'}} \theta_{u_t} \\= 1 - P(\tilde{s}' \mid \tilde{s}, \tilde{a}) = \sum_{s' \in \mathcal{S}\setminus\{\tilde{s}'\}}P(s' \mid \tilde{s}, \tilde{a})$, therefore it is possible to assign $\theta$ such that $\forall s' \in \mathcal{S}\setminus\{s_{t+1}\}, \sum_{\substack{u_t \in U_t \\ f(\tilde{s}, \tilde{a}, u_t) = s'}}{\theta_{u_t}} = P(s' \mid \tilde{s}, \tilde{a})$, satisfying \eqref{proofeq:interventional constraint}.
    
    \item From \eqref{eq: lbcase1b}, we get:
    
    \[\sum_{\substack{u_t \in U_t \\f(s_t, a_t, u_t) = s_{t+1} \\ f(\tilde{s}, \tilde{a}, u_t) = s_{t+1}}}{\theta_{u_t}} = P(s_{t+1} \mid s_t, a_t) \cdot (1 - \sum_{s' \in \mathcal{S}\setminus\{s_{t+1}\}}{\tilde{P}_{t}^{UB}(s' \mid \tilde{s}, \tilde{a})})> P(s_{t+1} \mid s_t, a_t) \cdot P(s_{t+1} \mid \tilde{s}, \tilde{a})\]
    
    Therefore, \eqref{eq:overlapping constraint 9} holds, so Mon1 \eqref{proofeq:monotonicity1} is satisfied. 

    \item CS \eqref{proofeq:counterfactual stability} doesn't apply to the transition $\tilde{s}, \tilde{a} \rightarrow s_{t+1}$.
    
    \item Because $\sum_{s' \in \mathcal{S}\setminus\{s_{t+1}\}}\sum_{\substack{u_t \in U_t \\f(s_t, a_t, u_t) = s_{t+1} \\ f(\tilde{s}, \tilde{a}, u_t) = s'}}{\theta_{u_t}} = P(s_{t+1} \mid s_t, a_t) \cdot (\sum_{s' \in \mathcal{S}\setminus\{s_{t+1}\}}{\tilde{P}_{t}^{UB}(s' \mid \tilde{s}, \tilde{a})})$, this means that $\theta$ can be assigned such that $\forall s' \in \mathcal{S}\setminus\{s_{t+1}\}$, its counterfactual probability $\tilde{P}_t(s' \mid \tilde{s}, \tilde{a}) = \tilde{P}_{t}^{UB}(s' \mid \tilde{s}, \tilde{a})$, which we have already proven satisfies Mon2 \eqref{proofeq:monotonicity2} and CS \eqref{proofeq:counterfactual stability}.

    \item $0 \leq \theta_{u_t} \leq 1, \forall u_t$, satisfying \eqref{proofeq:valid prob1}.
    
    \item $\sum_{u_t = 1}^{|U_t|} \theta_{u_t} = \sum_{\substack{u_t \in U_t \\f(s_t, a_t, u_t) = s_{t+1} \\ f(\tilde{s}, \tilde{a}, u_t) = s_{t+1}}} \theta_{u_t} +  \sum_{s' \in \mathcal{S}\setminus\{ s_{t+1}\}}\sum_{\substack{u_t \in U_t \\f(s_t, a_t, u_t) = s_{t+1} \\ f(\tilde{s}, \tilde{a}, u_t) = s'}}{\theta_{u_t}} + \sum_{\substack{u_t \in U_t \\f(s_t, a_t, u_t) \neq s_{t+1} \\ f(\tilde{s}, \tilde{a}, u_t) = s_{t+1}}} \theta_{u_t} \\+ \sum_{s' \in \mathcal{S}\setminus\{ s_{t+1}\}}\sum_{\substack{u_t \in U_t \\f(s_t, a_t, u_t) \neq s_{t+1} \\ f(\tilde{s}, \tilde{a}, u_t) = s'}}{\theta_{u_t}} = 1$, satisfying \eqref{proofeq:valid prob2}.
\end{itemize}

All the constraints are satisfied, so this is a valid assignment of $\theta$ for this state-action pair.
\end{proof}

Therefore, these two cases can be simplified as:

\begin{align*}
\tilde{P}_{t}^{LB}(\tilde{s}' \mid \tilde{s}, \tilde{a}) &= \dfrac{\sum_{\substack{u_t \in U_t \\f(s_t, a_t, u_t) = s_{t+1} \\ f(\tilde{s}, \tilde{a}, u_t) = \tilde{s}'}}{\theta_{u_t}}}{P(s_{t+1} \mid s_t, a_t)}\\ &=
\dfrac{\max\left(P(s_{t+1} \mid s_t, a_t) \cdot P(s_{t+1} \mid \tilde{s}, \tilde{a}), P(s_{t+1} \mid s_t, a_t) \cdot (1 - \sum_{s' \in \mathcal{S}\setminus\{s_{t+1}\}}{\tilde{P}_{t}^{UB}(s' \mid \tilde{s}, \tilde{a})})\right)}{P(s_{t+1} \mid s_t, a_t)}\\ &= \max\left(P(\tilde{s}' \mid \tilde{s}, \tilde{a}), 1 - \sum_{s' \in \mathcal{S}\setminus\{s_{t+1}\}}{\tilde{P}_{t}^{UB}(s' \mid \tilde{s}, \tilde{a})}\right)
\end{align*}
\end{proof}

\paragraph{Case 2: $\tilde{s}' \neq s_{t+1}$, $\dfrac{P(s_{t+1} \mid \tilde{s}, \tilde{a})}{P(s_{t+1} \mid s_t, a_t)}\geq\dfrac{P(\tilde{s}' \mid \tilde{s}, \tilde{a})}{P(\tilde{s}' \mid s_t, a_t)}$ and $P(\tilde{s}' \mid s_t, a_t) > 0$}
\noindent
\begin{proof}
Because $\dfrac{P(s_{t+1} \mid \tilde{s}, \tilde{a})}{P(s_{t+1} \mid s_t, a_t)}\geq\dfrac{P(\tilde{s}' \mid \tilde{s}, \tilde{a})}{P(\tilde{s}' \mid s_t, a_t)}$ and $P(\tilde{s}' \mid s_t, a_t) > 0$, to satisfy CS \eqref{proofeq:counterfactual stability} $\tilde{P}_{t}^{LB}(\tilde{s}' \mid \tilde{s}, \tilde{a}) = 0$. We have already proven in Case 2 of the proof for Theorem \ref{proof theorem:ub overlapping} that there exists an assignment of $\theta$ that satisfies all the constraints, where $\tilde{P}_{t}(\tilde{s}' \mid \tilde{s}, \tilde{a}) = 0$. Therefore, the same assignment of $\theta$ will satisfy the constraints of the optimisation problem, and result in $\tilde{P}_{t}^{LB}(\tilde{s}' \mid \tilde{s}, \tilde{a}) = 0$.
\end{proof}

\paragraph{Case 3: $\tilde{s}' \neq s_{t+1}$, and $\dfrac{P(s_{t+1} \mid \tilde{s}, \tilde{a})}{P(s_{t+1} \mid s_t, a_t)}<\dfrac{P(\tilde{s}' \mid \tilde{s}, \tilde{a})}{P(\tilde{s}' \mid s_t, a_t)}$ or $P(\tilde{s}' \mid s_t, a_t) = 0$}
\noindent
\begin{proof}
Consider the following disjoint cases:
\begin{itemize}
    \item $1 - \sum_{s' \in \mathcal{S}\setminus\{\tilde{s}'\}}{\tilde{P}_{t}^{UB}(s' \mid \tilde{s}, \tilde{a})} \geq 0$
    \item $1 - \sum_{s' \in \mathcal{S}\setminus\{\tilde{s}'\}}{\tilde{P}_{t}^{UB}(s' \mid \tilde{s}, \tilde{a})} < 0$
\end{itemize}

\paragraph{Case 3(a): $1 - \sum_{s' \in \mathcal{S}\setminus\{\tilde{s}'\}}{\tilde{P}_{t}^{UB}(\tilde{s}' \mid \tilde{s}, \tilde{a})} \geq 0$}
\noindent
\begin{proof}
We can assign $\theta$ as follows:

\[
\begin{cases}
    \sum_{\substack{u_t \in U_t \\f(s_t, a_t, u_t) = s_{t+1} \\ f(\tilde{s}, \tilde{a}, u_t) = \tilde{s}'}}{\theta_{u_t}} = P(s_{t+1} \mid s_t, a_t) \cdot (1 - \sum_{s' \in \mathcal{S}\setminus\{\tilde{s}'\}}{\tilde{P}_{t}^{UB}(s' \mid \tilde{s}, \tilde{a})})\\
    
    \forall{s' \in \mathcal{S}\setminus\{\tilde{s}'\}}, \sum_{\substack{u_t \in U_t \\f(s_t, a_t, u_t) = s_{t+1} \\ f(\tilde{s}, \tilde{a}, u_t) = s'}}{\theta_{u_t}} = P(s_{t+1} \mid s_t, a_t) \cdot {\tilde{P}_{t}^{UB}(s' \mid \tilde{s}, \tilde{a})}\\
    
    \sum_{\substack{u_t \in U_t \\f(s_t, a_t, u_t) \neq s_{t+1} \\ f(\tilde{s}, \tilde{a}, u_t) = \tilde{s}'}}{\theta_{u_t}} = P(\tilde{s}' \mid \tilde{s}, \tilde{a}) - \sum_{\substack{u_t \in U_t \\f(s_t, a_t, u_t) = s_{t+1} \\ f(\tilde{s}, \tilde{a}, u_t) = \tilde{s}'}}{\theta_{u_t}} \\
    
    \forall{s' \in \mathcal{S}\setminus\{ \tilde{s}'\}}, \sum_{\substack{u_t \in U_t \\f(s_t, a_t, u_t) \neq s_{t+1} \\ f(\tilde{s}, \tilde{a}, u_t) = s'}}{\theta_{u_t}} = P(s' \mid \tilde{s}, \tilde{a})- \sum_{\substack{u_t \in U_t \\f(s_t, a_t, u_t) = s_{t+1} \\ f(\tilde{s}, \tilde{a}, u_t) = s'}}{\theta_{u_t}}\\
\end{cases}
\]

 This assignment of $\theta$ satisfies the constraints of the linear optimisation problem, as follows:

    \begin{itemize}        
    \item $\sum_{\substack{u_t \in U_t \\f(s_t, a_t, u_t) = s_{t+1}}} \theta_{u_t} = \sum_{\substack{u_t \in U_t \\f(s_t, a_t, u_t) = s_{t+1} \\ f(\tilde{s}, \tilde{a}, u_t) = \tilde{s}'}} \theta_{u_t} + \sum_{s' \in \mathcal{S}\setminus\{ \tilde{s}'\}}\sum_{\substack{u_t \in U_t \\f(s_t, a_t, u_t) = s_{t+1} \\ f(\tilde{s}, \tilde{a}, u_t) = s'}}{\theta_{u_t}} \\ = P(s_{t+1} \mid s_t, a_t)$, satisfying \eqref{proofeq:interventional constraint}.

    \item $\sum_{\substack{u_t \in U_t\\f(s_t, a_t, u_t) \neq s_{t+1}}}{\theta_{u_t}} = \sum_{\substack{u_t \in U_t \\f(s_t, a_t, u_t) \neq s_{t+1} \\ f(\tilde{s}, \tilde{a}, u_t) = \tilde{s}'}} \theta_{u_t} + \sum_{s' \in \mathcal{S}\setminus\{ \tilde{s}'\}}\sum_{\substack{u_t \in U_t \\f(s_t, a_t, u_t) \neq s_{t+1} \\ f(\tilde{s}, \tilde{a}, u_t) = s'}}{\theta_{u_t}} \\=1 -
    P(s_{t+1} \mid s_t, a_t)$, satisfying \eqref{proofeq:interventional constraint}.

    \item $\sum_{\substack{u_t \in U_t \\f(\tilde{s}, \tilde{a}, u_t) = \tilde{s}'}} \theta_{u_t} = \sum_{\substack{u_t \in U_t \\f(s_t, a_t, u_t) = s_{t+1} \\ f(\tilde{s}, \tilde{a}, u_t) = \tilde{s}'}} \theta_{u_t} + \sum_{\substack{u_t \in U_t \\f(s_t, a_t, u_t) \neq s_{t+1} \\ f(\tilde{s}, \tilde{a}, u_t) = \tilde{s}'}} \theta_{u_t} = P(\tilde{s}' \mid \tilde{s}, \tilde{a})$, \\satisfying \eqref{proofeq:interventional constraint}.
    
    \item $\sum_{s' \in \mathcal{S}\setminus\{\tilde{s}'\}}\sum_{\substack{u_t \in U_t \\f(\tilde{s}, \tilde{a}, u_t) = s'}} \theta_{u_t} = \sum_{s' \in \mathcal{S}\setminus\{ \tilde{s}'\}}\sum_{\substack{u_t \in U_t \\f(s_t, a_t, u_t) = s_{t+1} \\ f(\tilde{s}, \tilde{a}, u_t) = s'}} \theta_{u_t} + \sum_{s' \in \mathcal{S}\setminus\{\tilde{s}'\}}\sum_{\substack{u_t \in U_t \\f(s_t, a_t, u_t) \neq s_{t+1} \\ f(\tilde{s}, \tilde{a}, u_t) = s'}} \theta_{u_t} \\= 1 - P(\tilde{s}' \mid \tilde{s}, \tilde{a}) = \sum_{s' \in \mathcal{S}\setminus\{\tilde{s}'\}}P(s' \mid \tilde{s}, \tilde{a})$, therefore it is possible to assign $\theta$ such that $\forall s' \in \mathcal{S}\setminus\{\tilde{s}'\}, \sum_{\substack{u_t \in U_t \\ f(\tilde{s}, \tilde{a}, u_t) = s'}}{\theta_{u_t}} = P(s' \mid \tilde{s}, \tilde{a})$, satisfying \eqref{proofeq:interventional constraint}.

    \item $\forall s' \in \mathcal{S} \setminus \{\tilde{s}'\}, \tilde{P}_t(s' \mid \tilde{s}, \tilde{a}) = \dfrac{\sum_{\substack{u_t \in U_t \\f(s_t, a_t, u_t) = s_{t+1} \\ f(\tilde{s}, \tilde{a}, u_t) = s'}}{\theta_{u_t}}}{P(s_{t+1} \mid s_t, a_t)} = \dfrac{P(s_{t+1} \mid s_t, a_t) \cdot \tilde{P}_t^{UB}(s' \mid \tilde{s}, \tilde{a})}{P(s_{t+1} \mid s_t, a_t)} = \tilde{P}_t^{UB}(s' \mid \tilde{s}, \tilde{a})$, which we have already shown in the proof for Theorem \ref{proof theorem:ub overlapping} satisfies Mon1 \eqref{proofeq:monotonicity1}, Mon2 \eqref{proofeq:monotonicity2} and CS \eqref{proofeq:counterfactual stability}.

    \item Mon2 \eqref{proofeq:monotonicity2} is also satisfied for the transition $\tilde{s}, \tilde{a} \rightarrow \tilde{s}'$, as follows:

    \begin{itemize}
        \item If $P(\tilde{s}' \mid s_t, a_t) = 0$, then this assignment of $\theta$ vacuously satisfies Mon2 \eqref{proofeq:monotonicity2}.
        
        \item Otherwise, we have already proven that there exists an assignment of $\theta$ that satisfies all the constraints (as proven in Theorem \ref{proof theorem:ub overlapping}). If $1 - \sum_{s' \in \mathcal{S}\setminus\{\tilde{s}'\}}{\tilde{P}_{t}^{UB}(\tilde{s}' \mid \tilde{s}, \tilde{a})} > 0 $, then the minimum possible counterfactual transition probability of $\tilde{s}, \tilde{a} \rightarrow \tilde{s}'$ is $\tilde{P}_t(\tilde{s}' \mid \tilde{s}, \tilde{a}) = 1 - \sum_{s' \in \mathcal{S}\setminus\{\tilde{s}'}{\tilde{P}_{t}^{UB}(\tilde{s}' \mid \tilde{s}, \tilde{a})}$ (because the upper bounds must sum to $1$). Therefore, $1 - \sum_{s' \in \mathcal{S}\setminus\{\tilde{s}'\}}{\tilde{P}_{t}^{UB}(\tilde{s}' \mid \tilde{s}, \tilde{a})} \leq P(\tilde{s}' \mid \tilde{s}, \tilde{a})$, or there would be no valid assignment of $\theta$ that satisfies Mon2 \eqref{proofeq:monotonicity2}.
    \end{itemize}

    \item CS \eqref{proofeq:counterfactual stability} will be vacuously satisfied for the transition $\tilde{s}, \tilde{a} \rightarrow \tilde{s}'$, because $\dfrac{P(s_{t+1} \mid \tilde{s}, \tilde{a})}{P(s_{t+1} \mid s_t, a_t)}<\dfrac{P(\tilde{s}' \mid \tilde{s}, \tilde{a})}{P(\tilde{s}' \mid s_t, a_t)}$ or $P(\tilde{s}' \mid s_t, a_t) = 0$.

    \item $0 \leq \theta_{u_t} \leq 1, \forall u_t$, satisfying \eqref{proofeq:valid prob1}.
    
    \item $\sum_{u_t = 1}^{|U_t|} \theta_{u_t} = \sum_{\substack{u_t \in U_t \\f(s_t, a_t, u_t) = s_{t+1} \\ f(\tilde{s}, \tilde{a}, u_t) = s_{t+1}}} \theta_{u_t} +  \sum_{s' \in \mathcal{S}\setminus\{ s_{t+1}\}}\sum_{\substack{u_t \in U_t \\f(s_t, a_t, u_t) = s_{t+1} \\ f(\tilde{s}, \tilde{a}, u_t) = s'}}{\theta_{u_t}} + \sum_{\substack{u_t \in U_t \\f(s_t, a_t, u_t) \neq s_{t+1} \\ f(\tilde{s}, \tilde{a}, u_t) = s_{t+1}}} \theta_{u_t} \\+ \sum_{s' \in \mathcal{S}\setminus\{ s_{t+1}\}}\sum_{\substack{u_t \in U_t \\f(s_t, a_t, u_t) \neq s_{t+1} \\ f(\tilde{s}, \tilde{a}, u_t) = s'}}{\theta_{u_t}} = 1$, satisfying \eqref{proofeq:valid prob2}.
\end{itemize}
All the constraints are satisfied, so this is a valid assignment of $\theta$ for this state-action pair.
\end{proof}

\pagebreak
\paragraph{Case 3(b): $1 - \sum_{s' \in \mathcal{S}\setminus\{\tilde{s}'\}}{\tilde{P}_{t}^{UB}(\tilde{s}' \mid \tilde{s}, \tilde{a})} < 0$}
\noindent
\begin{proof}
Because $1 - \sum_{s' \in \mathcal{S}\setminus\{\tilde{s}'\}}{\tilde{P}_{t}^{UB}(\tilde{s}' \mid \tilde{s}, \tilde{a})} < 0$, we have:

\[
    P(s_{t+1} \mid s_t, a_t) \cdot \sum_{s' \in \mathcal{S}\setminus\{\tilde{s}'\}}{\tilde{P}_{t}^{UB}(s' \mid \tilde{s}, \tilde{a})} > P(s_{t+1} \mid s_t, a_t)
\]

Therefore, we can assign $\theta$ as follows:

\[
\begin{cases}
    \sum_{\substack{u_t \in U_t \\f(s_t, a_t, u_t) = s_{t+1} \\ f(\tilde{s}, \tilde{a}, u_t) = \tilde{s}'}}{\theta_{u_t}} = 0\\
    
    \sum_{s' \in \mathcal{S}\setminus\{\tilde{s}'\}}\sum_{\substack{u_t \in U_t \\f(s_t, a_t, u_t) = s_{t+1} \\ f(\tilde{s}, \tilde{a}, u_t) = s'}}{\theta_{u_t}} = P(s_{t+1} \mid s_t, a_t)\\
    
    \sum_{\substack{u_t \in U_t \\f(s_t, a_t, u_t) \neq s_{t+1} \\ f(\tilde{s}, \tilde{a}, u_t) = \tilde{s}'}}{\theta_{u_t}} = P(\tilde{s}' \mid \tilde{s}, \tilde{a})\\
    
    \sum_{s' \in \mathcal{S}\setminus\{ \tilde{s}'\}}\sum_{\substack{u_t \in U_t \\f(s_t, a_t, u_t) \neq s_{t+1} \\ f(\tilde{s}, \tilde{a}, u_t) = s'}}{\theta_{u_t}} = \sum_{s' \in \mathcal{S}\setminus\{\tilde{s}'\}}{P(s' \mid \tilde{s}, \tilde{a})} - \sum_{s' \in \mathcal{S}\setminus\{\tilde{s}'\}}\sum_{\substack{u_t \in U_t \\f(s_t, a_t, u_t) = s_{t+1} \\ f(\tilde{s}, \tilde{a}, u_t) = s'}}{\theta_{u_t}}\\
\end{cases}
\]

This assignment of $\theta$ satisfies the constraints of the linear optimisation problem, as follows:

    \begin{itemize}        
    \item $\sum_{\substack{u_t \in U_t \\f(s_t, a_t, u_t) = s_{t+1}}} \theta_{u_t} = \sum_{\substack{u_t \in U_t \\f(s_t, a_t, u_t) = s_{t+1} \\ f(\tilde{s}, \tilde{a}, u_t) = \tilde{s}'}} \theta_{u_t} + \sum_{s' \in \mathcal{S}\setminus\{ \tilde{s}'\}}\sum_{\substack{u_t \in U_t \\f(s_t, a_t, u_t) = s_{t+1} \\ f(\tilde{s}, \tilde{a}, u_t) = s'}}{\theta_{u_t}} \\ = P(s_{t+1} \mid s_t, a_t)$, satisfying \eqref{proofeq:interventional constraint}.

    \item $\sum_{\substack{u_t \in U_t\\f(s_t, a_t, u_t) \neq s_{t+1}}}{\theta_{u_t}} = \sum_{\substack{u_t \in U_t \\f(s_t, a_t, u_t) \neq s_{t+1} \\ f(\tilde{s}, \tilde{a}, u_t) = \tilde{s}'}} \theta_{u_t} + \sum_{s' \in \mathcal{S}\setminus\{ \tilde{s}'\}}\sum_{\substack{u_t \in U_t \\f(s_t, a_t, u_t) \neq s_{t+1} \\ f(\tilde{s}, \tilde{a}, u_t) = s'}}{\theta_{u_t}} \\=1 -
    P(s_{t+1} \mid s_t, a_t)$, satisfying \eqref{proofeq:interventional constraint}.

    \item $\sum_{\substack{u_t \in U_t \\f(\tilde{s}, \tilde{a}, u_t) = \tilde{s}'}} \theta_{u_t} = \sum_{\substack{u_t \in U_t \\f(s_t, a_t, u_t) = s_{t+1} \\ f(\tilde{s}, \tilde{a}, u_t) = \tilde{s}'}} \theta_{u_t} + \sum_{\substack{u_t \in U_t \\f(s_t, a_t, u_t) \neq s_{t+1} \\ f(\tilde{s}, \tilde{a}, u_t) = \tilde{s}'}} \theta_{u_t} = P(\tilde{s}' \mid \tilde{s}, \tilde{a})$ , \\satisfying \eqref{proofeq:interventional constraint}.
    
    \item $\sum_{s' \in \mathcal{S}\setminus\{\tilde{s}'\}}\sum_{\substack{u_t \in U_t \\f(\tilde{s}, \tilde{a}, u_t) = s'}} \theta_{u_t} = \sum_{s' \in \mathcal{S}\setminus\{ \tilde{s}'\}}\sum_{\substack{u_t \in U_t \\f(s_t, a_t, u_t) = s_{t+1} \\ f(\tilde{s}, \tilde{a}, u_t) = s'}} \theta_{u_t} + \sum_{s' \in \mathcal{S}\setminus\{\tilde{s}'\}}\sum_{\substack{u_t \in U_t \\f(s_t, a_t, u_t) \neq s_{t+1} \\ f(\tilde{s}, \tilde{a}, u_t) = s'}} \theta_{u_t} \\= 1 - P(\tilde{s}' \mid \tilde{s}, \tilde{a}) = \sum_{s' \in \mathcal{S}\setminus\{\tilde{s}'\}}P(s' \mid \tilde{s}, \tilde{a})$, therefore it is possible to assign $\theta$ such that $\forall s' \in \mathcal{S}\setminus\{\tilde{s}'\}, \sum_{\substack{u_t \in U_t \\ f(\tilde{s}, \tilde{a}, u_t) = s'}}{\theta_{u_t}} = P(s' \mid \tilde{s}, \tilde{a})$, satisfying \eqref{proofeq:interventional constraint}.

    \item From this assignment of $\theta$, $\tilde{P}_t(\tilde{s}' \mid \tilde{s}, \tilde{a}) = \dfrac{\sum_{\substack{u_t \in U_t \\f(s_t, a_t, u_t) = s_{t+1} \\ f(\tilde{s}, \tilde{a}, u_t) = \tilde{s}'}}{\theta_{u_t}}}{P(s_{t+1}\mid s_t, a_t)} = \dfrac{0}{P(s_{t+1} \mid s_t, a_t)} = 0$, which immediately satisfies the Mon2 \eqref{proofeq:monotonicity2} and CS \eqref{proofeq:counterfactual stability} constraints for the transition $\tilde{s}, \tilde{a} \rightarrow \tilde{s}'$.
    
    \item Because $1 - \sum_{s' \in \mathcal{S}\setminus\{\tilde{s}'\}}{\tilde{P}_{t}^{UB}(\tilde{s}' \mid \tilde{s}, \tilde{a})} < 0$:
    \[
    P(s_{t+1} \mid s_t, a_t) < \sum_{s' \in \mathcal{S}\setminus\{\tilde{s}'\}}{\tilde{P}_{t}^{UB}(\tilde{s}' \mid \tilde{s}, \tilde{a}) \cdot P(s_{t+1} \mid s_t, a_t)}
    \]

    This guarantees that we can assign $\theta$ such that the Mon1 \eqref{proofeq:monotonicity1}, Mon2 \eqref{proofeq:monotonicity2} and CS \eqref{proofeq:counterfactual stability} constraints are satisfied for all transitions $\tilde{s}', \tilde{a}' \rightarrow s', \forall s' \in \mathcal{S}\setminus\{\tilde{s}'\}$. In particular, we should assign $\theta$ such that $\tilde{P}_t(s_{t+1} \mid \tilde{s}, \tilde{a}) = \tilde{P}_{t}^{UB}(s_{t+1} \mid \tilde{s}, \tilde{a})$ (Theorem \ref{proof theorem:ub overlapping} proves this $\theta$ exists), as we have proven that Mon1 \eqref{proofeq:monotonicity1} and CS \eqref{proofeq:counterfactual stability} constraints are satisfied for $\tilde{P}_{t}^{UB}(s_{t+1} \mid \tilde{s}, \tilde{a})$. Then, we can assign $\theta$ such that $\forall s' \in \mathcal{S}\setminus\{\tilde{s}', s_{t+1}\}, \tilde{P}_t(\tilde{s}' \mid \tilde{s}, \tilde{a}) \leq \tilde{P}_{t}^{UB}(\tilde{s}' \mid \tilde{s}, \tilde{a})$, and because each $\tilde{P}_{t}^{UB}(\tilde{s}' \mid \tilde{s}, \tilde{a})$ has been proven to satisfy the Mon2 \eqref{proofeq:monotonicity2} and CS \eqref{proofeq:counterfactual stability} constraints, each $\tilde{P}_t(\tilde{s}' \mid \tilde{s}, \tilde{a})$ must as well. 

    \item $0 \leq \theta_{u_t} \leq 1, \forall u_t$, satisfying \eqref{proofeq:valid prob1}.
    
    \item $\sum_{u_t = 1}^{|U_t|} \theta_{u_t} = \sum_{\substack{u_t \in U_t \\f(s_t, a_t, u_t) = s_{t+1} \\ f(\tilde{s}, \tilde{a}, u_t) = s_{t+1}}} \theta_{u_t} +  \sum_{s' \in \mathcal{S}\setminus\{ s_{t+1}\}}\sum_{\substack{u_t \in U_t \\f(s_t, a_t, u_t) = s_{t+1} \\ f(\tilde{s}, \tilde{a}, u_t) = s'}}{\theta_{u_t}} + \sum_{\substack{u_t \in U_t \\f(s_t, a_t, u_t) \neq s_{t+1} \\ f(\tilde{s}, \tilde{a}, u_t) = s_{t+1}}} \theta_{u_t} \\+ \sum_{s' \in \mathcal{S}\setminus\{ s_{t+1}\}}\sum_{\substack{u_t \in U_t \\f(s_t, a_t, u_t) \neq s_{t+1} \\ f(\tilde{s}, \tilde{a}, u_t) = s'}}{\theta_{u_t}} = 1$, satisfying \eqref{proofeq:valid prob2}.
\end{itemize}
All the constraints are satisfied, so this is a valid assignment of $\theta$ for this state-action pair.

\end{proof}
Therefore, these two cases can be simplified as:

\begin{align*}
\tilde{P}_{t}^{LB}(\tilde{s}' \mid \tilde{s}, \tilde{a}) &= \dfrac{\sum_{\substack{u_t \in U_t \\f(s_t, a_t, u_t) = s_{t+1} \\ f(\tilde{s}, \tilde{a}, u_t) = \tilde{s}'}}{\theta_{u_t}}}{P(s_{t+1} \mid s_t, a_t)}\\ &=
\dfrac{\max\left(0, P(s_{t+1} \mid s_t, a_t) \cdot (1 - \sum_{s' \in \mathcal{S}\setminus\{\tilde{s}'\}}{\tilde{P}_{t}^{UB}(s' \mid \tilde{s}, \tilde{a})})\right)}{P(s_{t+1} \mid s_t, a_t)}\\ &= \max\left(0, 1 - \sum_{s' \in \mathcal{S}\setminus\{\tilde{s}'\}}{\tilde{P}_{t}^{UB}(s' \mid \tilde{s}, \tilde{a})}\right)
\end{align*}
\end{proof}

We have proven all the cases for Theorem \ref{proof theorem:lb overlapping}, therefore Theorem \ref{proof theorem:lb overlapping} holds.

\end{proof}

\pagebreak
\subsection{Any CFMDP Entailed by an Interval CFMDP is a Valid CFMDP}
\label{sec: sampled interval cfmdp proof}
In our experiments, we evaluate the policies by sampling example CFMDPs from the interval CFMDPs. But, even if $\theta$s exist that would produce each sampled CF transition probability separately, this does not necessarily guarantee that there exists a single $\theta$ that produces all of the CF transition probabilities in the CFMDP. We can prove that, given any CFMDP $\tilde{\mathcal{M}}_t$, there exists a $\theta$ that produces all the counterfactual probabilities in $\tilde{\mathcal{M}}_t$, as follows.

\begin{theorem}
    Given any arbitrary interval counterfactual MDP, any counterfactual MDP $\tilde{\mathcal{M}}_t$ entailed by this interval counterfactual MDP will be a valid counterfactual MDP.
\end{theorem}

\begin{proof}
Let us assume that we have an interval CFMDP, and an arbitrary CFMDP $\tilde{\mathcal{M}_t}$ entailed by this interval CFMDP for each time-step $t$ from this interval CFMDP. First, we need to prove that, for every transition $s, a \rightarrow s'$ in $\tilde{\mathcal{M}_t}$, there exists a $\theta$ that produces its counterfactual probability $\tilde{P}_t(s' \mid s, a)$. Next, we need to show that, for all state-action pairs $(s, a)$, there exists a $\theta$ that satisfies the constraints of the optimisation problem \eqref{proofeq:interventional constraint}-\eqref{proofeq:valid prob2} and produces the entailed CF probability distribution $\tilde{P}_t(\cdot \mid s, a)$. Finally, similarly to the proof in Section \ref{sec: probability bounds proof}, we can prove by induction that these assignments of $\theta$ for each state-action pair can be combined to form a single $\theta$ that produces $\tilde{\mathcal{M}_t}$.

\subsubsection{Existence of $\theta$ for Each $\tilde{P}_t(s' \mid s, a)$}
The counterfactual transition probability bounds are produced by a convex optimisation problem (or using the analytical bounds, which we have proven produce the same exact bounds as the optimisation problem). Because the optimisation problem is convex, the feasible set for each of the counterfactual transition probabilities is also convex (i.e., there are no gaps in the probability bounds). This means there is guaranteed to be at least one $\theta$ that can produce each of the entailed counterfactual transition probabilities in $\tilde{\mathcal{M}_t}$.

\subsubsection{Existence of $\theta$ for Each $\tilde{P}_t(\cdot \mid s, a)$}
\label{sec: theta for sampled distribution}
Next, we need to show that, for all state-action pairs $(s, a)$, there exists a $\theta$ that satisfies the constraints of the optimisation problem \eqref{proofeq:interventional constraint}-\eqref{proofeq:valid prob2} and can produce the entailed CF probability distribution $\tilde{P}_t(\cdot \mid s, a)$. Take any state-action pair $(s, a)$ in $\tilde{\mathcal{M}_t}$ arbitrarily. Each CF transition probability $\tilde{P}_t(\cdot \mid s, a)$ arises from the probability over the mechanisms that produce each transition $s, a \rightarrow s'$ and observed transition $s_t, a_t \rightarrow s_{t+1}$ vs. the total probability assigned to $P(s_{t+1} \mid s_t, a_t)$. Therefore, we need to show that the total probability assigned across all mechanisms containing each transition and the observed transition sums to $P(s_{t+1} \mid s_t, a_t)$. Because $\tilde{P}_t(\cdot \mid s, a)$ is a valid probability distribution, we know

\[
\sum_{s' \in \mathcal{S}}\tilde{P}_t(s' \mid s, a) = 1
\]

This means

\[
\sum_{s' \in \mathcal{S}}\dfrac{\sum_{u_t = 1}^{|U_t|} \mu_{s, a, u_t, s'} \cdot \mu_{s_{t}, a_{t}, u_t, s_{t+1}} \cdot \theta_{u_t}}{P(s_{t+1} \mid s_t, a_t)} = 1
\]

and 

\[
\sum_{s' \in \mathcal{S}}\sum_{u_t = 1}^{|U_t|} \mu_{s, a, u_t, s'} \cdot \mu_{s_{t}, a_{t}, u_t, s_{t+1}} \cdot \theta_{u_t} = {P(s_{t+1} \mid s_t, a_t)}
\]

Because $\mu_{s, a, u_t, s'} = 1$ for exactly one next state, $s'$, and $0$ for the remaining states, we must have

\[
\sum_{u_t = 1}^{|U_t|} \mu_{s_{t}, a_{t}, u_t, s_{t+1}} \cdot \theta_{u_t} = {P(s_{t+1} \mid s_t, a_t)}
\]

so the total probability assigned across all mechanisms containing each transition and the observed transition sums to $P(s_{t+1} \mid s_t, a_t)$.

Next, we need to show that there exists an assignment of $\theta$ which satisfies the interventional probability constraints \eqref{proofeq:interventional constraint} of the optimisation problem. We know that the interventional probabilities of all transitions from $(s, a)$ must sum to $1$:

\[
\sum_{s' \in \mathcal{S}}P(s' \mid s, a) = 1
\]

Because $0 \leq \theta_{u_t} \leq 1, \forall u_t$ \eqref{proofeq:valid prob1} and $\sum_{u_t=1}^{|U_t|}\theta_{u_t} = 1$ \eqref{proofeq:valid prob2} it must be possible to assign $\theta$ such that 

\[
\sum_{u_t = 1}^{|U_t|} \mu_{s, a, u_t, s'} \cdot \theta_{u_t} = P(s' \mid s, a)
\]

and

\[
\forall s' \in \mathcal{S}, \sum_{u_t = 1}^{|U_t|} \mu_{s, a, u_t, s'} \cdot \mu_{s_{t}, a_{t}, u_t, s_{t+1}} \cdot \theta_{u_t} = \tilde{P}_t(s' \mid s, a) \cdot P(s_{t+1} \mid s_t, a_t)
\]

and 

\[
\forall s' \in \mathcal{S}, \sum_{\tilde{s}' \neq s_{t+1}}\sum_{u_t = 1}^{|U_t|} \mu_{s, a, u_t, s'} \cdot \mu_{s_{t}, a_{t}, u_t, \tilde{s}'} \cdot \theta_{u_t} = P(s' \mid s, a) - \tilde{P}_t(s' \mid s, a) \cdot P(s_{t+1} \mid s_t, a_t)
\]

Therefore, this assignment satisfies
\[
\sum_{u_t = 1}^{|U_t|} \mu_{s, a, u_t, s'} \cdot \theta_{u_t} = P(s' \mid s, a)
\]

\[
\sum_{u_t = 1}^{|U_t|} \mu_{s_{t}, a_{t}, u_t, s_{t+1}} \cdot \theta_{u_t} = {P(s_{t+1} \mid s_t, a_t)}
\]

and 

\[
\sum_{s' \in \mathcal{S}\setminus\{s_{t+1}\}}\sum_{u_t = 1}^{|U_t|} \mu_{s_{t}, a_{t}, u_t, s'} \cdot \theta_{u_t} = 1 - P(s_{t+1} \mid s_t, a_t)
\]

which satisfies the interventional probability constraints \eqref{proofeq:interventional constraint}. Finally, because each counterfactual transition probability was sampled from its bounds, we know that the entailed counterfactual transition probabilities satisfy the Mon1 \eqref{proofeq:monotonicity1}, Mon \eqref{proofeq:monotonicity2} and CS \eqref{proofeq:counterfactual stability} constraints as well.

\subsubsection{Existence of $\theta$ for $\tilde{\mathcal{M}_t}$}

We have shown that there exists a $\theta$ that meets the constraints of the optimisation problem for each state-action pair separately in the MDP. Now, 
we need to prove that these assignments can be combined to form a valid $\theta$ for all state-action pairs. Let us prove inductively that adding a new state-action pair from the MDP to the causal model (and changing $\theta$ to meet the constraints for this new state-action pair) does not affect whether the constraints of the other state-action pairs are satisfied. If this is the case, we can consider how $\theta$ can be assigned to satisfy the constraints for each state-action pair separately, as we know these can be combined to form a valid assignment of $\theta$ over all of the state-action pairs. 

\paragraph{Base Case 1} Assume $(s_t, a_t)$ (the observed state-action pair) is the only state-action pair currently considered in the causal model (this always exists, as we always have an observed transition). We require a structural equation mechanism for each possible next state (of which there are $|\mathcal{S}|$). This leads to an assignment of $\theta^1$ as follows, which is the only possible assignment that satisfies the constraints \eqref{proofeq:interventional constraint}-\eqref{proofeq:valid prob2}:

\begin{table}[h]
\centering
\begin{tabular}{c|c|c|c}
$\theta^1_{u_1}$                                  & $\theta^1_{u_2}$                                 & ... & $\theta^1_{u_{|\mathcal{S}|}}$\\
\hline
$P(s_1 \mid s_t, a_t)$ & $P(s_2 \mid s_t, a_t)$ & ... & $P(s_{|\mathcal{S}|} \mid s_t, a_t)$
\end{tabular}
\end{table}

$\theta^1 = [P(s_1 \mid s_t, a_t), P(s_2 \mid s_t, a_t), ..., P(s_{|\mathcal{S}|} \mid s_t, a_t)]$. Given $\theta^1$, the counterfactual transition probability for each transition $s_t, a_t \rightarrow s'$ can be calculated with:

\[
\tilde{P}_t(s' \mid s_t, a_t) = \dfrac{\sum_{u_t = 1}^{|U_t|} \mu_{s_t, a_t, u_t, s'} \cdot \mu_{s_{t}, a_{t}, u_t, s_{t+1}} \cdot \theta_{u_t}}{P(s_{t+1} \mid s_t, a_t)}, \forall s'
\]

Therefore, given $\theta^1$, $\tilde{P}_t(s_{t+1} \mid s_t, a_t) = 1$, and $\forall s' \in \mathcal{S}\setminus\{s_{t+1}\}$, $\tilde{P}_t(s' \mid s_t, a_t) = 0$. This is the only possible counterfactual probability distribution for $(s_t, a_t)$, so this is guaranteed to be the distribution entailed from the interval CFMDP. These counterfactual probabilities also satisfy the monotonicity and counterfactual stability constraints (6-8) of the optimisation problem, as follows:

\begin{itemize}
    \item Because $\dfrac{P(s_{t+1} \mid s_t, a_t)}{P(s_{t+1} \mid s_t, a_t)} = 1$ and $\forall s' \in \mathcal{S}, \dfrac{P(s' \mid s_t, a_t)}{P(s' \mid s_t, a_t)} = 1$,

    \[
    \forall s' \in \mathcal{S}, \dfrac{P(s_{t+1} \mid s_t, a_t)}{P(s_{t+1} \mid s_t, a_t)} = \dfrac{P(s' \mid s_t, a_t)}{P(s' \mid s_t, a_t)}
    \]

    so CS \eqref{proofeq:counterfactual stability} is always vacuously true for all transitions from $(s_t, a_t)$.\\

    \item Because $\tilde{P}_{t}(s_{t+1} \mid s_t, a_t) = 1 \geq P(s_{t+1} \mid s_t, a_t)$, this satisfies Mon1 \eqref{proofeq:monotonicity1}.
    
    \item Because $\forall s' \in \mathcal{S}\setminus \{s_{t+1}\}, \tilde{P}_{t}(s' \mid s_t, a_t) = 0 \leq P(s' \mid s_t, a_t)$, this satisfies Mon2 \eqref{proofeq:monotonicity2}.
\end{itemize} 

\paragraph{Base Case 2} Let us assume that we only have the observed state-action pair in the causal model, and assume we have a valid $\theta^1$ with $|\mathcal{S}|$ structural equation mechanisms. 

Now, take another state-action pair $(s, a)$ arbitrarily from the MDP to add to the causal model. We now require $|\mathcal{S}|^2$ structural equation mechanisms, one for each possible combination of transitions from the two state-action pairs. We can view this as separating each of the existing structural equation mechanisms from Base Case 1 into $|\mathcal{S}|$ new structural equation mechanisms, one for every possible transition from the new state-action pair, $(s, a)$:

\begin{figure}[!ht]
\centering
\resizebox{0.5\textwidth}{!}{%
\begin{circuitikz}
\tikzstyle{every node}=[font=\LARGE]
\node [font=\LARGE] at (3,16.25) {$u_1$};
\node [font=\LARGE] at (6.0,16.25) {$...$};
\node [font=\LARGE] at (8.75,16.25) {$u_{|S|}$};
\draw [short] (3,15.75) -- (1.25,13.5);
\draw [short] (3,15.75) -- (5,13.5);
\draw [short] (3,15.75) -- (3,13.5);
\node [font=\LARGE] at (1,13) {$u_{1,1}$};
\node [font=\LARGE] at (5,13) {$u_{1,|S|}$};
\node [font=\LARGE] at (3,13) {$u_{1,2}$};
\node [font=\LARGE] at (4.0,13) {$...$};
\node [font=\LARGE] at (9.7,13) {$...$};

\draw [short] (8.75,15.75) -- (7,13.5);
\draw [short] (8.75,15.75) -- (10.75,13.5);
\draw [short] (8.75,15.75) -- (8.75,13.5);
\node [font=\LARGE] at (6.75,13) {$u_{|S|,1}$};
\node [font=\LARGE] at (10.75,13) {$u_{|S|,|S|}$};
\node [font=\LARGE] at (8.75,13) {$u_{|S|,2}$};
\end{circuitikz}
}%
\end{figure}

where, for $1 \leq i \leq |\mathcal{S}|, 1 \leq n \leq |\mathcal{S}|$, $u_{i, n}$ leads to the same next state for $(s_t, a_t)$ as $u_i$, and produces the transition $s, a \rightarrow s_n$.\\

In the same way, we can split each $\theta^1_{u_1}, ..., \theta^1_{u_{|\mathcal{S}|}}$ across these new structural equation mechanisms to obtain $\theta^2$:

\begin{figure}[!ht]
\centering
\resizebox{0.5\textwidth}{!}{%
\begin{circuitikz}
\tikzstyle{every node}=[font=\LARGE]
\node [font=\LARGE] at (3,16.25) {$\theta^1_{u_1}$};
\node [font=\LARGE] at (6.0,16.25) {$...$};
\node [font=\LARGE] at (8.75,16.25) {$\theta^1_{u_{|S|}}$};
\draw [short] (3,15.75) -- (1.25,13.5);
\draw [short] (3,15.75) -- (5,13.5);
\draw [short] (3,15.75) -- (3,13.5);
\node [font=\LARGE] at (1,13) {$\theta^2_{u_{1,1}}$};
\node [font=\LARGE] at (5,13) {$\theta^2_{u_{1,|S|}}$};
\node [font=\LARGE] at (3,13) {$\theta^2_{u_{1,2}}$};
\node [font=\LARGE] at (3.9,13) {$...$};
\node [font=\LARGE] at (9.6,13) {$...$};
\draw [short] (8.75,15.75) -- (7,13.5);
\draw [short] (8.75,15.75) -- (10.75,13.5);
\draw [short] (8.75,15.75) -- (8.75,13.5);
\node [font=\LARGE] at (6.75,13) {$\theta^2_{u_{|S|,1}}$};
\node [font=\LARGE] at (10.75,13) {$\theta^2_{u_{|S|,|S|}}$};
\node [font=\LARGE] at (8.75,13) {$\theta^2_{u_{|S|,2}}$};
\end{circuitikz}
}%
\end{figure}
By splitting $\theta^1$ in this way, we guarantee that the total probability across each set of mechanisms (where each set produces a different transition from $(s_t, a_t)$) remains the same, i.e., \[\forall i \in \{1, .., |\mathcal{S}|\}, \sum_{n=1}^{|\mathcal{S}|}\theta^2_{u_{i,n}} = \theta_{u_i}^1\]

This means that the counterfactual probabilities of all transitions from $(s_t, a_t)$ will be exactly the same when using $\theta^2$ as with $\theta^1$. Because we assume all the constraints were satisfied when using $\theta^1$, this guarantees that all the constraints of the optimisation problem will continue to hold for $(s_t, a_t)$ with $\theta^2$.\\

Now, we only need to ensure we assign $\theta^2$ such that all the constraints hold for the transitions from $(s, a)$. Firstly, we need $\sum_{i=1}^{|\mathcal{S}|}{\theta^2_{u_{i, n}}} = P(s_n \mid s, a)$ for every possible next state $s_n \in \mathcal{S}$, to satisfy the interventional constraint \eqref{proofeq:interventional constraint}. For each transition $s, a \rightarrow s_n$, there is a $\theta^2_{u_{i, n}}$ for every $u_i$, and for every $\theta^2_{u_{i, n}}$, $0 \leq \theta^2_{u_{i, n}} \leq \theta^1_{u_n}$. Therefore, $0 \leq \sum_{i=1}^{|\mathcal{S}|}{\theta^2_{u_{i, n}}} \leq 1, \forall s_n \in \mathcal{S}$. We also know that:

\begin{align*}
\sum_{n=1}^{|\mathcal{S}|}\sum_{i=1}^{|\mathcal{S}|}\theta^2_{u_i,n} &= \sum_{i=1}^{|\mathcal{S}|}\theta^1_{u_i} \text{ (because we split the probabilities of each $\theta^1_{u_i}$)}\\
&= \sum_{i=1}^{|\mathcal{S}|}P(s_i|s_t, a_t) \\&= 1
\end{align*}

and 

\[
\sum_{n=1}^{|\mathcal{S}|}\sum_{i=1}^{|\mathcal{S}|}{\theta^2_{u_{i, n}}} = 
\sum_{i=1}^{|\mathcal{S}|}P(s_i|s,a) = 1
\]

Therefore, no matter the assignment of $\theta^1$, we can satisfy all the constraints \eqref{proofeq:interventional constraint}-\eqref{proofeq:valid prob2} for all transitions from $(s, a)$. In Section \ref{sec: theta for sampled distribution}, we prove that, for any $(s, a)$, we can assign $\theta^2$ such that it produces the entailed counterfactual probability distribution.

\paragraph{Inductive Case} Let us assume that we now have $k$ state-action pair values from the MDP $\mathcal{M}$ in our causal model, and we have found a valid assignment $\theta^k$ that satisfies all the constraints for $k$ state-action pairs from the MDP $\mathcal{M}$, including $(s_t, a_t)$. Because there are $k$ state-action pairs, there are $|\mathcal{S}|^{k}$ possible unique structural equation mechanisms for these state-action pairs \footnote{Note we have changed the indexing of each $\theta^k_{i, n}$ to $\theta^k_{j}$, where $j = (i-1)\cdot|\mathcal{S}| + n$}, with probabilities as follows:

\begin{table}[h]
\centering
\begin{tabular}{lllllll}
$\theta^k_{u_1}$ & ... & $\theta^k_{u_{|\mathcal{S}|}}$ & ... & $\theta^k_{u_{|\mathcal{S}|^2}}$ & ... & $\theta^k_{u_{|\mathcal{S}|^{k}}}$ \\
\end{tabular}
\end{table}

Now, we wish to add the $k+1^{th}$ state-action pair from the MDP to the causal model, resulting in $|\mathcal{S}|^{k+1}$ possible structural equation mechanisms. Let the $k+1^{th}$ state-action pair be $(s, a)$ arbitrarily. We can view this as separating the existing structural equation mechanisms each into $|\mathcal{S}|$ new structural equation mechanisms, one for every possible transition from the $k+1^{th}$ state-action pair:

\begin{figure}[!ht]
\centering
\resizebox{1\textwidth}{!}{%
\begin{circuitikz}
\tikzstyle{every node}=[font=\LARGE]
\node [font=\LARGE] at (3.0,16.25) {$u_1$};
\node [font=\LARGE] at (6.0,16.25) {$...$};
\node [font=\LARGE] at (8.75,16.25) {$u_{|S|}$};
\node [font=\LARGE] at (12,16.25) {$...$};
\node [font=\LARGE] at (14.75,16.25) {$u_{|S|^2}$};
\node [font=\LARGE] at (20.75,16.25) {$u_{|S|^k}$};
\node [font=\LARGE] at (18,16.25) {$...$};
\draw [short] (3,15.75) -- (1.25,13.5);
\draw [short] (3,15.75) -- (5,13.5);
\draw [short] (3,15.75) -- (3,13.5);
\node [font=\LARGE] at (1,13) {$u_{1,1}$};
\node [font=\LARGE] at (5,13) {$u_{1,|S|}$};
\node [font=\LARGE] at (3,13) {$u_{1,2}$};
\node [font=\LARGE] at (3.9,13) {$...$};
\draw [short] (8.75,15.75) -- (7,13.5);
\draw [short] (8.75,15.75) -- (10.75,13.5);
\draw [short] (8.75,15.75) -- (8.75,13.5);
\node [font=\LARGE] at (6.75,13) {$u_{|S|,1}$};
\node [font=\LARGE] at (10.75,13) {$u_{|S|,|S|}$};
\node [font=\LARGE] at (8.75,13) {$u_{|S|,2}$};
\node [font=\LARGE] at (9.7,13) {$...$};
\draw [short] (14.75,15.75) -- (13,13.5);
\draw [short] (14.75,15.75) -- (16.75,13.5);
\draw [short] (14.75,15.75) -- (14.75,13.5);
\node [font=\LARGE] at (12.75,13) {$u_{|S|^2,1}$};
\node [font=\LARGE] at (17,13) {$u_{|S|^2,|S|}$};
\node [font=\LARGE] at (14.75,13) {$u_{|S|^2,2}$};
\node [font=\LARGE] at (15.8,13) {$...$};
\draw [short] (20.75,15.75) -- (19,13.5);
\draw [short] (20.75,15.75) -- (22.75,13.5);
\draw [short] (20.75,15.75) -- (20.75,13.5);
\node [font=\LARGE] at (18.75,13) {$u_{|S|^{k},1}$};
\node [font=\LARGE] at (23,13) {$u_{|S|^{k},|S|}$};
\node [font=\LARGE] at (20.75,13) {$u_{|S|^{k}, 2}$};
\node [font=\LARGE] at (21.8,13) {$...$};
\end{circuitikz}
}%
\end{figure}

where each $u_{i, n}$ produces the same next states for all of the first $k$ state-action pairs as $u_i$, and produces the transition $s, a \rightarrow s_n$.

In the same way, we can split each $\theta^k_{u_1}, ..., \theta^k_{u_{|\mathcal{S}|}}, ...$ across these new structural equation mechanisms to find $\theta^{k+1}$.

\pagebreak
\begin{figure}[!ht]
\centering
\resizebox{1\textwidth}{!}{%
\begin{circuitikz}
\tikzstyle{every node}=[font=\LARGE]
\node [font=\LARGE] at (3,16.25) {$\theta^k_{u_1}$};
\node [font=\LARGE] at (6.0,16.25) {$...$};
\node [font=\LARGE] at (8.75,16.25) {$\theta^k_{u_{|S|}}$};
\node [font=\LARGE] at (11.8,16.25) {$...$};
\node [font=\LARGE] at (14.75,16.25) {$\theta^k_{u_{|S|^2}}$};
\node [font=\LARGE] at (20.75,16.25) {$\theta^k_{u_{|S|^k}}$};
\node [font=\LARGE] at (17.8,16.25) {$...$};
\draw [short] (3,15.75) -- (1.25,13.5);
\draw [short] (3,15.75) -- (5,13.5);
\draw [short] (3,15.75) -- (3,13.5);
\node [font=\LARGE] at (1,13) {$\theta^{k+1}_{u_{1,1}}$};
\node [font=\LARGE] at (5,13) {$\theta^{k+1}_{u_{1,|S|}}$};
\node [font=\LARGE] at (3,13) {$\theta^{k+1}_{u_{1,2}}$};
\node [font=\LARGE] at (3.9,13) {$...$};
\draw [short] (8.75,15.75) -- (7,13.5);
\draw [short] (8.75,15.75) -- (10.75,13.5);
\draw [short] (8.75,15.75) -- (8.75,13.5);
\node [font=\LARGE] at (6.75,13) {$\theta^{k+1}_{u_{|S|,1}}$};
\node [font=\LARGE] at (10.9,13) {$\theta^{k+1}_{u_{|S|,|S|}}$};
\node [font=\LARGE] at (8.75,13) {$\theta^{k+1}_{u_{|S|,2}}$};
\node [font=\LARGE] at (9.7,13) {$...$};
\draw [short] (14.75,15.75) -- (13,13.5);
\draw [short] (14.75,15.75) -- (16.75,13.5);
\draw [short] (14.75,15.75) -- (14.75,13.5);
\node [font=\LARGE] at (12.75,13) {$\theta^{k+1}_{u_{|S|^2,1}}$};
\node [font=\LARGE] at (17,13) {$\theta^{k+1}_{u_{|S|^2,|S|}}$};
\node [font=\LARGE] at (14.75,13) {$\theta^{k+1}_{u_{|S|^2,2}}$};
\node [font=\LARGE] at (15.75,13) {$...$};
\draw [short] (20.75,15.75) -- (19,13.5);
\draw [short] (20.75,15.75) -- (22.75,13.5);
\draw [short] (20.75,15.75) -- (20.75,13.5);
\node [font=\LARGE] at (18.75,13) {$\theta^{k+1}_{u_{|S|^{k},1}}$};
\node [font=\LARGE] at (23,13) {$\theta^{k+1}_{u_{|S|^{k},|S|}}$};
\node [font=\LARGE] at (20.75,13) {$\theta^{k+1}_{u_{|S|^{k}, 2}}$};
\node [font=\LARGE] at (21.75,13) {$...$};
\end{circuitikz}
}%
\end{figure}

By splitting $\theta^k$ in this way, we have $\forall i \in \{1, .., |\mathcal{S}|\}, \sum_{n=1}^{|\mathcal{S}|}\theta^{k+1}_{u_{i,n}} = \theta_{u_i}^k$. Each $\theta^{k+1}_{u_{i, n}}$ is only different to $\theta^k_{u_i}$ in the transition from $(s, a)$. Therefore, this guarantees that for every state-action pair $(s', a') \neq (s, a)$ in the first $k$ pairs added to the causal model, the total probability across each set of mechanisms (where each set produces a different transition from $(s', a')$) remains the same in $\theta^{k+1}$ as in $\theta^k$. As a result, the counterfactual probabilities of all transitions from each $(s', a')$ will be exactly the same when using $\theta^{k+1}$ as with $\theta^k$. Because we assume all the constraints were satisfied when using $\theta^k$, this guarantees that all the constraints of the optimisation problem will continue to hold for $(s', a')$ with $\theta^{k+1}$.\\

Now, we only need to make sure that we can assign $\theta^{k+1}$ such that the constraints also hold for all the transitions from $(s, a)$. We need $\sum_{i=1}^{|\mathcal{S}|^k}{\theta^{k+1}_{u_{i, n}}} = P(s_n \mid s, a)$ for every possible next state $s_n \in \mathcal{S}$, to satisfy the interventional probability constraints \eqref{proofeq:interventional constraint}. For each transition $s, a \rightarrow s_n$, there is a $\theta^{k+1}_{u_{i, n}}$ for every $u_i$, and for every $\theta^{k+1}_{u_{i, n}}$, $0 \leq \theta^{k+1}_{u_{i, n}} \leq \theta^k_{u_n}$, therefore $0 \leq \sum_{i=1}^{|\mathcal{S}|^k}{\theta^{k+1}_{u_{i, n}}} \leq 1, \forall s_n \in \mathcal{S}$. We also know that:

\begin{align*}
\sum_{n=1}^{|\mathcal{S}|}\sum_{i=1}^{|\mathcal{S}|^k}\theta^{k+1}_{u_i,n} &= \sum_{i=1}^{|\mathcal{S}|^k}\theta^k_{u_i} \text{ (because we split the probabilities of each $\theta^k_{u_i}$)}\\&= 1
\end{align*}

and 

\[
\sum_{n=1}^{|\mathcal{S}|}\sum_{i=1}^{|\mathcal{S}|^k}{\theta^2_{u_{i, n}}} = 
\sum_{n=1}^{|\mathcal{S}|}P(s_n|s,a) = 1
\]

Therefore, no matter the assignment of $\theta^k$, we can satisfy the interventional probability constraints \eqref{proofeq:interventional constraint} for all transitions from $(s, a)$. In Section \ref{sec: theta for sampled distribution}, we prove that, for any $(s, a)$, we can assign $\theta^{k+1}$ such that it produces the entailed counterfactual probability distribution (which we prove always satisfies the constraints of the optimisation problem). Therefore, we know we can assign $\theta^{k+1}$ such that it produces the counterfactual probability distributions for all $k+1$ state-action pairs in the causal model.

\paragraph{Conclusion} By induction, we have proven that as long as we can find a $\theta$ that produces the entailed counterfactual probability distribution for each state-action pair separately, we can combine these into a single valid $\theta$ that produces the entailed CFMDP $\tilde{\mathcal{M}_t}$.
\end{proof}

\pagebreak
\subsection{Lemmas}
\begin{lemma}\label{lemma: constraints do not apply if disjoint}
If $(s_t, a_t) \neq (\tilde{s}, \tilde{a})$ and $(s_t, a_t)$ and $(\tilde{s}, \tilde{a})$ have disjoint support, the monotonicity constraints \eqref{proofeq:monotonicity1} and \eqref{proofeq:monotonicity2} and counterfactual stability \eqref{proofeq:counterfactual stability} will be satisfied for all possible counterfactual transition probabilities for transitions from $(\tilde{s}, \tilde{a})$.
\end{lemma}

\begin{proof}
Because $(s_t, a_t)$ and $(\tilde{s}, \tilde{a})$ have disjoint support, $P(s_{t+1} \mid \tilde{s}, \tilde{a}) = 0$ and, for all $\tilde{s}'$ where $P(\tilde{s}' \mid \tilde{s}, \tilde{a}) > 0$, $P(\tilde{s}' \mid s_t, a_t) = 0$.
For all $\tilde{s}'$ where $P(\tilde{s}' \mid \tilde{s}, \tilde{a}) > 0$, $P(\tilde{s}' \mid s_t, a_t) = 0$, so CS \eqref{proofeq:counterfactual stability} and Mon2 \eqref{proofeq:monotonicity2} will be vacuously true for all transitions from $(\tilde{s}, \tilde{a})$. Also, because $P(s_{t+1} \mid \tilde{s}, \tilde{a}) = 0$, Mon1 \eqref{proofeq:monotonicity1} is vacuously true.
\end{proof}

\begin{lemma}
\label{lemma:absolute max cf prob}
For any observed transition $s_t, a_t \rightarrow s_{t+1}$ and counterfactual transition $\tilde{s}, \tilde{a} \rightarrow \tilde{s}'$:
\[
    \sum_{u_t = 1}^{|U_t|} \mu_{\tilde{s}, \tilde{a}, u_t, \tilde{s}'} \cdot \mu_{s_t, a_t, u_t, s_{t+1}} \cdot \theta_{u_t} \leq \min(P(s_{t+1} \mid s_t, a_t), P(\tilde{s}' \mid \tilde{s}, \tilde{a}))
\]
\end{lemma}

\begin{proof}
From our interventional probability constraints \eqref{proofeq:interventional constraint}, we have:

\begin{equation}
    \label{eq: lemma 4.2 constraint 1}
   \sum_{u_t = 1}^{|U_t|} \mu_{s_t, a_t, u_t, s_{t+1}} \cdot \theta_{u_t} = P(s_{t+1} \mid s_t, a_t) 
\end{equation}

and

\begin{equation}
    \label{eq: lemma 4.2 constraint 2}
   \sum_{u_t = 1}^{|U_t|} \mu_{\tilde{s}, \tilde{a}, u_t, \tilde{s}'}\cdot \theta_{u_t} = P(\tilde{s}' \mid \tilde{s}, \tilde{a}) 
\end{equation}

$\mu_{\tilde{s}, \tilde{a}, u, \tilde{s}'} \in \{0, 1\}$ (from its definition). $\sum_{u_t = 1}^{|U_t|} \mu_{\tilde{s}, \tilde{a}, u_t, \tilde{s}'} \cdot \mu_{s_t, a_t, u_t, s_{t+1}} \cdot \theta_{u_t}$ is maximal when the probability assigned to all $u_t$ where $\mu_{\tilde{s}, \tilde{a}, u_t, \tilde{s}'} = 1$ and $\mu_{s_t, a_t, u_t, s_{t+1}}=1$ is maximal.

\[
\sum_{u_t = 1}^{|U_t|} \mu_{\tilde{s}, \tilde{a}, u_t, \tilde{s}'} \cdot (\mu_{s_t, a_t, u_t, s_{t+1}} \cdot \theta_{u_t}) \leq \sum_{u_t = 1}^{|U_t|} \mu_{s_t, a_t, u_t, s_{t+1}} \cdot \theta_{u_t} = P(s_{t+1} \mid s_t, a_t)
\]

\[
\sum_{u_t = 1}^{|U_t|} (\mu_{\tilde{s}, \tilde{a}, u_t, \tilde{s}'} \cdot \theta_{u_t}) \cdot \mu_{s_t, a_t, u_t, s_{t+1}} \leq \sum_{u_t = 1}^{|U_t|} \mu_{\tilde{s}, \tilde{a}, u_t, \tilde{s}'}\cdot \theta_{u_t} = P(\tilde{s}' \mid \tilde{s}, \tilde{a})
\]

 Because of Eq. \eqref{eq: lemma 4.2 constraint 1} and Eq. \eqref{eq: lemma 4.2 constraint 2}, the probability assigned to all $u_t$ where $\mu_{\tilde{s}, \tilde{a}, u_t, \tilde{s}'} = 1$ and $\mu_{s_t, a_t, u_t, s_{t+1}}=1$ cannot be greater than $P(s_{t+1} \mid s_t, a_t)$ or $P(\tilde{s}' \mid \tilde{s}, \tilde{a})$. Therefore,

\[
\sum_{u_t = 1}^{|U_t|} \mu_{\tilde{s}, \tilde{a}, u_t, \tilde{s}'} \cdot \mu_{s_t, a_t, u_t, s_{t+1}} \cdot \theta_{u_t} \leq \min(P(s_{t+1} \mid s_t, a_t), P(\tilde{s}' \mid \tilde{s}, \tilde{a}))
\]

This also means that the upper bound of any counterfactual transition probability \\$\tilde{P}_t^{UB}(\tilde{s}' \mid \tilde{s}, \tilde{a}) \leq \dfrac{\min(P(\tilde{s}' \mid \tilde{s}, \tilde{a}), P(s_{t+1} \mid s_t, a_t))}{P(s_{t+1} \mid s_t, a_t)}$.
\end{proof}

\begin{lemma}
    \label{lemma: CF probs of observed state-action pair}
    For any observed transition $s_t, a_t \rightarrow s_{t+1}$, $\tilde{P}_{t}(s_{t+1} \mid s_t, a_t) = 1$ and, $\forall \tilde{s}' \in \mathcal{S}$ where $\tilde{s}' \in \mathcal{S}\setminus \{s_{t+1}\}$, $\tilde{P}_{t}(\tilde{s}' \mid s_t, a_t) = 0$.
\end{lemma}

\begin{proof}
From \eqref{eq: counterfactual probability}, we have:
    \begin{align*}
      \tilde{P}_{t}(s_{t+1} \mid s_t, a_t)
      &= \frac{\sum_{u_t = 1}^{|U_t|} \mu_{s_t, a_t, u_t, s_{t+1}} \cdot \mu_{s_t, a_t, u_t, s_{t+1}} \cdot \theta_{u_t}}{P(s_{t+1} \mid s_t, a_t)} 
      \\ &= \frac{\sum_{u_t = 1}^{|U_t|} \mu_{s_t, a_t, u_t, s_{t+1}} \cdot \theta_{u_t}}{P(s_{t+1} \mid s_t, a_t)}
      \\ &= \frac{P(s_{t+1} \mid s_t, a_t)}{P(s_{t+1} \mid s_t, a_t)}
      \tag{from the interventional probability constraint \eqref{proofeq:interventional constraint}}
      \\ &= 1 \tag{since $s_t, a_t \rightarrow s_{t+1}$ was observed, $P(s_{t+1} \mid s_t, a_t) > 0$}
    \end{align*}

    \begin{align*}
      \tilde{P}_{t}(\tilde{s}' \mid s_t, a_t)
      &= \frac{\sum_{u_t = 1}^{|U_t|} \mu_{s_t, a_t, u_t, \tilde{s}'} \cdot \mu_{s_t, a_t, u_t, s_{t+1}} \cdot \theta_{u_t}}{P(s_{t+1} \mid s_t, a_t)} 
      \\ &=\frac{0}{P(s_{t+1} \mid s_t, a_t)}
      \tag{since at most one of $\mu_{s_t, a_t, u_t, \tilde{s}'}$ and $\mu_{s_t, a_t, u_t, s_{t+1}}$ can equal 1}
      \\ &= 0 \tag{since $s_t, a_t \rightarrow s_{t+1}$ was observed, $P(s_{t+1} \mid s_t, a_t) > 0$}
    \end{align*}

    These results are as expected, because, given the same exogenous noise and input, we must get the same output in the counterfactual world.
\end{proof}

\begin{lemma}
\label{lemma: existence of theta for overlapping case.}
Take an arbitrary counterfactual state-action pair $(\tilde{s}, \tilde{a})$ and observed state-action pair from an MDP, where $(\tilde{s}, \tilde{a}) \neq (s_t, a_t)$ and $(\tilde{s}, \tilde{a})$ has overlapping support with $(s_t, a_t)$ and $P(s_{t+1} \mid \tilde{s}, \tilde{a}) < P(s_{t+1} \mid s_t, a_t)$. Let $s_{t+1}$ be the observed next state. It can be shown that:

\[\max_{\theta} \left(\sum_{s' \in \mathcal{S}\setminus \{s_{t+1}\}}\sum_{\substack{u_t \in U_t\\f(s_t, a_t, u_t) = s_{t+1} \\ f(\tilde{s}, \tilde{a}, u_t) = s'}}{\theta_{u_t}} \geq P(s_{t+1} \mid s_t, a_t) - P(s_{t+1} \mid \tilde{s}, \tilde{a})\right)\]
\end{lemma}
\begin{proof}
First, let $S_{CS} \subset \mathcal{S}\setminus\{s_{t+1}\}$ be the set of states $s'\in S_{CS}$ which must have a counterfactual transition probability $\tilde{P}_{t}(s' \mid \tilde{s}, \tilde{a}) = 0$ because of the counterfactual stability constraint (i.e., where $\forall s' \in S_{CS}, \dfrac{P(s_{t+1} \mid \tilde{s}, \tilde{a})}{P(s_{t+1} \mid s_t, a_t)}\geq\dfrac{P(s' \mid \tilde{s}, \tilde{a})}{P(s' \mid s_t, a_t)}$ and $P(s' \mid s_t, a_t) > 0$). Consider the total interventional probability of across all states $s'\in S_{CS}$. $\forall s' \in S_{CS}$:

\begin{align*}
    \dfrac{P(s' \mid \tilde{s}, \tilde{a})}{P(s' \mid s_t, a_t)} &\leq \dfrac{P(s_{t+1} \mid \tilde{s}, \tilde{a})}{P(s_{t+1} \mid s_t, a_t)} \text{  (from CS \eqref{proofeq:counterfactual stability})}\\
    P(s' \mid \tilde{s}, \tilde{a}) &\leq \dfrac{P(s_{t+1} \mid \tilde{s}, \tilde{a}) \cdot P(s' \mid s_t, a_t)}{P(s_{t+1} \mid s_t, a_t)}
\end{align*}

The sum of the interventional probabilities of these transitions is:

\[
    \sum_{s' \in S_{CS}}P(s' \mid \tilde{s}, \tilde{a}) \leq \sum_{s' \in S_{CS}}\dfrac{P(s'\mid s_t, a_t) \cdot P(s_{t+1} \mid \tilde{s}, \tilde{a})}{P(s_{t+1} \mid s_t, a_t)}
\]

Let $S_{other}$ be the set of all other states $s' \in \mathcal{S}\setminus{\{s_{t+1}\}\cup S_{CS}}$. We can calculate the sum of the interventional probabilities of the transitions going to states $s' \in S_{other}$:

\begin{align*}
    \sum_{s' \in S_{other}}P(s' \mid \tilde{s}, \tilde{a}) &= 1 - P(s_{t+1} \mid \tilde{s}, \tilde{a}) - \sum_{s'' \in S_{CS}}P(s'' \mid \tilde{s}, \tilde{a})\\
    &\geq 1 - P(s_{t+1} \mid \tilde{s}, \tilde{a}) - \sum_{s'' \in S_{CS}}\dfrac{P(s'' \mid s_t, a_t) \cdot P(s_{t+1} \mid \tilde{s}, \tilde{a})}{P(s_{t+1} \mid s_t, a_t)}\\
    &= 1 - P(s_{t+1} \mid \tilde{s}, \tilde{a}) - \dfrac{P(s_{t+1} \mid \tilde{s}, \tilde{a})}{P(s_{t+1} \mid s_t, a_t)}\cdot\sum_{s'' \in S_{CS}}P(s'' \mid s_t, a_t)\\
    &= 1 - P(s_{t+1} \mid \tilde{s}, \tilde{a}) \\&- \dfrac{P(s_{t+1} \mid \tilde{s}, \tilde{a})}{P(s_{t+1} \mid s_t, a_t)}\cdot(1 - P(s_{t+1} \mid s_t, a_t) - \sum_{s' \in S_{other}}P(s' \mid s_t, a_t))\\
    &= 1 - P(s_{t+1} \mid \tilde{s}, \tilde{a}) - \dfrac{P(s_{t+1} \mid \tilde{s}, \tilde{a})}{P(s_{t+1} \mid s_t, a_t)}\\ &+ P(s_{t+1} \mid \tilde{s}, \tilde{a}) + \dfrac{P(s_{t+1} \mid \tilde{s}, \tilde{a}) \cdot \sum_{s' \in S_{other}}P(s' \mid s_t, a_t)}{P(s_{t+1} \mid s_t, a_t)}\\
    &= 1 - \dfrac{P(s_{t+1} \mid \tilde{s}, \tilde{a})}{P(s_{t+1} \mid s_t, a_t)} + \dfrac{P(s_{t+1} \mid \tilde{s}, \tilde{a}) \cdot \sum_{s' \in S_{other}}P(s' \mid s_t, a_t)}{P(s_{t+1} \mid s_t, a_t)}
\end{align*}

Because of the Mon2 \eqref{proofeq:monotonicity2} and CS \eqref{proofeq:counterfactual stability} constraints, we know that $\forall s'' \in S_{CS}, \sum_{\substack{u_t \in U_t\\f(s_t, a_t, u_t) = s_{t+1} \\ f(\tilde{s}, \tilde{a}, u_t) = s''}}{\theta_{u_t}} =0$, and $s' \in S_{other}$, we have to assume $\sum_{\substack{u_t \in U_t\\f(s_t, a_t, u_t) = s_{t+1} \\ f(\tilde{s}, \tilde{a}, u_t) = s'}}{\theta_{u_t}} \leq P(s' \mid \tilde{s}, \tilde{a}) \cdot P(s_{t+1} \mid s_t, a_t)$ as we cannot know which states are affected by the Mon2 constraint. So,

\[\max_{\theta}\left(\sum_{s'' \in S_{CS}}\sum_{\substack{u_t \in U_t\\f(s_t, a_t, u_t) = s_{t+1} \\ f(\tilde{s}, \tilde{a}, u_t) = s''}}{\theta_{u_t}}\right) = 0
\]
\[\max_{\theta}\left(\sum_{s' \in S_{other}}\sum_{\substack{u_t \in U_t\\f(s_t, a_t, u_t) = s_{t+1} \\ f(\tilde{s}, \tilde{a}, u_t) = s'}}{\theta_{u_t}}\right) \leq \sum_{s' \in S_{other}}P(s' \mid \tilde{s}, \tilde{a}) \cdot P(s_{t+1} \mid s_t, a_t)\]

Therefore, the maximum $\sum_{s' \in \mathcal{S}\setminus \{s_{t+1}\}}\sum_{\substack{u_t \in U_t\\f(s_t, a_t, u_t) = s_{t+1} \\ f(\tilde{s}, \tilde{a}, u_t) = s'}}{\theta_{u_t}}$ is as follows:

\begin{equation}
    \begin{aligned}
        \max_{\theta}\left(\sum_{s' \in \mathcal{S}\setminus \{s_{t+1}\}}\sum_{\substack{u_t \in U_t\\f(s_t, a_t, u_t) = s_{t+1} \\ f(\tilde{s}, \tilde{a}, u_t) = s'}}{\theta_{u_t}}\right) &= \max_{\theta}\left(\sum_{s'' \in S_{CS}}\sum_{\substack{u_t \in U_t\\f(s_t, a_t, u_t) = s_{t+1} \\ f(\tilde{s}, \tilde{a}, u_t) = s''}}{\theta_{u_t}}\right) + \max_{\theta}\left(\sum_{s' \in S_{other}}\sum_{\substack{u_t \in U_t\\f(s_t, a_t, u_t) = s_{t+1} \\ f(\tilde{s}, \tilde{a}, u_t) = s'}}{\theta_{u_t}}\right)\\
        &= \sum_{s' \in S_{other}}P(s'' \mid \tilde{s}, \tilde{a}) \cdot P(s_{t+1} \mid s_t, a_t)\\
        &\geq P(s_{t+1} \mid s_t, a_t) - P(s_{t+1} \mid \tilde{s}, \tilde{a}) + P(s_{t+1} \mid \tilde{s}, \tilde{a}) \cdot \sum_{s' \in S_{other}}P(s' \mid s_t, a_t)\\
        &\geq P(s_{t+1} \mid s_t, a_t) - P(s_{t+1} \mid \tilde{s}, \tilde{a})
    \end{aligned}
\end{equation} 
\end{proof}

\pagebreak
\section{Equivalence with Existing Work}
\label{app: equivalence proofs}
\citet{li2024probabilities}'s bounds for the probability of causation are provided below. We can show that these are equivalent to our analytical bounds where the support of the counterfactual state-action pair $(\tilde{s}, \tilde{a})$ is disjoint from the support of the observed state-action pair $(s_t, a_t)$, given in Theorems \ref{proof theorem: ub disjoint} and \ref{proof theorem: lb disjoint}.

\begin{theorem}[Bounds for Probability of Causation]
    Suppose X has m values $x_1, ..., x_m$ and Y has n values $y_1, ..., y_n$, the probability of causation $P(y_{i_{x_j}}, y_k, x_p)$ where $1 \leq i$, $k \leq n$, $1 \leq j$, $p \leq m$, $j \neq p$ is given by:

    \[
    P^{LB}(y_{i_{x_j}}, y_k, x_p) = \max(0, P(y_{i_{x_j}}) + P(x_p, y_k)  -1 + P(x_j) - P(x_j, y_i))
    \]

    \[
    P^{UB}(y_{i_{x_j}}, y_k, x_p) = \min(P(y_{i_{x_j}}) - P(x_j, y_i), P(x_p, y_k))
    \]
\end{theorem}

\begin{proof}
We show that \cite{li2024probabilities}'s upper and lower bounds for the probability of causation are equivalent to our bounds in Theorems \ref{proof theorem: ub disjoint} and \ref{proof theorem: lb disjoint}, as follows.

\begin{equation}
\begin{aligned}
    P^{UB}(y_{i_{x_j}}, y_k, x_p) &= \min(P(x_p, y_k), P(y_{i_{x_j}}) - P(x_j, y_i))\\
    P^{UB}(y_{i_{x_j}} \mid y_k, x_p) &= \dfrac{\min(P(x_p, y_k), P(y_{i_{x_j}}) - P(x_j, y_i))}{P(x_p, y_k)}\\
    &= \min(1, \dfrac{P(y_{i_{x_j}}) - P(x_j, y_i)}{P(x_p, y_k)})\\
    &= \min(1, \dfrac{P(y_{i_{x_j}}) - P(y_{i_{x_j}}, x_j)}{P({y_k}_{x_p}, x_p)})\\
    &= \min(1, \dfrac{P(y_{i_{x_j}}) - P(y_{i_{x_j}})\cdot P(x_j)}{P({y_k}_{x_p})\cdot P(x_p)})\\
    &= \min(1, \dfrac{P(y_{i_{x_j}})}{P({y_k}_{x_p})}) \text{ because $P(x_p) = 1$ and $P(x_j) = 0$}
\end{aligned}
\end{equation}
which is equivalent to the upper bound in Theorem \ref{proof theorem: ub disjoint} where $x_p = (s_t, a_t)$, $y_k = s_{t+1}$, $x_j = (\tilde{s}, \tilde{a})$ and $y_i = \tilde{s}'$.

\begin{equation}
\begin{aligned}
    P^{LB}(y_{i_{x_j}}, y_k, x_p) &= \max(0, P(y_{i_{x_j}}) + P(x_p, y_k) -1 + P(x_j) - P(x_j, y_i))\\
    P^{LB}(y_{i_{x_j}}\mid y_k, x_p) &= \dfrac{\max(0, P(y_{i_{x_j}}) + P(x_p, y_k) -1 + P(x_j) - P(x_j, y_i))}{P(x_p, y_k)}\\
    &= \max(0, \dfrac{P(y_{i_{x_j}}) + P(x_p, y_k) -1 + P(x_j) - P(x_j, y_i)}{P(x_p, y_k)})\\
    &= \max(0, \dfrac{P(y_{i_{x_j}}) + P(y_{k_{x_p}}, x_p) -1 + P(x_j) - P(y_{i_{x_j}}, x_j)}{P({y_k}_{x_p}, x_p)})\\
    &= \max(0, \dfrac{P(y_{i_{x_j}}) + P(y_{k_{x_p}})\cdot P(x_p) -1 + P(x_j) - P(y_{i_{x_j}})\cdot P(x_j)}{P({y_k}_{x_p})\cdot P(x_p)})\\
    &= \max(0, \dfrac{P(y_{i_{x_j}}) + P(y_{k_{x_p}}) -1}{P({y_k}_{x_p})}) \text{ because $P(x_p) = 1$ and $P(x_j) = 0$}\\
    &= \max(0, \dfrac{P(y_{i_{x_j}}) - (1 - P(y_{k_{x_p}}))}{P({y_k}_{x_p})})\\
\end{aligned}
\end{equation}
which is equivalent to the lower bound in Theorem \ref{proof theorem: lb disjoint} where $x_p = (s_t, a_t)$, $y_k = s_{t+1}$, $x_j = (\tilde{s}, \tilde{a})$ and $y_i = \tilde{s}'$.

\end{proof}

\section{Training Details}
\label{app: training details}
\jl{Our algorithm was executed on a 128-core machine with an Intel Xeon CPU and 512 GB RAM. The interval CFMDP generation part of the algorithm is implemented in Python 3.10, and the value iteration algorithm over the interval CFMDP was implemented in Julia 1.9.4, using the Julia Interval MDP package (v.0.6.0) \citep{mathiesen2024intervalmdp}. Most of the algorithm was run single-threaded, except for generating the Sepsis and Aircraft interval CFMDPs and Gumbel-max SCM CFMDPs, which were run in parallel with $32$ threads.}

\camera{
\section{Environment Details}
\label{app: environments}
Below, we provide brief descriptions of each environment used in our experiments.

\paragraph{GridWorld}
The GridWorld MDP is a $4 \times 4$ grid where an agent must navigate from the top-left corner to the goal state in the bottom-right corner, avoiding a dangerous terminal state in the centre. At each step, the agent can move up, down, left, or right, but there is a small probability (controlled by hyperparameter $p$) of moving in an unintended direction. As the agent nears the goal, the reward for each state increases, culminating in a reward of $+100$ for reaching the goal. Entering the dangerous state results in a penalty of $-100$. We use two versions of GridWorld: a less stochastic version with $p=0.9$ (i.e., $90$\% chance of moving in the chosen direction) and a more stochastic version with $p=0.4$.

\paragraph{Sepsis}
The Sepsis MDP \citep{oberst2019counterfactual}
simulates trajectories of Sepsis patients. Each state consists of four vital signs (heart rate, blood pressure, oxygen concentration, and glucose levels), categorised as low, normal, or high. At each time step, three treatments can be toggled on or off, yielding a total of 8 possible actions. Unlike \citet{oberst2019counterfactual}, we scale rewards based on the number of out-of-range vital signs, between $-1000$ (patient dies) and $1000$ (patient discharged). %

\paragraph{Frozen Lake}
\jl{The Frozen Lake MDP is a PRISM environment \citep{prism} adapted from \citep{cool-mc}. It consists of a $4 \times 4$ grid where an agent moves from the top-left corner to the goal in the bottom-right corner while avoiding several holes. As the PRISM environment is formulated as a reachability problem, we adapt it to a reward-based setting: the agent receives a reward proportional to its distance from the goal, a $+100$ reward for reaching the goal, and a $-100$ penalty for falling into a hole. Additionally, the environment has a terminal state (reachable from the goal and hole states) with a reward of $0$.}

\paragraph{Aircraft}
\jl{The Aircraft MDP is also a PRISM environment \citep{prism}, adapted from \citep{suilen2022robust}. In this environment, two aircraft approach each other, with one controlled by the agent. Both aircraft can increase, decrease, or maintain their current altitude with a fixed probability of success. As the PRISM environment is formulated as a reachability problem, we adapt it to a reward-based setting: the agent receives a reward inversely proportional to how close it is to the other aircraft, a $+100$ reward for reaching the goal, and a $-100$ penalty for colliding with the other aircraft.}
}

\section{Additional Experiments}
\label{app: additional experiments}
\jl{We observe similar results for the Frozen Lake experiments as with the GridWorld and Sepsis experiments. For the slightly suboptimal trajectory, the interval CFMDP follows the observed policy exactly, as it cannot guarantee higher rewards than those observed. In contrast, Gumbel-max policy deviates from the observed trajectory: while it can achieve higher rewards, it can also attain much lower rewards than the observation. For the almost catastrophic and catastrophic paths, both policies improve upon the observed trajectory. They are identical for the almost catastrophic path, but for the catastrophic path, the interval CFMDP policy demonstrates slightly better robustness, as indicated by the higher lower end of its error bars, compared to the Gumbel-max policy.}

\begin{figure}[h]
    \centering
    \begin{subfigure}{0.75\linewidth}
        \begin{tikzpicture}[scale=1.0, every node/.style={scale=1.0}]
            \draw[thick, black] (-3, -0.25) rectangle (10, 0.25);
            \draw[black, line width=1pt] (-2.5, 0.0) -- (-2,0.0);
            \fill[black] (-2.25,0.0) circle (2pt);
            \node[right] at (-2,0.0) {\small Observed Path};
            \draw[blue, line width=1pt] (1.0,0.0) -- (1.5,0.0);
            \node[draw=blue, circle, minimum size=4pt, inner sep=0pt] at (1.25,0.0) {};
            \node[right] at (1.5,0.0) {\small Interval CFMDP Policy};
            \draw[red, line width=1pt] (5.5,0) -- (6,0);
            \node[red] at (5.75,0) {$\boldsymbol{\times}$};
            \node[right] at (6,0) {\small Gumbel-max SCM Policy};
        \end{tikzpicture}
    \end{subfigure}
    \vspace{0.2cm}
    \begin{minipage}{0.4\linewidth}
    \centering
    \begin{subfigure}{\linewidth}
    \resizebox{\linewidth}{!}{
    \begin{tikzpicture}
        \begin{axis}[
            xlabel={$t$},
            ylabel={Mean reward at time step $t$},
            every axis/.style={font=\Huge},
            grid=both,
            ymin=-40, ymax=40,
            height=10cm,
        ]
        \addplot[color=black, mark=*, line width=2pt, mark size=3pt]
        coordinates {(0,0)(1,1)(2,0)(3,0)(4,1)(5,2)(6,3)(7,4)(8,3)(9,4)(9,4)(10,3)(11,2)(12,3)(13,2)(14,2)(15,2)(16,3)(17,2)(18,1)};
        \addplot[
            color=blue, mark=o, line width=2pt, mark size=3pt,
            error bars/.cd, y dir=both, y explicit,
            error bar style={line width=1pt,solid}
        ]
        coordinates {
            (0,0)+-(0,0)
            (1,1)+-(0,0)
            (2,0)+-(0,0)
            (3,0)+-(0,0)
            (4,1)+-(0,0)
            (5,2)+-(0,0)
            (6,3)+-(0,0)
            (7,4)+-(0,0)
            (8,3)+-(0,0)
            (9,4)+-(0,0)
            (10,4)+-(0,0)
            (11,3)+-(0,0)
            (12,2)+-(0,0)
            (13,3)+-(0,0)
            (14,2)+-(0,0)
            (15,2)+-(0,0)
            (16,2)+-(0,0)
            (17,3)+-(0,0)
            (18,2)+-(0,0)
            (19,1.6601485)+-(0,0.86641471)
        };
        \addplot[
            color=red, mark=x, line width=2pt, mark size=6pt,
            error bars/.cd, y dir=both, y explicit,
            error bar style={line width=1pt,solid}
        ]
        coordinates {
            (0,0)+-(0,0)
            (1,1)+-(0,0)
            (2,0.3039025)+-(0,0.45994105)
            (3,0.375832)+-(0,0.63672781)
            (4,1.088197)+-(0,0.41922463)
            (5,2.0099625)+-(0,0.31775659)
            (6,2.9603105)+-(0,0.43287902)
            (7,3.9794295)+-(0,3.30906201)
            (8,3.2871535)+-(0,7.35608186)
            (9,3.922944)+-(0,2.15212485)
            (10,3.9377285)+-(0,3.90658006)
            (11,2.9447885)+-(0,1.71822618)
            (12,2.5103525)+-(0,3.65447805)
            (13,3.16213)+-(0,1.42651879)
            (14,5.0262315)+-(0,14.78230535)
            (15,4.3854185)+-(0,12.62448055)
            (16,7.296955)+-(0,20.94793132)
            (17,6.949505)+-(0,19.93265088)
            (18,4.768139)+-(0,16.85506059)
            (19,3.7749075)+-(0,15.15030676)
        };
        \end{axis}
    \end{tikzpicture}}
    \subcaption{Slightly Suboptimal Path}
    \end{subfigure}
    \end{minipage}
    \hspace{1cm}
    \begin{minipage}{0.4\linewidth}
    \centering
    \begin{subfigure}{\linewidth}
    \resizebox{\linewidth}{!}{
    \begin{tikzpicture}
        \begin{axis}[
            xlabel={$t$},
            ylabel={Mean reward at time step $t$},
            every axis/.style={font=\Huge},
            grid=both,
            ymin=-20, ymax=105,
            height=10cm,
        ]
        \addplot[color=black, mark=*, line width=2pt, mark size=3pt]
        coordinates {(0,0)(1,0)(2,1)(3,0)(4,1)(5,0)(6,0)(7,1)(8,1)(9,2)(9,2)(10,3)(11,2)(12,3)(13,4)(14,5)(15,4)(16,3)(17,2)(18,2)};
        \addplot[
            color=blue, mark=o, line width=2pt, mark size=3pt,
            error bars/.cd, y dir=both, y explicit,
            error bar style={line width=1pt,solid}
        ]
        coordinates {
            (0,0)+-(0,0)
            (1,0)+-(0,0)
            (2,1)+-(0,0)
            (3,0)+-(0,0)
            (4,1)+-(0,0)
            (5,0)+-(0,0)
            (6,0)+-(0,0)
            (7,1)+-(0,0)
            (8,1)+-(0,0)
            (9,2)+-(0,0)
            (10,2)+-(0,0)
            (11,3)+-(0,0)
            (12,2)+-(0,0)
            (13,3)+-(0,0)
            (14,4)+-(0,0)
            (15,5)+-(0,0)
            (16,69.348238)+=(0,30.651762)-=(0,44.58915774)
            (17,6.4397385)+=(0,22.14008082)-=(0,22.14008082)
            (18,4.4237545)+=(0,19.69424148)-=(0,19.69424148)
        };
        \addplot[
            color=red, mark=x, line width=2pt, mark size=6pt,
            error bars/.cd, y dir=both, y explicit,
            error bar style={line width=1pt,solid}
        ]
        coordinates {
            (0,0)+-(0,0)
            (1,0.162166)+-(0,0.36860302)
            (2,1.051055)+-(0,0.22010994)
            (3,0.4224175)+-(0,0.62691064)
            (4,1.0818475)+-(0,0.44471394)
            (5,0.8073005)+-(0,1.05064856)
            (6,0.9823115)+-(0,1.17209156)
            (7,1.4205475)+-(0,0.7044901)
            (8,1.336679)+-(0,0.73447958)
            (9,2.052118)+-(0,0.2347205)
            (10,2.0159025)+-(0,0.14615612)
            (11,2.9917545)+-(0,0.10675913)
            (12,1.9990315)+-(0,0.61831025)
            (13,2.9920555)+-(0,0.11945872)
            (14,3.98468)+-(0,0.32815133)
            (15,4.976689)+-(0,0.28653376)
            (16,68.84109)+=(0,31.15891)-=(0,44.80179426)
            (17,6.4594135)+=(0,22.17157048)-=(0,22.17157048)
            (18,4.7171065)+=(0,18.91450078)-=(0,18.91450078)
        };
        \end{axis}
    \end{tikzpicture}}
    \subcaption{Almost Catastrophic Path}
    \end{subfigure}

    \end{minipage}

    \vspace{0.5cm}

    \begin{subfigure}{0.4\linewidth}
    \resizebox{\linewidth}{!}{
    \begin{tikzpicture}
        \begin{axis}[
            xlabel={$t$},
            ylabel={Mean reward at time step $t$},
            every axis/.style={font=\Huge},
            grid=both,
            ymin=-100, ymax=20,
            height=10cm,
        ]
        \addplot[color=black, mark=*, line width=2pt, mark size=3pt]
        coordinates {(0,0)(1,1)(2,-100)(3,0)(4,0)(5,0)(6,0)(7,0)(8,0)(9,0)(9,0)(10,0)(11,0)(12,0)(13,0)(14,0)(15,0)(16,0)(17,0)(18,0)};
        \addplot[
            color=blue, mark=o, line width=2pt, mark size=3pt,
            error bars/.cd, y dir=both, y explicit,
            error bar style={line width=1pt,solid}
        ]
        coordinates {
            (0,0)+-(0,0)
            (1,1)+-(0,0)
            (2,0.640064)+-(0,0.8044856)
            (3,0.8297265)+-(0,0.91989317)
            (4,0.9809975)+-(0,1.04405766)
            (5,1.106431)+-(0,1.12895369)
            (6,1.3400475)+-(0,6.32019376)
            (7,1.521675)+-(0,6.36155148)
            (8,1.843053)+-(0,8.63182696)
            (9,2.0676395)+-(0,9.45060609)
            (10,2.2726135)+-(0,10.65147578)
            (11,2.434154)+-(0,11.31111344)
            (12,2.563164)+-(0,11.92749602)
            (13,2.673002)+-(0,12.46639396)
            (14,2.7199575)+-(0,12.57135027)
            (15,2.72112)+-(0,12.61092574)
            (16,2.708079)+-(0,12.53375272)
            (17,2.6903665)+-(0,12.45170416)
            (18,2.64189)+-(0,12.3211224)
        };
        \addplot[
            color=red, mark=x, line width=2pt, mark size=6pt,
            error bars/.cd, y dir=both, y explicit,
            error bar style={line width=1pt,solid}
        ]
        coordinates {
            (0,0)+-(0,0)
            (1,0.691236)+-(0,0.46198354)
            (2,0.852259)+-(0,0.76089197)
            (3,0.9774095)+-(0,0.94471592)
            (4,1.129891)+-(0,1.1337585)
            (5,1.2636345)+-(0,1.26118054)
            (6,1.6293815)+-(0,9.16282601)
            (7,1.937078)+-(0,10.01683567)
            (8,2.3133385)+-(0,11.80638003)
            (9,2.527989)+-(0,12.36570187)
            (10,2.7170125)+-(0,13.12232562)
            (11,2.820189)+-(0,13.36544911)
            (12,2.901524)+-(0,13.74094063)
            (13,2.936518)+-(0,13.90834358)
            (14,2.91954)+-(0,13.83110278)
            (15,2.8959655)+-(0,13.6948715)
            (16,2.8408925)+-(0,13.53754436)
            (17,2.783275)+-(0,13.310052)
            (18,2.7082105)+-(0,13.05259347)
        };
        \end{axis}
    \end{tikzpicture}}
    \subcaption{Catastrophic Path}
    \end{subfigure}

\caption{Average instant reward of CF paths induced by policies on Frozen Lake. Error bars denote the standard deviation in reward at each time step.}
\label{fig:reward_frozen_lake}
\end{figure}

\jl{In the Aircraft experiment, we select the observed trajectories slightly differently from the other experiments, because in this environment the aircraft are guaranteed to either collide or the agent will reach the goal. Therefore, we define the slightly suboptimal path as a trajectory that reaches the goal but accumulates less reward than the optimal path, and the almost catastrophic path as a trajectory where the aircraft collide but accumulates higher reward than the catastrophic path.

Since the slightly suboptimal path still reaches the goal, neither the interval CFMDP nor the Gumbel-max policy deviates from the observed trajectory, as they cannot guarantee higher cumulative rewards by doing so. For the almost catastrophic and catastrophic paths, both policies improve upon the observed trajectories and achieve identical rewards. This is likely because all policies that reach the goal will achieve very similar rewards: under the environment constraints, once a policy reaches the goal without collision, its value function will closely match that of other effective policies, as they all require the same number of timesteps to reach the goal.}

\begin{figure}[h]
    \centering
    \begin{subfigure}{0.75\linewidth}
        \begin{tikzpicture}[scale=1.0, every node/.style={scale=1.0}]
            \draw[thick, black] (-3, -0.25) rectangle (10, 0.25);
            \draw[black, line width=1pt] (-2.5, 0.0) -- (-2,0.0);
            \fill[black] (-2.25,0.0) circle (2pt);
            \node[right] at (-2,0.0) {\small Observed Path};
            \draw[blue, line width=1pt] (1.0,0.0) -- (1.5,0.0);
            \node[draw=blue, circle, minimum size=4pt, inner sep=0pt] at (1.25,0.0) {};
            \node[right] at (1.5,0.0) {\small Interval CFMDP Policy};
            \draw[red, line width=1pt] (5.5,0) -- (6,0);
            \node[red] at (5.75,0) {$\boldsymbol{\times}$};
            \node[right] at (6,0) {\small Gumbel-max SCM Policy};
        \end{tikzpicture}
    \end{subfigure}
    \vspace{0.2cm}
    \begin{minipage}{0.4\linewidth}
    \centering
    \begin{subfigure}{\linewidth}
    \resizebox{\linewidth}{!}{
    \begin{tikzpicture}
        \begin{axis}[
            xlabel={$t$},
            ylabel={Mean reward at time step $t$},
            every axis/.style={font=\Huge},
            grid=both,
            ymin=0, ymax=105,
            height=10cm,
        ]
        \addplot[color=black, mark=*, line width=2pt, mark size=3pt]
        coordinates {(0,10)(1,7)(2,5)(3,2)(4,5)(5,2)(6,0)(7,0)(8,0)(9,0)(10,100)(11,100)};
        \addplot[
            color=blue, mark=o, line width=2pt, mark size=3pt,
            error bars/.cd, y dir=both, y explicit
        ]
        coordinates {
            (0,10)+-(0,0)
            (1,7)+-(0,0)
            (2,5)+-(0,0)
            (3,2)+-(0,0)
            (4,5)+-(0,0)
            (5,2)+-(0,0)
            (6,0)+-(0,0)
            (7,0)+-(0,0)
            (8,0)+-(0,0)
            (9,0)+-(0,0)
            (10,100)+-(0,0)
            (11,100)+-(0,0)
        };
        \addplot[
            color=red, mark=x, line width=2pt, mark size=6pt,
            error bars/.cd, y dir=both, y explicit
        ]
        coordinates {
            (0,10)+-(0,0)
            (1,7)+-(0,0)
            (2,5)+-(0,0)
            (3,2)+-(0,0)
            (4,5)+-(0,0)
            (5,2)+-(0,0)
            (6,0)+-(0,0)
            (7,0)+-(0,0)
            (8,0)+-(0,0)
            (9,0)+-(0,0)
            (10,100)+-(0,0)
            (11,100)+-(0,0)
        };
        \end{axis}
    \end{tikzpicture}}
    \subcaption{Slightly Suboptimal Path}
    \end{subfigure}
    \end{minipage}
    \hspace{1cm}
    \begin{minipage}{0.4\linewidth}
    \centering
    \begin{subfigure}{\linewidth}
    \resizebox{\linewidth}{!}{
    \begin{tikzpicture}
        \begin{axis}[
            xlabel={$t$},
            ylabel={Mean reward at time step $t$},
            every axis/.style={font=\Huge},
            grid=both,
            ymin=-100, ymax=105,
            height=10cm,
        ]
        \addplot[color=black, mark=*, line width=2pt, mark size=3pt]
        coordinates {(0,10)(1,7)(2,5)(3,2)(4,2)(5,-100)(6,-100)(7,-100)(8,-100)(9,-100)(10,-100)(11,-100)};
        \addplot[
            color=blue, mark=o, line width=2pt, mark size=3pt,
            error bars/.cd, y dir=both, y explicit
        ]
        coordinates {
            (0,10)+-(0,0)
            (1,7)+-(0,0)
            (2,5)+-(0,0)
            (3,2)+-(0,0)
            (4,-1.2500485)+=(0,20.20977769)-=(0,20.20977769)
            (5,-4.253525)+=(0,22.87481093)-=(0,22.87481093)
            (6,-6.22175)+=(0,24.1550166)-=(0,24.1550166)
            (7,-6.22175)+=(0,24.1550166)-=(0,24.1550166)
            (8,-6.22175)+=(0,24.1550166)-=(0,24.1550166)
            (9,-6.22175)+=(0,24.1550166)-=(0,24.1550166)
            (10,87.5565)+=(0,12.4435)-=(0,48.3100332)
            (11,87.5565)+=(0,12.4435)-=(0,48.3100332)
        };
        \addplot[
            color=red, mark=x, line width=2pt, mark size=6pt,
            error bars/.cd, y dir=both, y explicit
        ]
        coordinates {
            (0,10)+-(0,0)
            (1,7)+-(0,0)
            (2,5)+-(0,0)
            (3,2)+-(0,0)
            (4,-0.967246)+=(0,19.52786735)-=(0,19.52786735)
            (5,-3.965066)+=(0,22.29262702)-=(0,22.29262702)
            (6,-5.93025)+=(0,23.61899945)-=(0,23.61899945)
            (7,-5.93025)+=(0,23.61899945)-=(0,23.61899945)
            (8,-5.93025)+=(0,23.61899945)-=(0,23.61899945)
            (9,-5.93025)+=(0,23.61899945)-=(0,23.61899945)
            (10,88.1395)+=(0,11.8605)-=(0,47.2379989)
            (11,88.1395)+=(0,11.8605)-=(0,47.2379989)
        };
        \end{axis}
    \end{tikzpicture}}
    \subcaption{Almost Catastrophic Path}
    \end{subfigure}

    \end{minipage}

    \vspace{0.5cm}

    \begin{subfigure}{0.4\linewidth}
    \resizebox{\linewidth}{!}{
    \begin{tikzpicture}
        \begin{axis}[
            xlabel={$t$},
            ylabel={Mean reward at time step $t$},
            every axis/.style={font=\Huge},
            grid=both,
            ymin=-100, ymax=105,
            height=10cm,
        ]
        \addplot[color=black, mark=*, line width=2pt, mark size=3pt]
        coordinates {(0,10)(1,7)(2,5)(3,2)(4,-100)(5,-100)(6,-100)(7,-100)(8,-100)(9,-100)(10,-100)(11,-100)};
        \addplot[
            color=blue, mark=o, line width=2pt, mark size=3pt,
            error bars/.cd, y dir=both, y explicit
        ]
        coordinates {
            (0,10)+-(0,0)
            (1,7)+-(0,0)
            (2,5)+-(0,0)
            (3,2.002412)+=(0,0.08503048)-=(0,0.08503048)
            (4,-10.8982525)+=(0,34.84110891)-=(0,34.84110891)
            (5,-13.966235)+=(0,35.9090185)-=(0,35.9090185)
            (6,-16.0943)+=(0,36.74783677)-=(0,36.74783677)
            (7,-16.0943)+=(0,36.74783677)-=(0,36.74783677)
            (8,-16.0943)+=(0,36.74783677)-=(0,36.74783677)
            (9,-16.0943)+=(0,36.74783677)-=(0,36.74783677)
            (10,67.8114)+=(0,32.1886)-=(0,73.49567355)
            (11,67.8114)+=(0,32.1886)-=(0,73.49567355)
        };
        \addplot[
            color=red, mark=x, line width=2pt, mark size=6pt,
            error bars/.cd, y dir=both, y explicit
        ]
        coordinates {
            (0,10)+-(0,0)
            (1,7)+-(0,0)
            (2,5)+-(0,0)
            (3,2.0023865)+=(0,0.08458017)-=(0,0.08458017)
            (4,-10.9361515)+=(0,34.88349832)-=(0,34.88349832)
            (5,-13.988216)+=(0,35.92995381)-=(0,35.92995381)
            (6,-16.1184)+=(0,36.77005822)-=(0,36.77005822)
            (7,-16.1184)+=(0,36.77005822)-=(0,36.77005822)
            (8,-16.1184)+=(0,36.77005822)-=(0,36.77005822)
            (9,-16.1184)+=(0,36.77005822)-=(0,36.77005822)
            (10,67.7632)+=(0,32.2368)-=(0,73.54011644)
            (11,67.7632)+=(0,32.2368)-=(0,73.54011644)
        };
        \end{axis}
    \end{tikzpicture}}
    \subcaption{Catastrophic Path}
    \end{subfigure}

\caption{Average instant reward of CF paths induced by policies on Aircraft. Error bars denote the standard deviation in reward at each time step.}
\label{fig:reward_aircraft}
\end{figure}

\begin{table}
\centering
\begin{tabular}{|c|c|c|c|}
\hline
\multirow{2}{*}{\textbf{Environment}} 
& \multicolumn{3}{c|}{\textbf{Pessimistic $V(s_0)$}} \\ \cline{2-4} 
& \textbf{CS+M} & \textbf{CS} & \textbf{None} \\ \hline

1. GridWorld ($p=0.9$)   & $345 \pm 114$ & $333 \pm 124$ & $332 \pm 125$ \\ \hline
2. GridWorld ($p=0.4$)   & $-84.5 \pm 192$ & $-102 \pm 220$ & $-102 \pm 220$ \\ \hline
3. Sepsis                & $1640 \pm 745$ & $1630 \pm 755$ & $1630 \pm 756$\\ \hline
4. Frozen Lake           & $37.1 \pm 16.3$ & $32.8 \pm 14.7$ & $32.8 \pm 14.7$ \\ \hline
5. Aircraft              & $-104 \pm 373$ & $-182 \pm 414.7$ & $-187 \pm 413$\\ \hline

\end{tabular}
\caption{Average worst-case counterfactual $V(s_0)$ for the ICFMDP policies (given varying assumptions) over $100$ randomly sampled observed trajectories.
\label{tab:robustness assumptions}
}
\end{table}

\camera{
We also further examine the impact of the counterfactual stability and counterfactual monotonicity assumptions on the robustness of the approach by repeating the experiment from Section \ref{sec: robustness experiment} under different assumptions. We take an additional $100$ randomly sampled paths, and for each path, we derive ICFMDPs under three settings: (i) both counterfactual stability and counterfactual monotonicity, (ii) counterfactual stability only, and (iii) no assumptions. Table~\ref{tab:robustness assumptions} reports the average worst-case counterfactual value $V(s_0)$ of the policies derived for each corresponding ICFMDP.

As expected, relaxing the assumptions leads to a decrease in the worst-case value $V(s_0)$, due to the wider counterfactual transition bounds. However, in most environments, this performance drop is relatively small. This is consistent with Table~\ref{tab: bounds}, which shows that the counterfactual transition bounds, on average, are only marginally tightened by the additional assumptions. In the Aircraft environment, the performance drop is greater, which we suspect is due to the specific transitions that have been affected by the addition of the counterfactual monotonicity assumption, as empirically we observe that many transitions leading to terminal states become significantly tighter when counterfactual monotonicity is imposed.

Nevertheless, across all environments, the worst-case performance of the derived policies over their respective ICFMDPs remains higher than the worst-case performance of the Gumbel-max SCM reported in Table~\ref{tab: pessimistic value random paths}, demonstrating the robustness of our approach with and without assumptions.
}

\end{document}